\documentclass{article} % For LaTeX2e
\usepackage{iclr2026_conference,times}

\usepackage{amsmath,amsfonts,bm}

\def\eqref#1{equation~\ref{#1}}
\def\1{\bm{1}}

\def\eps{{\epsilon}}

\DeclareMathAlphabet{\mathsfit}{\encodingdefault}{\sfdefault}{m}{sl}
\SetMathAlphabet{\mathsfit}{bold}{\encodingdefault}{\sfdefault}{bx}{n}

\usepackage{hyperref}
\usepackage{url}

\usepackage{booktabs}       % professional-quality tables
\usepackage{amsfonts}       % blackboard math symbols
\usepackage{nicefrac}       % compact symbols for 1/2, etc.
\usepackage{microtype}      % microtypography
\usepackage{rotating}
\usepackage{tabularx}
\usepackage{pdflscape}
\usepackage{amsmath,amsthm,amssymb, bm}
\usepackage[capitalize,noabbrev]{cleveref}
\usepackage{blkarray}  

\usepackage{amsmath}
\usepackage{pifont}
\usepackage{subcaption} 
\usepackage{dsfont}
\usepackage{float}
\usepackage{cancel}
\usepackage{enumerate, cases}
\usepackage{thmtools,thm-restate}
\usepackage{mathtools}
\usepackage[textsize=tiny]{todonotes}
\usepackage[none]{hyphenat}
\usepackage{multirow}
\usepackage{makecell}

\usepackage{algorithm}
\usepackage{algpseudocode}

\usepackage{array}
\usepackage{enumitem}
\usepackage{tikz}
\usetikzlibrary{decorations.pathreplacing}
\usepackage{pgfkeys}
\usetikzlibrary{intersections}
\usepackage[textsize=tiny]{todonotes}
\usepackage{etoolbox} % For better programming logic
\usepackage{pdfpages}

\newcommand{\indicator}{\mathbb{I}}
\newcommand{\calS}{\mathcal{S}}
\newcommand{\calA}{\mathcal{A}}

\usepackage{pifont}
\newtheorem{thm}{Theorem}
\newtheorem{lem}[thm]{Lemma}
\newtheorem{prop}[thm]{Proposition}
\newtheorem{cor}[thm]{Corollary}

\newtheorem{rem}[thm]{Remark}
\newtheorem{assu}[thm]{Assumption}

\newtheorem{example}[thm]{Example}

\newcounter{keylemma}

\newcounter{auxlemma}

\newcounter{techlemma}

\makeatletter
\@ifundefined{definition}{}{%

}
\makeatother
\theoremstyle{definition}
\newtheorem{defi}{Definition}
\usepackage{newfloat}
\usepackage{listings}
\DeclareCaptionStyle{ruled}{labelfont=normalfont,labelsep=colon,strut=off} % DO NOT CHANGE THIS
\floatstyle{ruled}
\newfloat{listing}{tb}{lst}{}
\floatname{listing}{Listing}
\newtheorem{theorem}{Theorem}

\usepackage{amsthm}

\newcommand{\up}[1]{\overline{#1}}
\newcommand{\low}[1]{\underline{#1}}
\newcommand{\Reg}{{\rm Regret}}

\newcommand{\Eq}{\textsc{Equilibrium}}

\newcommand{\Algoname}{Multiplayer Optimistic Robust Nash Value Iteration for $f$-Divergence Uncertainty Set (\textsc{$f$-MORNAVI})}
\newcommand{\Algonamef}{\textsc{$f$-MORNAVI}}
\newcommand{\AlgonameTV}{\text{TV}-MORNAVI}
\newcommand{\AlgonameKL}{\text{KL}-MORNAVI}

\newcommand{\RMGTV}{\text{TV}-DRMG}
\newcommand{\RMGKL}{\text{KL}-DRMG}

\newcommand{\BibTeX}{Bib\TeX}

\title{Sample-Efficient Distributionally Robust Multi-Agent Reinforcement Learning via Online Interaction}

\author{Zain Ulabedeen Farhat$^{1,*}$ , Debamita Ghosh$^{1,}$\thanks{The first two authors contributed equally.}~, George K. Atia$^{1,2}$, Yue Wang$^{1,2}$  \\
$^{1}$ Department of Electrical \& Computer Engineering \quad $^{2}$ Department of Computer Science\\
University of Central Florida, Orlando, FL 32816, USA \\
\texttt{\{za464241, de881780, george.atia, yue.wang\}@ucf.edu}% \\
}%

\iclrfinalcopy % Uncomment for camera-ready version, but NOT for submission.
\begin{document}

\newcommand{\pmat}[1]{\begin{pmatrix} #1 \end{pmatrix}}
\newcommand{\paren}[1]{{\left( #1 \right)}}
\newcommand{\brac}[1]{{\left[ #1 \right]}}
\newcommand{\set}[1]{{\left\{ #1 \right\}}}

% Operator
\newcommand{\defeq}{\mathrel{\mathop:}=}
\newcommand{\vect}[1]{\ensuremath{\mathbf{#1}}}
\newcommand{\mat}[1]{\ensuremath{\mathbf{#1}}}
\newcommand{\dd}{\mathrm{d}}
\newcommand{\grad}{\nabla}
\newcommand{\hess}{\nabla^2}
\newcommand{\argmin}{\mathop{\rm argmin}}
\newcommand{\argmax}{\mathop{\rm argmax}}
\newcommand{\fracpar}[2]{\frac{\partial #1}{\partial  #2}}

\newcommand{\sigmin}{\sigma_{\min}}
\newcommand{\tr}{\mathrm{tr}}
\renewcommand{\det}{\mathrm{det}}
\newcommand{\rank}{\mathrm{rank}}
\newcommand{\logdet}{\mathrm{logdet}}
\newcommand{\trans}{^{\top}}
\newcommand{\poly}{\mathrm{poly}}
\newcommand{\polylog}{\mathrm{polylog}}
\newcommand{\st}{\mathrm{s.t.~}}
\newcommand{\proj}{\mathcal{P}}
\newcommand{\projII}{\mathcal{P}_{\parallel}}
\newcommand{\projT}{\mathcal{P}_{\perp}}

\DeclarePairedDelimiter\abs{\lvert}{\rvert}%
\makeatletter
\let\oldabs\abs
\def\abs{\@ifstar{\oldabs}{\oldabs*}}

\newcommand{\norm}[1]{\Big\lVert #1 \Big\rVert}

\newcommand{\fnorm}[1]{\|{#1} \|_{\text{F}}}
\newcommand{\spnorm}[2]{\| {#1} \|_{\text{S}({#2})}}

\newcommand{\E}{\mathbb{E}}
\newcommand{\D}{\mathbb{D}}
\newcommand{\Var}{\text{Var}}
\newcommand{\la}{\langle}
% \newcommand{\ra}{\rangle}

% newly added by Yu
\newcommand{\wbar}[1]{\overline{#1}}
\newcommand{\prox}{{\mathop{\rm prox}}}
\newcommand{\<}{\left\langle}
\renewcommand{\>}{\right\rangle}

% newly added by Qinghua
\newcommand{\minH}[1]{{\min\left\{#1,H\right\}}}
\newcommand{\epsplan}{{\eps_{\rm plan}}}
\newcommand{\bigO}{\mathcal{O}}

\newcommand{\diag}{{\rm diag}}

\newcommand{\one}{{\mathbf1}}
\newcommand{\bb}{{\mathbf b}}
\newcommand{\br}{{\mathbf r}}
\newcommand{\bc}{{\mathbf c}}

\newcommand{\Ohat}{\hat{\O}}
\newcommand{\Phat}{\widehat{\P}}
\newcommand{\That}{\hat{\T}}
\newcommand{\rhat}{\widehat{r}}
\newcommand{\Prhat}{\hat{\Pr}}
\newcommand{\Qhat}{\widehat{Q}}
\newcommand{\Vhat}{\widehat{V}}
\newcommand{\fhat}{\hat{f}}
\newcommand{\pihat}{{\hat{\pi}}}
\newcommand{\bhat}{{\hat{\mathbf b}}}
\newcommand{\ahat}{\hat{a}}
\newcommand{\chat}{\hat{c}}
\newcommand{\ohat}{\hat{o}}
\newcommand{\that}{{\hat{t}}}
\newcommand{\uhat}{\hat{u}}
\newcommand{\vhat}{\hat{v}}
\newcommand{\sigmahat}{\hat{\sigma}}
\newcommand{\thetahat}{\hat{\theta}}
\newcommand{\Prt}{\tilde{\Pr}}
\newcommand*\circled[1]{\tikz[baseline=(char.base)]{
            \node[shape=circle,draw,inner sep=1pt] (char) {#1};}}
\newcommand{\Cpoly}{C_{\rm poly}}

\newcommand{\Bhat}{\hat{B}}
\newcommand{\Ahat}{\hat{A}}
\newcommand{\Chat}{\hat{C}}
\newcommand{\what}{\hat{w}}

\newcommand{\wmin}{w_{\min}}

\newcommand{\Qstar}{Q^{\star}}
\newcommand{\Qt}{{\widetilde{Q}}}
\newcommand{\Vt}{{\widetilde{V}}}
\newcommand{\Vmint}{{\widetilde{V}_{\min}}}
\newcommand{\Vstar}{V^{\star}}
\newcommand{\Tt}{\tilde{\T}}
\newcommand{\Ot}{\tilde{\O}}
\newcommand{\Pt}{\tilde{P}}
\newcommand{\bt}{\tilde{b}}
\newcommand{\Bt}{\tilde{B}}
\newcommand{\At}{\tilde{A}}
\newcommand{\at}{\tilde{a}}
\newcommand{\cC}{\mathcal{C}}

\newcommand{\Qbar}{{\bar{Q}}^\star}
\newcommand{\Vbar}{{\bar{V}}^\star}
\newcommand{\Tbar}{\bar{\T}}
\newcommand{\rbar}{\bar{r}}
\newcommand{\pibar}{{\bar{\pi}}}

\newcommand{\Vh}{\mathcal{V}}
\newcommand{\Qh}{\mathcal{Q}}

\newcommand{\Ox}{{\diag(\O(x\mid \cdot))}}

% Big O
\newcommand{\cO}{\mathcal{O}}
\newcommand{\tlO}{\mathcal{\tilde{O}}}
\newcommand{\tlOmega}{\tilde{\Omega}}
\newcommand{\tlTheta}{\tilde{\Theta}}

% Set

\newcommand{\Z}{\mathbb{Z}}
\newcommand{\N}{\mathbb{N}}
\newcommand{\R}{\mathbb{R}}
\renewcommand{\S}{\mathbb{S}}
\newcommand{\C}{\mathbb{C}}

% Matrix and Vector

\newcommand{\A}{\mat{A}}
\newcommand{\B}{\mat{B}}
\newcommand{\Q}{\mat{Q}}

\newcommand{\I}{\mat{I}}
\newcommand{\U}{\mat{U}}
\newcommand{\V}{\mat{V}}
\newcommand{\W}{\mat{W}}
\newcommand{\X}{\mat{X}}
\newcommand{\Y}{\mat{Y}}
\newcommand{\e}{\vect{e}}
\renewcommand{\u}{\vect{u}}
\renewcommand{\v}{\vect{v}}
\newcommand{\w}{\vect{w}}
\newcommand{\x}{\vect{x}}
\newcommand{\y}{\vect{y}}
\newcommand{\z}{\vect{z}}
\newcommand{\g}{\vect{g}}
\newcommand{\zero}{\vect{0}}
\newcommand{\fI}{\mathfrak{I}}
\newcommand{\fS}{\mathfrak{S}}
\newcommand{\fE}{\mathfrak{E}}
\newcommand{\fF}{\mathfrak{F}}

% Special Character

\newcommand{\cF}{\mathcal{F}}
\newcommand{\cM}{\mathcal{M}}
\newcommand{\cN}{\mathcal{N}}
\newcommand{\cD}{\mathcal{D}}
\newcommand{\cL}{\mathcal{L}}
\newcommand{\cH}{\mathcal{H}}
\newcommand{\cS}{\mathcal{S}}
\newcommand{\cA}{\mathcal{A}}
\newcommand{\cB}{\mathcal{B}}

\newcommand{\cn}{\kappa}
\newcommand{\nn}{\nonumber}

% low bound

\newcommand{\minone}[1]{\max\{#1,1\}}

\newcommand{\astar}{\bm{a}^\star}
\newcommand{\bstar}{\bm{b}^\star}
\newcommand{\vstar}{\bm{v}^\star}
\newcommand{\nw}{n_{\rm wrong}}
\newcommand{\er}[1]{{\rm error}_{#1}}
\newcommand{\Jtcal}{{\tilde{\Jcal}}}

\renewcommand{\circ}{\diamond}

\maketitle

\begin{abstract}
\label{sec:Abstract}
Well-trained multi-agent systems can fail when deployed in real-world environments due to model mismatches between the training and deployment environments, caused by environment uncertainties including noise or adversarial attacks. Distributionally Robust Markov Games (DRMGs) enhance system resilience by optimizing for worst-case performance over a defined set of environmental uncertainties. However, current methods are limited by their dependence on simulators or large offline datasets, which are often unavailable.
This paper pioneers the study of online learning in DRMGs, where agents learn directly from environmental interactions without prior data. We introduce the {\it Multiplayer Optimistic Robust Nash Value Iteration (MORNAVI)} algorithm and provide the first provable guarantees for this setting. Our theoretical analysis demonstrates that the algorithm achieves low regret and efficiently finds the optimal robust policy for uncertainty sets measured by Total Variation divergence and  Kullback-Leibler divergence. These results establish a new, practical path toward developing truly robust multi-agent systems.

\end{abstract}

\section{Introduction}
\label{sec:Introduction}
Multi-agent reinforcement learning (MARL), along with its stochastic game-based mathematical formulation \citep{shapley1953stochastic, littman1994markov}, has emerged as a cornerstone paradigm for intelligent multi-agent systems capable of complex, coordinated behavior. It provides the theoretical and algorithmic foundation for enabling multiple agents to learn, adapt, and make sequential decisions in shared, dynamic environments. Its practical impacts span strategic gaming, where MARL agents have achieved superhuman mastery \citep{GO_DavidSilver, Vinyals2019GrandmasterLI}; autonomous transportation, where it is used to coordinate fleets of vehicles to navigate complex traffic scenarios \citep{shalev2016safe, hua2024multiagentreinforcementlearningconnected}; and distributed robotics, where teams of robots learn to execute tasks \citep{lowe2017multi, matignon2012independent}. 

Despite the remarkable progress in MARL, a fundamental and pervasive challenge severely restricts its reliable deployment in practice: the {\it Sim-to-Real} gap \citep{zhao2020sim, peng2018sim}. A standard pipeline of RL involves training extensively within a high-fidelity simulator and then deploying in practice. However, any simulator inevitably fails to capture the full richness and complexity of the real world, omitting subtle physical effects, unpredictable sensor noise, unmodeled system dynamics, or latent environmental factors \citep{padakandla2020reinforcement, rajeswaran2016epopt}. Consequently, a policy that appears optimal within the simulation can be brittle and perform poorly—or even fail catastrophically—when deployed into the noisy, unpredictable environment.

This vulnerability to model mismatch is magnified exponentially in the multi-agent context: this uncertainty is amplified through a cascading feedback loop of agent interactions. A minor, unmodeled perturbation that affects one agent can cause it to deviate from its expected behavior. This deviation alters the environment for its peers, who in turn must adapt their policies. Their adaptations further change the dynamics for all other agents, including the one first affected. This can trigger a chain of unpredictable responses, destabilizing the collective strategy and leading to a highly non-stationary learning environment far more volatile than that caused by strategic adaptation alone \citep{papoudakis2019dealing, canese2021multi, wong2023deep}. The entire multi-agent system becomes fragile, as the intricate inter-agent dependencies act as amplifiers for even the smallest model inaccuracies.

To enable MARL against such uncertainty, the framework of Distributionally Robust Markov Games (DRMGs) offers a principled and powerful solution \citep{zhang2020robust, kardecs2011discounted}. DRMG approach embraces a principle of pessimism. It defines an uncertainty set of plausible environment models centered around the nominal one, and the goal is to maximize the worst-case expected returns across the entire uncertainty set. This robust optimization strategy yields two profound benefits. First, it provides a formal performance guarantee: if the true environment lies within the uncertainty set, the policy's performance is guaranteed to be no worse than the optimized worst-case value. Second, it acts as a powerful regularizer, forcing agents to discover more generalizable policies that are inherently less sensitive to perturbations, thereby enhancing generalization even to environments outside the set \citep{vinitsky2020robust, abdullah2019wasserstein,liu2025distributionally}.

However, despite its theoretical appeal, the current body of research on DRMGs is built upon assumptions that create a critical disconnect from the realities of many high-stakes applications. The prevailing algorithmic frameworks fall into two main categories: those that assume access to a generative model \citep{Arxiv2024_SampleEfficientMARL_Shi,jiao_minimax-optimal_2024}, which is tantamount to having a perfect, queryable oracle or simulator, and those designed for the offline setting \citep{li2025sample,NeuRIPS2023_DoublePessimismDROfflineRL_Blanchet}, which presuppose the existence of a large, static, and sufficiently comprehensive dataset collected beforehand.
These assumptions are untenable in precisely the domains where robustness is most crucial. Consider applications in autonomous systems \citep{demontis2022surveyreinforcementlearningsecurity} or personalized healthcare \citep{alaa2023drlhealthcare, lu2020deepreinforcementlearningready}. In these settings, creating a high-fidelity simulator is often impossible, and pre-collecting a dataset that covers all critical scenarios is infeasible. Agents have no choice but to learn online, through direct, sequential interaction with the complex and unknown real world. In this online paradigm, data is not a free commodity to be sampled at will; it is earned through experience, where every action has a real cost and naive exploration can lead to severe or irreversible outcomes. This necessitates a new class of algorithms that can navigate the exploration-exploitation tradeoff under the additional burden of worst-case environmental uncertainty.

% We are thus faced with a formidable challenge at the intersection of robustness and practicality. Agents must be resilient to model misspecification, but they must achieve this resilience while learning through direct interactions, without any simulator or a comprehensive prior dataset. This critical need exposes a fundamental gap in the literature and motivates the central question of our work:
We aim for robustness that survives contact with reality: agents must cope with misspecification while learning purely from experience. Without simulators or sizable offline datasets, existing approaches struggle to bridge theory and practice. This shortfall clarifies the gap we address and motivates the central question of our work: 
    \textbf{\textit{How to design provably effective online algorithms for distributionally robust Markov games?}}

In this paper, we answer the above question by designing a model-based online algorithm for  DRMGs and providing corresponding theoretical guarantees.
Our contributions are summarized as follows.

 \textbf{Hardness in Online DRMGs:} We first reveal the inherent hardness of online learning in DRMGs. Specifically, we first show that the online learning can suffer from the support shifting issue, where the support of the worst-case kernel is not fully covered by the support of the nominal environment, by constructing a hard instance that incurs an $\Omega\big(K \min\{H,\prod_{i} A_i\}\big)$-regret for any algorithm. Moreover, we use another example to show that even without the support shifting issue, the regret can still have a minimax lower bound of $\Omega(\sqrt{K\prod_{i } A_i})$. Here, $K$ is the number of iteration episodes, $H$ is the DRMG horizon, and $\prod_{i} A_i$ is the size of the joint action space.  These results directly imply the hardness of online learning, compared to other well-posed learning schemes, including generative model \citep{shi2024breaking,jiao_minimax-optimal_2024} or offline learning \citep{li2025sample}.  
   
  \textbf{A Framework for Online Robust MARL:} We introduce {{\Algonamef}}, a novel model-based meta-algorithm designed specifically for online learning in  DRMGs. Our framework pioneers a dual approach that synergizes the \textit{pessimism} required for robust optimization with the \textit{optimism} essential for provably efficient online exploration. At its core, {\Algonamef} learns the nominal environment model from online interactions and then incorporates a carefully constructed, data-driven bonus term, $\beta$. This bonus term is uniquely tailored to the geometry of the chosen uncertainty set, guiding exploration while guaranteeing that the learned policy is robust to worst-case model perturbations. We further present two concrete instantiations of our framework for uncertainty sets defined by Total Variation (TV) distance and Kullback-Leibler (KL) divergence.

  \textbf{Near-Optimal Regret Bounds for Online DRMGs:} We establish the first known theoretical guarantees for online learning in general-sum DRMGs by providing rigorous, high-probability regret bounds for our algorithms. The regret measures the performance gap between our algorithm and an optimal robust policy, thus formally characterizing the sample complexity needed to solve the DRMG. We futher  prove that our algorithms converge to an $\epsilon$-optimal robust policy with high sample efficiency (see \Cref{cor:Sample_Complexity_bound}). Our results are significant as they are the first to demonstrate that finding a robust equilibrium in a general-sum DRMG is achievable in a sample-efficient manner through online interaction, without requiring a simulator or a pre-collected dataset.

\section{Problem Formulation}
\label{sec:Problem_setup}

\subsection{Distributionally Robust Markov Games}

A \textit{Distributionally Robust Markov Game} (DRMG) can be specified as $\mathcal{MG}_{\text{rob}} = \big\{\mathcal{M}, \mathcal{S}, \mathcal{A}, H, \{\mathcal{P}_i\}_{i\in\mathcal{M}}, r \big\}$,
% \begin{align}
% \label{eq:robust_markov_games}
% \mathcal{MG}_{\text{rob}} = \Big\{\mathcal{M}, \mathcal{S}, \mathcal{A}, H, \{\mathcal{P}_i\}_{i\in\mathcal{M}}, r \Big\},
% \end{align}
where $\mathcal{M}=\{1,...,m\}$ is the set of  $m$ agents, $\mathcal{S}=\{1,2,\dots,S\}$ denotes the finite state space, $\mathcal{A}$ denotes the joint action space for all agents as $\mathcal{A}= \mathcal{A}_1 \times \cdots \times \mathcal{A}_m$, where $\mathcal{A}_i = \{1,2,\dots,A_i\}$ is the action space of agent $i$, and $H$ denotes the horizon length. We consider non-stationary DRMGs, i.e.,  $r$ is the reward function: $ r= \{r_{i,h}\}_{1 \leq i \leq m, 1 \leq h \leq H}$ with $r_{i,h} : \mathcal{S} \times \mathcal{A} \mapsto [0, 1]$. Specifically, for any $(i, h, s, \mathbf{a}) \in \mathcal{M} \times [H] \times \mathcal{S} \times \mathcal{A}$, $r_{i,h}(s, \mathbf{a})$ is the immediate (deterministic) reward received by the $i$-th agent in state $s$ when the joint action profile is $\mathbf{a}$. Agents in a DRMG maintain their own uncertainty sets of transition kernels $\mathcal{P}_i$, to capture the potential environment uncertainties in their perspective. At each step, the environment transits following an arbitrary kernel from the uncertainty set.

% Each agent aims to maximize the worst-case performance within this uncertainty set. 

%\paragraph{Rectangular uncertainty sets with $f$-divergence.}

% To accommodate individual robustness preferences, each agent is permitted to tailor its own uncertainty set $\mathcal{P}^{\rho_i}_D(P^\star)$ by choosing different sizes $\rho_i$ and even the shape determined by different divergence functions $D$. Here, we consider the same divergence function for all agents for simplicity. We focus on the discussion of the transition kernel's uncertainty, though similar uncertainty can also be considered for each agent’s reward function.

% In this work, we mainly consider  uncertainty sets specified by $f$-divergence \citep{sason2016f}.

Drawing inspiration from the rectangularity condition in robust single-agent RL \citep{INFORM2005_RobusDP_Iyengar, INFORMS2013_RobustMDP_wiesemann, PMLR2021_DROTabularRL_Zhou, NeuRIPS2023_CuriousPriceDRRLGenerativeMdel_Shi}, and following standard DRMG studies \citep{Arxiv2024_SampleEfficientMARL_Shi,shi2024breaking,zhang2020robust}, we consider the \textit{agent-wise $(s,\bm{a})$-rectangular uncertainty set}, due to its computational tractability\footnote{\textcolor{black}{Robust MDPs without rectangular assumption can be NP-hard to solve \citep{INFORMS2013_RobustMDP_wiesemann}.}}. Specifically, for each agent $i$, the DRMG specify an uncertainty set $\mathcal{P}_i$, which is independently defined over all horizons, states, and joint actions:
\begin{equation}
\mathcal{P}_i= \bigotimes_{(h, s, {\bf a})\in[H]\times \mathcal{S}\times\mathcal{A}} \mathcal{P}_{i,h,f}^{\rho_i}(s,{\bf a}),
    \label{eq:Uncertainty_set}
\end{equation}
where $\otimes$ denotes the Cartesian product. At step $h$, if all agents take a joint action $\bf{a}_h$ at the state $s_h$, the transition kernel can be chosen arbitrarily from the prescribed uncertainty set $\mathcal{P}^{\rho_i}_{i,h,f} (s_h,\bf{a}_h)$. We consider the uncertainty set $\mathcal{P}^{\rho_i}_{i,h,f}(s,\bf a)$ centered on a \textit{nominal kernel} $P^\star$:
\begin{defi}[$f$-Divergence Uncertainty Set]
\label{def:f_divergence_uncertainty}
   % For each \((s,\bm{a})\) pair, we denote the center probability by \(P^{\star}_{h}(\cdot|s,\bm{a})\) and the size of the set by \(\rho_i > 0\). 
   The \(f\)-divergence uncertainty set is defined as:
\begin{align}
\label{eq:f_divg_Uncertainty}
\mathcal{P}^{\rho_i}_{i,h,f}(s,{\bf a})  = \left\{ P_h \in \Delta(\mathcal{S}): f\Big(P_{h},P^{\star}_{h}(\cdot|s,{\bf a})\Big)\leq \rho_i \right\},
\end{align}
where the $f$-divergence is  $f\big(P_h,\ P_{h}^{\star}(\cdot|s,\bm{a})\big) = \sum\limits_{s' \in \mathcal{S}} f\left( \frac{ P_h(s') }{ P_{h}^{\star}(s'|s,\bm{a}) } \right) P_{h}^{\star}(s'|s,\bm{a})$.
%, $P_{h,s,a}$ is vector of the transition kernel $P$ at state-action pair $(s,{\bf a}) \in \mathcal{S} \times \mathcal{A}$ for step $h$ where $P_{h,s,{\bf a}}\triangleq P_h(\cdot|s,{\bf a}) \in \mathbb{R}^{1\times S}$. 
\end{defi}
% We assume that all agents use the same $f$-divergence for their uncertainty sets\footnote{Generally, each agent can decide their own (possibly different) distance metric. We consider the same ${f}$-divergence for simplicity.} but different uncertainty levels $\rho_i > 0$ independently.  
The $f$-divergence uncertainty sets with different $f$ have been extensively studied in distributionally robust RL \citep{Arxiv2023_TowardsMinimaxOptimalityRobustRL_Clavier,NeuRIPS2023_CuriousPriceDRRLGenerativeMdel_Shi,panaganti2022robust,AnnalsStat2022_TheoreticalUnderstandingRMDP_Yang,NeuRIPS2024_UnifiedPessimismOfflineRL_Yue,zhang2025modelfree}. In this work, we focus on TV and KL-divergence.

\paragraph{Robust Value Functions.} For a DRMG, each agent aims to maximize its own worst-case performance over all possible transition kernels in its own (possibly different) prescribed uncertainty set. The strategy of agent $i$ taking actions is captured by a policy $\pi_i=\{ \pi_{i,h}: \mathcal{S}\to\Delta(\mathcal{A}_i)\}_{h=1}^H$. Since the immediate rewards and transition kernels are determined by the joint actions, the worst-case performance of the $i$-th agent over its own uncertainty set $\mathcal{P}_i$ is determined by a joint policy $\pi=\{\pi_h: \mathcal{S}\to\Delta(\mathcal{A}) \}_{h=1}^H$, which we refer to as the robust value function $V^{\pi, \rho_i}_{{i,h}}$ and the robust $Q$-function $Q^{\pi, \rho_i}_{{i,h}}$, for an initial state $s$ and initial action $\bf a$: 
$Q^{\pi, \rho_i}_{{i,h}}(s,{\bf a})\triangleq \inf_{\tilde{P} \in \mathcal{P}_i} \mathbb{E}_{\pi, \tilde{P}} \Biggl[ \sum_{t=h}^H r_{i,t}(s_t, \mathbf{a}_t) \,\Biggm|\, s_h = s, \mathbf{a}_h = \mathbf{a} \Biggr], $ and $V^{\pi, \rho_i}_{{i,h}}(s)  \triangleq \sum_{\mathbf{a}} \pi(\mathbf{a}|s)Q^{\pi, \rho_i}_{{i,h}}(s,\mathbf{a})$.

% \begin{align*}
%     Q^{\pi, \rho_i}_{{i,h}}(s,{\bf a})\triangleq \inf_{\tilde{P} \in \mathcal{P}_i} \mathbb{E}_{\pi, \tilde{P}} \Biggl[ \sum_{t=h}^H r_{i,t}(s_t, \mathbf{a}_t) \,\Biggm|\, s_h = s, \mathbf{a}_h = \mathbf{a} \Biggr],  V^{\pi, \rho_i}_{{i,h}}(s)  \triangleq \sum_{\mathbf{a}} \pi(\mathbf{a}|s)Q^{\pi, \rho_i}_{{i,h}}(s,\mathbf{a}) ,  %\label{eq:robust_Q_value_agent_i} 
%      %V^{\pi, \tilde{P}}_{{i,h}}(s),   
%  %   Q^{\pi, \tilde{P}}_{{i,h}}(s,{\bf a})\label{eq:robust_Q_value_agent_i}.
% \end{align*}
% where  $V^{\pi, \tilde{P}}_{i,h} : \mathcal{S} \mapsto \mathbb{R}$ and $Q^{\pi,\tilde{P}}_{i,h} : \mathcal{S} \times \mathcal{A} \mapsto \mathbb{R}$ are the standard value function and $Q$-function of the $i$-th agent under the transition kernel $\tilde{P}$:
% \begin{align*}
% &V^{\pi, \tilde{P}}_{i,h}(s) \triangleq \mathbb{E}_{\pi, \tilde{P}} \left[ \sum_{t=h}^H r_{i,t}(s_t, \mathbf{a}_t) \,\middle|\, s_h = s \right], \\%\label{eq:non_robust_V_value_agent_i}\\
% &Q^{\pi, \tilde{P}}_{i,h}(s, \mathbf{a}) \triangleq \mathbb{E}_{\pi, \tilde{P}} \Biggl[ \sum_{t=h}^H r_{i,t}(s_t, \mathbf{a}_t) \,\Biggm|\, s_h = s, \mathbf{a}_h = \mathbf{a} \Biggr],
% %\label{eq:non_robust_Q_value_agent_i}
% \end{align*}
where the expectation is taken over the randomness of the joint policy $\pi$ and the kernel $\tilde{P}$.

% , i.e., $\mathbf{a}_t \sim \pi_t(\cdot \mid s_t)$ and $s_{t+1} \sim \tilde{P}(\cdot \mid s_t, \mathbf{a}_t)$. 

\paragraph{Solutions to DRMGs.} Due to different objectives among players, the goal of a DRMG is to achieve some notions of equilibrium \citep{fudenberg1991game}. 
For any given joint policy $\pi$, $\pi_{-i}$ is the marginal policies of all agents excluding the $i$-th agent. The agent $i$'s best response policy to $\pi_{-i}$, $\pi^{\dagger,\rho_i}_i(\pi_{-i})$, is the policy that maximizes its own robust value function, at the given step $h$ and state $s$: 
$\pi^{\dagger,\rho_i}_i(\pi_{-i})\triangleq \arg\max_{\pi^{\prime}_i\in\Delta(\mathcal{A}_i)}V^{(\pi_{-i}\times\pi^{\prime}_i), \rho_i}_{i,h}(s). $ The corresponding robust value function is 
\begin{align}
\label{eq:robust_best_response_agent_i}
    V^{\dagger, \pi_{-i}, \rho_i}_{{i,h}}(s) 
    &\triangleq \max_{\pi^{\prime}_i \in  \Delta(\mathcal{A}_i)} V^{\pi^{\prime}_i \times \pi_{-i}, \rho_i}_{{i,h}}(s).
\end{align}

% \begin{align}\pi^{\dagger,\rho_i}_i(\pi_{-i})\triangleq \arg\max_{\pi^{\prime}_i\in\Delta(\mathcal{A}_i)}V^{(\pi_{-i}\times\pi^{\prime}_i), \rho_i}_{i,h}(s). 
% \end{align}

% It is shown that the best response policy always exists, and there exists some policies can simultaneously attain $V^{\dagger, \pi_{-i}, \rho_i}_{{i,h}}(s)$ value for all $s \in \mathcal{S}$ and $h \in [H]$ \citep{INFORM2005_RobusDP_Iyengar,NeuRIPS2023_DoublePessimismDROfflineRL_Blanchet}.
The goal of a DRMG is to compute an equilibrium policy \citep{fudenberg1991game}, such that each agent’s policy is the best response to the others, so that no single agent can improve its robust value by deviating while the rest remain fixed. Standard notions of equilibria include {\it robust Nash Equilibrium (NE)}, {\it robust Coarse Correlated Equilibrium (CCE)}, and {\it robust Correlated Equilibrium (CE)} (their existence is shown \citep{NeuRIPS2023_DoublePessimismDROfflineRL_Blanchet}), defined as follows:
% As mentioned, the goal of a DRMG is to obtain some equilibrium policy \citep{fudenberg1991game}, in the sense that any agent's policy is a best response policy to the remaining agents' joint policy, or equivalently, no agent can gain or improve its robust value function by deviating from that equilibrium policy while others sticking to it. Specially, there are different notions of equilibrium, including {\it robust Nash Equilibrium (NE)}, {\it robust Coarse Correlated Equilibrium (CCE)}\footnote{Since computing exact robust equilibria is often intractable \citep{inproceedings}, we generally consider approximate equilibrium solutions.}, and {\it robust Correlated Equilibrium (CE)}, and DRMG aims to find any of them:
% % 
% \footnote{Since computing exact robust equilibria is often intractable \citep{inproceedings}, we generally consider approximate equilibrium solutions.}

 \textbf{Robust $\varepsilon$-NE.} A \textcolor{black}{product policy $\pi \in \Delta(\mathcal{A}_1) \times \cdots \times \Delta(\mathcal{A}_m)$} is a \textit{robust-$\varepsilon$ NE} if for any $s\in\mathcal{S}$: $\text{gap}_{\text{NE}}(\pi,s) \triangleq \max_{\substack{i\in \mathcal{M}}} \left\{ V^{\dagger, \pi_{-i}, \rho_i}_{i,1}(s) - V^{\pi, \rho_i}_{i,1}(s) \right\} \leq \varepsilon.$
% \begin{equation*} %s \in \mathcal{S}, 1 \leq i \leq m
%     \text{gap}_{\text{NE}}(\pi,s) \triangleq \max_{\substack{i\in \mathcal{M}}} \left\{ V^{\dagger, \pi_{-i}, \rho_i}_{i,1}(s) - V^{\pi, \rho_i}_{i,1}(s) \right\} \leq \varepsilon.
% \end{equation*}
% \begin{align*}
% \text{gap}_{\text{NE}}(\pi) 
% &\triangleq \max_{s \in \mathcal{S},\, 1 \leq i \leq m} 
% \Big\{ V^{\dagger, \pi_{-i}, \rho_i}_{i,1}(s) 
% - V^{\pi, \rho_i}_{i,1}(s) \Big\} \\
% &\leq \varepsilon
% \end{align*}

Robust NE ensures that, the agent $i$'s policy induced by the NE is a best response policy to the remaining agents' joint policy (up to $\epsilon$), thus no agent can improve its worst-case performance—evaluated over its own uncertainty set $\mathcal{P}_i$—by unilaterally deviating from the NE. 

% Computing a NE is generally PPAD \citep{daskalakis2009complexity}, thus we also consider the following two tractable solutions. 

\textbf{Robust $\varepsilon$-CCE.}
A \textcolor{black}{(possibly correlated) joint policy $\pi \in \Delta(\mathcal{A})$} is a \textit{robust-$\varepsilon$ CCE} if for any $s\in\mathcal{S}$: $ \text{gap}_{\text{CCE}}(\pi,s) \triangleq \max_{\substack{i\in \mathcal{M}}} \left\{ V^{\dagger, \pi_{-i}, \rho_i}_{i,1}(s) - V^{\pi, \rho_i}_{i,1}(s) \right\} \leq \varepsilon.$
% \begin{equation*} %s \in \mathcal{S}, 1 \leq i \leq m
%     \text{gap}_{\text{CCE}}(\pi,s) \triangleq \max_{\substack{i\in \mathcal{M}}} \left\{ V^{\dagger, \pi_{-i}, \rho_i}_{i,1}(s) - V^{\pi, \rho_i}_{i,1}(s) \right\} \leq \varepsilon.
% \end{equation*}
Robust CCE relaxes the notion of NE by allowing for potentially correlated policies, while still ensuring that no agent has an incentive to unilaterally deviate from it.

\textbf{Robust $\varepsilon$-CE.}
A joint policy $\pi \in \Delta(\mathcal{A})$ is a \textit{robust-$\varepsilon$ CE} if for any $s\in\mathcal{S}$: $\text{gap}_{\text{CE}}(\pi,s) \triangleq \max_{\substack{i\in \mathcal{M}}} \left\{ \max_{\phi \in \Phi_i} V^{\phi \circ \pi, \rho_i}_{i,1}(s) - V^{\pi, \rho_i}_{i,1}(s) \right\} \leq \varepsilon.$
% \begin{equation*} %s \in \mathcal{S}, 1 \leq i \leq m
%     \text{gap}_{\text{CE}}(\pi,s) \triangleq \max_{\substack{i\in \mathcal{M}}} \left\{ \max_{\phi \in \Phi_i} V^{\phi \circ \pi, \rho_i}_{i,1}(s) - V^{\pi, \rho_i}_{i,1}(s) \right\} \leq \varepsilon.
% \end{equation*}
Here, a strategy modification $\phi \triangleq \{\phi_{h,s}\}_{(h,s)\in [H] \times \mathcal{S}}$ for player $i$ is a set of $[H] \times \mathcal{S}$ functions from $\mathcal{A}_i$ to itself. Let $\Phi_i$ denote the
set of all possible strategy modifications for player $i$.
Given a joint policy $\pi$, applying a modification $\phi$ yields a new joint policy $\phi \circ \pi$, which matches $\pi$ everywhere except that at each state $s$ and timestep $h$, player $i$’s action $a_i$ is replaced by $\phi_{h,s}(a_i)$.

\paragraph{Online Learning in DRMGs.}
We consider online learning in DRMGs, aiming to compute equilibria $\{\sf NASH, CCE, CE\}$ via interaction with the nominal environment $P^\star$ over $K \in \mathbb{N}$ episodes. Each episode starts from $s_1^k$, proceeds with a policy $\pi^k$ chosen from experience, and ends with an update for the next round. We use \emph{robust regret} as our performance metric, which compares the learned outcome to the target equilibrium in the presence of model error.
%To reflect the cost of interaction, we use robust regret as our performance criterion, comparing learned outcomes to the desired equilibrium in the presence of model error.

% We study the online learning problem in DRMGs, where agents aim to reach one of the equilibria in $\{\sf NASH, CCE, CE\}$ through interaction with the nominal environment $P^\star$ over $K \in \mathbb{N}$ episodes. In each episode $k$, all agents observe an initial state $s_1^k$, select a joint policy $\pi^k$ based on past experience, execute it in $P^\star$ to collect a trajectory, and update their policy for the next round. Since interacting with the environment is generally expensive, we introduce robust regret to quantify the learning cost.

% must strive to optimize their respective policies through direct interaction with the environment, balancing exploration and exploitation in real time while collectively achieving robust policy improvement despite the underlying model uncertainty.

\begin{defi}[Robust Regret]
  Let $\pi^k$ be the execution policy 
  in the $k^{\text{th}}$ episode. After a total of $K$ episodes, the corresponding robust regret is defined as $ \Reg_{\{\sf NASH, CCE, CE\}}(K)=\sum_{k=1}^K \text{gap}_{\{\sf NASH, CCE, CE\}}(\pi^k,s^k_1). $
\end{defi}
Notably, if an algorithm has a sub-linear regret, it achieves a robust equilibrium as $K\to\infty$.

% \textbf{Relationship between Robust NASH, CCE and CE.} The existence of these equilibria under general divergence functions is established in \citep{NeuRIPS2023_DoublePessimismDROfflineRL_Blanchet}. The robust equilibria defined here reduce to standard ones over robust Q and V value functions, which satisfy ${\text{NE}} \subseteq {\text{CE}} \subseteq {\text{CCE}}$, as noted in \citep{ACM2010_AlgorithmicGameTheory_roughgarden}. Thus, the existence of a robust NE ensures the existence of robust CE and CCE.

% \section{Performance Metric}
% \label{sec:Performance_metric}
% \input{Performance_metric}

\section{Optimistic Robust Nash Value Iteration}
\label{sec:Algorithm}

We then present {\Algoname}, a meta-algorithm for episodic, finite-horizon DRMGs with interactive data collection. {\Algonamef} handles general 
$f$-divergences, with emphasis on KL and TV. 

% As shown in \Cref{alg:Robust-Multi-Nash-VI}, it balances exploration and exploitation by placing confidence bounds on the robust value function, thereby avoiding explicit modeling of the full transition dynamics.

% In this section, we introduce {\Algoname}, a meta-algorithm designed for episodic finite-horizon DRMGs with interactive data collection. {\Algonamef} is a flexible framework that accommodates a range of $f$-divergences, with particular focus on KL-divergence and TV-divergence. The algorithm, presented in \Cref{alg:Robust-Multi-Nash-VI}, achieves a balance between exploration and exploitation by constructing confidence intervals directly on the robust value function, thereby circumventing the complexity of modeling the full transition dynamics. 

% By exploiting the structure of the $f$-divergence, it incorporates adaptive bonus terms that capture both uncertainty and robustness. Building on ideas from Multi-Nash-VI \citep{liu2021sharp}, this approach yields tighter confidence bounds, reduced dependence on the size of the State space, and more efficient learning in the face of uncertainty.

\begin{algorithm}[ht]
\caption{$f$-MORNAVI}
\label{alg:Robust-Multi-Nash-VI}
\begin{algorithmic}[1]
\State \textbf{Input:} Uncertainty level $\rho_i>0$ for all $i \in \mathcal{M}$.
\State \textbf{Initialize:} Dataset \( \mathbb{D} = \emptyset \)
\For{episode \( k = 1, \dots, K \)}
    % \Statex * {\bf \texttt{ Nominal Transition Estimation}} *
    \State Compute the transition kernel estimator $ \widehat{P}_h^k(s,{\bf a},s')$ as given in eq. \ref{eq:transition_estimate}.
      % \Statex * {\bf \texttt{ Optimistic Robust Planning}} *
    \State Set \( \overline{V}_{H+1}^{k,\rho_i}(\cdot) = \underline{V}_{H+1}^{k,\rho_i}(\cdot) = 0 \) for all $i \in \mathcal{M}$.
    \For{step \( h = H, \dots, 1 \)}
        %\For{$\forall (s,{\bf a})\in \mathcal{S}\times \mathcal{A}$}
                \State For all  $(s,{\bf a})\in \mathcal{S}\times \mathcal{A}$ and $i \in \mathcal{M}$, update \( \overline{Q}_{i,h}^{k,\rho_i}(s,{\bf a}) \) [eq. \ref{eq:robust_Qupper_k}] and \( \underline{Q}_{i,h}^{k,\rho_i}(s,{\bf a}) \) [eq. \ref{eq:robust_Qlower_k}]. 
                % \State Update \( \underline{Q}_{i,h}^{k,\rho_i}(s,{\bf a}) \) as in \ref{eq:robust_Qlower_k} for all  $i \in \mathcal{M}$.
       % \EndFor
        % \For{$\forall s\in \mathcal{S}$}
            \State  For all $s\in \mathcal{S}$, update $\pi^k_h(\cdot|s)$ by eq. \ref{eq:policy_k_h}, update $\overline{V}_{i,h}^{k,\rho_i}(s)$ and $\underline{V}_{i,h}^{k,\rho_i}(\cdot)$ by eq. \ref{eq:robust_V_values_k}.
        % \EndFor
    \EndFor
   % \Statex * {\bf \texttt{Execution of policy and data collection}} *
    \State Receive initial State \( s_1^k \in \mathcal{S} \)
   %  \If{$\max_{i\in\mathcal{M}}(\overline{V}_{i,1}^{k, \rho_i} - \underline{V}_{i,1}^{k, \rho_i})(s_1) < \Delta$} \label{line:multi_term_start}
   % \State $\Delta \leftarrow \max_{i\in\mathcal{M}}(\overline{V}_{i,1}^{k, \rho_i} - \underline{V}_{i,1}^{k, \rho_i})(s_1)$ and $\pi^{\text{out}} \leftarrow \{ \pi^k \}_{k=1}^K$. \label{line:multi_term_end}
   % \ENDIF
   \For{step \( h = 1, \dots, H \)}
        \State Take action \( {\bf a}_h^k \sim \pi_h^k(\cdot \mid s_h^k) \), observe reward \( r_h(s_h^k, {\bf a}_h^k) \) and next state \( s_{h+1}^k \).
    \EndFor
    \State Set \( \mathbb{D} = \mathbb{D} \cup \{(s_h^k, {\bf a}_h^k, s_{h+1}^k)\}_{h=1}^H \).
\EndFor
\State \textbf{Output:} Return policy $\pi^{\text{out}} =  \{\pi^k\}_{k=1}^K$.
\end{algorithmic}
\end{algorithm}

\subsection{Algorithm Design}
Our algorithm has the following three stages. 

\textbf{Stage 1: Nominal Transition Estimation (Line 4).} At the start of each episode \( k \in [K] \), we maintain an estimate of the nominal kernel \( P^\star \) using the historical data \( \mathbb{D} = \{(s_h^\tau, {\bf a}_h^\tau, s_{h+1}^\tau)\}_{\tau=1,h=1}^{k-1,H} \) collected from past interactions with the training environment. Specifically, {\Algonamef} updates the empirical transition kernel for each tuple \( (h, s, {\bf a}, s') \in [H] \times \mathcal{S} \times \mathcal{A} \times \mathcal{S} \) as follows:
\begin{align}
\label{eq:transition_estimate}
\widehat{P}_h^k(s'|s,{\bf a}) = \frac{N_h^k(s,{\bf a},s') }{N_h^k(s,{\bf a})} (\text{if } N_h^k(s,{\bf a})> 0), \text{ and } \widehat{P}_h^k(s'|s,{\bf a})=\frac{1}{|\mathcal{S}|} (\text{if } N_h^k(s,{\bf a})=0),
\end{align}
% \textcolor{red}{need to modify}
% \begin{align}
% \label{eq:transition_estimate}
% \widehat{P}_h^k(s'|s,{\bf a}) = \frac{N_h^k(s,{\bf a},s')}{N_h^k(s,{\bf a}) \vee 1},
% \end{align}
where \( N_h^k(s,{\bf a},s') \) and \( N_h^k(s,{\bf a}) \), are calculated on the current dataset \( \mathbb{D} \) by $ N^k_h(s,{\bf a},s')  = \sum\limits_{\tau=1}^{k-1}\mathbf{1}\{(s^{\tau}_h, {\bf a}^{\tau}_h, s^{\tau}_{h+1}) = (s,{\bf a},s^{\prime})\}, \quad \text{ and }\quad 
    N^k_h(s,{\bf a})=\sum\limits_{s^{\prime}\in \mathcal{S}}N^k_h(s,{\bf a},s^{\prime}).$ Note that we adopt a model-based approach that estimates transition kernels. Although this leads to higher memory consumption, model-free DRMGs are inherently challenging due to the non-linearity of worst-case expectation w.r.t. nominal kernels, which makes model-free estimators biased or sample-inefficient \citep{liu2022dQlearning,wang2023model,wang2024modelfree,zhang2025modelfree}.
% \begin{align}
% \label{eq:counts_on_D}
%     N^k_h(s,{\bf a},s') &= \sum\limits_{\tau=1}^{k-1}\mathbf{1}\{(s^{\tau}_h, {\bf a}^{\tau}_h, s^{\tau}_{h+1}) = (s,{\bf a},s^{\prime})\}, \quad \text{ and }\quad 
%     N^k_h(s,{\bf a})=\sum\limits_{s^{\prime}\in \mathcal{S}}N^k_h(s,{\bf a},s^{\prime}).
% \end{align}

% Notably, our algorithm adopts a model-based approach, as it explicitly requires estimating the transition model. Although this leads to higher memory consumption, we highlight that distributionally robust MARL is fundamentally difficult in the model-free setting: the worst-case expectation is a non-linear function of the nominal transition kernel for each agent, rendering model-free estimation either biased or highly sample-inefficient~\citep{liu2022dQlearning, wang2023model, wang2024modelfree, zhang2025modelfree}.

\textbf{Stage 2: Optimistic Robust Planning (Lines 5--9).}
The {\Algonamef} constructs the episode policy $\pi^k$ via optimistic robust planning based on the empirical model $\widehat{P}^k$. This involves estimating an upper bound on the robust value function, following the principle of Upper-Confidence-Bound (UCB) methods, which are well-established in online vanilla RL \citep{auer2010ucb, PMLR2017_MinimxRegretBoundNonRobustRL_Azar, PMLR2019_TighterProblemDependentRegretRL_Zanette, PMLR2021_RLDifficultThanBandits_Zhang, PMLR2021_UCBMomentumQLearning_Menard, zhang2024settling}, and this optimism encourages exploration of less-visited state–action pairs.
%. Specifically, optimistic estimates encourages the agents to explore the less visited state-action pairs.
%NeurIPS2020_AlmostOptimalModelFreeRL_Zhang, li2021breaking, domingues2021episodic, 

To this end, {\Algonamef} maintains a bonus term at each episode \( k \), capturing the gap between the robust value function under $\widehat{P}^k$ and that under the true model. This bonus is added to the robust Bellman estimate to ensure its optimism. Specifically, for each \( (h, s, {\bf a}) \in [H] \times \mathcal{S} \times \mathcal{A} \), we set
\begin{align}
\overline{Q}_{i,h}^{k, \rho_i}(s, \bm{a}) = &  
\min \big\{ r_{i,h}(s,\bm{a}) + 
\sigma_{\widehat{\mathcal{P}}_{i,h,f}^{\rho_i}(s, \bm{a})} [\overline{V}_{i,h+1}^{k, \rho_i}] + \beta^k_{i,h,f}(s,\bm{a}),\ 
H \big\}.\label{eq:robust_Qupper_k}\\
\underline{Q}_{i,h}^{k, \rho_i}(s, \bm{a}) = & 
\max \big\{ r_{i,h}(s,\bm{a}) + 
\sigma_{\widehat{\mathcal{P}}_{i,h,f}^{\rho_i}(s, \bm{a})}[\underline{V}_{i,h+1}^{k, \rho_i}]  - \beta^k_{i,h,f}(s,\bm{a}),\ 0 \big\}.\label{eq:robust_Qlower_k}
\end{align}
Here, $\sigma_{\mathcal{P}}[V]=\inf_{P\in\mathcal{P}} \mathbb{E}_{P}[V]$ is the support function of $V$ over the uncertainty set $\mathcal{P}$, and can be calculated through its dual representation (see \Cref{lem:strong_duality}); $\widehat{\mathcal{P}}_{i,h,f}^{\rho_i}$ is the uncertainty set centered at $\widehat{P}^k$ from eq. \ref{eq:transition_estimate}:  $\widehat{\mathcal{P}}_{i,h,f}^{\rho_i}(s,\mathbf{a})  = \left\{ P_h \in \Delta(\mathcal{S}): f\Big(P_{h},\widehat{P}^k_{h}(\cdot|s,{\bf a})\Big)\leq \rho_i \right\}$.

% \begin{align}
% \overline{Q}_{i,h}^{k, \rho_i}(s, \bm{a}) &=  
% \min\left\{r_{i,h}(s,\bm{a}) + 
% \sigma_{\widehat{\mathcal{P}}_{i,h,f}^{\rho_i}(s, \bm{a})} [\overline{V}_{i,h+1}^{k, \rho_i}] + \beta^k_{i,h,f}(s,\bm{a}),\ 
% H \right\}.\label{eq:robust_Qupper_k}\\
% \underline{Q}_{i,h}^{k, \rho_i}(s, \bm{a}) &= 
% \max\left\{r_{i,h}(s,\bm{a}) + 
% \sigma_{\widehat{\mathcal{P}}_{i,h,f}^{\rho_i}(s, \bm{a})}[\underline{V}_{i,h+1}^{k, \rho_i}] - 
% \beta^k_{i,h,f}(s,\bm{a}),\ 0 \right\}.\label{eq:robust_Qlower_k}
% \end{align}
Each of these estimates in eq. \ref{eq:robust_Qupper_k} and eq. \ref{eq:robust_Qlower_k} are based on estimated robust Bellman operators (see \Cref{sec:DRMG Uncertainty Set} for details) and a bonus term $\beta^k_{i,h,f}(s,\bm{a}) \geq 0$. The bonus term is constructed (we will discuss the construction later) to ensure the estimation becomes a confidence interval of the true robust value function, i.e., $Q^{\dagger,\pi_{-i},\rho_i}_{i,h} (s, {\bf a})\in [\underline{Q}_{i,h}^{k,\rho_i}(s,{\bf a}), \overline{Q}_{i,h}^{k,\rho_i}(s,{\bf a})]$, with high probability.

% Under these constructions, we will show that \( \overline{Q}_{i,h}^{k,\rho_i} \) and \( \underline{Q}_{i,h}^{k,\rho_i} \) remain valid confidence bounds (as shown in $\{\sf NASH, CCE, CE\}=\{\text{\Cref{lem:Optimistic_pessimism_TV_NE_Bernstein}},\text{\Cref{lem:Optimistic_pessimism_TV_CCE_Bernstein}}, \text{\Cref{lem:Optimistic_pessimism_TV_CE_Bernstein}}\}$ for {\RMGTV} and $\{\sf NASH, CCE, CE\}=\{\text{\Cref{lem:Optimistic_pessimism_KL_NE}},\text{\Cref{lem:Optimistic_pessimism_KL_CCE}}, \text{\Cref{lem:Optimistic_pessimism_KL_CE}}\}$ for {\RMGKL} in Appendix). Importantly, we will also show that the carefully designed bonuses \ref{eq:Bonus_term_TV_Bernstein} and \ref{eq:Bonus_term_KL} ensure the confidence region is tight, resulting in a near-optimal regret bound of our algorithm. 

\textbf{EQUILIBRIUM subroutine (Line 8).} Given robust $Q$-function estimates $\underline{Q}_{i,h}^{k,\rho_i}(s,{\bf a})$ and $\overline{Q}_{i,h}^{k,\rho_i}(s,{\bf a})$ for \( i \in \mathcal{M} \) at step \( h \), the sub-routine $\Eq \in \{\sf NASH, CCE, CE\}$ finds a corresponding equilibrium $\pi^k_h(\cdot|s)$ for the matrix-form game with pay-off matrices $\{\overline{Q}_{i,h}^{k,\rho_i}(s,\cdot)\}_{i\in \mathcal{M}}$:
\begin{align}
\label{eq:policy_k_h}
\pi^k_h(\cdot|s) \gets \Eq \Big(\set{\overline{Q}_{i,h}^{k,\rho_i}(s,\cdot)}_{i\in \mathcal{M}} \Big). 
\end{align} 
\textcolor{black}{Note that finding a NE can be PPAD-hard \citep{daskalakis2009complexity}, but computing CE or CCE remains tractable in polynomial time \citep{liu2021sharp}. We follow standard MG studies, assuming $\Eq$ can be executed, and mainly focus on sample complexity and statistic efficiency.}

% With the optimistic estimation $\overline{Q}_{i,h}^{k,\rho_i}$, we then set the execution policy for the $k$-th episode as the greedy policy with respect to the optimistic Q-estimate:
% \begin{align}
% \label{eq:policy_k_h}
% % \overline{Q}_{1,h}^{k, \rho_1}(s, \cdot), \cdots, \overline{Q}_{m,h}^{k, \rho_m}(s, \cdot)
% \pi^k_h(\cdot|s) \gets \Eq\left(\set{\overline{Q}_{i,h}^{k,\rho_i}(s,\cdot)}_{i\in \mathcal{M}} \right), 
% \end{align}
We then update the estimation of $V^{\dagger,\pi_{-i},\rho}_{h}$ as 
\begin{align}
\label{eq:robust_V_values_k}
\overline{V}_{i,h}^{k, \rho_i}(s) = 
\mathbb{E}_{\bm{a} \sim \pi^k(\cdot|s)}\left[\overline{Q}_{i,h}^{k, \rho_i}(s,{\bf a}) \right] \quad \text{ and }\quad 
\underline{V}_{i,h}^{k, \rho_i}(s) &= \mathbb{E}_{\bm{a} \sim \pi^k(\cdot|s)}\left[\underline{Q}_{i,h}^{k, \rho_i}(s,{\bf a}) \right].
\end{align}

Note that while the lower estimate in eq. \ref{eq:robust_Qlower_k} does not influence policy execution directly, it plays a crucial role in constructing valid exploration bonuses and ensuring strong theoretical guarantees. By leveraging both upper and lower bounds, the algorithm performs optimistic robust planning, enabling structured, uncertainty-aware exploration that balances exploration, exploitation, and robustness.

\textbf{Stage 3: Execution of Policy and Data Collection (Lines 10--16).}
After evaluating the policy $\{\pi^k_h\}_{h=1}^H$ for episode $k$, the learner takes action based on $\pi^k_h$ and observes the reward $r_h(s^k_h,{\bf a}^k_h)$ and next state $s^k_{h+1}$, which get appended to the historical dataset collected till episode $k-1$.

% \subsection{Bonus of {\Algonamef} under DRMG}
% We then instantiate our meta-algorithm for DRMGs under both TV and KL-divergence by explicitly constructing the corresponding bonus terms and estimation procedures.

\section{Hardness of Online Learning}
\label{sec:challengs}

In this section, we aim to discuss the inherent hardness of online learning in DRMGs from two aspects: (1) When there is a support shift issue, no MARL algorithm can obtain a sub-linear regret on a certain DRMG; \textcolor{black}{(2) Even if there is no support shift issue, there exists a DRMG such that any online algorithm suffers from the curse of multi-agency.} This is a separation between DRMGs with interactive data collection and generative model/offline data, and also between DRMGs with non-robust MGs, showing the inherent challenges of online DRMGs.

\subsection{Hardness with Support Shift}  
\label{subsec:hardness_support_shift}
Support shift \citep{Arxiv2024_DRORLwithInteractiveData_Lu} refers to the case that the support of the worst-case transition kernel is not covered by the support of the nominal kernel. It can happen when, for instance, the uncertainty set is defined through TV. It will result in a challenge that, for those states that are not covered by the nominal kernel, there is no data available, so that the agent can never learn the optimal robust policy efficiently. Specifically, we derive the following result to illustrate the hardness. 
\begin{thm}
There exists a TV-DRMG, such that any online learning algorithm satisfies that:
$$\inf_{\mathcal{ALG}}\E[\text{Regret}_{\textsf{NASH}}(K)] \geq \Omega\Big(\rho K \cdot \min\{H, \prod_{i\in \mathcal{M}}A_i\}\Big).$$
\end{thm}
Our construction is deferred to \Cref{ex:initial_shock} in Appendix. This regret bound is linear in the number of episodes $K$, creating a combinatorial explosion that makes the problem information-theoretically intractable. Moreover, our result shows that when the game horizon $H$ is large enough, the minimax lower bound depends on the joint action space, showing the hardness of online learning compared to generative models and offline settings.

\subsection{Hardness without support shift}
\label{subsec:hardness_without_support_shift}
We then illustrate the hardness of online DRMGs when there is no support shift. Note that when the uncertainty set is defined through, e.g., KL divergence, the worst-case support will be covered by the nominal one, so there will not be any support shift. However, we construct another example to show that, even without the support shift, online learning can still be challenging and inefficient. 

\begin{thm}[Lower Bound for Robust Learning without Support Shift]
There exists a DRMG, such that any learning algorithm suffers the following cumulative regret lower bound over $K$ episodes:
$$\inf_{\mathcal{ALG}}   \E[\text{Regret}_{\textsf{NASH}}(K)] \geq \Omega\Big(\sqrt{ K\prod_{i\in \mathcal{M}}A_i}\Big).$$
\end{thm}
% Our construction is in \Cref{ex:corrupted_bandit} in Appendix.
This result illustrates that, even without any support shift, some hard instance can require at least $\Omega\left(\sqrt{K\prod_iA_i}\right)$ regret. \textcolor{black}{Our result thus suggests that the dependence on the joint action space may be inevitable in online DRMGs, which suffer from the curse of multi-agency. Specifically, in DRMGs, agents need to solve the robust optimization (i.e., estimate the support function $\sigma_{\mathcal{P}}(\cdot)$), which requires knowledge of the whole transition kernels to find the worst-case from the uncertainty set. Thus, the agents have to explore the whole model, introducing an inevitable dependence on $\prod_iA_i$. In non-robust MGs, however, agents can estimate the single nominal performance merely from samples instead of model estimations, thus the multi-agency curse can be broken.}

\section{Theoretical Guarantees}
\label{sec:Theoretical_guarantees}

% We then develop the theoretical results of our algorithm under both TV and KL sets.

\subsection{Regret Bound for Total Variation}
As discussed in \Cref{sec:challengs}, no efficient algorithm can be expected due to the support shifting issue. We hence adopt a standard fail-state assumption \citep{Arxiv2024_DRORLwithInteractiveData_Lu,Arxiv2024_UpperLowerDRRL_Liu} to ensure the worst-case kernel support will be covered by the nominal one, bypassing the issue.
% hence enabling sample-efficient robust RL through interactive data collection. 
\begin{assu}[Failure States]
\label{ass:vanisin_minimal}
For any agent $i$, there exists an (agent-specified) set of failure states $\mathcal{S}_f^i\subseteq\mathcal{S}$, such that $r_i(s,\bm{a})=0$,  and $P^\star_h(s'|s,\bm{a})=0$, $\forall \bm{a}\in\mathcal{A},\forall s\in\mathcal{S}^i_f,\forall s'\notin\mathcal{S}^i_f$. 
\end{assu}
% \begin{assu}[Vanishing minimal value]
% \label{ass:vanisin_minimal}
%    We assume that the underlying {\RMGTV} satisfies that 
%    \begin{align}
%        \min_{s \in \mathcal{S}} V^{\dagger,\pi_{-i},\rho_i}_{i,1}(s) = 0.
%    \end{align}
%    % WLOG, we assume that the initial state $s_1 \notin \arg\min_{s \in \mathcal{S}} V^{\dagger,\pi_{-i},\rho_i}_{i,1}(s)$.
% \end{assu}
% \noindent
This assumption is only needed for the TV case.  \Cref{ass:vanisin_minimal} is a standard assumption in single-agent robust RL studies \citep{panaganti2022robust,Arxiv2024_DRORLwithInteractiveData_Lu}, and we adapt it to multi-agent cases. 

% Assuming that the initial state $s_1 \notin \arg\min_{s \in \mathcal{S}} V^{\dagger,\pi_{-i},\rho_i}_{i,1}(s)$ avoids making the problem trivial.  A close look at \Cref{ass:vanisin_minimal} actually gives that the minimal robust value function of any joint policy $\pi$ at any step is zero, i.e., $\min_{s \in \mathcal{S}} V^{\pi,\rho_i}_{i,1}(s) = 0$ for any joint policy $\pi$ and any step $h \in [H]$. 

We then present our threotical guarantees. 
% design of the bonus term and regret. The bonus term and the lemmas used for its proof are adapted from \citep{Arxiv2024_DRORLwithInteractiveData_Lu} to our multi-agent setting.
\begin{thm}[Upper bound of {\AlgonameTV}]
\label{thm:Regret_TV_bound_Bernstein}
Denote $\rho_{\min}:=\min_{i \in \mathcal{M}}\rho_i$. For any $\delta\in (0,1)$, we set $\beta^k_{i,h,f}(s,{\bf a})$ as 
 $\sqrt{ \frac{
        c_1 \iota \, \mathrm{Var}_{\widehat{P}_h^k(\cdot|s,{\bf a})} \left[
            \frac{\overline{V}_{i,h+1}^{k,\rho_i} + \underline{V}_{i,h+1}^{k,\rho_i}}{2}
        \right]
    }{
        N^k_h(s,{\bf a}) \vee 1
    }
}+ \frac{c_2 H^2 S \iota}{\sqrt{N_h^k(s,{\bf a}) \vee 1}} + \frac{2 \mathbb{E}_{\widehat{P}^k_h(\cdot|s, {\bf a})}\left[\up{V}^{k,\rho_i}_{i,h+1} - \low{V}^{k,\rho_i}_{i,h+1}\right]}{H} + \frac{1}{\sqrt{K}}$,
% \begin{align}
% \label{eq:Bonus_term_TV_Bernstein}
%     \beta^k_{i,h}(s,{\bf a}) &= \sqrt{\frac{c_1\iota \mathrm{Var}_{\widehat{P}_h^k(\cdot|s,{\bf a})} \left[ \left( \frac{ \overline{V}_{i,h+1}^{k,\rho_i} + \underline{V}_{i,h+1}^{k,\rho_i} }{2} \right) \right]}{\{N^k_h(s,{\bf a})\vee 1\}}} + \frac{2\mathbb{E}_{\widehat{P}^k_h(\cdot|s, {\bf a})}\left[\up{V}^{k,\rho_i}_{i,h+1}- \low{V}^{k,\rho_i}_{i,h+1}\right]}{H} + \frac{c_2H^2 S\iota}{\sqrt{\{N_h^k(s,{\bf a}) \vee 1\}}} + \frac{1}{\sqrt{K}},
% \end{align}
where $\iota = \log\Big(S^2(\prod_{i=1}^m A_i)H^2K^{3/2}/\delta\Big)$ and $c_1, c_2$ are absolute constants. Then, under \Cref{ass:vanisin_minimal}, for $\Eq$ being one of $\{\text{NASH}, \text{CE}, \text{CCE}\}$, with probability at least $1-\delta$, the regret of our {\AlgonameTV} algorithm can be bounded as: $\Reg_{\{\sf NASH, CCE, CE\}}(K) =\tilde{\mathcal{O}}\left(
\sqrt{ \min\left\{\rho_{\min}^{-1}, H \right\} H^2 S K  \Big( \prod_{i \in \mathcal{M}} A_i \Big)}\right)$.
% \begin{align*}
% \Reg_{\{\sf NASH, CCE, CE\}}(K) =\tilde{\mathcal{O}}\left(
% \sqrt{ \min\left\{\rho_{\min}^{-1}, H \right\} H^2 S K  \Big( \prod_{i \in \mathcal{M}} A_i \Big)}\right), 
% \end{align*}
% where $f(K)=\tilde{\mathcal{O}}(g(K))$ means $f(K)\leq \textbf{Poly}(\log(K))\cdot g(K)$ for sufficiently large $K$ and some polynomial of $\log(K)$. 
% where $C_1>0$ is an absolute constant, and $\iota^{\prime} = \log^2\Big(\frac{SHK\big(\prod_{i\in \mathcal{M}}A_i\big)}{\delta}\Big)$.
% \begin{itemize}
%     \item $\pi^{out}$ is an $\epsilon$-approximate $\{\text{NASH}, \text{CE}, \text{CCE}\}$, if the number of episodes 
%         \begin{equation*}
%         K = \Omega\Bigg( \frac{\min\Big\{\frac{1}{\rho_{\min}},H\Big\}H^2S\Big(\prod_{i\in \mathcal{M}}A_i\Big)\iota^{\prime\prime}}{\epsilon^2}\Bigg), \quad \text{ where } \quad \iota^{\prime\prime} = \log^2\Bigg(\frac{SH\big(\prod_{i\in \mathcal{M}}A_i\big)}{\delta}\Bigg).
%     \end{equation*} 
%     \item  $\Reg_{\sf Nash}(K)= \mathcal{O}\Bigg( \sqrt{\min\Big\{\frac{1}{\rho_{\min}},H\Big\}H^2SK\Big(\prod_{i\in \mathcal{M}}A_i\Big)\iota^{\prime}}\Bigg),$ where $\iota^{\prime} = \log^2\Big(\frac{SHK\big(\prod_{i\in \mathcal{M}}A_i\big)}{\delta}\Big)$.
% \end{itemize}

%   \begin{align}
% \label{eq:Regret_TV_bound}
%     \Reg_{\{\sf NASH, CE, CCE\}}(K) \leq \mathcal{O}\Bigg( \min\bigg\{\frac{1}{\min_{i\in \mathcal{M}} \rho_i},H\bigg\}\sqrt{H^4S^2K\iota\prod_{i\in \mathcal{M}}A_i}\Bigg),
% \end{align}
% where $\iota = \log\Big(\frac{S^2H^2K^{3/2}\prod_{i\in \mathcal{M}}A_i}{\delta}\Big)$. 
\end{thm}

\subsection{Regret Bound for KL-Divergence}
We then study the regret bound of the KL-divergence set. As discussed, the KL set is free from the support issue hence no additional assumption is required. Our regret bound result is as follows.

% The bonus terms and lemmas used for its proof are adapted from \citep{ghosh2025provablynearoptimaldistributionallyrobust}.  We derive the regret bound as follows.

% \citep{Arxiv2022_OfflineDRRLLinearFunctionApprox_Ma, AnnalsStat2022_TheoreticalUnderstandingRMDP_Yang,NeuRIPS2023_CuriousPriceDRRLGenerativeMdel_Shi},  
% \begin{def}\label{def:KL_P_min}
    
% \end{def}

%     There exists a constant
% $P^{\star}_{\min}>0$, such that for any $(h,s,\bm{a},s') \in [H]\times \mathcal{S}\times \mathcal{A} \times \mathcal{S}$, if $P^\star_h(s'|s,\bm{a})>0$, then $P^\star_h(s'|s,\bm{a})>P^\star_{\min}$.
% \end{assu}
\begin{thm}
\label{thm:Regret_KL_bound}
For any $\delta$, set $\beta^{k}_{i,h,f}(s,\bm{a})$ in {\RMGKL} as  
$  \frac{2 c_f H}{\rho_i} 
\sqrt{
    \frac{\iota}{
        \big(N_h^k(s,\bm{a}) \vee 1\big) \widehat{P}^k_{\min,h}(s,\bm{a})
    }
} + \sqrt{\frac{1}{K}}$,
where $\widehat{P}^k_{\min,h}(s,\bm{a}) = \min\limits_{s'\in \mathcal{S}}\{\widehat{P}^k_{h}(s'|s,\bm{a}): \widehat{P}^k_{h}(s'|s,\bm{a})>0\}$, $\iota = \log\Big(S^2(\prod_{i=1}^m A_i)H^2K^{3/2}/\delta\Big)$, and  $c_f$ is an absolute constant. Then for $\Eq$ being one of $\{\sf NASH, CE, CCE\}$, with probability at least $1-\delta$, it holds that $\Reg_{\{\sf NASH, CCE, CE\}}(K)  =\tilde{\mathcal{O}}\left(\sqrt{ H^4 \exp(2H^2) K S\Big(\prod_{i \in \mathcal{M}} A_i\Big)\Big(\rho^2_{\min} P^\star_{\min}\Big)^{-1}}\right),$ 
where $P^\star_{\min}\triangleq \min_{(s,\bm{a},s',h):P_h(s'|s,\bm{a})>0} P(s'|s,\bm{a})$ is the smallest positive entry of the nominal kernel. 
% where $C_2>0$ is an absolute constant, and  $\iota^{\prime} = \log^3\left(\frac{S H K \left(\prod_{i=1}^{m} A_i\right)}{\delta}\right)$.

% \begin{itemize}
%     \item $\pi^{out}$ is an $\epsilon$-approximate $\{\text{NASH}, \text{CE}, \text{CCE}\}$, if the number of episodes 
%     \begin{equation*}
%         K = \Omega\Bigg(\frac{H^4\exp(2H^2)S\big(\prod_{i\in \mathcal{M}}A_i\big)\iota^{\prime} }{\rho^2_{\min}P^\star_{\min}}\Bigg), \quad \text{ where }\quad  \iota^{\prime\prime} = \log^3\Bigg(\frac{SH\big(\prod_{i=1}^m A_i\big)}{\delta}\Bigg).
%     \end{equation*}
%     \item $\Reg_{\{\sf NASH, CE, CCE\}}(K) = \mathcal{O}\Bigg(\sqrt{\frac{H^4\exp(2H^2)KS\big(\prod_{i\in \mathcal{M}}A_i\big)\iota^{\prime} }{\rho^2_{\min}P^\star_{\min}}}\Bigg),$ where $\iota^\prime = \log^3\Bigg(\frac{SHK\big(\prod_{i=1}^m A_i\big)}{\delta}\Bigg)$.
% \end{itemize}
\end{thm}
\textcolor{black}{We note that the $\exp(H)$ term is inherently from the duality form of the distributionally robust optimization with KL-ball (see \eqref{eq:dual_KL}). It is standard in existing robust RL studies under KL settings, and can be directly replaced by $(P_{\min}^\star)^{-1}$ (see, e.g.,  \citep{PMLR2022_SampleComplexityRORL_Panaganti,NeuRIPS2023_DoublePessimismDROfflineRL_Blanchet, ghosh2025provablynearoptimaldistributionallyrobust,si2020distributionally,PMLR2023_ImprovedSampleComplexityDRRL_Xu,PMLR2021_DROTabularRL_Zhou}). It reflects the inherent hardness of the KL-based robust RL, and is inevitable in sample complexity. In practice, for moderate horizons, $P_{\min}^\star > 0$, and non-vanishing $\sigma$, these worst-case factors remain controlled and do not pose serious issues.}

% \color{blue}
We then briefly discuss the construction of $\beta$ under the two cases. Recall that in our meta–algorithm $f$-MORNAVI, for each agent $i$, episode $k$ and step $h$,
we maintain an optimistic and a pessimistic robust $Q$–estimate
\(
Q_{k,i,h}^{\rho_i}(s,a), \underline Q_{k,i,h}^{\rho_i}(s,a),
\)
defined via the empirical robust Bellman operators as in eqs \ref{eq:robust_Qupper_k}-\ref{eq:robust_Qlower_k}, 
and shifted by an exploration bonus $\beta^k_{i,h,f}(s,a)\ge 0$. 
We use $\sigma_{\mathcal P}[V] := \inf_{P\in\mathcal P}\mathbb E_{P}[V]$ for the support function over the uncertainty set.
The purpose of the bonus is to make these estimates form a tight, uniform high–probability
confidence interval around the true robust $Q$–values, i.e.
\begin{equation}
  Q^{\dagger,\pi_{-i},\rho_i}_{i,h}(s,a)
  \in \Big[\underline Q_{k,i,h}^{\rho_i}(s,a),\; \overline{Q}_{k,i,h}^{\rho_i}(s,a)\Big]
  \quad\text{for all } (i,h,k,s,a).
  \label{eq:robust-ci}
\end{equation}
% Once \eqref{eq:robust-ci} holds, we could apply the telescoping argument bounds the regret by the cumulative width
% $\sum_{k,h}(V_{k,i,h}^{\rho_i}-\underline V_{k,i,h}^{\rho_i})(s_1^k)$,
% and it remains to ensure that this width is (i) summable in $k$ and (ii) scales optimally with the
% problem parameters. 
% We now sketch how the concrete forms of the bonuses are obtained.

\textbf{TV–uncertainty.}
For TV–balls, we use the dual representation of the robust Bellman operator in \eqref{eq:dual_TV}.
Under Assumption~3 (failure states), it holds that $\min_s V(s)=0$, and the deviation between the true and empirical robust operators at $(h,s,a)$ then decomposes as
\begin{align*}
\Big|
  \sigma_{\mathcal P^{\rho_i}_{\mathrm{TV}}(P^\star_{h}(\cdot|s,a))}[V]
 -\sigma_{\mathcal P^{\rho_i}_{\mathrm{TV}}(\widehat P^{k}_{h}(\cdot|s,a))}[V]
\Big|\le \max_{\eta\in[0,H/\rho_{\min}]}
  \Big|\mathbb E_{P^\star_h(\cdot|s,a)}[V_\eta] - \mathbb E_{\widehat P^{k}_h(\cdot|s,a)}[V_\eta]\Big|.
\end{align*}

% Through the standard $1/(S\sqrt K)$-net technique \citep{JMLR2024_DROfflineRLNearOptimalSampelComplexity_Shi,li2024settling}, it remains to consider the two errors with a vector $U$ independent of $\hat{P}$. 

To simultaneously control the estimation error for all $(i,h,k,s,a)$ and all value functions of the form
$V=V^{\rho_i}_{k,i,h+1}$ and $\underline V^{\rho_i}_{k,i,h+1}$, we utilize the standard $\epsilon$-net \citep{JMLR2024_DROfflineRLNearOptimalSampelComplexity_Shi,li2024settling} of the interval $[0,H/\rho_{\min}]$, 
 and construct a Bernstein–type concentration inequality for empirical expectations of the random functions
 $V_\eta$ as
\begin{equation}
\label{eq:TV-local-dev}
\Big|
\mathbb E_{P^\star_h(\cdot|s,a)}[U] - \mathbb E_{\widehat P^{k}_h(\cdot|s,a)}[U]
\Big|
\lesssim
\sqrt{\frac{\Var_{\widehat P^{k}_h(\cdot|s,a)}(U)\,\iota}{N^k_h(s,a)\vee 1}}
\;+\;
\frac{H^2\sqrt{S\iota}}{\sqrt{N^k_h(s,a)\vee 1}},
\end{equation}
for all $U$ with $\|U\|_\infty\le H$.
In our algorithm we set \(
U = \frac{\overline{V}^{\rho_i}_{k,i,h+1}+\underline V^{\rho_i}_{k,i,h+1}}{2},
\) and \(
\Delta V := \overline{V}^{\rho_i}_{k,i,h+1}-\underline V^{\rho_i}_{k,i,h+1},
\)
which allows us to relate the variance under $P^\star$ and $\widehat P^{k}$ and to control the gap
$\mathbb E[\Delta V]$ that appears in the robustness amplification term.
Combining \eqref{eq:TV-local-dev} with these comparisons yields
\begin{multline*}
\Big|
\sigma_{\mathcal P^{\rho_i}_{\mathrm{TV}}(P^\star_{h}(\cdot|s,a))}[V^{\rho_i}_{k,i,h+1}]
-\sigma_{\mathcal P^{\rho_i}_{\mathrm{TV}}(\widehat P^{k}_{h}(\cdot|s,a))}[V^{\rho_i}_{k,i,h+1}]
\Big|
\\
\lesssim
\sqrt{
  \frac{\Var_{\widehat P^{k}_h(\cdot|s,a)}
    \big[\tfrac12(V^{\rho_i}_{k,i,h+1}+\underline V^{\rho_i}_{k,i,h+1})\big]\,
    \iota}{N^k_h(s,a)\vee 1}
}
+ \frac{H^2\sqrt{S\iota}}{\sqrt{N^k_h(s,a)\vee 1}}
+ \frac{1}{H}\,\mathbb E_{\widehat P^{k}_h(\cdot|s,a)}\big[\Delta V\big].
\end{multline*}
This motivates choosing the TV–bonus as
$\beta^k_{i,h,f}(s,a)
=
\sqrt{
  \frac{c_1\iota\,\Var_{\widehat P^{k}_h(\cdot|s,a)}
   \big[\tfrac12(\overline{V}^{\rho_i}_{k,i,h+1}+\underline V^{\rho_i}_{k,i,h+1})\big]}
       {N^k_h(s,a)\vee 1}
}
+
\frac{2}{H}\,
\mathbb E_{\widehat P^{k}_h(\cdot|s,a)}\big[\Delta V\big]
+
\frac{c_2H^2\sqrt{S\iota}}{\sqrt{N^k_h(s,a)\vee 1}}
+
\frac{1}{\sqrt K}.$
% Here the first term is an empirical–Bernstein variance term, the second term compensates the
% robustness–induced bias propagation (and will telescope across $h$), the third term is the
% second–order correction coming from the union bound over $(s,a,h,k)$, and the small
% $K^{-1/2}$–term dominates residual lower–order fluctuations.
With this choice, Lemma~\ref{lem:Optimistic_pessimism_TV_NE_Bernstein} shows that
\eqref{eq:robust-ci} holds under TV–uncertainty.

\textbf{KL–uncertainty.}
For KL–balls we again appeal to the dual formulation \eqref{eq:dual_KL}.
Thus, the robust Bellman operator becomes a \emph{log–moment generating function} of $V$.
The key difficulty is that we now need to control the deviation between the true and empirical
log–MGFs,
\[
\left|
-\frac{1}{\lambda}\log\mathbb E_{P^\star_h(\cdot|s,a)}\!\big[\exp(-\lambda V)\big]
+
\frac{1}{\lambda}\log\mathbb E_{\widehat P^k_h(\cdot|s,a)}\!\big[\exp(-\lambda V)\big]
\right|,
\]
uniformly over all $(i,h,k,s,a)$ and the random value functions $V=V^{\rho_i}_{k,i,h+1}$ generated by
the algorithm. We utilize Hoeffding's inequality to derive a self–normalized concentration inequality for
empirical MGFs:
\(
\left|
\log\mathbb E_{P^\star}[e^{-\lambda V}] - \log\mathbb E_{\widehat P^k}[e^{-\lambda V}]
\right|
\lesssim
\sqrt{\frac{\iota}{(N^k_h(s,a)\vee 1)\,P^\star_{\min,h}(s,a)}}\,.
\)
% where $P^\star_{\min,h}(s,a)$ is the smallest non–zero transition probability from $(s,a)$ at step $h$(cf. eqs.~(100)–(102)), and $\iota$ is the same logarithmic factor as before.

Multiplying both sides by $H/\rho_i$ (since $\lambda\asymp \rho_i/H$) and using the boundedness 
$\|V\|_\infty\le H$ to control higher–order terms in the MGF expansion, we obtain the local deviation
\[
\Big|
\sigma_{\mathcal P^{\rho_i}_{\mathrm{KL}}(P^\star_h(\cdot|s,a))}[V]
-
\sigma_{\mathcal P^{\rho_i}_{\mathrm{KL}}(\widehat P^k_h(\cdot|s,a))}[V]
\Big|
\lesssim
\frac{H}{\rho_i}
\sqrt{\frac{\iota}{(N^k_h(s,a)\vee 1)\,P^\star_{\min,h}(s,a)}}.
\]
Since only the support of $P^\star$ matters, and we only observe empirical transitions, we replace
$P^\star_{\min,h}(s,a)$ by its empirical counterpart
$\widehat P^k_{\min,h}(s,a)$, at the cost of an extra factor that is absorbed into the constants
(cf. Lemma \ref{lem:Control_Bonus_KL}).
This leads to the KL–bonus as $\beta^k_{i,h,f}(s,a)
~=~
2c_f\frac{H}{\rho_i}
\sqrt{\frac{\iota}{(N^k_h(s,a)\vee 1)\,\widehat P^k_{\min,h}(s,a)}}
\;+\;
\sqrt{\frac{1}{K}}.$
% \begin{equation*}
% \beta^k_{i,h,f}(s,a)
% ~=~
% 2c_f\frac{H}{\rho_i}
% \sqrt{\frac{\iota}{(N^k_h(s,a)\vee 1)\,\widehat P^k_{\min,h}(s,a)}}
% \;+\;
% \sqrt{\frac{1}{K}}.
% \end{equation*}
% The factor $H/\rho_i$ reflects the curvature of the KL–robust Bellman operator (through the dual
% parameter $\lambda$), while the $1/\widehat P^k_{\min,h}(s,a)$ dependence captures the intrinsic
% difficulty of robustly estimating rare transitions: the worst–case kernel must consider the full support
% of $P^\star$, so small probabilities inevitably amplify estimation error.

% Summing these bonuses over $(k,h)$ and using Lemmas~\ref{lem:KL-optimism-NE}–\ref{lem:KL-optimism-CE}
% yields the regret bound in Theorem~\ref{thm:KL-regret}, with the characteristic dependence
% $\big(\rho_{\min}^{-2} (P^\star_{\min})^{-1} H^4\exp(2H^2)\big)$ that is known to be unavoidable
% for KL–robust RL \cite{Blanchet2023,Shi2024,Ghosh2025}.

\color{black}

\subsection{Sample Complexity}
As a direct corollary, we derive the sample complexity to learn an $\varepsilon$-equilibrium. Using a standard online-to-batch conversion \citep{NeuRIPS2001_OnlineLearningAlgo_Cesa}, we have the following results. 

% \begin{cor} (Sample Complexity).
% \label{cor:Sample_Complexity_bound}
% Under the same setup in \Cref{thm:Regret_TV_bound_Bernstein} and \Cref{thm:Regret_KL_bound}, with probability at least $1 - \delta$, the sample complexity of finding an $\epsilon$-equilibrium is
% \begin{align*}
% % &T = KH \\
% % &\geq 
% KH=
% \begin{cases}
% \tilde{\mathcal{O}}\bigg( \dfrac{ \min\left\{\rho_{\min}^{-1}, H \right\} H^3 S \Big( \prod_{i \in \mathcal{M}} A_i \Big)}{ \epsilon^2 } \bigg),  &\text{{\RMGTV}} \\
% \tilde{\mathcal{O}}\bigg( \dfrac{ H^5 \exp(2H^2) S \Big( \prod_{i \in \mathcal{M}} A_i \Big)}{ \rho_{\min}^2 P^\star_{\min} } \bigg), &\text{{\RMGKL}}
% \end{cases}.
% \end{align*} 
% \end{cor}

\begin{cor}[Sample Complexity]
\label{cor:Sample_Complexity_bound}
With  probability at least $1-\delta$, and under the settings of Theorem~\ref{thm:Regret_TV_bound_Bernstein} and Theorem~\ref{thm:Regret_KL_bound}, the number of samples required to find an $\epsilon$-approximate equilibrium is bounded as:
\begin{align*}
% &T = KH \\
% &\geq
KH =
\begin{cases}
\tilde{\mathcal{O}}\Big(\epsilon^{-2} \min\left\{\rho_{\min}^{-1}, H \right\} H^3 S \big( \prod_{i \in \mathcal{M}} A_i \big) \Big), &\text{for {\RMGTV}} \\
\tilde{\mathcal{O}}\Big( \epsilon^{-2}H^5 \exp(2H^2) S \big( \prod_{i \in \mathcal{M}} A_i \big)\big(\rho_{\min}^2 P^\star_{\min}\big)^{-1} \big), &\text{for {\RMGKL}}
\end{cases}.
\end{align*}
\end{cor}

Our results hence imply that, despite the inherent hardness of online learning in DRMGs, our algorithm is able to learn an equilibrium with efficient sample complexity. As we shall discuss in the next section, our complexity bounds are near-optimal (except the term $\prod_{i \in \mathcal{M}} A_i$).

% , which hence implies the efficiency of our method. 

% \begin{align*}
% &T = KH \\
% &\geq 
% \begin{cases}
% C_3\bigg( \dfrac{ \min\left\{ 1/\rho_{\min}, H \right\} H^3 S \Big( \prod_{i \in \mathcal{M}} A_i \Big) \iota^{\prime\prime} }{ \epsilon^2 } \bigg),  &\text{{\RMGTV}} \\
% C_4\bigg( \dfrac{ H^5 \exp(2H^2) S \Big( \prod_{i \in \mathcal{M}} A_i \Big) \iota^{\prime} }{ \rho_{\min}^2 P^\star_{\min} } \bigg), &\text{{\RMGKL}}
% \end{cases}
% \end{align*}

\section{Comparison with Prior Works and Discussion}
We then compare our results with prior works (the detailed Comparisons are shown in \Cref{table:comparison_table}).

% In \citep{li2025sample}, $f(H, \rho)=\big(H\rho - 1 + (1 - \rho)^H\big)/\rho^2$.
\begin{table*}[!thb]
\caption{Comparison with prior results. $C^\star_{u/p}$ are coverage coefficients for offline learning. }
\label{table:comparison_table}
\centering
\renewcommand{\arraystretch}{1.3}
\begin{tabular}{|c|c|c|}
\hline
\makecell{\textbf{Setting} \& \\\textbf{Algorithm}} & \textbf{Uncertainty Set} & \textbf{Sample Complexity} \\
\hline
\makecell{\textbf{Generative} \\ \citep{Arxiv2024_SampleEfficientMARL_Shi}} & TV & $\tilde{\mathcal{O}}\left({\epsilon^{-2}}{H^3 S (\prod_{i\in \mathcal{M}} A_i)}\min\left\{ \rho_{\min}^{-1},H\right\}\right)$\\
\hline
\makecell{\textbf{Generative} \\ \citep{jiao_minimax-optimal_2024}} & Contamination & \scalebox{1.0}{$\tilde{\mathcal{O}}(\epsilon^{-2} H^3 S (\sum_{i\in \mathcal{M}}A_i)\min\left\{\rho_{\min}^{-1},H\right\})$}\\
\hline
\makecell{\textbf{Generative} \\ \citep{shi2024breaking}} & TV (fictitious) & \scalebox{1.0}{$
\tilde{\mathcal{O}}\left(\epsilon^{-4} H^6 S (\sum_{i\in \mathcal{M}} A_i) \min\left\{\rho_{\min}^{-1},H\right\} \right) $}\\
\hline
\multirow{2}{*}{\makecell{\textbf{Offline} \\ \citep{NeuRIPS2023_DoublePessimismDROfflineRL_Blanchet}}} & KL &$\tilde{\mathcal{O}}\left({\epsilon^{-2} \rho_{\min}^{-2}}{C^\star_u H^4 \exp(H) S^2 (\prod_{i\in \mathcal{M}}A_i)}\right)$\\
\cline{2-3}
& TV &$\tilde{\mathcal{O}}\Big(\epsilon^{-2}{C^\star_uH^4 S^2 (\prod_{i\in \mathcal{M}}A_i)}\Big)$\\
\hline
\makecell{\textbf{Offline} \\ \citep{li2025sample}} & TV & \scalebox{1.0}{$
\tilde{\mathcal{O}}\left( \epsilon^{-2}{C^\star_p H^4 S (\sum_{i=1}^m A_i) \min\left\{f(H,\rho), H \right\} } \right)$}\\
\hline
\makecell{\textbf{Online} \\ \citep{ma2023decentralized}} & KL & $\tilde{\mathcal{O}}(\epsilon^{-2}H^5S(\max_i\{ A_i\})^2)$ (with an oracle)\\
\hline
\multirow{2}{*}{\makecell{\textbf{Online} \\ (\textbf{Our work})}} & TV & $\tilde{\mathcal{O}}\left( \epsilon^{-2}  H^3 S (\prod_{i \in \mathcal{M}} A_i)\min\left\{\rho_{\min}^{-1},H\right\} \right)$\\
\cline{2-3}
& KL & $\tilde{\mathcal{O}}\left(\epsilon^{-2}\rho^{-2}_{\min}(P^\star_{\min})^{-1}{H^5\exp(2H^2)S\big(\prod_{i\in \mathcal{M}}A_i\big)}\right)$\\
\hline
\makecell{\textbf{Generative} \\ \textit{Lower bound} \\ \citep{Arxiv2024_SampleEfficientMARL_Shi}} & TV & $\tilde{\Omega}\left( \epsilon^{-2} H^3 S (\max_{i \in \mathcal{M}} A_i)\min\left\{\rho_{\min}^{-1},H\right\} \right)$\\
\hline
\end{tabular}
\end{table*}
% \vspace{-0.5 in}

A substantial body of research on DRMGs has focused on two primary settings: (i) generative model setting,  where the agents can freely sample from all state-action pairs \citep{shi2024breaking, Arxiv2024_SampleEfficientMARL_Shi, jiao_minimax-optimal_2024}; (ii) offline setting, which relies on a comprehensive, pre-collected dataset \citep{NeuRIPS2023_DoublePessimismDROfflineRL_Blanchet,li2025sample}. As we discuss in \Cref{sec:challengs}, both of these avoid exploration and are therefore easier than the online regime we consider. Despite this added difficulty, our algorithm attains complexities comparable to those reported for the generative and offline settings.

For both uncertainty sets, our results match or improve upon previous results and the minimax lower bound in all parameters except for the action-product term, $\prod_iA_i$, under the generative model setting. In the offline setting, if the dataset is generated uniformly, the convergence coefficients $C^\star_{u/p}$ from \citep{li2025sample,NeuRIPS2023_DoublePessimismDROfflineRL_Blanchet} introduce an additional $\prod_iA_i$ term into the sample complexity. Consequently, our results also match or surpass the offline complexity in all parameter dependence.
% A substantial body of research on DRMGs has focused on two primary settings: the generative model setting,  and the offline setting. In the generative model setting, agents can freely sample from all state-action pairs, as seen in works like \citep{shi2024breaking, Arxiv2024_SampleEfficientMARL_Shi, jiao_minimax-optimal_2024}. The offline setting, by contrast, relies on a comprehensive, pre-collected dataset \citep{NeuRIPS2023_DoublePessimismDROfflineRL_Blanchet,li2025sample}. As we discuss in \Cref{sec:challengs}, both of these are significantly simpler than the online setting we consider because they do not require exploration. Despite the added difficulty of online learning, our algorithm achieves complexity results comparable to those found in the generative model and offline settings.
% For both uncertainty sets, our results either match or exceed previous results and the minimax lower bound in all parameters except for the product of the number of actions, $\prod_iA_i$, under the generative model setting. In the offline setting, if the dataset is generated uniformly, the convergence coefficients $C^\star_{u/p}$ from \citep{li2025sample,NeuRIPS2023_DoublePessimismDROfflineRL_Blanchet} introduce an additional $\prod_iA_i$ term into the sample complexity. Consequently, our results also match or surpass the offline complexity in all parameter dependence.
This raises an important open question: \textbf{Can any DRMG learning algorithm overcome the curse of multi-agency and eliminate the dependence on $\prod_iA_i$ under general settings?}

While some works \citep{shi2024breaking,jiao_minimax-optimal_2024,li2025sample,ma2023decentralized} have achieved independence from $\prod_iA_i$, it remains unclear whether these improvements are applicable to general DRMGs. Specifically, the results in \citep{shi2024breaking} and \citep{jiao_minimax-optimal_2024} are developed for special uncertainty sets with desirable properties. For instance, the fictitious TV uncertainty set in \citep{shi2024breaking} allows the global transition kernel to be estimated from a single agent's local information; and robust RL under contamination models is known to be equivalent to a non-robust problem with a specific discount factor \citep{wang2023policy}. And the improvement in the offline setting is attributed to the benefits of the coverage coefficient.

The only online method (which also breaks the curse of multi-agency) is presented in \citep{ma2023decentralized}. However, their algorithm relies on additional assumptions about uncertainty sets and a powerful oracle. This oracle is required to provide an $\epsilon$-accurate estimation of the worst-case performance, $\sigma_{\mathcal{P}_i}[V]$ (see Theorem 12 of \citep{ma2023decentralized}), without any need for exploration. A central challenge in the analysis of robust learning algorithms is precisely quantifying this estimation error, as demonstrated in works like \citep{NeuRIPS2023_CuriousPriceDRRLGenerativeMdel_Shi,PMLR2023_ImprovedSampleComplexityDRRL_Xu,PMLR2022_SampleComplexityRORL_Panaganti,NeurIPS2024_MinimaxOptimalOfflineRL_Liu}. By assuming the existence of such an oracle, they bypass this core challenge, which significantly reduces their sample complexity. Moreover, their results need additional assumptions on the radius $\rho$. For instance, it is assumed that $\rho\leq \frac{P^\star_{\min}}{H}$, whereas ours do not require any of them.

% \color{blue}
Therefore, the complexity reduction in these works is in fact a blessing of their specific  uncertainty set structures, the properties of offline coverage coefficients, or the use of an impractical oracle. Given our lower bound derived in \Cref{sec:challengs}, we argue that the dependence on the joint action space may be inevitable in DRMGs. In the robust settings, agents need to estimate the entire nominal kernel so that they can learn the worst-case from the uncertainty set through distributionally robust optimization, which requires samples from all joint actions to estimate the whole transition kernel; whereas in the non-robust case, there is only one transition kernel and agents can use samples to directly estimate the performance under it, instead of estimating the whole transition model. 
We leave the exploration of this direction, including whether practical relaxations and techniques can avoid it,  for future work. 
\color{black}

\section{Conclusion}
\label{sec:Conclusion}
% In this paper, we introduced the {\it Multiplayer Optimistic Robust Nash Value Iteration} algorithm, pioneering the study of online learning in DRMGs. Our work provides the first provable guarantees for this challenging setting, demonstrating that MORNAVI achieves low regret and efficiently identifies optimal robust policies for TV-divergence and KL-divergence uncertainty sets. These results establish a practical path toward developing truly robust multi-agent systems that learn directly from environmental interactions without reliance on simulators or large offline datasets. Despite the inherent hardness of online DRMGs, our algorithm achieves complexity results comparable to those in generative model and offline settings, often matching or surpassing prior benchmarks. This research, however, highlights a critical open question: whether online DRMG learning algorithms can overcome the curse of multi-agency and eliminate the dependence on the joint action space size. Future work will explore this fundamental challenge, aiming to advance the scalability of robust MARL. This work will pave the way for future research on scalable and theoretically grounded algorithms for robust MARL.

In this paper, we introduced the Multiplayer Optimistic Robust Nash Value Iteration (MORNAVI) algorithm, pioneering the study of online learning in DRMGs. Our work provides the first provable guarantees for this challenging setting, demonstrating that MORNAVI achieves low regret and efficiently identifies optimal robust policies for TV-divergence and KL-divergence uncertainty sets. This research establishes a practical path toward developing truly robust multi-agent systems that learn directly from environmental interactions. Despite the inherent hardness of online DRMGs, our algorithm achieves complexity results comparable to the generative model and offline settings. This work also highlights a critical open question: whether online DRMG learning algorithms can overcome the curse of multi-agency and eliminate the dependence on the joint action space size. Future work will explore this fundamental challenge to advance the scalability of robust MARL.

\section*{Acknowledgment}
This work was supported by DARPA under Agreement No. HR0011-24-9-0427 and NSF under Award CCF-2106339.

\bibliographystyle{ICLR/iclr2026_conference}
\bibliography{ICLR/ref}

@article{INFORM2005_RobusDP_Iyengar,
  title={{Robust Dynamic Programming}},
  author={Iyengar, Garud N},
  journal={Mathematics of Operations Research},
  volume={30},
  number={2},
  pages={257--280},
  year={2005},
  publisher={INFORMS}
}

@article{NeuRIPS2023_CuriousPriceDRRLGenerativeMdel_Shi,
  title={{The Curious Price of Distributional Robustness in Reinforcement Learning with a Generative Model}},
  author={Shi, Laixi and Li, Gen and Wei, Yuting and Chen, Yuxin and Geist, Matthieu and Chi, Yuejie},
  journal={Advances in Neural Information Processing Systems},
  volume={36},
  pages={79903--79917},
  year={2023}
}

@article{INFORMS2013_RobustMDP_wiesemann,
  title={{Robust Markov Decision Processes}},
  author={Wiesemann, Wolfram and Kuhn, Daniel and Rustem, Ber{\c{c}}},
  journal={Mathematics of Operations Research},
  volume={38},
  number={1},
  pages={153--183},
  year={2013},
  publisher={INFORMS}
}

@inproceedings{PMLR2021_DROTabularRL_Zhou,
  title={{Finite-Sample Regret Bound for Distributionally Robust Offline Tabular Reinforcement Learning}},
  author={Zhou, Zhengqing and Zhou, Zhengyuan and Bai, Qinxun and Qiu, Linhai and Blanchet, Jose and Glynn, Peter},
  booktitle={International Conference on Artificial Intelligence and Statistics},
  pages={3331--3339},
  year={2021},
  organization={PMLR}
}

@article{NeuRIPS2023_DoublePessimismDROfflineRL_Blanchet,
  title={{Double Pessimism is Provably Efficient for Distributionally Robust Offline Reinforcement Learning: Generic Algorithm and Robust Partial Coverage}},
  author={Blanchet, Jose and Lu, Miao and Zhang, Tong and Zhong, Han},
  journal={Advances in Neural Information Processing Systems},
  volume={36},
  pages={66845--66859},
  year={2023}
}

@article{AnnalsStat2022_TheoreticalUnderstandingRMDP_Yang,
  title={{Toward Theoretical Understandings of Robust Markov Decision Processes: Sample Complexity and Asymptotics}},
  author={Yang, Wenhao and Zhang, Liangyu and Zhang, Zhihua},
  journal={The Annals of Statistics},
  volume={50},
  number={6},
  pages={3223--3248},
  year={2022},
  publisher={Institute of Mathematical Statistics}
}

@article{NeuRIPS2024_UnifiedPessimismOfflineRL_Yue,
  title={{A Unified Principle of Pessimism for Offline Reinforcement Learning under Model Mismatch}},
  author={Wang, Yue and Sun, Zhongchang and Zou, Shaofeng},
  journal={Advances in Neural Information Processing Systems},
  volume={37},
  pages={9281--9328},
  year={2024}
}

@article{Arxiv2023_TowardsMinimaxOptimalityRobustRL_Clavier,
  title={{Towards Minimax Optimality of Model-based Robust Reinforcement Learning}},
  author={Clavier, Pierre and Pennec, Erwan Le and Geist, Matthieu},
  journal={arXiv preprint arXiv:2302.05372},
  year={2023}
}

@inproceedings{PMLR2023_ImprovedSampleComplexityDRRL_Xu,
  title={{Improved Sample Complexity Bounds for Distributionally Robust Reinforcement Learning}},
  author={Xu, Zaiyan and Panaganti, Kishan and Kalathil, Dileep},
  booktitle={International Conference on Artificial Intelligence and Statistics},
  pages={9728--9754},
  year={2023},
  organization={PMLR}
}

@article{Arxiv2009_EmpBernsteinBounds_Maurer,
  title={{Empirical Bernstein Bounds and Sample Variance Penalization}},
  author={Maurer, Andreas and Pontil, Massimiliano},
  journal={arXiv preprint arXiv:0907.3740},
  year={2009}
}

@inproceedings{PMLR2022_SampleComplexityRORL_Panaganti,
  title={{Sample Complexity of Robust Reinforcement Learning with a Generative Model}},
  author={Panaganti, Kishan and Kalathil, Dileep},
  booktitle={International Conference on Artificial Intelligence and Statistics},
  pages={9582--9602},
  year={2022},
  organization={PMLR}
}

@string{aaai="Proc. Conference on Artificial Intelligence (AAAI)"}

@string{colt="Proc. Annual Conference on Learning Theory (CoLT)"}

@string{el = "Electron. Lett."}

@string{iclr="Proc. International Conference on Learning Representations (ICLR)"}

@string{icml="Proc. International Conference on Machine Learning (ICML)"}

@string{nips="Proc. Advances in Neural Information Processing Systems (NIPS)"}

@string{nipsnew="Proc. Advances in Neural Information Processing Systems (NeurIPS)"}

@string{uai="Proc. International Conference on Uncertainty in Artificial Intelligence (UAI)"}

@article{Arxiv2024_DRORLwithInteractiveData_Lu,
  title={Distributionally robust reinforcement learning with interactive data collection: Fundamental hardness and near-optimal algorithm},
  author={Lu, Miao and Zhong, Han and Zhang, Tong and Blanchet, Jose},
  journal={The Thirty-eighth Annual Conference on Neural Information Processing Systemss},
  year={2024}
}

@book{Book2009_ConcIneq_Dubhashi,
  title={{Concentration of Measure for the Analysis  of Randomized Algorithms}},
  author={Dubhashi, Devdatt P and Panconesi, Alessandro},
  year={2009},
  publisher={Cambridge University Press}
}

@inproceedings{PMLR2017_MinimxRegretBoundNonRobustRL_Azar,
  title={{Minimax Regret Bounds for Reinforcement Learning}},
  author={Azar, Mohammad Gheshlaghi and Osband, Ian and Munos, R{\'e}mi},
  booktitle={International conference on machine learning},
  pages={263--272},
  year={2017},
  organization={PMLR}
}

@article{JMLR2024_DROfflineRLNearOptimalSampelComplexity_Shi,
  title={{Distributionally Robust Model-Based Offline Reinforcement Learning with Near-Optimal Sample Complexity}},
  author={Shi, Laixi and Chi, Yuejie},
  journal={Journal of Machine Learning Research},
  volume={25},
  number={200},
  pages={1--91},
  year={2024}
}

@article{NeuRIPS2001_OnlineLearningAlgo_Cesa,
  title={{On the Generalization Ability of On-Line Learning Algorithms}},
  author={Cesa-Bianchi, Nicol{\'o} and Conconi, Alex and Gentile, Claudio},
  journal={Advances in neural information processing systems},
  volume={14},
  year={2001}
}

@article{INFORMS2005_RMDP_Nilim,
  title={{Robust Control of Markov Decision Processes with Uncertain Transition Matrices}},
  author={Nilim, Arnab and El Ghaoui, Laurent},
  journal={Operations Research},
  volume={53},
  number={5},
  pages={780--798},
  year={2005},
  publisher={INFORMS}
}

@article{Arxiv2024_UpperLowerDRRL_Liu,
  title={Upper and lower bounds for distributionally robust off-dynamics reinforcement learning},
  author={Liu, Zhishuai and Wang, Weixin and Xu, Pan},
  journal={arXiv preprint arXiv:2409.20521},
  year={2024}
}

@article{NeurIPS2024_MinimaxOptimalOfflineRL_Liu,
  title={{Minimax Optimal and Computationally Efficient Algorithms for Distributionally Robust Offline Reinforcement Learning}},
  author={Liu, Zhishuai and Xu, Pan},
  journal={Advances in Neural Information Processing Systems},
  volume={37},
  pages={86602--86654},
  year={2024}
}

@article{Arxiv2023_SoftRMDPRiskSensitive_Zhang,
  title={{Soft Robust MDPs and Risk-Sensitive MDPs: Equivalence, Policy Gradient, and Sample Complexity}},
  author={Zhang, Runyu and Hu, Yang and Li, Na},
  journal={arXiv preprint arXiv:2306.11626},
  year={2023}
}

@inproceedings{PMLR2023_SampleComplexityDRQLearning_Wang,
  title={{A Finite Sample Complexity Bound for Distributionally Robust Q-learning}},
  author={Wang, Shengbo and Si, Nian and Blanchet, Jose and Zhou, Zhengyuan},
  booktitle={International Conference on Artificial Intelligence and Statistics},
  pages={3370--3398},
  year={2023},
  organization={PMLR}
}

@article{JMLR2024_SampleComplexityVarianceReducedDRQLearning_Wang,
  title={{Sample Complexity of Variance-Reduced Distributionally Robust Q-Learning}},
  author={Wang, Shengbo and Si, Nian and Blanchet, Jose and Zhou, Zhengyuan},
  journal={Journal of Machine Learning Research},
  volume={25},
  number={341},
  pages={1--77},
  year={2024}
}

@article{Arxiv2022_OfflineDRRLLinearFunctionApprox_Ma,
  title={{Distributionally Robust Offline Reinforcement Learning with Linear Function Approximation}},
  author={Ma, Xiaoteng and Liang, Zhipeng and Blanchet, Jose and Liu, Mingwen and Xia, Li and Zhang, Jiheng and Zhao, Qianchuan and Zhou, Zhengyuan},
  journal={arXiv preprint arXiv:2209.06620},
  year={2022}
}

@article{NeurIPS2021_OnlineRobustRLModelUncertainty_Wang,
  title={{Online Robust Reinforcement Learning with Model Uncertainty}},
  author={Wang, Yue and Zou, Shaofeng},
  journal={Advances in Neural Information Processing Systems},
  volume={34},
  pages={7193--7206},
  year={2021}
}

@inproceedings{PMLR2019_TighterProblemDependentRegretRL_Zanette,
  title={{Tighter Problem-Dependent Regret Bounds in Reinforcement Learning without Domain Knowledge using Value Function Bounds}},
  author={Zanette, Andrea and Brunskill, Emma},
  booktitle={International Conference on Machine Learning},
  pages={7304--7312},
  year={2019},
  organization={PMLR}
}

@inproceedings{PMLR2021_UCBMomentumQLearning_Menard,
  title={{UCB Momentum Q-learning: Correcting the bias without forgetting}},
  author={M{\'e}nard, Pierre and Domingues, Omar Darwiche and Shang, Xuedong and Valko, Michal},
  booktitle={International Conference on Machine Learning},
  pages={7609--7618},
  year={2021},
  organization={PMLR}
}

@inproceedings{PMLR2021_RLDifficultThanBandits_Zhang,
  title={{Is Reinforcement Learning More Difficult Than Bandits? A Near-optimal Algorithm Escaping the Curse of Horizon}},
  author={Zhang, Zihan and Ji, Xiangyang and Du, Simon},
  booktitle={Conference on Learning Theory},
  pages={4528--4531},
  year={2021},
  organization={PMLR}
}

@article{vinitsky2020robust,
  title={Robust Reinforcement Learning using Adversarial Populations},
  author={Vinitsky, Eugene and Du, Yuqing and Parvate, Kanaad and Jang, Kathy and Abbeel, Pieter and Bayen, Alexandre},
  journal={arXiv preprint arXiv:2008.01825},
  year={2020}
}

@article{abdullah2019wasserstein,
  title={Wasserstein Robust Reinforcement Learning},
  author={Abdullah, Mohammed Amin and Ren, Hang and Ammar, Haitham Bou and Milenkovic, Vladimir and Luo, Rui and Zhang, Mingtian and Wang, Jun},
  journal={arXiv preprint arXiv:1907.13196},
  year={2019}
}

@inproceedings{zhang2020robust,
  title={Robust Multi-Agent Reinforcement Learning with Model Uncertainty},
  author={Zhang, Kaiqing and Sun, Tao and Tao, Yunzhe and Genc, Sahika and Mallya, Sunil and Basar, Tamer},
  booktitle=nipsnew,
  volume={33},
  year={2020}
}

@inproceedings{si2020distributionally,
  title={Distributionally robust policy evaluation and learning in offline contextual bandits},
  author={Si, Nian and Zhang, Fan and Zhou, Zhengyuan and Blanchet, Jose},
  booktitle=icml,
  pages={8884--8894},
  year={2020},
  organization={PMLR}
}

@InProceedings{wang2022policy,
  title ={Policy Gradient Method For Robust Reinforcement Learning},
  author ={Wang, Yue and Zou, Shaofeng},
  booktitle = icml,
  pages ={23484--23526},
  year ={2022},
  volume = 	 {162},
  publisher =    {PMLR},
}

@article{kardecs2011discounted,
  title={Discounted robust stochastic games and an application to queueing control},
  author={Karde{\c{s}}, Erim and Ord{\'o}{\~n}ez, Fernando and Hall, Randolph W},
  journal={Operations research},
  volume={59},
  number={2},
  pages={365--382},
  year={2011},
  publisher={INFORMS}
}

@inproceedings{liu2022dQlearning,
  title={Distributionally Robust {Q}-Learning},
  author={Liu, Zijian and Bai, Qinxun and Blanchet, Jose and Dong, Perry and Xu, Wei and Zhou, Zhengqing and Zhou, Zhengyuan},
  booktitle=icml,
  pages={13623--13643},
  year={2022},
  organization={PMLR}
}

@article{panaganti2022robust,
  title={Robust reinforcement learning using offline data},
  author={Panaganti, Kishan and Xu, Zaiyan and Kalathil, Dileep and Ghavamzadeh, Mohammad},
  journal={arXiv preprint arXiv:2208.05129},
  year={2022}
}

@inproceedings{jin2018q,
  title={Is {Q}-learning provably efficient?},
  author={Jin, Chi and Allen-Zhu, Zeyuan and Bubeck, Sebastien and Jordan, Michael I},
  booktitle=nipsnew,
  pages={4868--4878},
  year={2018}
}

@article{li2024settling,
  title={Settling the sample complexity of model-based offline reinforcement learning},
  author={Li, Gen and Shi, Laixi and Chen, Yuxin and Chi, Yuejie and Wei, Yuting},
  journal={The Annals of Statistics},
  volume={52},
  number={1},
  pages={233--260},
  year={2024},
  publisher={Institute of Mathematical Statistics}
}

@article{auer2010ucb,
  title={{UCB} revisited: Improved regret bounds for the stochastic multi-armed bandit problem},
  author={Auer, Peter and Ortner, Ronald},
  journal={Periodica Mathematica Hungarica},
  volume={61},
  number={1-2},
  pages={55--65},
  year={2010},
  publisher={Akad{\'e}miai Kiad{\'o}, co-published with Springer Science+ Business Media BV~…}
}

@inproceedings{wang2023model,
  title={Model-free robust average-reward reinforcement learning},
  author={Wang, Yue and Velasquez, Alvaro and Atia, George K and Prater-Bennette, Ashley and Zou, Shaofeng},
  booktitle=icml,
  pages={36431--36469},
  year={2023},
  organization={PMLR}
}

@article{li2022first,
  title={First-order policy optimization for robust markov decision process},
  author={Li, Yan and Lan, Guanghui and Zhao, Tuo},
  journal={arXiv preprint arXiv:2209.10579},
  year={2022}
}

@article{auer2002nonstochastic,
  title={The nonstochastic multiarmed bandit problem},
  author={Auer, Peter and Cesa-Bianchi, Nicolo and Freund, Yoav and Schapire, Robert E},
  journal={SIAM journal on computing},
  volume={32},
  number={1},
  pages={48--77},
  year={2002},
  publisher={SIAM}
}

@inproceedings{liu2025distributionally,
title={Distributionally Robust Multi-Agent Reinforcement Learning for Dynamic Chute Mapping},
author={Guangyi Liu and Suzan Iloglu and Michael Caldara and Joseph W Durham and Michael M. Zavlanos},
booktitle=icml,
year={2025}
}

@inproceedings{wang2023policy,
  title={Policy gradient in robust mdps with global convergence guarantee},
  author={Wang, Qiuhao and Ho, Chin Pang and Petrik, Marek},
  booktitle=icml,
  pages={35763--35797},
  year={2023},
  organization={PMLR}
}

@article{dong2022online,
  title={Online policy optimization for robust MDP},
  author={Dong, Jing and Li, Jingwei and Wang, Baoxiang and Zhang, Jingzhao},
  journal={arXiv preprint arXiv:2209.13841},
  year={2022}
}

@article{wang2024sample,
  title={Sample complexity of offline distributionally robust linear markov decision processes},
  author={Wang, He and Shi, Laixi and Chi, Yuejie},
  journal={arXiv preprint arXiv:2403.12946},
  year={2024}
}

@inproceedings{zhao2020sim,
  title={Sim-to-real transfer in deep reinforcement learning for robotics: a survey},
  author={Zhao, Wenshuai and Queralta, Jorge Pe{\~n}a and Westerlund, Tomi},
  booktitle={2020 IEEE symposium series on computational intelligence (SSCI)},
  pages={737--744},
  year={2020},
  organization={IEEE}
}

@inproceedings{
wang2024modelfree,
title={Model-Free Robust Reinforcement Learning with Sample Complexity Analysis},
author={Yudan Wang and Shaofeng Zou and Yue Wang},
booktitle=uai,
year={2024}
}

@inproceedings{
zhang2025modelfree,
title={Model-Free Offline Reinforcement Learning with Enhanced Robustness},
author={Chi Zhang and Zain Ulabedeen Farhat and George K. Atia and Yue Wang},
booktitle=iclr,
year={2025}
}

@article{wang2024robust,
  title={Robust Average-Reward Reinforcement Learning},
  author={Wang, Yue and Velasquez, Alvaro and Atia, George and Prater-Bennette, Ashley and Zou, Shaofeng},
  journal={Journal of Artificial Intelligence Research},
  volume={80},
  pages={719--803},
  year={2024}
}

@inproceedings{wang2022data,
  title={Data-driven robust multi-agent reinforcement learning},
  author={Wang, Yudan and Wang, Yue and Zhou, Yi and Velasquez, Alvaro and Zou, Shaofeng},
  booktitle={2022 IEEE 32nd International Workshop on Machine Learning for Signal Processing (MLSP)},
  pages={1--6},
  year={2022},
  organization={IEEE}
}

@article{liang2023single,
  title={Single-trajectory distributionally robust reinforcement learning},
  author={Liang, Zhipeng and Ma, Xiaoteng and Blanchet, Jose and Zhang, Jiheng and Zhou, Zhengyuan},
  journal={arXiv preprint arXiv:2301.11721},
  year={2023}
}

@inproceedings{zhang2024settling,
  title={Settling the sample complexity of online reinforcement learning},
  author={Zhang, Zihan and Chen, Yuxin and Lee, Jason D and Du, Simon S},
  booktitle=colt,
  pages={5213--5219},
  year={2024},
  organization={PMLR}
}

@inproceedings{peng2018sim,
  title={Sim-to-real transfer of robotic control with dynamics randomization},
  author={Peng, Xue Bin and Andrychowicz, Marcin and Zaremba, Wojciech and Abbeel, Pieter},
  booktitle={2018 IEEE international conference on robotics and automation (ICRA)},
  pages={3803--3810},
  year={2018},
  organization={IEEE}
}

@article{padakandla2020reinforcement,
  title={Reinforcement learning algorithm for non-stationary environments},
  author={Padakandla, Sindhu and KJ, Prabuchandran and Bhatnagar, Shalabh},
  journal={Applied Intelligence},
  volume={50},
  number={11},
  pages={3590--3606},
  year={2020},
  publisher={Springer}
}

@article{rajeswaran2016epopt,
  title={Epopt: Learning robust neural network policies using model ensembles},
  author={Rajeswaran, Aravind and Ghotra, Sarvjeet and Ravindran, Balaraman and Levine, Sergey},
  journal={arXiv preprint arXiv:1610.01283},
  year={2016}
}

@article{jiao_minimax-optimal_2024,
  title={Minimax-Optimal Multi-Agent Robust Reinforcement Learning},
  author={Jiao, Yuchen and Li, Gen},
  journal={arXiv preprint arXiv:2412.19873},
  year={2024}
}

@article{shi2024breaking,
  title={Breaking the Curse of Multiagency in Robust Multi-Agent Reinforcement Learning},
  author={Shi, Laixi and Gai, Jingchu and Mazumdar, Eric and Chi, Yuejie and Wierman, Adam},
  journal={arXiv preprint arXiv:2409.20067},
  year={2024}
}

@unpublished{li2025sample,
  title={Sample Efficient Robust Offline Self-Play for Model-based Reinforcement Learning},
  author={Li, Na and Zheng, Zewu and Ni, Wei and Shan, Hangguan and Zhang, Wenjie and Li, Xinyu},
  note={Manuscript, OpenReview preprint},
  year={2025},
  url={https://openreview.net/forum?id=3lXZjsir0e}
}

@article{ma2023decentralized,
  title={Decentralized robust v-learning for solving markov games with model uncertainty},
  author={Ma, Shaocong and Chen, Ziyi and Zou, Shaofeng and Zhou, Yi},
  journal={Journal of Machine Learning Research},
  volume={24},
  number={371},
  pages={1--40},
  year={2023}
}

@inproceedings{liu2021sharp,
  title={A sharp analysis of model-based reinforcement learning with self-play},
  author={Liu, Qinghua and Yu, Tiancheng and Bai, Yu and Jin, Chi},
  booktitle=icml,
  pages={7001--7010},
  year={2021},
  organization={PMLR}
}

@article{
han2024what,
title={What is the Solution for State-Adversarial Multi-Agent Reinforcement Learning?},
author={Songyang Han and Sanbao Su and Sihong He and Shuo Han and Haizhao Yang and Shaofeng Zou and Fei Miao},
journal={Transactions on Machine Learning Research},
issn={2835-8856},
year={2024},
url={https://openreview.net/forum?id=HyqSwNhM3x},
note={}
}

@article{shalev2016safe,
  title={Safe, multi-agent, reinforcement learning for autonomous driving},
  author={Shalev-Shwartz, Shai and others},
  journal={arXiv preprint arXiv:1610.03295},
  year={2016}
}

@inproceedings{lowe2017multi,
  title={Multi-agent actor-critic for mixed cooperative-competitive environments},
  author={Lowe, Ryan and Wu, Yi and Tamar, Aviv and Harb, Jean and Abbeel, Pieter and Mordatch, Igor},
  booktitle=nips,
  pages={6379--6390},
  year={2017}
}

@inproceedings{li2019robust,
  title={Robust multi-agent reinforcement learning via minimax deep deterministic policy gradient},
  author={Li, Shihui and Wu, Yi and Cui, Xinyue and Dong, Honghua and Fang, Fei and Russell, Stuart},
  booktitle=aaai,
  volume={33},
  number={01},
  pages={4213--4220},
  year={2019}
}

@article{matignon2012independent,
  title={Independent reinforcement learners in cooperative markov games: a survey regarding coordination problems},
  author={Matignon, Laetitia and Laurent, Guillaume J and Le Fort-Piat, Nadine},
  journal={The Knowledge Engineering Review},
  volume={27},
  number={1},
  pages={1--31},
  year={2012},
  publisher={Cambridge University Press}
}

@article{papoudakis2019dealing,
title={Dealing with non-stationarity in multi-agent deep reinforcement learning},
author={Papoudakis, Georgios and Christianos, Filippos and Rahman, Arrasy and Albrecht, Stefano V},
journal={arXiv preprint arXiv:1906.04737},
year={2019}
}

@inproceedings{bukharin2023robust,
  title={Robust multi-agent reinforcement learning via adversarial regularization: Theoretical foundation and stable algorithms},
  author={Bukharin, Alexander and Li, Yan and Yu, Yue and Zhang, Qingru and Chen, Zhehui and Zuo, Simiao and Zhang, Chao and Zhang, Songan and Zhao, Tuo},
  booktitle=nipsnew,
  volume={36},
  pages={68121--68133},
  year={2023}
}

@incollection{littman1994markov,
  title={Markov games as a framework for multi-agent reinforcement learning},
  author={Littman, Michael L},
  booktitle={Machine learning proceedings 1994},
  pages={157--163},
  year={1994},
  publisher={Elsevier}
}

@article{hu2003nash,
  title={Nash Q-learning for general-sum stochastic games},
  author={Hu, Junling and Wellman, Michael P},
  journal={Journal of machine learning research},
  volume={4},
  number={Nov},
  pages={1039--1069},
  year={2003}
}

@article{shapley1953stochastic,
  title={Stochastic games},
  author={Shapley, Lloyd S},
  journal={Proceedings of the national academy of sciences},
  volume={39},
  number={10},
  pages={1095--1100},
  year={1953},
  publisher={National Academy of Sciences}
}

@article{jin2021v,
  title={V-Learning--A Simple, Efficient, Decentralized Algorithm for Multiagent RL},
  author={Jin, Chi and Liu, Qinghua and Wang, Yuanhao and Yu, Tiancheng},
  journal={arXiv preprint arXiv:2110.14555},
  year={2021}
}

@article{roch2025finite,
  title={A Finite-Sample Analysis of Distributionally Robust Average-Reward Reinforcement Learning},
  author={Roch, Zachary and Zhang, Chi and Atia, George and Wang, Yue},
  journal={arXiv preprint arXiv:2505.12462},
  year={2025}
}

@article{yang2023robust,
  title={Robust markov decision processes without model estimation},
  author={Yang, Wenhao and Wang, Han and Kozuno, Tadashi and Jordan, Scott M and Zhang, Zhihua},
  journal={arXiv preprint arXiv:2302.01248},
  year={2023}
}

@article{zhang2021multi,
  title={Multi-agent reinforcement learning: A selective overview of theories and algorithms},
  author={Zhang, Kaiqing and Yang, Zhuoran and Ba{\c{s}}ar, Tamer},
  journal={Handbook of reinforcement learning and control},
  pages={321--384},
  year={2021},
  publisher={Springer}
}

@inproceedings{xie2020learning,
  title={Learning zero-sum simultaneous-move markov games using function approximation and correlated equilibrium},
  author={Xie, Qiaomin and Chen, Yudong and Wang, Zhaoran and Yang, Zhuoran},
  booktitle=colt,
  pages={3674--3682},
  year={2020},
  organization={PMLR}
}

@inproceedings{Roch2025reduction,
  title={A Reduction Framework for Distributionally Robust Reinforcement Learning under Average Reward},
  author={Roch, Zachary and Atia, George and Wang, Yue},
  booktitle=icml,
  year={2025},
  organization={PMLR}
}

@book{fudenberg1991game,
  title={Game theory},
  author={Fudenberg, Drew and Tirole, Jean},
  year={1991},
  publisher={MIT press}
}

@inproceedings{badrinath2021robust,
  title={Robust Reinforcement Learning using Least Squares Policy Iteration with Provable Performance Guarantees},
  author={Badrinath, Kishan Panaganti and Kalathil, Dileep},
  booktitle=icml,
  pages={511--520},
  year={2021},
  organization={PMLR}
}

@article{daskalakis2009complexity,
  title={The complexity of computing a Nash equilibrium},
  author={Daskalakis, Constantinos and Goldberg, Paul W and Papadimitriou, Christos H},
  journal={Communications of the ACM},
  volume={52},
  number={2},
  pages={89--97},
  year={2009},
  publisher={ACM New York, NY, USA}
}

@book{shoham_multiagent_2008,
	title = {Multiagent systems: {Algorithmic}, game-theoretic, and logical foundations},
	publisher = {Cambridge University Press},
	author = {Shoham, Yoav and Leyton-Brown, Kevin},
	year = {2008},
}

@article{xu2025efficient,
  title={Efficient $ Q $-Learning and Actor-Critic Methods for Robust Average Reward Reinforcement Learning},
  author={Xu, Yang and Ganesh, Swetha and Aggarwal, Vaneet},
  journal={arXiv preprint arXiv:2506.07040},
  year={2025}
}

@article{wong2023deep,
  title={Deep multiagent reinforcement learning: Challenges and directions},
  author={Wong, Annie and B{\"a}ck, Thomas and Kononova, Anna V and Plaat, Aske},
  journal={Artificial Intelligence Review},
  volume={56},
  number={6},
  pages={5023--5056},
  year={2023},
  publisher={Springer}
}

@article{canese2021multi,
  title={Multi-agent reinforcement learning: A review of challenges and applications},
  author={Canese, Lorenzo and Cardarilli, Gian Carlo and Di Nunzio, Luca and Fazzolari, Rocco and Giardino, Daniele and Re, Marco and Span{\`o}, Sergio},
  journal={Applied Sciences},
  volume={11},
  number={11},
  pages={4948},
  year={2021},
  publisher={MDPI}
}

@article{deng2023complexity,
  title={On the complexity of computing markov perfect equilibrium in general-sum stochastic games},
  author={Deng, Xiaotie and Li, Ningyuan and Mguni, David and Wang, Jun and Yang, Yaodong},
  journal={National Science Review},
  volume={10},
  number={1},
  pages={nwac256},
  year={2023},
  publisher={Oxford University Press}
}

@article{jin2022complexity,
  title={The complexity of infinite-horizon general-sum stochastic games},
  author={Jin, Yujia and Muthukumar, Vidya and Sidford, Aaron},
  journal={arXiv preprint arXiv:2204.04186},
  year={2022}
}

@article{Arxiv2024_SampleEfficientMARL_Shi,
  title={{Sample-Efficient Robust Multi-Agent Reinforcement Learning in the Face of Environmental Uncertainty}},
  author={Shi, Laixi and Mazumdar, Eric and Chi, Yuejie and Wierman, Adam},
  journal={arXiv preprint arXiv:2404.18909},
  year={2024}
}

@inproceedings{PMLR2020_ProvableSelfPlayCompetitiveRL_Bai,
  title={{Provable Self-Play Algorithms for Competitive Reinforcement Learning}},
  author={Bai, Yu and Jin, Chi},
  booktitle={International conference on machine learning},
  pages={551--560},
  year={2020},
  organization={PMLR}
}

@misc{he2023robustmultiagentreinforcementlearning,
      title={Robust Multi-Agent Reinforcement Learning with State Uncertainty}, 
      author={Sihong He and Songyang Han and Sanbao Su and Shuo Han and Shaofeng Zou and Fei Miao},
      year={2023},
      eprint={2307.16212},
      archivePrefix={arXiv},
      primaryClass={cs.LG},
      url={https://arxiv.org/abs/2307.16212}, 
}

@article{article,
author = {Kardes, Erim and Ordóñez, Fernando and Hall, Randolph},
year = {2011},
month = {04},
pages = {365-382},
title = {Discounted Robust Stochastic Games and an Application to Queueing Control},
volume = {59},
journal = {Operations Research},
doi = {10.2307/23013175}
}

@inproceedings{article1,
  title={Correlated Q-learning},
  author={Greenwald, Amy and Hall, Keith and Serrano, Roberto and others},
  booktitle={ICML},
  volume={3},
  pages={242--249},
  year={2003}
}

@InProceedings{pmlr-v167-chen22d,
  title = 	 {Almost Optimal Algorithms for Two-player Zero-Sum Linear Mixture Markov Games},
  author =       {Chen, Zixiang and Zhou, Dongruo and Gu, Quanquan},
  booktitle = 	 {Proceedings of The 33rd International Conference on Algorithmic Learning Theory},
  pages = 	 {227--261},
  year = 	 {2022},
  editor = 	 {Dasgupta, Sanjoy and Haghtalab, Nika},
  volume = 	 {167},
  series = 	 {Proceedings of Machine Learning Research},
  month = 	 {29 Mar--01 Apr},
  publisher =    {PMLR},
  url = 	 {https://proceedings.mlr.press/v167/chen22d.html}
}

@article{GO_DavidSilver,
author = {Silver, David and Huang, Aja and Maddison, Christopher and Guez, Arthur and Sifre, Laurent and Driessche, George and Schrittwieser, Julian and Antonoglou, Ioannis and Panneershelvam, Veda and Lanctot, Marc and Dieleman, Sander and Grewe, Dominik and Nham, John and Kalchbrenner, Nal and Sutskever, Ilya and Lillicrap, Timothy and Leach, Madeleine and Kavukcuoglu, Koray and Graepel, Thore and Hassabis, Demis},
year = {2016},
month = {01},
pages = {484-489},
title = {Mastering the game of Go with deep neural networks and tree search},
volume = {529},
journal = {Nature},
doi = {10.1038/nature16961}
}

@article{Vinyals2019GrandmasterLI,
  title={Grandmaster level in StarCraft II using multi-agent reinforcement learning},
  author={Oriol Vinyals and Igor Babuschkin and Wojciech M. Czarnecki and Micha{\"e}l Mathieu and Andrew Dudzik and Junyoung Chung and David Choi and Richard Powell and Timo Ewalds and Petko Georgiev and Junhyuk Oh and Dan Horgan and Manuel Kroiss and Ivo Danihelka and Aja Huang and L. Sifre and Trevor Cai and John P. Agapiou and Max Jaderberg and Alexander Sasha Vezhnevets and R{\'e}mi Leblond and Tobias Pohlen and Valentin Dalibard and David Budden and Yury Sulsky and James Molloy and Tom Le Paine and Caglar Gulcehre and Ziyun Wang and Tobias Pfaff and Yuhuai Wu and Roman Ring and Dani Yogatama and Dario W{\"u}nsch and Katrina McKinney and Oliver Smith and Tom Schaul and Timothy P. Lillicrap and Koray Kavukcuoglu and Demis Hassabis and Chris Apps and David Silver},
  journal={Nature},
  year={2019},
  volume={575},
  pages={350 - 354},
  url={https://api.semanticscholar.org/CorpusID:204972004}
}

@misc{hua2024multiagentreinforcementlearningconnected,
      title={Multi-Agent Reinforcement Learning for Connected and Automated Vehicles Control: Recent Advancements and Future Prospects}, 
      author={Min Hua and Dong Chen and Xinda Qi and Kun Jiang and Zemin Eitan Liu and Quan Zhou and Hongming Xu},
      year={2024},
      eprint={2312.11084},
      archivePrefix={arXiv},
      primaryClass={cs.RO},
      url={https://arxiv.org/abs/2312.11084}, 
}

@inproceedings{lu2020deepreinforcementlearningready,
  title={{Is deep reinforcement learning ready for practical applications in healthcare? A sensitivity analysis of duel-DDQN for hemodynamic management in sepsis patients}},
  author={Lu, MingYu and Shahn, Zachary and Sow, Daby and Doshi-Velez, Finale and Lehman, Li-wei H},
  booktitle={AMIA annual symposium proceedings},
  volume={2020},
  pages={773},
  year={2021}
}

@misc{alaa2023drlhealthcare,
  author       = {Alaa Eldin, Baraa},
  title        = {Why Applying Deep Reinforcement Learning in Healthcare is Hard},
  year         = {2023},
  howpublished = {\url{https://medium.com/@baraa.alaa.eldin/why-applying-deep-reinforcement-learning-in-healthcare-is-hard-ffc6e05ab7ca}},
  note         = {Accessed: 2025-07-28}
}

@misc{demontis2022surveyreinforcementlearningsecurity,
      title={A Survey on Reinforcement Learning Security with Application to Autonomous Driving}, 
      author={Ambra Demontis and Maura Pintor and Luca Demetrio and Kathrin Grosse and Hsiao-Ying Lin and Chengfang Fang and Battista Biggio and Fabio Roli},
      year={2022},
      eprint={2212.06123},
      archivePrefix={arXiv},
      primaryClass={cs.LG},
      url={https://arxiv.org/abs/2212.06123}, 
}

@misc{ramesh2023distributionallyrobustmodelbasedreinforcement,
      title={Distributionally Robust Model-based Reinforcement Learning with Large State Spaces}, 
      author={Shyam Sundhar Ramesh and Pier Giuseppe Sessa and Yifan Hu and Andreas Krause and Ilija Bogunovic},
      year={2023},
      eprint={2309.02236},
      archivePrefix={arXiv},
      primaryClass={cs.LG},
      url={https://arxiv.org/abs/2309.02236}, 
}

@misc{wang2024foundationdistributionallyrobustreinforcement,
      title={On the Foundation of Distributionally Robust Reinforcement Learning}, 
      author={Shengbo Wang and Nian Si and Jose Blanchet and Zhengyuan Zhou},
      year={2024},
      eprint={2311.09018},
      archivePrefix={arXiv},
      primaryClass={cs.LG},
      url={https://arxiv.org/abs/2311.09018}, 
}

@misc{lin2020robustnesscooperativemultiagentreinforcement,
      title={On the Robustness of Cooperative Multi-Agent Reinforcement Learning}, 
      author={Jieyu Lin and Kristina Dzeparoska and Sai Qian Zhang and Alberto Leon-Garcia and Nicolas Papernot},
      year={2020},
      eprint={2003.03722},
      archivePrefix={arXiv},
      primaryClass={cs.LG},
      url={https://arxiv.org/abs/2003.03722}, 
}

@article{Zhou_2024,
   title={A Robust Mean-Field Actor-Critic Reinforcement Learning Against Adversarial Perturbations on Agent States},
   volume={35},
   ISSN={2162-2388},
   url={http://dx.doi.org/10.1109/TNNLS.2023.3278715},
   DOI={10.1109/tnnls.2023.3278715},
   number={10},
   journal={IEEE Transactions on Neural Networks and Learning Systems},
   publisher={Institute of Electrical and Electronics Engineers (IEEE)},
   author={Zhou, Ziyuan and Liu, Guanjun and Zhou, Mengchu},
   year={2024},
   month=oct, pages={14370–14381} }

@article{busoniu2008comprehensive,
  title={A comprehensive survey of multiagent reinforcement learning},
  author={Busoniu, Lucian and Babuska, Robert and De Schutter, Bart},
  journal={IEEE Transactions on Systems, Man, and Cybernetics, Part C (Applications and Reviews)},
  volume={38},
  number={2},
  pages={156--172},
  year={2008},
  publisher={IEEE}
}

@article{oroojlooy2023review,
  title={A review of cooperative multi-agent deep reinforcement learning},
  author={Oroojlooy, Afshin and Hajinezhad, Davood},
  journal={Applied Intelligence},
  volume={53},
  number={11},
  pages={13677--13722},
  year={2023},
  publisher={Springer}
}

@inproceedings{littman1996generalized,
  title={A generalized reinforcement-learning model: Convergence and applications},
  author={Littman, Michael L and Szepesv{\'a}ri, Csaba},
  booktitle={ICML},
  volume={96},
  pages={310--318},
  year={1996}
}

@inproceedings{littman2001friend,
  title={Friend-or-foe Q-learning in general-sum games},
  author={Littman, Michael L and others},
  booktitle={ICML},
  volume={1},
  number={2001},
  pages={322--328},
  year={2001}
}

@article{fink1964equilibrium,
  title={Equilibrium in a stochastic $ n $-person game},
  author={Fink, Arlington M},
  journal={Journal of science of the hiroshima university, series ai (mathematics)},
  volume={28},
  number={1},
  pages={89--93},
  year={1964},
  publisher={Hiroshima University, Mathematics Program}
}

@article{daskalakis2013complexity,
  title={On the complexity of approximating a Nash equilibrium},
  author={Daskalakis, Constantinos},
  journal={ACM Transactions on Algorithms (TALG)},
  volume={9},
  number={3},
  pages={1--35},
  year={2013},
  publisher={ACM New York, NY, USA}
}

@article{hansen2013strategy,
  title={Strategy iteration is strongly polynomial for 2-player turn-based stochastic games with a constant discount factor},
  author={Hansen, Thomas Dueholm and Miltersen, Peter Bro and Zwick, Uri},
  journal={Journal of the ACM (JACM)},
  volume={60},
  number={1},
  pages={1--16},
  year={2013},
  publisher={ACM New York, NY, USA}
}

@article{mao2023provably,
  title={Provably efficient reinforcement learning in decentralized general-sum markov games},
  author={Mao, Weichao and Ba{\c{s}}ar, Tamer},
  journal={Dynamic Games and Applications},
  volume={13},
  number={1},
  pages={165--186},
  year={2023},
  publisher={Springer}
}

@article{song2021can,
  title={When can we learn general-sum Markov games with a large number of players sample-efficiently?},
  author={Song, Ziang and Mei, Song and Bai, Yu},
  journal={arXiv preprint arXiv:2110.04184},
  year={2021}
}

@inproceedings{cui2023breaking,
  title={Breaking the curse of multiagents in a large state space: Rl in markov games with independent linear function approximation},
  author={Cui, Qiwen and Zhang, Kaiqing and Du, Simon},
  booktitle={The Thirty Sixth Annual Conference on Learning Theory},
  pages={2651--2652},
  year={2023},
  organization={PMLR}
}

@article{feng2023improving,
  title={Improving sample efficiency of model-free algorithms for zero-sum Markov games},
  author={Feng, Songtao and Yin, Ming and Wang, Yu-Xiang and Yang, Jing and Liang, Yingbin},
  journal={arXiv preprint arXiv:2308.08858},
  year={2023}
}

@inproceedings{li2024provable,
  title={Provable Memory Efficient Self-Play Algorithm for Model-free Reinforcement Learning},
  author={Li, Na and Jiao, Yuchen and Shan, Hangguan and Yan, Shefeng},
  booktitle={The Twelfth International Conference on Learning Representations},
  year={2024}
}

@article{ghosh2025provablynearoptimaldistributionallyrobust,
  title={{ORVIT: Near-Optimal Online Distributionally Robust Reinforcement Learning}},
  author={Ghosh, Debamita and Atia, George K and Wang, Yue},
  journal={arXiv preprint arXiv:2508.03768},
  year={2025}
}

\appendix
% \tableofcontents
\onecolumn
\label{appendix}

\section{Use of Large Language Models}
We used ChatGPT only as a general-purpose assistant for language editing and typesetting. Its role was limited to (i) improving grammar, style, and readability, and (ii) LaTeX support—adjusting algorithm placement, tidying {\BibTeX{}} entries and citation styles, and resolving compile issues (e.g., Type-3 font warnings and package conflicts). All ideas, derivations, and final claims were conceived, checked, and validated by the authors, who bear full responsibility for the paper’s content.

\section{Related Works}
\label{sec:Related Works}
In this section we discuss other related works. 

\textbf{Single-Agent Robust RL.} Robust RL for single-agent settings has been extensively studied across a wide range of formulations. In particular, a substantial body of work has examined the generative-model setting \citep{Arxiv2023_TowardsMinimaxOptimalityRobustRL_Clavier, liu2022dQlearning, PMLR2022_SampleComplexityRORL_Panaganti, ramesh2023distributionallyrobustmodelbasedreinforcement, NeuRIPS2023_CuriousPriceDRRLGenerativeMdel_Shi, PMLR2023_SampleComplexityDRQLearning_Wang, wang2024foundationdistributionallyrobustreinforcement,wang2023model,JMLR2024_SampleComplexityVarianceReducedDRQLearning_Wang,wang2024robust, wang2023policy,wang2022policy,PMLR2023_ImprovedSampleComplexityDRRL_Xu, AnnalsStat2022_TheoreticalUnderstandingRMDP_Yang, yang2023robust,roch2025finite,Roch2025reduction,xu2025efficient}, where the agent is assumed to have access to a simulator. These studies develop distributionally robust RL algorithms under various uncertainty sets, including TV, KL, $\chi^2$, and Wasserstein divergences. 
    Another, and arguably more challenging, line of research focuses on the offline setting \citep{NeuRIPS2023_DoublePessimismDROfflineRL_Blanchet, Arxiv2022_OfflineDRRLLinearFunctionApprox_Ma, panaganti2022robust, JMLR2024_DROfflineRLNearOptimalSampelComplexity_Shi, Arxiv2023_SoftRMDPRiskSensitive_Zhang,NeurIPS2024_MinimaxOptimalOfflineRL_Liu,NeuRIPS2024_UnifiedPessimismOfflineRL_Yue,NeuRIPS2023_DoublePessimismDROfflineRL_Blanchet,wang2024sample}. In this setting, the agent must learn exclusively from a fixed offline dataset, without the ability to collect additional online samples. 
    Finally, we consider the online setting \citep{badrinath2021robust, dong2022online, li2022first, liang2023single, NeurIPS2021_OnlineRobustRLModelUncertainty_Wang}, where the agent learns exclusively through direct interaction with the environment. Prior work spans model-based, model-free, and policy-gradient approaches, with some methods, such as the policy optimization algorithm of \citep{dong2022online}, achieving sublinear regret guarantees.

\textbf{Robust MARL.} Besides the distributionally robust Markov games we considered in our paper, there are also other works that investigate robustness in MARL for cooperative tasks, where all agents share a unified objective. \citep{bukharin2023robust} enhance robustness through adversarial regularization, perturbing the environment to encourage Lipschitz-continuous policies. \citep{lin2020robustnesscooperativemultiagentreinforcement} explore adversarial attacks on MARL agents as a means of improving resilience, while \citep{li2019robust} extend this approach to continuous action spaces by modifying the MADDPG algorithm \citep{lowe2017multi} to focus on worst-case actions—a narrower interpretation of worst-case optimization in robust RL.  \citep{wang2022data} studied robust MARL with network agents.

    Another line of research focuses on the robustness in MARL under observation uncertainty, under the formulation of partially observable MDPs. The framework of  observation-robust games is proposed in \citep{he2023robustmultiagentreinforcementlearning, han2024what}. Observation-robust cooperative MARL is studied in  \citep{Zhou_2024}.
    
% \citep{bukharin2023robust} - Robust Cooperative MARL\\
% \citep{lin2020robustnesscooperativemultiagentreinforcement} \\
% \citep{li2019robust}\\
% \\
% \citep{}\\
%\\

 \textbf{Non-Robust Markov Games.} Markov games (MGs), or stochastic games, introduced by \citep{shapley1953stochastic}, form the standard foundation for MARL, particularly in equilibrium learning. Comprehensive surveys such as \citep{busoniu2008comprehensive, oroojlooy2023review, zhang2021multi} offer thorough coverage of the field’s evolution. Early work in MARL focused on asymptotic convergence guarantees \citep{littman2001friend, littman1996generalized}, whereas recent research emphasizes finite-sample analyses to establish non-asymptotic guarantees, especially for learning Nash equilibria (NE)—a central solution concept. The existence of NE in general-sum MGs was shown by \citep{fink1964equilibrium}, and the algorithmic foundation was laid by the seminal work of \citep{littman1994markov}. Classical algorithms such as Nash-Q \citep{hu2003nash}, FF-Q \citep{littman2001friend}, and correlated-Q learning \citep{article1} were proposed to compute NE and its variants. However, computing NE in general-sum multi-player settings remains PPAD-complete \citep{daskalakis2013complexity}, and no polynomial-time algorithms exist for this case \citep{jin2022complexity, deng2023complexity}. In contrast, the two-player zero-sum setting admits tractable solutions, with the first polynomial-time algorithm developed by \citep{hansen2013strategy}. To address the computational intractability in general-sum MGs, attention has shifted to weaker notions like CE and CCE, with polynomial-time algorithms such as V-learning \citep{jin2021v, mao2023provably, song2021can} and Nash value iteration \citep{liu2021sharp} enabling efficient computation. Furthermore, significant progress in finite-sample analysis—spanning both model-based and model-free algorithms—has been achieved in the two-player zero-sum setting, as evidenced by \citep{PMLR2020_ProvableSelfPlayCompetitiveRL_Bai, xie2020learning, cui2023breaking, pmlr-v167-chen22d, liu2021sharp, feng2023improving, li2024provable}, advancing the theoretical understanding of equilibrium learning in standard MARL without robustness considerations.

\section{DRMG with $f$-Divergence Uncertainty Set}
\label{sec:DRMG Uncertainty Set}

% In this section we briefly review the formulation of DRMG with $f$-divergence uncertainty sets. In this work, we specifically focus on general $f$-divergence uncertainty set under $\mathcal{S} \times \mathcal{A}$-rectangularity assumption, as defined in Definition ~\ref{def:f_divergence_uncertainty}, where $P^{\star}$ is the nominal transition probability and $\rho_i$ determines the radius of the set for each agetn $i \in \mathcal{M}$.

We review the formulation of DRMG with $f$-divergence uncertainty sets. This framework operates under the $\mathcal{S} \times \mathcal{A}$-rectangularity assumption, where the nominal transition probability $P^{\star}$ and the agent-specific radius $\rho_i$ for $i \in \mathcal{M}$ define the robust problem as per Definition~\ref{def:f_divergence_uncertainty}.

% old proposition

% \begin{prop}[Dual representation of $f$-divergence uncertainty set]
% \label{prop:strong_duality}
%    Under Definition~\ref{def:f_divergence_uncertainty}, for any $V_i:\mathcal{S}\to \mathbb{R}_{+}$ and $P^\star:\mathcal{S}\times \mathcal{A}\to \Delta(\mathcal{S})$, the dual representation for $     \sigma_{\mathcal{P}^{\rho_i}_f(s,\bm{a})}[V_i]:= \inf_{P \in \mathcal{P}^{\rho_i}_f(s,\bm{a})}\left[\mathbb{P}V_i\right](s,\bm{a})$, 
% can be formulated as 
% \begin{align*}
% \sigma_{\mathcal{P}^{\rho_i}_{i,h,f}(s,{\bf a})}[V] =
% \sup_{\lambda \geq 0,\, \eta \in \mathbb{R}} \bigg\{ 
%  -\lambda \sum_{s \in \mathcal{S}} P^\star(s) 
% f^{\star}\left( \frac{\eta - V(s)}{\lambda} \right) - \lambda \rho_i + \eta 
% \bigg\},
% \end{align*}
% % \begin{align*}        \sigma_{\mathcal{P}^{\rho_i}_{i,h,f}(s,{\bf a})}[V]:= \sup_{\lambda \geq 0, \eta \in \mathbb{R}} \bigg\{&-\lambda\sum\limits_{s\in\mathcal{S}}P^\star(s)f^{\star}\bigg(\frac{\eta-V(S)}{\lambda}\bigg)  -\lambda\rho_i + \eta\bigg\},
% % \end{align*}
%    where $f^{\star}(t) = -\inf_{y\geq 0}(f(y)-yt)$ is the convex conjugate function \citep{Book2004_ConvexOpt_Boyd} of $f$ with restriction to $y\geq 0$.
% \end{prop}

% Zain, this is the new proposition 1
\begin{lem}[Strong duality for $f$-divergence]
\label{lem:strong_duality}
Let $\mathcal{P}^{\rho_i}_{f}(s,\bm{a})$ be an $f$-divergence uncertainty set as defined in Definition~\ref{def:f_divergence_uncertainty}. For any value function $V_i:\mathcal{S}\to \mathbb{R}_{+}$ and a nominal transition kernel $P^\star:\mathcal{S}\times \mathcal{A}\to \Delta(\mathcal{S})$, the worst-case expected value,  $\sigma_{\mathcal{P}^{\rho_i}_f(s,\bm{a})}[V_i]:= \inf_{P \in \mathcal{P}^{\rho_i}_f(s,\bm{a})}\left[\mathbb{P}V_i\right](s,\bm{a})$, admits a dual representation given by:
\begin{align*}
\sigma_{\mathcal{P}^{\rho_i}_{i,h,f}(s,{\bf a})}[V] =
\sup_{\lambda \geq 0,\, \eta \in \mathbb{R}} \bigg\{
-\lambda \sum_{s \in \mathcal{S}} P^\star(s)
f^{\star}\left( \frac{\eta - V(s)}{\lambda} \right) - \lambda \rho_i + \eta
\bigg\},
\end{align*}
where $f^{\star}$ is the convex conjugate of $f$.
\end{lem}

The detailed proof is given in Lemma B.1 of \citep{AnnalsStat2022_TheoreticalUnderstandingRMDP_Yang}.

% The detailed proof of~\Cref{lem:strong_duality} is given in \citep[Lemma B.1]{AnnalsStat2022_TheoreticalUnderstandingRMDP_Yang}.

% \begin{cor}[Special cases of $f$-divergence sets: KL-divergence and TV-divergence]
%     Under $\mathcal{S} \times \mathcal{A}$-rectangularity assumption and Lemma~\ref{lem:strong_duality}, the duality representation for the robust expectation for any $V:\mathcal{S}\to [0,H]$ and $P^{\star}_h:\mathcal{S}\times \mathcal{A}\to \Delta(\mathcal{S})$ can be reformulated as 
%     \begin{enumerate}
%             \item \textbf{TV-Divergence: $f(t) = \frac{1}{2}\Big|t-1\Big|$}, and 
% \begin{align}
% \label{eq:dual_TV}
% \sigma_{\mathcal{P}^{\rho_i}_{i,h}(s,{\bf a})}[V_i] :=\,
% & \sigma_{\mathcal{P}^{\rho_i}_{i,h, \text{TV}}(s,{\bf a})}[V_i] = \sup_{\eta \in [0,H]} \bigg\{ 
%   - \mathbb{E}_{P^\star_h(\cdot|s,{\bf a})}\Big[(\eta - V_i)_{+}\Big] - \frac{\rho}{2} \left( \eta - \min_{s \in \mathcal{S}} V_i(s) \right)_{+} 
% + \eta \bigg\}.
% \end{align}

% \item \textbf{KL-Divergence: $f(t) = t\log(t)$}, and
% \begin{align}
% \label{eq:dual_KL}
% \sigma_{\mathcal{P}^{\rho_i}_{i,h}(s,{\bf a})}[V_i] := \sigma_{\mathcal{P}^{\rho_i}_{i,h, \text{KL}}(s,{\bf a})}[V_i] = \sup_{\eta \in [\underline{\eta}, H/\rho_i]} \bigg\{
% -\eta \log\left(\mathbb{E}_{P^\star_h(\cdot|s,{\bf a})}\left[\exp\bigg\{-\frac{V_i}{\eta}\bigg\}\right]\right) - \eta \rho_i
% \bigg\}.
% \end{align} 
% \end{enumerate}
%  \end{cor}

 \begin{cor}[Dual representation for TV and KL-divergence]
\label{cor:special_cases}
Under the assumption of $\mathcal{S} \times \mathcal{A}$-rectangularity, the dual representation from Lemma~\ref{lem:strong_duality} simplifies to the following for two specific cases of $f$-divergence. For any value function $V: \mathcal{S} \to [0,H]$ and a nominal distribution $P^{\star}_h$ over the next states:

\textbf{TV-Divergence.}
For an uncertainty set defined by TV-divergence, where $f(t) = \frac{1}{2}\Big|t-1\Big|$, the robust expectation $\sigma_{\mathcal{P}^{\rho_i}_{i,h,\text{TV}}(s,{\bf a})}[V_i]$ is expressed as:
\begin{align}
\label{eq:dual_TV}
\sigma_{\mathcal{P}^{\rho_i}_{i,h,TV}(s,{\bf a})}[V_i] = \sup_{\eta \in [0,H]} \bigg\{
 - \mathbb{E}_{P^\star_h(\cdot|s,{\bf a})}\Big[\max(0, \eta - V_i)\Big]-\frac{\rho}{2} \max(0, \eta - \min_{s' \in \mathcal{S}} V_i(s'))
+ \eta \bigg\}.
\end{align}

\textbf{KL-Divergence.}
For an uncertainty set defined by KL-divergence, with $f(t) = t\log(t)$, the robust expectation $\sigma_{\mathcal{P}^{\rho_i}_{i,h,\text{KL}}(s,{\bf a})}[V_i]$ is expressed as:
\begin{align}
\label{eq:dual_KL}
\sigma_{\mathcal{P}^{\rho_i}_{i,h,KL}(s,{\bf a})}[V_i] = \sup_{\eta \in [\underline{\eta}, H/\rho_i]} \bigg\{
-\eta \log\left(\mathbb{E}_{P^\star_h(\cdot|s,{\bf a})}\left[\exp\bigg\{-\frac{V_i}{\eta}\bigg\}\right]\right) - \eta \rho_i
\bigg\}.
\end{align}
\end{cor}

% \begin{rem}
% \label{ass:Regularity_KL}
%     For regularity assumption of KL-divergence duality variable, we assume that the optimal dual variable \( \eta^\star \) is lower bounded by \( \underline{\eta} > 0 \) for any nominal transition kernels \( P^\star_h\), and step \( h \in [H] \) \citep{NeuRIPS2023_DoublePessimismDROfflineRL_Blanchet, ICML2025_OnlineDRMDPSampleComplexity_He}.
% \end{rem}

\subsubsection{Robust Bellman Equations.}
Analogous to standard MGs, the following proposition provides the robust Bellman equation for DRMGs. In particular,  the robust value functions $V^{\pi, \rho_i}_{{i,h}}(s)$ associated with any joint policy $\pi$ for all $(i,h,s) \in \mathcal{M}\times [H]\times \mathcal{S}$ obeys the following proposition given below:

% \begin{prop}[Robust Bellman equation]
% \label{prop:Robust_Bellman_eq}
% Under $\mathcal{S} \times \mathcal{A}$-rectangularity assumption, for any nominal transition kernel $P^{\star}:=\{P^{\star}_h\}_{h=1}^H$ and any joint policy \( \pi = \{ \pi_h \}_{h=1}^H \), the following robust Bellman equation holds for any $(i,h,s,\bm{a})\in \mathcal{M}\times[H]\times \mathcal{S}\times \mathcal{A}$, as 
% \begin{align}
%     Q^{\pi,\rho_i}_{i,h}(s,\bm{a}) &= r_{i,h}(s,\bm{a}) + \sigma_{\mathcal{P}^{\rho_i}_{i,h}(s,\bm{a})}\left[V^{\pi,\rho_i}_{i,h+1}\right].\label{eq:Robust_bellman_Q_fn}\\
%     V^{\pi,\rho_i}_{i,h}(s) &= \mathbb{E}_{\bm{a} \sim \pi_h(\cdot|s)} \left[ Q^{\pi,\rho_i}_{i,h}(s,\bm{a}) \right].\label{eq:Robust_bellman_V_fn} 
% \end{align}
% \end{prop}

\begin{prop}[Robust Bellman Equation]
\label{prop:Robust_Bellman_eq}
Under the $\mathcal{S} \times \mathcal{A}$-rectangularity assumption, for any nominal transition kernel $P^{\star}$ and joint policy $\pi$, the robust Bellman equation holds for any $(i,h,s,\bm{a})$:
\begin{align}
Q^{\pi,\rho_i}_{i,h}(s,\bm{a}) &= r_{i,h}(s,\bm{a}) + \sigma_{\mathcal{P}^{\rho_i}_{i,h}(s,\bm{a})}\left[V^{\pi,\rho_i}_{i,h+1}\right]\label{eq:Robust_bellman_Q_fn}\\
V^{\pi,\rho_i}_{i,h}(s) &= \mathbb{E}_{\bm{a} \sim \pi_h(\cdot|s)} \left[ Q^{\pi,\rho_i}_{i,h}(s,\bm{a}) \right]\label{eq:Robust_bellman_V_fn}
\end{align}
\end{prop}

The detailed proof of Proposition \ref{prop:Robust_Bellman_eq} for finite-horizon RMDP is given in \citep[Proposition 2.3]{NeuRIPS2023_DoublePessimismDROfflineRL_Blanchet}. We emphasize that the robust Bellman equation in \ref{eq:Robust_bellman_V_fn} is fundamentally grounded in the agent-wise 
$(s,{\bf a})$-rectangularity condition imposed on the uncertainty set. This condition decouples the dependencies of uncertainty across agents, state-action pairs, and time steps, thereby enabling the recursive structure of the Bellman equation.

% \zain{move this under algorithm}
% \begin{align}
% \overline{Q}_{i,h}^{k, \rho_i}(s, \bm{a}) &=
% \min\left\{r_{i,h}(s,\bm{a}) + 
% \mathbb{E}_{\widehat{\mathcal{P}}_{i,h,f}^{\rho_i}(s, \bm{a})} [\overline{V}_{i,h+1}^{k, \rho_i}] + \beta^k_{i,h,f}(s,\bm{a}),\ 
% H \right\} \label{eq:Upper_estimate_Q} \\
% \underline{Q}_{i,h}^{k, \rho_i}(s, \bm{a}) &=
% \max\left\{r_{i,h}(s,\bm{a}) + 
% \mathbb{E}_{\widehat{\mathcal{P}}_{i,h,f}^{\rho_i}(s, \bm{a})}[\underline{V}_{i,h+1}^{k, \rho_i}] - 
% \beta^k_{i,h,f}(s,\bm{a}),\ 0 \right\} \label{eq:Lower_estimate_Q}
% \end{align}
% Update the robust value function estimation of $V^{\dagger,\pi_{-i}^k,\rho_i}_{i,h}$ as 
% \begin{align}
% &\overline{V}_{i,h}^{k,\rho_i}(s) = \mathbb{E}_{\bm{a} \sim \pi_{h}^k(\cdot|s)}[\overline{Q}_{i,h}^{k,\rho_i}(s,\bm{a})] = \max_{\pi_{i}'} \mathbb{E}_{\bm{a} \sim \pi_{i}' \times \pi_{-i}^k} [\overline{Q}_{i,h}^{k,\rho_i}(s,\bm{a})],\label{eq:policy_up_value_epsiode_k}\\
% &\underline{V}_{i,h}^{k,\rho_i}(s) = \mathbb{E}_{\bm{a} \sim \pi_{h}^k(\cdot|s)}[\underline{Q}_{i,h}^{k,\rho_i}(s,\bm{a})] = \max_{\pi_{i}'} \mathbb{E}_{\bm{a} \sim \pi_{i}' \times \pi_{-i}^k} [\underline{Q}_{i,h}^{k,\rho_i}(s,\bm{a})]. \label{eq:policy_low_value_epsiode_k}
% \end{align}
% \max_{a \in \mathcal{A}} \overline{Q}_h^k(s,a)

\section{Numerical Experiments}\label{app:exp}
% \color{blue}
In this section, we develop numerical experiments to validate our theoretical results. We highlight that numerical experiment for Markov games can be significantly challenging due to, e.g., the equilibrium identification challenge and computational barrier \citep{shoham_multiagent_2008}, hence we use some small-scale experiments to validate our results.

\subsection{Fully Cooperative DRMG}
As the first step in numerical experiment, we design a 2-agent, 2-step fully cooperative DRMG (with identical rewards for both players), to illustrate the separation between our robust learning algorithm and the non-robust ones in standard Markov games. 

The game is formally defined by the following components:
\begin{itemize}
    \item \textbf{Agents ($\mathcal{M}$):} The set of agents is $\mathcal{M} = \{1, 2\}$.
    \item \textbf{Horizon ($H$):} The game has a finite horizon of $H=2$.
    \item \textbf{State Space ($\mathcal{S}$):} The state space is $\mathcal{S} = \{s_0, s_H, s_M, s_T\}$. The game always starts in state $s_0$ at $h=1$. The states $s_H$ (High), $s_M$ (Medium), and $s_T$ (Trap) are the potential states for $h=2$, and the episode terminates after this step.
    \item \textbf{Action Space ($\mathcal{A}$):} Each agent has two actions, $\mathcal{A}_i = \{0, 1\}$ for $i \in \mathcal{M}$. The joint action space is $\mathcal{A} = \mathcal{A}_1 \times \mathcal{A}_2$, with joint actions $a = (a_1, a_2) \in \{(0,0), (0,1), (1,0), (1,1)\}$.
\end{itemize}

In our game, agents receive no reward at the first step: $r_{i,1}(s_0, a) = 0$ for all $i, a$. At step $h=2$, the reward $r_{i,2}(s, a)$ for both agents is determined by the current state $s \in \{s_H, s_M, s_T\}$ and the joint action $a$. The rewards are defined as:
\begin{itemize}
    \item \textbf{At $s_H$ (High):} This is the high-reward state, where $r_{i,2}(s_H, a) = 1$ for all $i, a$. 

    \item \textbf{At $s_M$ (Medium):} This is a medium-reward state, where $r_{i,2}(s_M, a) = 0.6$ for all $i, a$. 

    \item \textbf{At $s_T$ (Trap):} This is the low-reward, trap state, where $r_{i,2}(s_T, a) = 0$ for all $i, a$. 
\end{itemize}

We then set the nominal transition kernel from $s_0$ at $h=1$, $P^\star_1(\cdot | s_0, a)$. The probabilities are detailed as follows:
\begin{table}[!htb]
\centering
\caption{Nominal transition probabilities $P^\star_1(\cdot | s_0, a)$ from the start state.}
\label{tab:transitions}
\begin{tabular}{@{}l c c c l@{}}
\toprule
Joint Action $a$ & $P^\star_1(s_H | s_0, a)$ & $P^\star_1(s_M | s_0, a)$ & $P^\star_1(s_T | s_0, a)$ & Description \\
\midrule
$a=(1,1)$ & 0.90 & 0.00 & 0.10 & Risky (high reward, trap support) \\
$a=(0,0)$ & 0.60 & 0.40 & 0.00 & Safe (no trap support) \\
$a=(1,0)$ & 0.50 & 0.25 & 0.25 & Mediocre \\
$a=(0,1)$ & 0.50 & 0.25 & 0.25 & Mediocre \\
\bottomrule
\end{tabular}
\end{table}

It can be seen that, under the nominal kernel, the risky action is preferred as it has higher probability to transit to $s_H$. However, when there are model mismatch between the training and deploying environment, and under the risky action, the probability of transiting to the Trap state $s_T$ becomes higher, then the non-robust equilibrium becomes sub-optimal. On the other hand, our robust learning considers the worst-case, so it prefers to take the safe action. We will numerically show that our robust learning algorithm will learn a more robust policy that performs better under model uncertainties or the sim-to-real gap.

% \subsubsection{Experiment Results}
We aim to numerically verify two of our claims: (1). Our MORNAVI algorithm converges to the robust equilibria; And (2). The robust equilibria learned are more robust against model uncertainty compared to non-robust ones.

Specifically, we construct the uncertainty set as a KL-divergence ball centered at $P^\star_h$ as in \Cref{eq:f_divg_Uncertainty}, which $\rho_i=\rho$. We then implement our algorithm (\Cref{alg:Robust-Multi-Nash-VI}) together with the non-robust Nash value iteration \citep{liu2021sharp} as the baseline. Due to the hardness of computing Nash equilibria (which is PPAD-hard in the worst-case \citep{deng2023complexity}), we compute the CCE for the games. 

We develop two experiments as follows. Firstly, we run both algorithms (we set $\rho=0.25$ in our algorithm) for 10 times, and plot the averaged robust value function of Player 1 against the total number of samples. We also plot the standard deviation to show statistical errors. Secondly, we test the learned equilibria from both algorithm under different uncertainty radii $\rho$. For different $\rho$, we compute the robust value function of Player 1 (since both players have identical performance) under the KL-ball, to showcase the robustness of our algorithm. The experiment results are shown in \Cref{fig:KL_MORNAVI_Perforamnce}.

\begin{figure}[H]
    \centering
    \begin{subfigure}[b]{0.44\textwidth}
        \centering
        \includegraphics[width=\linewidth]{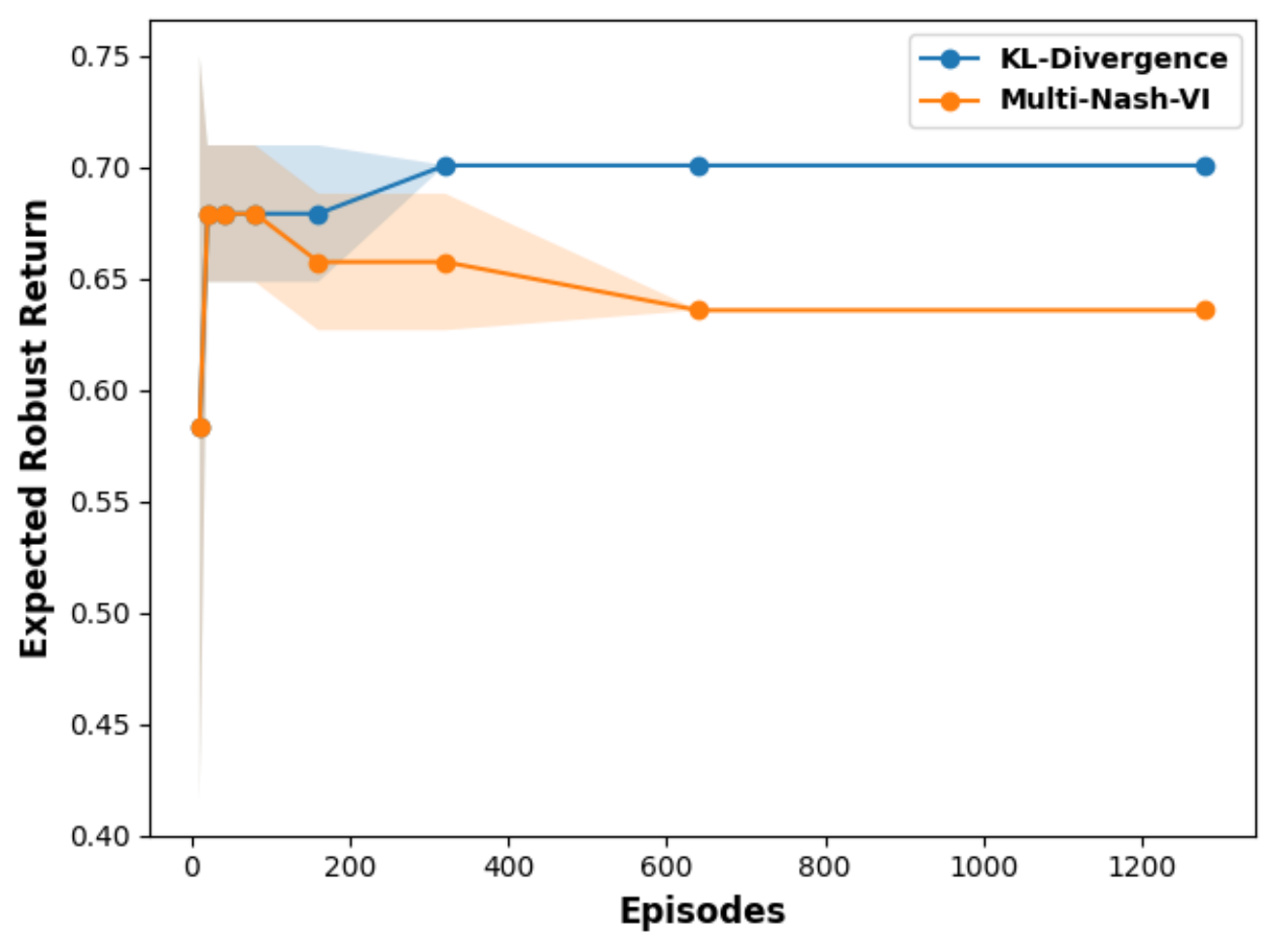}
        % \caption{Performance of {\AlgonameKL} vs no. of samples}
        \caption{Performance of {\AlgonameKL} vs. Episodes}
        \label{fig:KL-MORNAVI_samples}
    \end{subfigure}
    \hspace{0.04\textwidth}
    \begin{subfigure}[b]{0.44\textwidth}
        \centering
        \includegraphics[width=\linewidth]{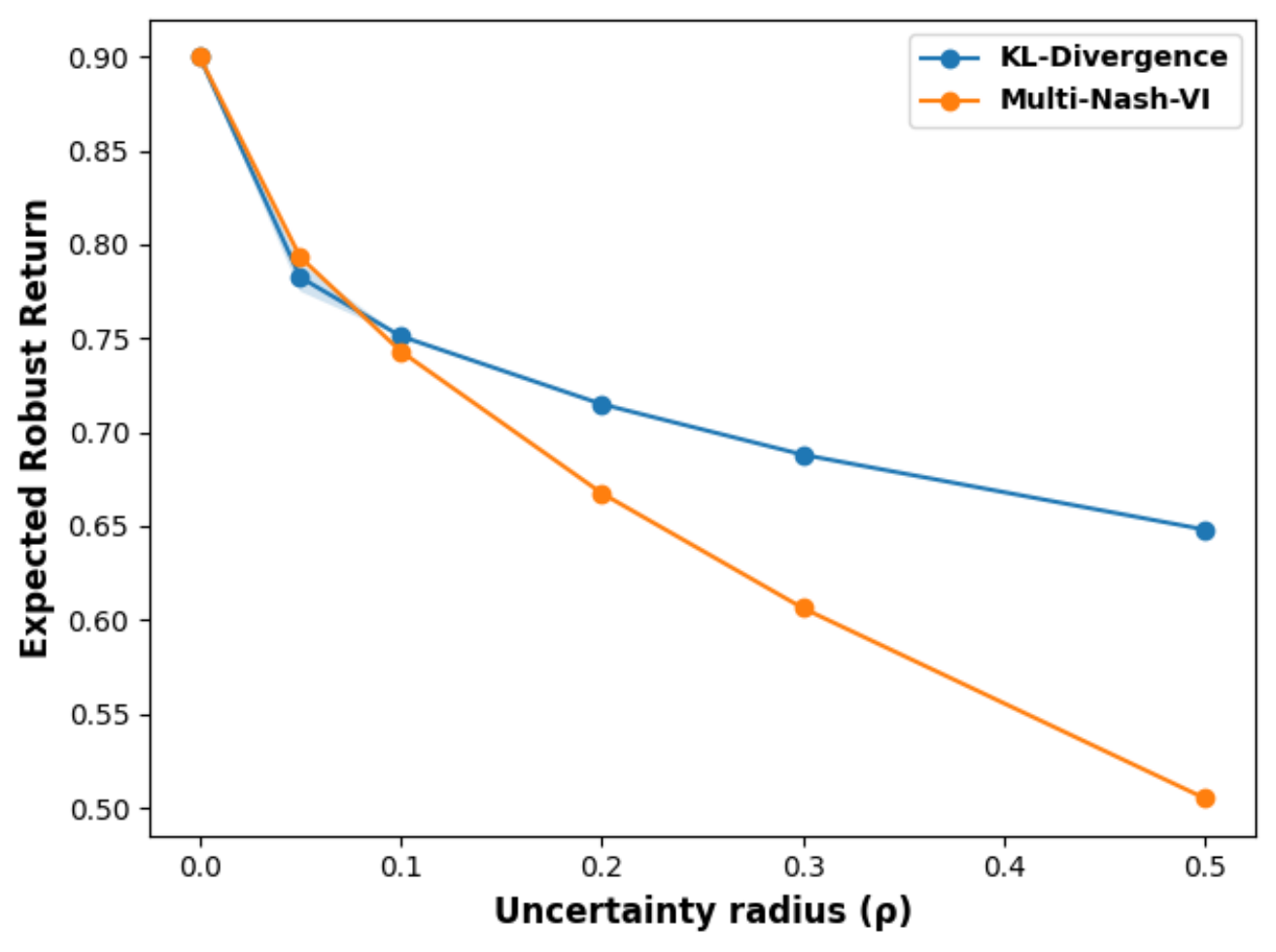}
        \caption{{\AlgonameKL} vs. Uncertainty Level ($\rho$)}
        \label{fig:KL-MORNAVI_rho}
    \end{subfigure}

    \vspace{-2pt} % reduce gap below subfigures & above main caption
    \caption{{\Algonamef} v.s. Multi-Nash-VI under KL-Divergence}
    \label{fig:KL_MORNAVI_Perforamnce}
\end{figure}
\vspace{-1.5em}
\begin{figure}[H]
    \centering
    \begin{subfigure}[b]{0.44\textwidth}
        \centering
        \includegraphics[width=\linewidth]{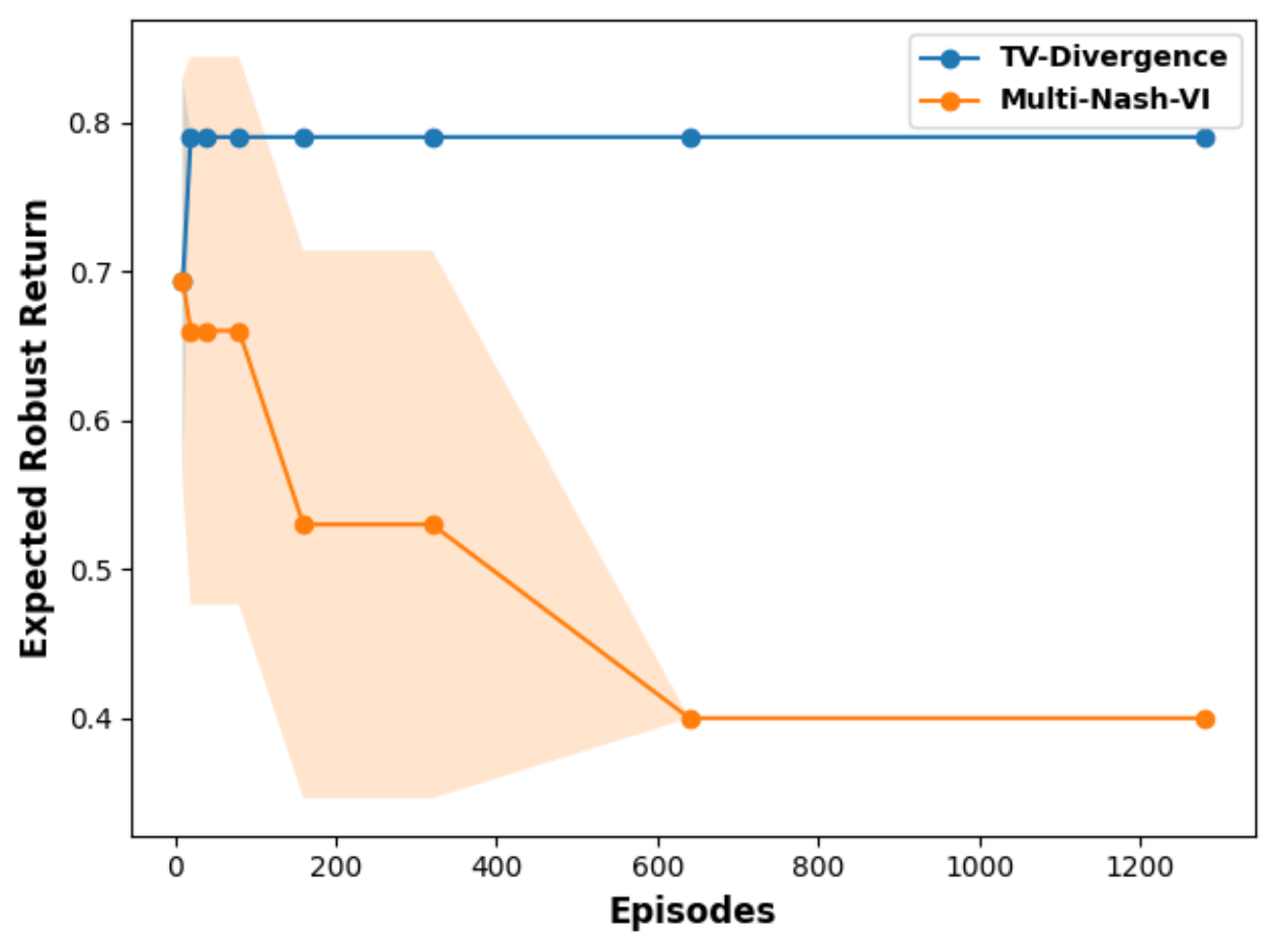}
        %Per-pair TV uncertainty:
        \caption{Performance of {\AlgonameTV} vs. Episodes}
        \label{fig:TV-MORNAVI_samples}
    \end{subfigure}
    \hspace{0.04\textwidth}
    \begin{subfigure}[b]{0.44\textwidth}
        \centering
        \includegraphics[width=\linewidth]{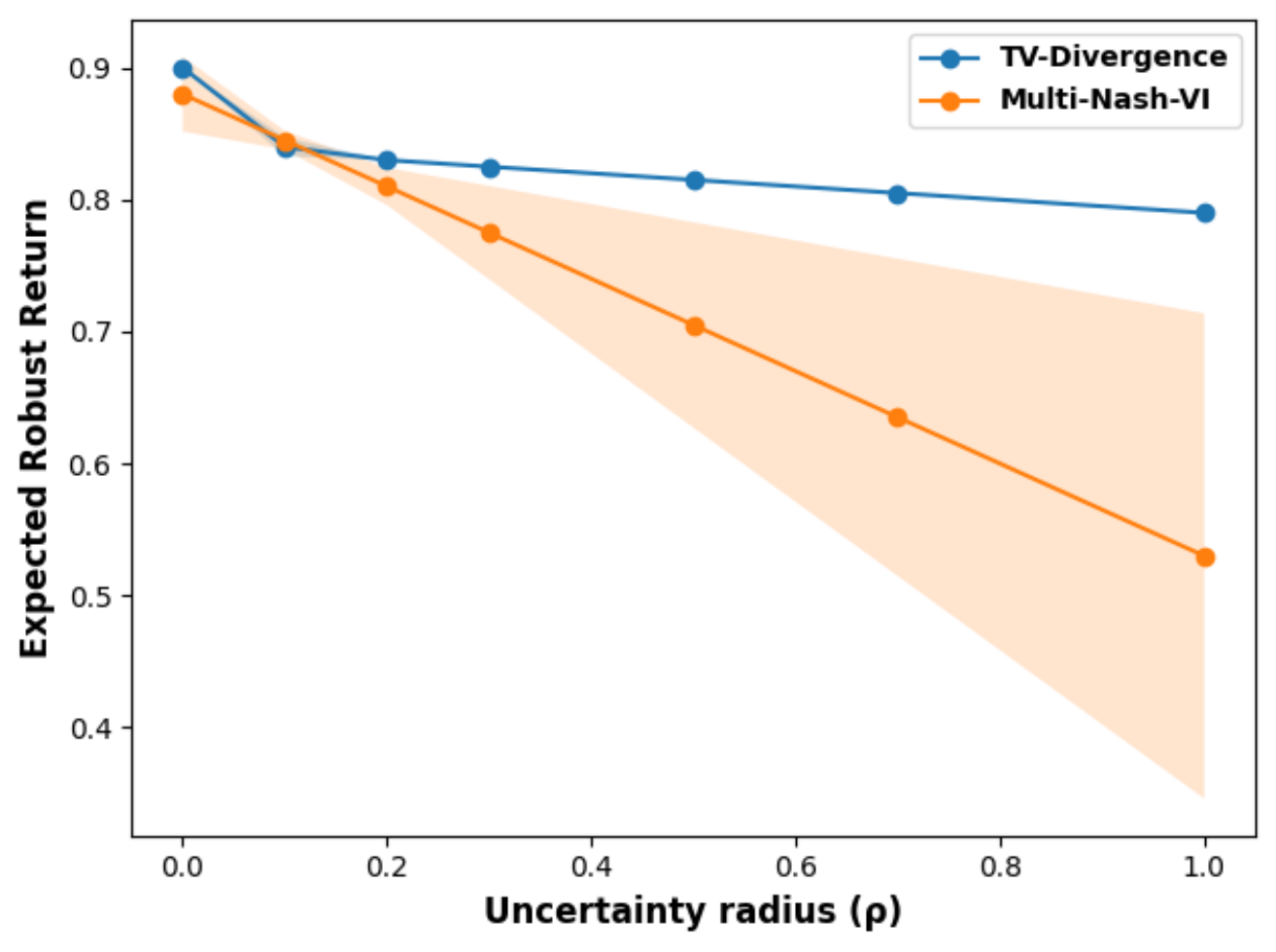}
        \caption{{\AlgonameTV} vs. Uncertainty Level ($\rho$)}
        \label{fig:TV-MORNAVI_rho}
    \end{subfigure}

    \vspace{-2pt}
    \caption{{\Algonamef} v.s. Multi-Nash-VI under TV-Divergence}
    \label{fig:TV_MORNAVI_Perforamnce}
\end{figure}

% \begin{figure}[!htb]
%     \centering
%     \begin{subfigure}[b]{0.44\textwidth}
%         \includegraphics[width=\linewidth]{ICLR/Appendix/trapdrmg_KL_MORNAVI_vs_samples.pdf}
%           \caption{Performance of {\AlgonameKL} vs no. of samples}
%         \label{fig:KL-MORNAVI_samples}
%     \end{subfigure}
%     \begin{subfigure}[b]{0.44\textwidth}
%         \includegraphics[width=\linewidth]{ICLR/Appendix/trapdrmg_KL_MORNAVI_vs_rho.pdf}
%           \caption{Performance of {\AlgonameKL} vs uncertainty level ($\rho$)}
%         \label{fig:KL-MORNAVI_rho}
%     \end{subfigure}

%     \caption{{\Algonamef} v.s. Multi-Nash-VI under KL-Divergence}
%     \label{fig:KL_MORNAVI_Perforamnce}
% \end{figure}
% \begin{figure}[!htb]
%     \centering
%     \begin{subfigure}[b]{0.44\textwidth}
%         \includegraphics[width=\linewidth]{ICLR/Appendix/trapdrmg_TV_MORNAVI_vs_samples.pdf}
%           \caption{Per-pair TV uncertainty: Performance of {\AlgonameTV} vs no. of samples}
%         \label{fig:TV-MORNAVI_samples}
%     \end{subfigure}
%     \begin{subfigure}[b]{0.44\textwidth}
%         \includegraphics[width=\linewidth]{ICLR/Appendix/trapdrmg_TV_MORNAVI_vs_rho.pdf}
%           \caption{Performance of {\AlgonameTV} vs  uncertainty level ($\rho$)}
%         \label{fig:TV-MORNAVI_rho}
%     \end{subfigure}

%     \caption{{\Algonamef} v.s. Multi-Nash-VI under TV-Divergence }
%     \label{fig:TV_MORNAVI_Perforamnce}
% \end{figure}

As the results shown, our algorithm converges to the robust equilibrium, validating the convergence of our theoretical results and convergence guarantees. Moreover, our robust equilibrium shows an enhanced robustness when model mismatch exists. Specifically, when $\rho\approx 0$ and there is no model mismatch, then the non-robust algorithm outperforms ours (as we are conservative and robust while non-robust is optimization for the nominal kernel); However, when the uncertainty radius increasing and model mismatch is introduced, performance of the non-robust equilibrium decreases significantly, whereas ours shows a more stable and robust performance. We trained the plots of uncertainty level radius for 400 episodes. Our results hence validate our theoretical results and claims.

Similarly, we develop experiments with TV-based uncertainty set, and plot results in \Cref{fig:TV_MORNAVI_Perforamnce}. As results shown, our algorithm converges to a robust equilibrium, which is more stable and robust against model uncertainties. Our results hence align with and validate our theoretical findings.

\subsection{General-Sum DRMG} 
We then slightly modify the fully cooperative DRMG considered, transferring it to a general-sum DRMG, to further validate our theoretical results.  

We set the nominal kernel as follows. At step~1, the nominal transition $P^\star_1(\cdot\,|\,s_0,a)$  is
\[
P^\star_1(\cdot\,|\,s_0,a)=
\begin{cases}
0.82\,\delta_{s_H} + 0.18\,\delta_{s_T}, & a=(1,1)\ \text{(risky)},\\[2pt]
0.60\,\delta_{s_H} + 0.40\,\delta_{s_M}, & a=(0,0)\ \text{(safe)},\\[2pt]
0.48\,\delta_{s_H} + 0.22\,\delta_{s_M} + 0.30\,\delta_{s_T}, & a\in\{(1,0),(0,1)\}\ \text{(off-diag)}.
\end{cases}
\]
At step~2 the kernel is absorbing:
$P^\star_2(s' \mid s,a)=\mathbf{1}\{s'=s\}$ for $s\in\{s_H,s_M,s_T\}$.

The rewards are settled as follows. At the terminal step (step~2), each terminal state induces a $2{\times}2$ matrix game;
let $R^{(1)}(s),R^{(2)}(s)\in\mathbb{R}^{2\times 2}$ denote the row/column players' payoffs.
We set
\[
\textbf{High:}\quad
R^{(1)}(s_H)=\begin{bmatrix}0.55&0.90\\[2pt]1.00&1.20\end{bmatrix},\qquad
R^{(2)}(s_H)=\begin{bmatrix}0.70&0.85\\[2pt]0.90&1.00\end{bmatrix},
\]
\[
\textbf{Medium:}\quad
R^{(1)}(s_M)=\begin{bmatrix}0.45&0.35\\[2pt]0.35&0.30\end{bmatrix},\qquad
R^{(2)}(s_M)=\begin{bmatrix}0.65&0.55\\[2pt]0.50&0.45\end{bmatrix},
\]
\[
\textbf{Trap:}\quad
R^{(1)}(s_T)=\mathbf{0},\qquad R^{(2)}(s_T)=\mathbf{0}.
\]

Both players then have different rewards and the game becomes a general-sum DRMG. 

Similarly, we implement our algorithms with non-robust baseline under both KL and TV uncertainty sets. We plot the performance of both players (as they are different). Our observations from the experiment results remain the same. In \Cref{fig:KL-MORNAVI_samples2} and \Cref{fig:TV-MORNAVI_samples2}, our robust algorithm converges to a robust equilibrium (sample) efficiently. And in  \Cref{fig:KL-MORNAVI_rho2} and \Cref{fig:TV-MORNAVI_rho2}, the robust equilibria learned by our algorithms maintain a more robust and stable performance under model mismatches, showcasing the enhanced robustness of our methods in MARL settings. 

\begin{figure}[!htb]
    \centering
    \begin{subfigure}[b]{0.44\textwidth}
        \includegraphics[width=\linewidth]{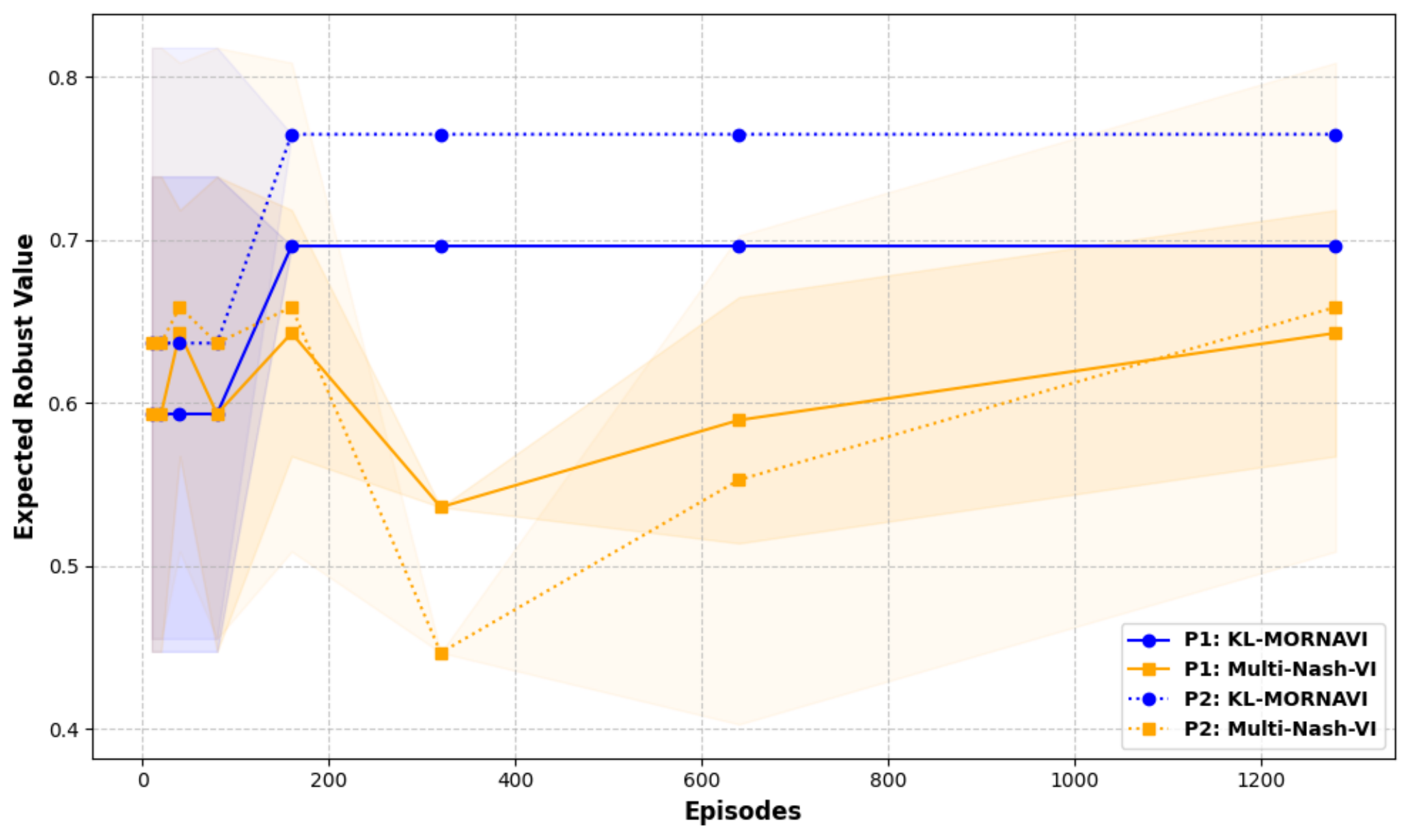}
        \caption{Performance of {\AlgonameKL} vs. Episodes}
        \label{fig:KL-MORNAVI_samples2}
    \end{subfigure}
    \hfill
    \begin{subfigure}[b]{0.44\textwidth}
        \includegraphics[width=\linewidth]{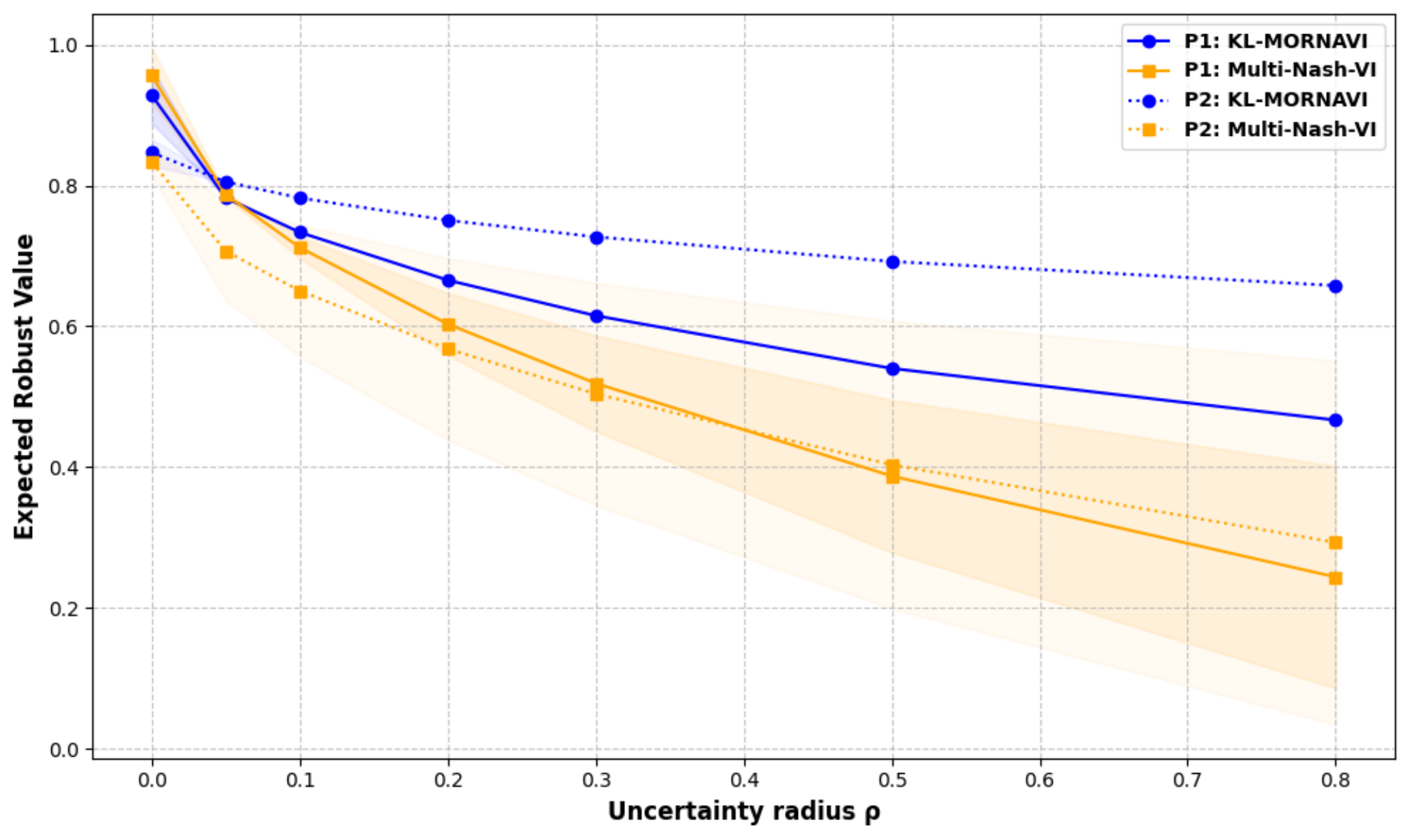}
        \caption{{\AlgonameKL} vs. Uncertainty Level ($\rho$)}
        \label{fig:KL-MORNAVI_rho2}
    \end{subfigure}
    \caption{{\Algonamef} v.s. Multi-Nash-VI under KL-Divergence}
    \label{fig:KL_MORNAVI_Perforamnce2}
\end{figure}
\begin{figure}[H]
    \centering
    \begin{subfigure}[b]{0.44\textwidth}
        \includegraphics[width=\linewidth]{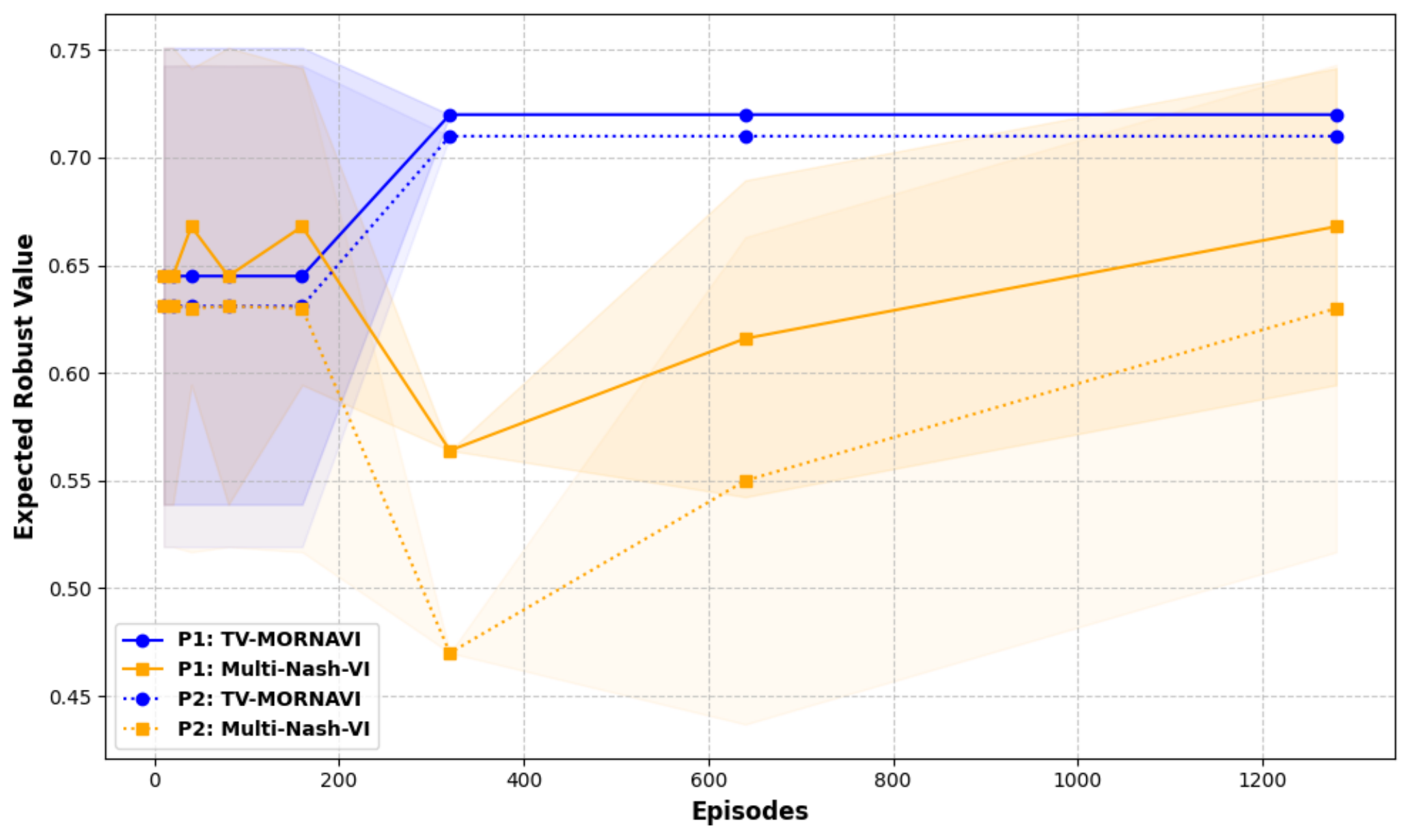}
        \caption{Performance of {\AlgonameTV} vs. Episodes}
        \label{fig:TV-MORNAVI_samples2}
    \end{subfigure}
    \hfill
    \begin{subfigure}[b]{0.44\textwidth}
        \includegraphics[width=\linewidth]{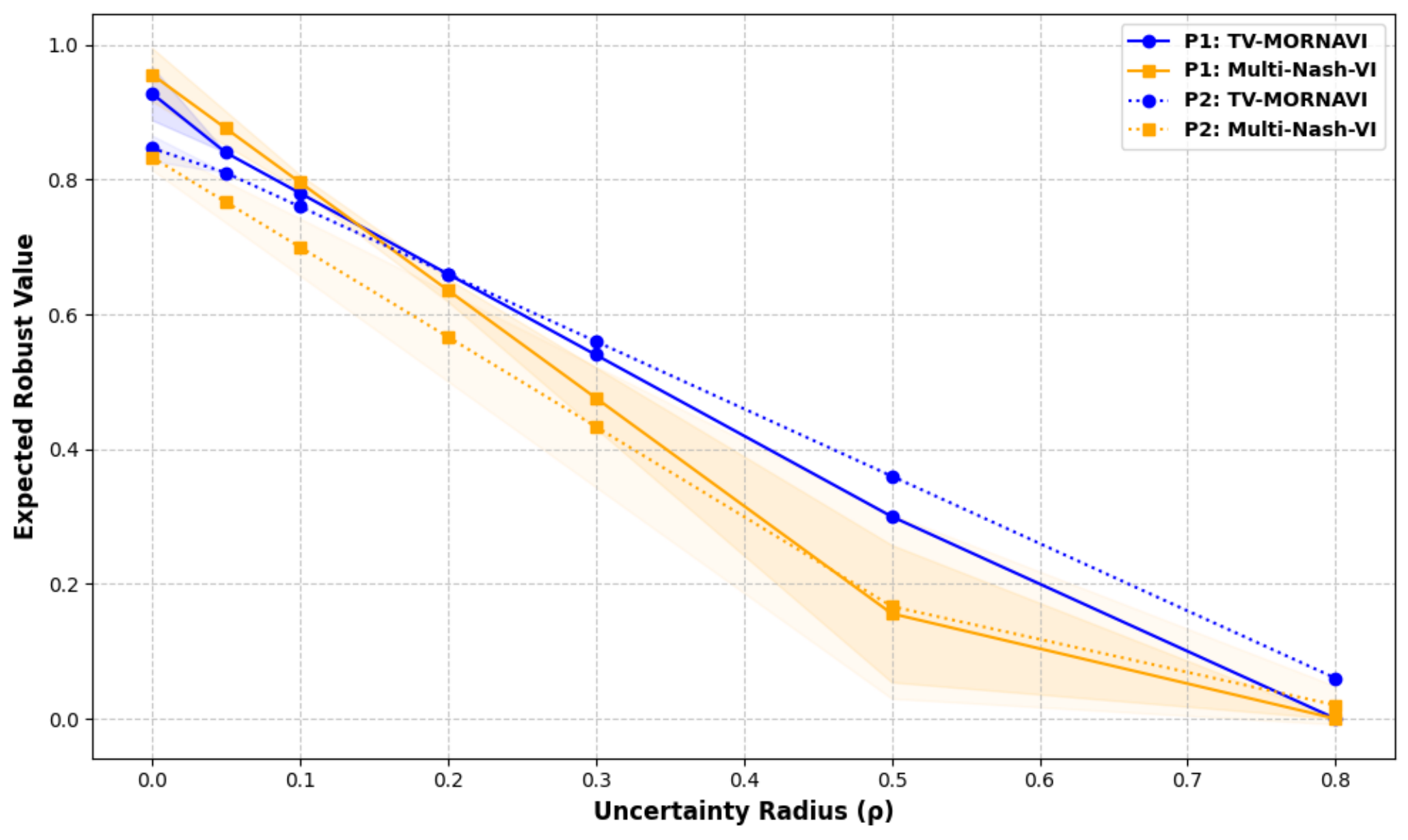}
        \caption{{\AlgonameTV} vs.  Uncertainty Level ($\rho$)}
        \label{fig:TV-MORNAVI_rho2}
    \end{subfigure}
    \caption{{\Algonamef} v.s. Multi-Nash-VI under TV-Divergence }
    \label{fig:TV_MORNAVI_Perforamnce2}
\end{figure}

\color{black}

\section{Hardness of Multi-Agent Online Learning}
\label{app:hardness multi agent}

\subsection{Hardness with Support Shift}
\label{app:hardness_with_support_shift'}
\begin{example}[The ``Initial Shock'' Game]
\label{ex:initial_shock}
Consider a class of $N$-agent DRMGs, $\{M_{\boldsymbol{a}^*}\}_{\boldsymbol{a}^* \in \mathcal{A}}$, parameterized by a ``secret escape route" $\boldsymbol{a}^* \in \mathcal{A}$.
\begin{itemize}
    \item \textbf{Action Spaces:} $A_i=M$ for each agent. The joint action space has size $|\calA| = \prod_{i \in [N]}A_i= M^N$.
    \item \textbf{States, Horizon, Rewards:} $\calS=\{s_{good}, s_{bad}\}$, horizon $H$, initial state $s_1=s_{good}$, and rewards are defined as $$r_i(s, {\bf a}) = \begin{cases} 1, & \text{if } s = s_{good} \text{ or if } (s=s_{bad} \text{ and } \bm{a} = \bm{a^*}) \\ 0, & \text{if } s = s_{bad} \text{ and } \bm{a} \ne \bm{a^*} \end{cases}.$$
     \item \textbf{Dynamics:} The system dynamics create the trap.
        \begin{itemize}
            \item From $\mathbf{s_{good}}$: Nominally, the system stays in $s_{good}$. An adversary can force a transition to $s_{bad}$ with probability $\rho$.
            \item From $\mathbf{s_{bad}}$: This is the trap. The only way to escape is to play the secret joint action:
            $$\text{Next State} = \begin{cases} s_{good}, & \text{if } \bm{a} = \bm{a^*} \\ s_{bad}, & \text{if } \bm{a} \ne \bm{a^*} \end{cases}.$$
        \end{itemize}
    \item \textbf{Uncertainty Set:} The uncertainty is non-zero only at the first step.
        \begin{itemize}
            \item \textbf{At $h=1$ and $s_1=s_{good}$:} The uncertainty set is a TV-ball with radius $\rho$.
            \item \textbf{For all $h>1$ or $s \ne s_{good}$:} There is no uncertainty ($\rho=0$). The transition is the nominal one.
        \end{itemize}
    % \item \textbf{Dynamics:} The system stays in $s_{bad}$ until $\boldsymbol{a}^*$ is played.
\end{itemize}
\end{example}

\begin{thm}
For the ``Initial Shock" DRMG, any decentralized online learning algorithm suffers the following best-response regret lower bound:
$$\inf_{\mathcal{ALG}} \sup_{\boldsymbol{a}^* \in \mathcal{A}} \E[\text{Regret}_i(K)] \geq \Omega\left(\rho K \cdot \min\left\{H, \prod_{i \in [N]}A_i\right\}\right).$$
\end{thm}

\begin{proof}
\textbf{Step 1: Decomposing the Per-Episode Regret.}
The best-response regret for Agent 1 in an episode is $\text{Regret}_1^k = V_{1,1}^{\dagger, \pi_{-i}, \rho} - V_{1,1}^{\pi,\rho}$. We expand this using the robust Bellman equation at $s_1=s_{good}$, where uncertainty exists.
\begin{align*}
\text{Regret}_1^k &= \left( 1 + (1-\rho)V_{1,2}^{\dagger,\pi_{-i},\rho}(s_{good}) + \rho V_{1,2}^{\dagger,\pi_{-i},\rho}(s_{bad}) \right) 
\\&- \left( 1 + (1-\rho)V_{1,2}^{\pi,\rho}(s_{good}) + \rho V_{1,2}^{\pi,\rho}(s_{bad}) \right) \\
&= (1-\rho)\left(V_{1,2}^{\dagger,\pi_{-i},\rho}(s_{good}) - V_{1,2}^{\pi,\rho}(s_{good})\right) + \rho\left(V_{1,2}^{\dagger,\pi_{-i},\rho}(s_{bad}) - V_{1,2}^{\pi,\rho}(s_{bad})\right).
\end{align*}
Since there is no uncertainty for $h>1$, the transition from $s_{good}$ at $h=2$ is deterministically to $s_{good}$ at $h=3$. Thus, $V_{1,2}(s_{good})$ is a constant independent of the policy in the trap state, which means $V_{1,2}^{\dagger,\pi_{-i},\rho_i}(s_{good}) = V_{1,2}^{\pi}(s_{good})$. The first term is exactly zero, and thus we have that
\begin{equation}
\text{Regret}_1^k = \rho \left( V_{1,2}^{\dagger,\pi_{-i},\rho}(s_{bad}) - V_{1,2}^{\pi,\rho}(s_{bad}) \right) = \rho \cdot \Delta V_{2}^{\rho}(s_{bad}).
\end{equation}

\textbf{Step 2: Formalizing the Value Gap $\Delta V_{2}^{\rho}(s_{bad})$.}
The value gap is the expected difference in total future rewards. This difference is precisely the expected number of steps wasted in the trap. Note that the value of state $s_{bad}$ at step $h$ under a policy $\pi'$ is the expected sum of future rewards. Let $\tau = \tau(\pi')$ be the random variable for the number of steps to escape (i.e., play $\boldsymbol{a}^*$), starting from step $h$. Let $C = H-h+1$ be the number of steps remaining in the episode, then the total reward collected from $h=2$ is $V_{1,2}^{\pi',\rho}(s_{bad}) = \E[\indicator[\tau \le C] \cdot (C-\tau+2)]$ as it will always receive $r=1$ when at $s_{good}$. 

% \begin{equation}
%     \sum_{t=h}^{H} R_t = \indicator[\tau \le C] \cdot \left(1 + (H - (h+\tau-1))\right) = \indicator[\tau \le C] \cdot (C - \tau + 1)\end{equation}

Moreover, note that the total number of available rewards is $C$, and since $C = \min(\tau-1, C) + \indicator[\tau \le C](C-\tau+1)$, the value can therefore be expressed as $V_{1,2}^{\pi',\rho}(s_{bad}) = C - \E[\min(\tau-1, C)]$.

Therefore, the value gap is the difference in the expected number of wasted steps:
\begin{equation*}
\begin{aligned}
\Delta V_{2}^{\rho}(s_{bad}) 
&= \left(C - \E[\min(\tau^*-1, C)]\right) - \left(C - \E[\min(\tau-1, C)]\right) \\
&= \E[\min(\tau-1, C)] - \E[\min(\tau^*-1, C)].
\end{aligned}
\end{equation*}

where $\tau^*$ is the escape probability of $\pi^*$. Since the best-response policy $\pi_1^*$ plays $a_1^*$ deterministically, so its escape time $\tau^*$ depends only on the other agents' policies, $\pi_{-1}$. The algorithm's escape time $\tau$ depends on its full policy $\pi$.

\textbf{Step 3: Lower Bounding the Value Gap.}
The best response for Agent 1 is to play $\bm{a_1}^*$, so $\tau^*$ does not involve any search for Agent 1. In contrast, 

However, the algorithm does not know $\bm{a_1}^*$ and must search. We are interested in the worst-case regret over the choice of $\boldsymbol{a}^*$.
The expected wasted steps for the algorithm is $\E[\min(\tau-1, C)]$. Let $p_1 = \Pr_{\pi_1}(a_1=a_1^*)$ and $p_{-1} = \Pr_{\pi_{-1}}(\boldsymbol{a}_{-1}=\boldsymbol{a}_{-1}^*)$. The algorithm's one-step escape probability is $p_1\cdot p_{-1}$. Its expected escape time is $\E[\tau] = 1/(p_1 \cdot p_{-1})$. The expected wasted steps is lower-bounded by:
$$\E[\min(\tau-1, C)] \ge \Omega(\min(\E[\tau-1], C)) = \Omega(\min(1/(p_1 \cdot p_{-1}), H-1)),$$
where the inequality is due to Lemma \ref{lem:4}. 

In the worst case over the unknown $\boldsymbol{a}^*$, the probabilities $p_1$ and $p_{-1}$ are minimized:
$$ \inf_{a_1^*} p_1 \le 1/A_1 \quad \text{and} \quad \inf_{\boldsymbol{a}_{-1}^*} p_{-1} \le 1\Big/\Big(\prod_{i=2}^N A_i\Big).$$
The best-response policy suffers much less waste. Thus, the value gap $\Delta V_2^{\rho}(s_{bad})$ is dominated by the algorithm's large number of wasted steps.
$$ \sup_{\boldsymbol{a}^*} \Delta V_{2}^{\rho}(s_{bad}) \ge \Omega\left(\min\left\{1\Big/\Big( (1/A_1)\cdot(1\Big/\Big(\prod_{i=2}^N A_i\Big) \Big), H\right\}\right) = \Omega\left(\min\left\{\prod_{i=1}^N A_i, H\right\}\right). $$

\textbf{Step 4: Finalizing the Bound.}
Substituting this back into the per-episode regret expression from Step 1:
$$ \sup_{\boldsymbol{a}^*} \E[\text{Regret}_1^k] \ge \rho \cdot \Omega\left(\min\left\{\prod_{i=1}^N A_i, H\right\}\right). $$
This per-episode regret is incurred because the information bottleneck prevents the algorithm from learning $\boldsymbol{a}^*$. Summing over $K$ episodes gives the final total regret bound:
$$ \inf_{\mathcal{ALG}} \sup_{\boldsymbol{a}^*} \E[\text{Regret}_1(K)] = \sum_{k=1}^K \sup_{\boldsymbol{a}^*} \E[\text{Regret}_1^k] \ge \Omega\left(\rho K \cdot\min\left\{\prod_{i=1}^N A_i, H\right\}\right). $$
This completes the   proof.
\end{proof}

% \section{Proofs of Theoretical Results of {\Algonamef}}
% We present the proofs of the regret bound for {\AlgonameTV} and {\AlgonameKL}.

 \begin{lem}\label{lem:4}
Let $\tau$ be the random variable for the escape time from the trap state, and let $C = H-1$ be the number of steps remaining in the episode. The true expected number of wasted steps, $\E[\min(\tau-1, C)]$, has the following asymptotic lower bound:
$$ \E[\min(\tau-1, C)] \ge \Omega(\min(\E[\tau-1], C)). $$
\end{lem}

\begin{proof}
Note that $\tau$ follows a Geometric distribution $\tau \sim \text{Geo}(p)$ and have the probability mass function  $P(\tau=k) = (1-p)^{k-1}p$ for $k \in \{1, 2, 3, \dots\}$. The random variable $\tau-1$ represents the number of failures before the first success. Its expectation is $\E[\tau-1] = \frac{1-p}{p}$.
 
We first derive an expression for $\E[\min(\tau-1, C)]$. We use the tail sum formula for the expectation of a non-negative, integer-valued random variable $X$, which states $\E[X] = \sum_{k=0}^{\infty} P(X > k)$.

Let $X = \min(\tau-1, C)$. The event $\{X > k\}$ is equivalent to the event $\{\tau-1 > k \text{ and } C > k\}$.
\begin{itemize}
    \item If $k \ge C$, then $P(X > k) = 0$.
    \item If $k < C$, then $P(X > k) = P(\tau-1 > k)$.
\end{itemize}
The event $\{\tau-1 > k\}$ means the first $k+1$ trials resulted in failure, so its probability is $P(\tau > k+1) = (1-p)^{k+1}$.

The expectation is therefore the sum over the non-zero probabilities:
\begin{align*}
    \E[\min(\tau-1, C)] = \sum_{k=0}^{\infty} P(\min(\tau-1, C) > k) = \sum_{k=0}^{C-1} P(\tau-1 > k) = \sum_{k=0}^{C-1} (1-p)^{k+1}.
\end{align*}
Letting $q = 1-p$, this is a finite geometric series:
$$
\sum_{j=1}^{C} q^j = q \frac{1-q^C}{1-q} = \frac{q(1-q^C)}{p}.
$$
Substituting $q=1-p$ back, we express the expectation in terms of $\E[\tau-1]$:
$$
\E[\min(\tau-1, C)] = \frac{1-p}{p} (1 - (1-p)^C) = \E[\tau-1] (1 - (1-p)^C).
$$

Let $\mu = \E[\tau-1] = \frac{1-p}{p}$. We want to show that there exists a universal constant $k>0$ such that:
$$
\mu(1 - (1-p)^C) \ge k \cdot \min(\mu, C).
$$
We proceed with a case analysis based on the relationship between $\mu$ and $C$.

\textbf{Case 1: $\mu \le C$:}
In this case, $\min(\mu, C) = \mu$. We need to show that $\mu(1 - (1-p)^C) \ge k \cdot \mu$, which simplifies to proving that $1 - (1-p)^C \ge k$.

The condition $\mu \le C$ implies a lower bound on $p$:
$$
\frac{1-p}{p} \le C \implies 1-p \le Cp \implies 1 \le (C+1)p \implies p \ge \frac{1}{C+1}.
$$
Using the standard inequality $1-x \le e^{-x}$, we have $(1-p)^C \le e^{-pC}$. Thus,
$$
1 - (1-p)^C \ge 1 - e^{-pC}.
$$
Since $p \ge \frac{1}{C+1}$, we have $pC \ge \frac{C}{C+1}$. As the function $f(x) = 1-e^{-x}$ is increasing for $x>0$,
$$
1 - e^{-pC} \ge 1 - e^{-C/(C+1)}.
$$
The function $g(C) = \frac{C}{C+1}$ is increasing for $C \ge 1$, with a minimum value of $g(1) = 1/2$. Therefore, for any integer $C \ge 1$,
$$
1 - (1-p)^C \ge 1 - e^{-1/2}.
$$
Thus, the inequality holds in this case with the constant $k_1 = 1-e^{-1/2} \approx 0.393$.

\textbf{Case 2: $\mu > C$:}
In this case, $\min(\mu, C) = C$. We need to show that $\mu(1 - (1-p)^C) \ge kC$.

The condition $\mu > C$ implies an upper bound on $p$:
$$
\frac{1-p}{p} > C \implies 1-p > Cp \implies 1 > (C+1)p \implies p < \frac{1}{C+1}.
$$
From our calculation of the expectation, we have a sum of $C$ positive, decreasing terms:
$$
\E[\min(\tau-1, C)] = \sum_{k=0}^{C-1} (1-p)^{k+1}.
$$
This sum is greater than $C$ times its smallest term, which is $(1-p)^C$:
$$
\E[\min(\tau-1, C)] > C(1-p)^C.
$$
From the condition $p < \frac{1}{C+1}$, it follows that $1-p > 1 - \frac{1}{C+1} = \frac{C}{C+1}$. Therefore,
$$
\E[\min(\tau-1, C)] > C \left(\frac{C}{C+1}\right)^C = C \left(1 - \frac{1}{C+1}\right)^C.
$$
The sequence $a_C = \left(1 - \frac{1}{C+1}\right)^C$ is decreasing for $C \ge 1$, and its limit as $C \to \infty$ is $1/e$. Hence, for all $C \ge 1$, the sequence is bounded below by its limit:
$$
\left(1 - \frac{1}{C+1}\right)^C \ge \lim_{n \to \infty} \left(1 - \frac{1}{n+1}\right)^n = \frac{1}{e}.
$$
This gives the lower bound:
$$
\E[\min(\tau-1, C)] > C \cdot \frac{1}{e}.
$$
So, the inequality holds in this case with the constant $k_2 = 1/e \approx 0.368$. By combining the two cases, the inequality is shown to hold for a universal constant $k = \min(k_1, k_2) = \min(1-e^{-1/2}, 1/e) = 1/e$.

Therefore, for all $p \in (0,1)$ and integers $C \ge 1$, we have established that:
$$
\E[\min(\tau-1, C)] \ge \frac{1}{e} \min(\E[\tau-1], C)=\Omega(\min(\E[\tau-1], C)),
$$
which hence completes the proof.
\end{proof}

\subsection{Hardness without Support Shift}
\label{app:hardness_without_support_shift}
% We then illustrate the hardness of online DRMGs when there is no support shift. Note that when the uncertainty set is defined through, e.g., KL divergence, the worst-case support will be covered by the nominal one, so there will not be any support shift. However, we construct another example to show that, even without the support shift, the online learning can still be challenging and inefficient. 

% \begin{theorem}[Lower Bound for Robust Learning without Support Shift]
% There exists a DRMG, such that any learning algorithm suffers the following cumulative regret lower bound over $K$ episodes:
% $$\inf_{\mathcal{ALG}}   \E[\text{Regret}_{\textsf{NASH}}(K)] \geq \Omega\Big(\sqrt{ K\prod_{i\in \mathcal{M}}A_i}\Big).$$
% \end{theorem}
% Our construction is in \Cref{ex:corrupted_bandit}. This result illustrates that, even without any support shift, some hard instance can require at least $\Omega\left(\sqrt{K\prod_iA_i}\right)$ regret. Our result hence suggests that the dependence on the joint action space may be inevitable in online robust learning, which suffer from the curse of multi-agency. 

\begin{example}[The ``Robust Corrupted Bandit'' Game]
\label{ex:corrupted_bandit}
Consider a class of $N$-agent DRMGs, $\{M_{\theta}\}_{\theta \in \mathcal{A}}$, where each game is parameterized by a secret ``best'' joint action $\theta \in \mathcal{A}$.
\begin{itemize}
    \item \textbf{States and Horizon:} A single state, $s$, and horizon $H=1$. This reduces the problem to a one-shot game, equivalent to a multi-armed bandit setting where each episode corresponds to a single step or arm pull.
    \item \textbf{Action Spaces:} The joint action space $\calA$ is the set of arms, with size $|\calA|= \prod_{i=1}^N A_i$.
    \item \textbf{Reward Function ($R \in \{0,1\}$):} The rewards are stochastic. Let $\epsilon \in (0, 1/2)$ be a small constant. The nominal model $M_\theta$ defines the following Bernoulli reward distributions for any agent $i$:
        $$ \mathbb{E}[R_i(s, \boldsymbol{a}) | M_\theta] = \begin{cases} 1/2 + \epsilon, & \text{if } \boldsymbol{a} = \theta \\ 1/2, & \text{if } \boldsymbol{a} \ne \theta .\end{cases}$$
    \item \textbf{KL-Divergence Uncertainty Set:} The true reward distribution for an action $\boldsymbol{a}$, denoted $\tilde{P}(\cdot|\boldsymbol{a})$, can be any distribution that is close to the nominal one $P^*(\cdot|\boldsymbol{a})$:
        $$ \mathcal{P}_{i,h,KL}^{\rho_i}(.,\bm{a}) = \left\{ \tilde{P} : \mathrm{KL}(\tilde{P}(\cdot|\boldsymbol{a}) \| P_{M_\theta}(\cdot|\boldsymbol{a})) \le \rho_i, \forall \boldsymbol{a} \in \calA \right\}.$$
    This uncertainty set does not have a support shift.
\end{itemize}
\end{example}
% replaced \Phi(P) with \mathcal{P}_{i,h,KL}^{\rho_i}(.,\bm{a})
% replaced $P_{M_\theta}(\cdot|\boldsymbol{a})$ by P*(\cdot|\boldsymbol{a})
 
The learning problem is to identify the best arm $\theta$ by observing noisy rewards that are actively corrupted by an adversary.

\begin{theorem}[Lower Bound for Robust Learning without Support Shift]
For the ``Robust Corrupted Bandit" game, any learning algorithm suffers the following cumulative regret lower bound over $K$ episodes (steps):
$$\inf_{\mathcal{ALG}} \sup_{\theta \in \mathcal{A}} \E[\text{Regret}_i(K)] \geq \Omega\left(\sqrt{\prod_{i=1}^N A_i K}\right).$$
\end{theorem}

\begin{proof}
 The proof proceeds by a formal reduction to the classic multi-armed bandit (MAB) problem.

 Let $\mathcal{M}_\rho = \{M_{\theta, \rho}\}_{\theta \in \mathcal{A}}$ denote the class of robust game instances from our example, with uncertainty radius $\rho > 0$. Let $\mathcal{M}_0 = \{M_{\theta, 0}\}_{\theta \in \mathcal{A}}$ be the corresponding class of non-robust instances, where the uncertainty radius is zero and the rewards are always drawn from the nominal distributions.

Note that since the horizon $H=1$, the robust problem reduces to a non-robust one, and thus the worst-case regret over the robust class $\mathcal{M}_\rho$ must be at least as high as the worst-case regret over the non-robust class $\mathcal{M}_0$:
$$\E[\text{Regret}(K; M_{\theta, \rho})] \ge \E[\text{Regret}(K; M_{\theta, 0})].$$And thus
\begin{equation}
\inf_{\mathcal{ALG}} \sup_{\theta \in \mathcal{A}} \E[\text{Regret}(K; M_{\theta, \rho})] \ge \inf_{\mathcal{ALG}} \sup_{\theta \in \mathcal{A}} \E[\text{Regret}(K; M_{\theta, 0})].\label{eq:dominance}
\end{equation}
Therefore, we can establish a lower bound for the robust problem by proving one for the simpler non-robust case.

 The non-robust problem instance, $\mathcal{M}_0$, is a classic stochastic multi-armed bandit problem with $M=|\mathcal{A}|$ arms. A foundational result in this area provides a strong lower bound on regret.

Note that following standard lemma:
\begin{lem}\citep{auer2002nonstochastic}
For any integer $M \ge 2$ and $K > M$, and for any bandit algorithm, there exists a multi-armed bandit problem instance with $M$ arms whose reward distributions are supported on $[0,1]$, such that the expected cumulative regret after $K$ steps is lower-bounded by:
$$ \E[\text{Regret}(K)] \ge \Omega(\sqrt{MK}). $$
\end{lem}
% This result is information-theoretic and demonstrates that a certain number of suboptimal arm pulls are necessary to distinguish the best arm from the others, leading to unavoidable regret.

 We apply the lemma to our non-robust problem instance $\mathcal{M}_0$.
\begin{itemize}
    \item The number of arms, $M$, is the size of the joint action space, $|\mathcal{A}|$.
    \item The number of steps is $K$.
    \item The reward distributions (Bernoulli) are supported on $[0,1]$.
\end{itemize}
The conditions of the lemma are met. Therefore, for the class of problems $\mathcal{M}_0$, the worst-case regret is lower-bounded:
\begin{equation}
\inf_{\mathcal{ALG}} \sup_{\theta \in \mathcal{A}} \E[\text{Regret}(K; M_{\theta, 0})] \ge \Omega\left(\sqrt{\prod_{i=1}^N A_i K}\right). \label{eq:non-robust-bound}
\end{equation}
Combining the regret dominance principle from eq. \ref{eq:dominance} with the specific lower bound from eq. \ref{eq:non-robust-bound}, we arrive at the final result for our robust problem:
\begin{equation}
\inf_{\mathcal{ALG}} \sup_{\theta \in \mathcal{A}} \E[\text{Regret}_i(K; M_{\theta, \rho})] \ge \Omega\left(\sqrt{\prod_{i=1}^N A_i K}\right).
\end{equation}
This completes the formal proof by reduction.

\end{proof}

\section{Proof of regret bound of {\AlgonameTV}}
\label{app:thm:Regret_TV_Bernstein_bound}
In this section, we prove our regret bound for {\RMGTV}. Before presenting all the proofs, we first denote $\pi^{\dagger}$ as the joint robust best responses over the agents, and is gven by
\begin{align}
\label{eq:joint_best_response}
 \pi^{\dagger} =\pi^{\dagger,\rho_1}_1(\pi_{-1})\times\dots\times \pi^{\dagger,\rho_m}_m(\pi_{-m}).
\end{align}
We will use the notation of $\pi^{\dagger}$ later on our proof-lines. In addition, we leverage \Cref{ass:vanisin_minimal}, which generalizes to the case where the minimal value vanishes, i.e., \(\min_{s \in \mathcal{S}} V(s) = 0\), to address the support shift or extrapolation challenge arising in interactive data collection, as discussed in Remark B.3 of \citep{Arxiv2024_DRORLwithInteractiveData_Lu}. Consequently, this allows us to eliminate the \(\min_{s \in \mathcal{S}} V(s)\) term in the dual formulation of the {\RMGTV} optimization problem, as shown in \ref{eq:dual_TV}.

We now recall the bonus term  used in {\AlgonameTV} for agent $i$ in episode $k$ at step $h$, as follows:
\begin{align}
\label{eq:Bonus_term_TV_Bernstein}
    \beta^k_{i,h}(s,{\bf a}) &= \sqrt{\frac{c_1\iota \mathrm{Var}_{\widehat{P}_h^k(\cdot|s,{\bf a})} \left[ \left( \frac{ \overline{V}_{i,h+1}^{k,\rho_i} + \underline{V}_{i,h+1}^{k,\rho_i} }{2} \right) \right]}{\{N^k_h(s,{\bf a})\vee 1\}}} + \frac{2\mathbb{E}_{\widehat{P}^k_h(\cdot|s, {\bf a})}\left[\up{V}^{k,\rho_i}_{i,h+1}- \low{V}^{k,\rho_i}_{i,h+1}\right]}{H} \nonumber\\
    &\qquad \qquad  + \frac{c_2H^2 S\iota}{\sqrt{\{N_h^k(s,{\bf a}) \vee 1\}}} + \frac{1}{\sqrt{K}},
\end{align}
where $\iota = \log\Big(S^2(\prod_{i=1}^m A_i)H^2K^{3/2}/\delta\Big)$ and $c_1, c_2$ are absolute constants.

We begin by defining the high-probability event 
$\mathcal{E}_{\text{TV}}$, stated in the next lemma. Our proof outline is inspired by \citep{Arxiv2024_DRORLwithInteractiveData_Lu} and \citep{ghosh2025provablynearoptimaldistributionallyrobust}.

\begin{lem}[Uniform Concentration Bound of event $\mathcal{E}_{\text{TV}}$]
\label{lem:Bernstein_TV_event_bound}
Let $\mathcal{E}_{\text{TV}}$ be the event in which, for all $(s, \mathbf{a}, s', h, k) \in \mathcal{S} \times \mathcal{A} \times \mathcal{S} \times [H] \times [K]$, and for all $\eta$ in a $1/(S\sqrt{K})$-cover of $[0,H]$, and is defined as
\begin{align}
\label{eq:Bernstein_TV_event}
    \mathcal{E}_{\text{TV}} := &\Bigg\{\abs{\left[\mathbb{E}_{\widehat{P}_h^k(\cdot|s,{\bf a})} - \mathbb{E}_{P_h^\star(\cdot|s,{\bf a})}\right]\left( \eta - V_{i,h+1}^{\dagger,\pi^k_{-i},\rho_i} \right)_+ } \leq \sqrt{ \frac{c_1\iota \text{Var}_{\widehat{P}_h^k}\left( \eta - V_{i,h+1}^{\dagger, \pi^k_{-i},\rho_i} \right)_+}{N_h^k(s,{\bf a})\vee 1 } } 
\nonumber\\&+ \frac{ c_2H \iota }{\{N_h^k(s,{\bf a}) \vee 1\}}, \nonumber\\
&\abs{\widehat{P}_h^k(s'\mid s, {\bf a}) - P^\star_h(s' \mid s,{\bf a})} 
\leq \sqrt{ \frac{ c_1\min\left\{ P^\star_h(s' \mid s,{\bf a}),\widehat{P}^k_h(s' \mid s,{\bf a}) \right\} \cdot \iota }
{\{N_h^k(s,{\bf a})\vee 1\}} } \nonumber\\&\qquad+ \frac{ c_2\iota }{\{N_h^k(s,{\bf a})\vee 1\}}, \nonumber\\
&\qquad \forall (s, {\bf a}, s', h,k)\in \mathcal{M}\times \mathcal{S}\times \mathcal{A}\times \mathcal{S}\times [H]\times [K], \forall \eta \in \mathcal{N}_{1/(S\sqrt{K})}([0,H])\Bigg\},
%&\left\| \left( \widehat{\mathbb{P}}_h^k - \mathbb{P}_h \right)(\cdot \mid s,{\bf a}) \right\|_1 \leq c_1 \sqrt{ \frac{ S \iota }{ \max\{N_h^k(s,a,b), 1\} } }.
\end{align}
where $\iota = \log\Big(S^3(\prod_{i=1}^m A_i)H^2K^{3/2}/\delta\Big)$, $c_1, c_2>0$ are two absolute constants, $\mathcal{N}_{1/(S\sqrt{K})}([0,H])$ denotes an $1/S\sqrt{K}$-cover of the interval $[0,H]$.

Then, this event $ \mathcal{E}_{\text{TV}}$ occurs with high probability, i.e., $\Pr(\mathcal{E}_{\text{TV}}) \geq 1 - \delta$.
\end{lem}

\begin{proof}
This proof builds upon standard techniques by applying classical concentration inequalities and a union bound. To simplify our analysis, we first consider a fixed state-action-time tuple $(s, \mathbf{a}, h)$ within a given episode $k$. We can then construct an equivalent stochastic process:
\begin{enumerate}[label=(\roman*)]
\item Before the agents' interaction, the environment draws a sequence of next states $\{s^{(1)}, s^{(2)},\dots, s^{(k-1)}\}$ independently from the nominal distribution $P^\star_h(\cdot|s,\mathbf{a})$, where $s^{(i)} \in \mathcal{S}$ represents the state sampled in episode $i$.
\item When the agents visit the $(s, \mathbf{a})$ tuple at time step $h$ for the $i$-th time, the environment causes a transition to the pre-sampled next state $s^{(i)}$.
\end{enumerate}
The randomness of this constructed process is identical to that of our original, interactive learning environment. Consequently, the probability of any event is the same in both contexts. This allows us to prove the required concentration inequalities within this more tractable, simplified setting.

Leveraging this fact, we directly apply \Cref{lem:self_bound_variance}, which presents a variant of Bernstein's inequality and its empirical counterpart from \citep{Arxiv2009_EmpBernsteinBounds_Maurer}. To establish a uniform bound, we apply a union bound across all tuples $(h, s, \mathbf{a}, s', k, \eta) \in [H] \times \mathcal{S} \times \mathcal{A} \times \mathcal{S} \times [K] \times \mathcal{N}_{1/(S\sqrt{K})}([0, H])$. The size of this $\epsilon$-cover, $\mathcal{N}_{1/(S\sqrt{K})}([0, H])$, is on the order of $\mathcal{O}(SH\sqrt{K})$.
\end{proof}

\subsection{Proof of \Cref{thm:Regret_TV_bound_Bernstein} (\RMGTV~ Setting)}
\label{app:proof_regret_TV_bound_Bernstein}
\begin{proof}
By leveraging \Cref{lem:Optimistic_pessimism_TV_NE_Bernstein}, we can establish an upper bound on the regret by considering the difference between the optimistic and pessimistic value functions:
\begin{align}
\Reg_{\sf NASH}(K) = \sum_{k=1}^K\max\limits_{i \in \mathcal{M}}\Big(V_{i,1}^{\dagger, \pi_{-i}^k, \rho_i}-V_{i,1}^{\pi^k, \rho_i} \Big)(s^k_1)\leq \sum_{k=1}^K \max\limits_{i \in \mathcal{M}}\Big(\up{V}_{i,1}^{k, \rho_i}-\low{V}_{i,1}^{k, \rho_i} \Big)(s^k_1).\label{eq:Regret_TV_step1_Bernstein}
\end{align}
For the TV-divergence uncertainty set, we begin by analyzing the difference between the upper and lower Q-values. Given our definitions for $\overline{Q}_h^k$, $\underline{Q}_{i,h}^{k,\rho_i}$, $\overline{V}_{i,h}^{k,\rho_i}$, and $\underline{V}_{i,h}^{k,\rho_i}$ (from eq. \ref{eq:robust_Qupper_k}- \ref{eq:robust_V_values_k}), along with the bonus term $\beta^k_{i,h}(s,\bm{a})$ defined in eq. \ref{eq:Bonus_term_TV_Bernstein}, we can establish a bound on this difference for any $(h,k) \in [H] \times [K]$ and $(s,\bm{a}) \in \mathcal{S} \times \mathcal{A}$:
\begin{align}
 \overline{Q}_h^k(s,{\bf a}) - \underline{Q}_h^k(s,{\bf a}) &\leq \sigma_{\widehat{\mathcal{P}^{\rho_i}_{i,h}}(s,\bm{a})}\left[ \overline{V}_{i,h+1}^{k,\rho_i} \right]
 - \sigma_{\widehat{\mathcal{P}^{\rho_i}_{i,h}}(s,\bm{a})}\left[ \underline{V}_{h+1}^{k,\rho_i} \right] + 2\beta^k_{i,h}(s,{\bf a}).\label{eq:Regret_TV_step2_Bernstein}
\end{align}
We introduce two key terms, $A$ and $B$, to simplify this expression:
\begin{align}
A:= &\sigma_{\widehat{\mathcal{P}^{\rho_i}_{i,h}}(s,\bm{a})}\left[ \overline{V}_{i,h+1}^{k,\rho_i} \right] - \sigma_{\mathcal{P}^{\rho_i}_{i,h}(s,{\bf a})}\left[ \overline{V}_{i,h+1}^{k,\rho_i} \right] + \sigma_{\mathcal{P}^{\rho_i}_{i,h}(s,{\bf a})}\left[ \underline{V}_{i,h+1}^{k,\rho_i} \right] - \sigma_{\widehat{\mathcal{P}^{\rho_i}_{i,h}}(s,\bm{a})}\left[ \underline{V}_{i,h+1}^{k,\rho_i} \right]. \label{eq:Regret_TV_A_Bernstein}\\
B := &\sigma_{\mathcal{P}^{\rho_i}_{i,h}(s,{\bf a})}\left[ \overline{V}_{i,h+1}^{k,\rho_i} \right] - \sigma_{\mathcal{P}^{\rho_i}_{i,h}(s,{\bf a})}\left[ \underline{V}_{i,h+1}^{k,\rho_i} \right].\label{eq:Regret_TV_B_Bernstein}
\end{align}
By substituting these definitions into eq. \ref{eq:Regret_TV_step2_Bernstein}, we obtain a new bound:
\begin{align}
\overline{Q}_{i,h}^{k,\rho_i}(s,{\bf a}) - \underline{Q}_{i,h}^{k,\rho_i}(s,{\bf a})&\leq A + B + 2\beta^k_{i,h}(s,{\bf a}). \label{eq:Regret_TV_step3_Bernstein}
\end{align}
We then proceed to bound each of these terms. A concentration bound argument tailored for TV robust expectations in \Cref{lem:Proper_bouns_optimism_pessimism_bound_TV_Bernstein} shows that $A \leq 2\beta^{k}_{i,h}(s,{\bf a})$. For term $B$, we use the dual representation of $\sigma_{\mathcal{P}_{i,h}^{\rho_i}(s,{\bf a})}[V]$ from eq. \ref{eq:dual_TV} and \Cref{ass:vanisin_minimal} to first establish that $B \leq \sup_{\eta \in [0,H]}\{\mathbb{E}_{P^\star_h(\cdot|s,{\bf a})}[\eta-\overline{V}^{k,\rho_i}_{i,h+1}]_+ - \mathbb{E}_{P^\star_h(\cdot|s,{\bf a})}[\eta-\underline{V}^{k,\rho_i}_{i,h+1}]_+\}$. Since $\overline{V}^{k,\rho_i}_{i,h+1}\geq \underline{V}^{k,\rho_i}_{i,h+1}$ (by \Cref{lem:Optimistic_pessimism_TV_NE_Bernstein}), we can simplify this further to $B\leq \mathbb{E}_{P^\star_h(\cdot|s,{\bf a})}[\overline{V}^{k,\rho_i}_{i,h+1}-\underline{V}^{k,\rho_i}_{i,h+1}]$.

By substituting the bounds for $A$ and $B$ back into eq. \ref{eq:Regret_TV_step3_Bernstein}, we arrive at the following inequality:
\begin{align}
\overline{Q}_{i,h}^{k,\rho_i}(s,{\bf a}) - \underline{Q}_{i,h}^{k,\rho_i}(s,{\bf a})&\leq \mathbb{E}_{P^\star_h(\cdot|s,{\bf a})}[\overline{V}^{k,\rho_i}_{i,h+1}-\underline{V}^{k,\rho_i}_{i,h+1}] + 4\beta^{k}_{i,h}(s,{\bf a}). \label{eq:Regret_TV_step4_Bernstein}
\end{align}
Using \Cref{lem:Control_Bonus_TV_Bernstein} to upper bound the bonus term, and rearranging the terms, we obtain:
\begin{align}
\overline{Q}_{i,h}^{k,\rho_i}(s,{\bf a}) &- \underline{Q}_{i,h}^{k,\rho_i}(s,{\bf a}) \leq \bigg(1 + \frac{20}{H}\bigg)\mathbb{E}_{P^\star_h(\cdot|s,{\bf a})}[\overline{V}^{k,\rho_i}_{i,h+1}-\underline{V}^{k,\rho_i}_{i,h+1}] \nonumber\\
&+ 4\sqrt{\frac{c_1\iota
\mathrm{Var}_{P_h^\star(\cdot|s,{\bf a})}\left[ V_{i,h+1}^{\pi^k,\rho_i} \right]}{\{N_h^k(s,{\bf a}) \vee 1\}}}  + \frac{4c_2 H^2 S\iota}{\{N_h^k(s,{\bf a}) \vee 1\}}
+ \sqrt{\frac{4}{K}},\label{eq:Regret_TV_step4.1_Bernstein}
\end{align}
where $c_1, c_2>0$ are absolute constants.
From the definitions in eq. \ref{eq:robust_V_values_k}, the difference in V-functions is given by:
\begin{align}
\overline{V}_{i,h}^{k,\rho_i}(s) - \underline{V}_{i,h}^{k,\rho_i}(s) & = \mathbb{E}_{{\bf a}\sim\pi^k(\cdot|s)}\bigg[\overline{Q}_{i,h}^{k,\rho_i}(s,{\bf a}) - \underline{Q}_{i,h}^{k,\rho_i}(s,{\bf a}) \bigg].\label{eq:Regret_TV_step5_Bernstein}
\end{align}
Now, let's define a new recursive value function $\widetilde{V}_h^{k,\rho_{\min}}$ and a corresponding Q-function $\widetilde{Q}_h^{k,\rho_{\min}}$ with $\widetilde{V}_{H+1}^{k,\rho_{\min}} = 0$, where $\rho_{\min}=\min\limits_{i \in \mathcal{M}}\rho_i$.
Furthermore, we consider that there exists a $\rho^\star$ such that $\rho^\star=\argmax\limits_{\rho_i, \forall i \in \mathcal M, h \in [H]}\mathrm{Var}_{P_h^\star(\cdot|s,{\bf a})}\left[ V_{i,h+1}^{\pi^k,\rho_i} \right]$. 
Using these facts we now define
\begin{align}
\widetilde{Q}_h^{k,\rho_{\min}}(s, {\bf a}) &= \bigg(1 + \frac{20}{H}\bigg)\mathbb{E}_{P^\star_h(\cdot|s,{\bf a})} \left[\widetilde{V}_{h+1}^{k,\rho_{\min}}\right] + {\color{black}4\sqrt{\frac{c_1\iota
\mathrm{Var}_{P_h^\star(\cdot|s,{\bf a})}\left[ V_{h+1}^{\pi^k,\rho^\star} \right]}{\{N_h^k(s,{\bf a}) \vee 1\}}}} \nonumber\\
&\qquad + \frac{4c_2 H^2 S\iota}{\{N_h^k(s,{\bf a}) \vee 1\}} + \sqrt{\frac{4}{K}}, \label{eq:tilde_Q_TV_Bernstein}\\
\widetilde{V}_h^{k,\rho_{\min}}(s) &= \mathbb{E}_{{\bf a}\sim\pi^k(\cdot|s)}\left[\widetilde{Q}_h^{k,\rho_{\min}}(s, {\bf a})\right].\label{eq:tilde_V_TV_Bernstein}
\end{align}
It is a well-known property of robust value functions under TV-divergence that they become more conservative as the uncertainty radius $\rho_i$ decreases (e.g., \citep{INFORM2005_RobusDP_Iyengar, INFORMS2005_RMDP_Nilim}). Given that $\rho_{\min} \leq \rho_i$ for all agents $i \in \mathcal{M}$, it follows that for every next state $s' \in \mathcal{S}$:
\[
V_{i,h+1}^{\pi^k,\rho_i}(s') \leq V_{h+1}^{\pi^k,\rho_{\min}}(s') \quad \forall i \in \mathcal{M} \text{ and } s\in \mathcal{S}.
\]
We can inductively prove that for any $(i, h, s, \mathbf{a}) \in \mathcal{M} \times [H] \times \mathcal{S} \times \mathcal{A}$:
\begin{align}
\max\limits_{i\in \mathcal{M}} \left( \overline{Q}_{i,h}^{k,\rho_i}(s,{\bf a}) - \underline{Q}_{i,h}^{k,\rho_i}(s,{\bf a}) \right) &\leq \widetilde{Q}_h^{k,\rho_{\min}}(s, a), \\
\max\limits_{i\in \mathcal{M}} \left( \overline{V}_{i,h}^{k,\rho_i}(s) - \underline{V}_{i,h}^{k,\rho_i}(s) \right) &\leq \widetilde{V}_h^{k,\rho_{\min}}(s)\label{eq:max_V_bound}.
\end{align}
 Inductive Proof: Let eq. \ref{eq:max_V_bound} hold for $h+1$. So, we will show that eq. (32) holds for h. For any $h$, by eq. \ref{eq:Regret_TV_step4.1_Bernstein} we have
\begin{align}
&\left( \overline{Q}_{i,h}^{k,\rho_i}(s,{\bf a}) - \underline{Q}_{i,h}^{k,\rho_i}(s,{\bf a}) \right) \nonumber\\
    &\leq \left(1 + \frac{20}{H}\right)\mathbb{E}_{P^\star_h(\cdot|s,{\bf a})} \left[\overline{V}^{k,\rho_i}_{i,h+1}-\underline{V}^{k,\rho_i}_{i,h+1}\right]  + 4\sqrt{\frac{c_1\iota \mathrm{Var}_{P_h^\star(\cdot|s,{\bf a})}\left[ V_{i,h+1}^{\pi^k,\rho_i} \right]}{N_h^k(s,{\bf a}) \vee 1}} + \frac{4c_2 H^2 S\iota}{N_h^k(s,{\bf a}) \vee 1} + \sqrt{\frac{4}{K}},\nonumber\\
&\overset{(a)}{\leq}  \bigg(1 + \frac{20}{H}\bigg)\mathbb{E}_{P^\star_h(\cdot|s,{\bf a})}[\widetilde{V}_{h+1}^{k,\rho_{\min}}] + 4\sqrt{\frac{c_1\iota
\mathrm{Var}_{P_h^\star(\cdot|s,{\bf a})}\left[ V_{i,h+1}^{\pi^k,\rho^\star} \right]}{\{N_h^k(s,{\bf a}) \vee 1\}}} + \frac{4c_2 H^2 S\iota}{\{N_h^k(s,{\bf a}) \vee 1\}}
+ \sqrt{\frac{4}{K}},\nonumber\\
&\overset{(b)}{=}  \widetilde{Q}_h^{k,\rho_{\min}}(s, {\bf a}) \nonumber\\
&\overline{V}_{i,h}^{k,\rho_i}(s) - \underline{V}_{i,h}^{k,\rho_i}(s)  \overset{(c)}{=} \mathbb{E}_{{\bf a}\sim\pi^k(\cdot|s)}\bigg[\overline{Q}_{i,h}^{k,\rho_i}(s,{\bf a}) - \underline{Q}_{i,h}^{k,\rho_i}(s,{\bf a}) \bigg]. \overset{(d)}{\leq} \mathbb{E}_{{\bf a}\sim\pi^k(\cdot|s)}\bigg[ \widetilde{Q}_h^{k,\rho_{\min}}(s, {\bf a}) \bigg] \nonumber\\
&\qquad \qquad \qquad \qquad \overset{(e)}{=}\widetilde{V}_h^{k,\rho_{\min}}(s).\nonumber
\end{align}
where ineq. (a) holds by eq. \ref{eq:max_V_bound} which is true for $h+1$ and by the def. of $\rho^\star$ which implies $\mathrm{Var}_{P_h^\star(\cdot|s,{\bf a})}\left[ V_{i,h+1}^{\pi^k,\rho_i} \right] \leq \mathrm{Var}_{P_h^\star(\cdot|s,{\bf a})}\left[ V_{i,h+1}^{\pi^k,\rho^\star} \right]$; equality (b) holds by the definition of $\widetilde{Q}_h^{k,\rho_{\min}}(s, {\bf a})$ as in eq. \ref{eq:tilde_Q_TV_Bernstein}; equality (c) holds by eq. \ref{eq:Regret_TV_step5_Bernstein}; ineq. (d) holds by (b); equality (e) holds by eq. \ref{eq:tilde_V_TV_Bernstein}.
Therefore, we only need to upper bound the sum $\sum_{k=1}^K \widetilde{V}_1^{k,\rho_{\min}}(s^k_1)$. For simplicity, we define the following notations for the differences at any $(h,k) \in [H] \times [K]$:
\begin{align}
\Delta^k_{h} &:= \widetilde{V}_h^{k,\rho_{\min}}(s^k_h),\label{eq:Delta_k_h_TV_Bernstein}\\
\zeta_{h}^k &:= \Delta_{h}^k - \widetilde{Q}_h^{k,\rho_{\min}}(s_h^k, {\bf a}_h^k), \label{eq:zeta_h_k_TV_Bernstein}\\
\xi_{h}^k &:= \mathbb{E}_{P^\star_h(\cdot|s^k_h,{\bf a}^k_h)}[\widetilde{V}_{h+1}^{k,\rho_{\min}}] - \Delta_{h+1}^k. \label{eq:xi_h_k_TV_Bernstein}
\end{align}
We can confirm that $\{\zeta_{h}^k\}_{(h,k)}$ and $\{\xi_{h}^k\}_{(h,k)}$ are martingale difference sequences with respect to their respective filtrations. By substituting eq. \ref{eq:tilde_Q_TV_Bernstein} into eq. \ref{eq:zeta_h_k_TV_Bernstein}, we get:
\begin{align}
\Delta_{h}^k &= \zeta_{h}^k + \widetilde{Q}_h^{k,\rho_{\min}}(s_h^k, {\bf a}_h^k) \nonumber \\
&\le \zeta_{h}^k + \bigg(1 + \frac{20}{H}\bigg)\mathbb{E}_{P^\star_h(\cdot|s^k_h,{\bf a}^k_h)} \left[\widetilde{V}_{h+1}^{k,\rho_{\min}}\right] + 4\sqrt{\frac{c_1\iota
\mathrm{Var}_{P_h^\star(\cdot|s^k_h,{\bf a}^k_h)}\left[ V_{h+1}^{\pi^k,\rho^\star} \right]}{\{N_h^k(s^k_h,{\bf a}^k_h) \vee 1\}}} \nonumber\\
&\qquad \qquad + \frac{4c_2 H^2 S\iota}{\{N_h^k(s^k_h,{\bf a}^k_h) \vee 1\}} + \sqrt{\frac{4}{K}}\nonumber \\
&= \zeta_{h}^k + \bigg(1 + \frac{20}{H}\bigg)\xi_{h}^k + \bigg(1 + \frac{20}{H}\bigg)\Delta_{h+1}^k + 4\sqrt{\frac{c_1\iota
\mathrm{Var}_{P_h^\star(\cdot|s,{\bf a})}\left[ V_{h+1}^{\pi^k,\rho^\star} \right]}{\{N_h^k(s^k_h,{\bf a}^k_h) \vee 1\}}} \nonumber\\
&\qquad \qquad + \frac{4c_2 H^2 S\iota}{\{N_h^k(s^k_h,{\bf a}^k_h) \vee 1\}} + \sqrt{\frac{4}{K}}.\label{eq:Regret_TV_step6_Bernstein}
\end{align}
By recursively applying eq. \ref{eq:Regret_TV_step6_Bernstein} and noting that $\Big(1 + \frac{20}{H}\Big)^h\leq \Big(1 + \frac{20}{H}\Big)^H\leq c$ for some constant $c\geq 0$, we can upper bound the right-hand side of eq. \ref{eq:Regret_TV_step1_Bernstein} as:
\begin{align}
\label{eq:Regret_TV_step7_Bernstein}
&\Reg_{\sf NASH}(K) \leq \sum_{k=1}^K \Delta_{1}^k \leq c\sum_{k=1}^K \sum_{h=1}^H \Bigg\{(\zeta_{h}^k + \xi_{h}^k) \nonumber\\
& \qquad + \Bigg(4\sqrt{\frac{c_1\iota
\mathrm{Var}_{P_h^\star(\cdot|s,{\bf a})}\left[ V_{h+1}^{\pi^k,\rho^{*}} \right]}{\{N_h^k(s,{\bf a}) \vee 1\}}} + \frac{4c_2 H^2 S\iota}{\{N_h^k(s,{\bf a}) \vee 1\}}\Bigg)  + \sqrt{\frac{4}{K}}\Bigg\}.
\end{align}
The first term, a sum of martingale differences, is bounded using the Azuma-Hoeffding inequality from \Cref{lem:Azuma-Hoeffding}, yielding:
\begin{align}
\label{eq:Regret_Term_i_TV_step1_Bernstein}
\sum_{k=1}^{K} \sum_{h=1}^{H} (\zeta_{h}^k + \xi_{h}^k)
\leq c_1 \min\Bigg\{\frac{1}{\rho_{\min}},H\Bigg\}\sqrt{HK\iota},
\end{align}
where $c_1 > 0$ is an absolute constant.
For the second term, we apply the Cauchy–Schwarz inequality to the summation of the variance terms:
% \begin{align}
% \label{eq:Regret_Term_ii_TV_step1_Bernstein}
% &\sum_{k=1}^K \sum_{h=1}^H \sqrt{\frac{
% \mathrm{Var}_{P_h^\star(\cdot|s^k_h,{\bf a}^k_h)}\left[ V_{h+1}^{\pi^k,\rho_{\min}} \right]}{\{N_h^k(s^k_h,{\bf a}^k_h) \vee 1\}}} \le \sqrt{\left( \sum_{k=1}^K \sum_{h=1}^H \text{Var}_{P_h^{\star}(\cdot \mid s_h^k, {\bf a}_h^k)}\left[ V_{h+1}^{\pi^k,\rho_{\min}} \right] \right)\cdot
% \left( \sum_{k=1}^K \sum_{h=1}^H \frac{1}{\{N_h^k(s_h^k, a_h^k) \vee 1\}} \right)}.
% \end{align}
\begin{align}
\label{eq:Regret_Term_ii_TV_step1_Bernstein}
\sum_{k=1}^K \sum_{h=1}^H 
\sqrt{\frac{\mathrm{Var}_{P_h^\star(\cdot|s^k_h,{\bf a}^k_h)}
   \!\left[ V_{h+1}^{\pi^k,\rho^{*}} \right]}
   {\,N_h^k(s^k_h,{\bf a}^k_h) \vee 1}} \nonumber
\le& 
\sqrt{\left( \sum_{k=1}^K \sum_{h=1}^H 
   \mathrm{Var}_{P_h^{\star}(\cdot \mid s_h^k, {\bf a}_h^k)}
   \!\left[ V_{h+1}^{\pi^k,\rho^{*}} \right] \right)
   \cdot} \nonumber \\& \sqrt{\left( \sum_{k=1}^K \sum_{h=1}^H 
   \frac{1}{N_h^k(s_h^k, a_h^k) \vee 1} \right)}.
\end{align}

The second factor on the right-hand side is bounded by $c_2HS(\prod_{i=1}^m A_i)\iota$, as shown in \citep[Theorem 3]{liu2021sharp}, while the first factor is bounded using the Law of Total Variation and standard martingale concentration arguments (from \citep{jin2018q} and \citep{Arxiv2024_DRORLwithInteractiveData_Lu}):
\begin{align}
\label{eq:Regret_Term_ii_TV_step3_Bernstein}
\sum_{k=1}^{K} \sum_{h=1}^{H} \text{Var}_{P^\star_h(\cdot \mid s_h^k, {\bf a}_h^k)} \left[ V_{h+1}^{\pi^k,\rho^\star} \right]
&\leq c_3 \cdot \bigg( \min\left\{\frac{1}{\rho^\star},H\right\}HK + \min\left\{\frac{1}{\rho^\star},H\right\}^3H\iota \bigg).\nonumber\\
&\overset{(a)}{\leq } c_3 \cdot \bigg( \min\left\{\frac{1}{\rho_{\min}},H\right\}HK + \min\left\{\frac{1}{\rho_{\min}},H\right\}^3H\iota \bigg).
\end{align}
where (a) uses the fact that $\rho^\star\geq \rho_{\min}$, which implies that $\min\{1/\rho^\star,H\}\leq \min\{1/\rho_{\min},H\}$.
By combining these bounds and substituting them into eq. \ref{eq:Regret_Term_ii_TV_step1_Bernstein}, we can obtain a final bound for the second term. The third term, $\sum_{k=1}^K\sum_{h=1}^H\sqrt{\frac{4}{K}}$, is straightforwardly bounded by $c_5\sqrt{H^2K}$. By combining the bounds for all three terms, we arrive at the final regret bound for $\Reg_{\sf Nash}(K)$:
\begin{align}
\label{eq:Regret_TV_Bernstein}
\Reg_{\sf NASH}(K)= \mathcal{O}\Bigg( \sqrt{\min\bigg\{\frac{1}{\rho_{\min}},H\bigg\}H^2SK\Big(\prod_{i\in \mathcal{M}}A_i\Big)\iota^{\prime}}\Bigg),
\end{align}
where $\iota^{\prime} = \log^2\Big(\frac{SHK\prod_{i\in \mathcal{M}}A_i}{\delta}\Big)$.
This completes the proof of \Cref{thm:Regret_TV_bound_Bernstein}. \qedhere

\begin{rem}
\label{rem:proof_regret_TV_bound_Bernstein_CCE_CE}
The methodology for bounding the regret for Correlated Equilibrium (CE) and Coarse Correlated Equilibrium (CCE) settings mirrors the approach outlined here for the Nash equilibrium in the \RMGTV~ context. The proofs leverage \Cref{lem:Optimistic_pessimism_TV_CCE_Bernstein} and \Cref{lem:Optimistic_pessimism_TV_CE_Bernstein}, respectively.
\end{rem}
\end{proof}

\subsection{Key Lemmas for {\RMGTV}}
\label{subsubsec:key_lemma_TV_Bernstein}

\begin{lem}[Gap between maximum and minimum \citep{Arxiv2024_DRORLwithInteractiveData_Lu}]
\label{lem:gap_TV_Bernstein}
Consider any RMG $\mathcal{MG}_{\text{rob}} = \left\{ \mathcal{S}, \mathcal{A}, H, \{\mathcal{P}^{\rho_i}_{\text{TV}} (P^\star)\}_{i=1}^m, r \right\}$. The robust value function $V^{\pi,\rho_i}_{i,h}$ for all $i \in \mathcal{M}$ and $h\in [H]$ associated with any joint policy $\pi$ satisfies
\begin{align*}
    \forall (i,h)\in \mathcal{M}\times [H]:
 \max_{s\in \mathcal{S}}V^{\pi,\rho_i}_{i,h}(s) - \min_{s\in \mathcal{S}}V^{\pi,\rho_i}_{i,h}(s) \leq \nu^{\rho_i}_H,
 \end{align*}
 where $\nu^{\rho_i}_H := \min\left\{\frac{1}{\rho_i}, H-h+1\right\}\leq \min\left\{\frac{1}{\rho_i}, H\right\}$.
\end{lem}

\begin{proof}
    Refer to the proof-lines of Lemma 3 in \citep{Arxiv2024_SampleEfficientMARL_Shi}.
\end{proof}

\begin{lem}[Bound of optimistic and pessimistic value estimators with bonus for {\RMGTV}] 
\label{lem:Proper_bouns_optimism_pessimism_bound_TV_Bernstein}
Under the typical event $\mathcal{E}_{\text{TV}}$ defined in eq. \ref{eq:Bernstein_TV_event}  and by setting the bonus $ \beta^{k}_{i,h}$ as in eq. \ref{eq:Bonus_term_TV_Bernstein}, it holds that
\begin{align*}
&\sigma_{\widehat{\mathcal{P}^{\rho_i}_{i,h}}(s,\bm{a})}
   \!\left[ \overline{V}_{i,h+1}^{k,\rho_i} \right]  
 - \sigma_{\mathcal{P}^{\rho_i}_{i,h}(s,{\bf a})}
   \!\left[ \overline{V}_{i,h+1}^{k,\rho_i} \right]\nonumber\\ 
&\qquad + \sigma_{\mathcal{P}^{\rho_i}_{i,h}(s,{\bf a})}
   \!\left[ \underline{V}_{i,h+1}^{k,\rho_i} \right] 
 - \sigma_{\widehat{\mathcal{P}^{\rho_i}_{i,h}}(s,\bm{a})}
   \!\left[ \underline{V}_{i,h+1}^{k,\rho_i} \right] 
 \;\leq\; 2\beta^k_{i,h}(s,{\bf a}).
\end{align*}

\end{lem}
% \begin{proof}
%         Let us denote
%     \begin{align}
%     \label{eq:A_TV_Bernstein}
%         A:= \sigma_{\widehat{\mathcal{P}^{\rho_i}_{i,h}}(s,\bm{a})}\left[ \overline{V}_{i,h+1}^{k,\rho_i} \right]  - \sigma_{\mathcal{P}^{\rho_i}_{i,h}(s,{\bf a})}\left[ \overline{V}_{i,h+1}^{k,\rho_i} \right] + \sigma_{\mathcal{P}^{\rho_i}_{i,h}(s,{\bf a})}\left[ \underline{V}_{i,h+1}^{k,\rho_i} \right] - \sigma_{\widehat{\mathcal{P}^{\rho_i}_{i,h}}(s,\bm{a})}\left[ \underline{V}_{i,h+1}^{k,\rho_i} \right]. 
%     \end{align}

% We upper bound $A$ by using the concentration inequality given in \Cref{lem:Bernstein_concentration_TV_optimism_pessimism}, as follows
% \begin{align}
% \label{eq:A_TV_Bernstein_bound}
% A \leq  2\sqrt{ \frac{c_1\mathrm{Var}_{\widehat{P}_h^k}\left(V_{i,h+1}^{\dagger, \pi^k_{-i},\rho_i} \right) \cdot \iota }{\{N^k_h(s,{\bf a})\vee 1\}} } 
% + \frac{2\mathbb{E}_{\widehat{P}^k_h(\cdot|s, {\bf a})}\left[\up{V}^{k,\rho_i}_{i,h+1}- \low{V}^{k,\rho_i}_{i,h+1}\right]}{H} + \frac{2c'_2H^2S\iota}{\{N^k_h(s, {\bf a})\vee 1\}} + \frac{2}{\sqrt{K}},
% \end{align}
% where $\iota = \log\bigg(\frac{S^2(\prod_{i=1}^m A_i)H^2K^{3/2}}{\delta}\bigg)$ and $c_1, c'_2 >0$ are absolute constants. Now by applying \Cref{lem:variance_analysis_1_Bernstein} in the variance term in \ref{eq:A_TV_Bernstein_bound}, we get the required bound in \Cref{lem:Proper_bouns_optimism_pessimism_bound_TV_Bernstein}.
% \end{proof}
\begin{proof}
Let's denote the term to be bounded as $A$.
\begin{align}
\label{eq:A_TV_Bernstein}
A :=\; &\sigma_{\widehat{\mathcal{P}^{\rho_i}_{i,h}}(s,\bm{a})}
   \!\left[ \overline{V}_{i,h+1}^{k,\rho_i} \right]  
 - \sigma_{\mathcal{P}^{\rho_i}_{i,h}(s,{\bf a})}
   \!\left[ \overline{V}_{i,h+1}^{k,\rho_i} \right] + \sigma_{\mathcal{P}^{\rho_i}_{i,h}(s,{\bf a})}
   \!\left[ \underline{V}_{i,h+1}^{k,\rho_i} \right] 
 - \sigma_{\widehat{\mathcal{P}^{\rho_i}_{i,h}}(s,\bm{a})}
   \!\left[ \underline{V}_{i,h+1}^{k,\rho_i} \right].
\end{align}
Under the high-probability event $\mathcal{E}_{\text{TV}}$ (as defined in eq. \ref{eq:Bernstein_TV_event}), we can apply the concentration inequality from \Cref{lem:Bernstein_concentration_TV_optimism_pessimism} to upper bound $A$ as follows:
\begin{align}
\label{eq:A_TV_Bernstein_bound}
A \leq\;& 
2\sqrt{\frac{c_1\,\mathrm{Var}_{\widehat{P}_h^k}
   \!\left(V_{i,h+1}^{\dagger, \pi^k_{-i},\rho_i} \right)\,\iota}
   {\,N^k_h(s,{\bf a}) \vee 1}} + \frac{2\,\E_{\widehat{P}^k_h(\cdot|s, {\bf a})}
   \!\left[\up{V}^{k,\rho_i}_{i,h+1}- \low{V}^{k,\rho_i}_{i,h+1}\right]}{H} + \frac{2c'_2 H^2 S \,\iota}{\,N^k_h(s, {\bf a}) \vee 1}
 + \frac{2}{\sqrt{K}}.
\end{align}
where $\iota = \log\big(S^2(\prod_{i=1}^m A_i)H^2K^{3/2}/\delta\big)$ and $c_1, c'_2 >0$ are absolute constants. By applying the result from \Cref{lem:variance_analysis_1_Bernstein} to the variance term in eq. \ref{eq:A_TV_Bernstein_bound}, we obtain the required bound presented in the lemma statement. This concludes the proof.
\end{proof}

\begin{lem}[Bound of the bonus term for {\RMGTV}]
\label{lem:Control_Bonus_TV_Bernstein}
\textit{Under the typical event \( \mathcal{E}_{\text{TV}}\), the bonus term defined in \ref{eq:Bonus_term_TV_Bernstein} is bounded by}
\begin{align*}
\beta^{k}_{i,h}(s,{\bf a}) \leq\;&
\sqrt{\frac{c_1 \iota \,
   \mathrm{Var}_{P_h^\star(\cdot|s,{\bf a})}
   \!\left[ V_{i,h+1}^{\pi^k,\rho_i} \right]}
   {\,N_h^k(s,{\bf a}) \vee 1}} + \frac{5\,\E_{P^\star_h(\cdot|s, {\bf a})}
   \!\left[\up{V}^{k,\rho_i}_{i,h+1}
   - \low{V}^{k,\rho_i}_{i,h+1}\right]}{H} \\
&+ \frac{c_2 H^2 S \,\iota}{\,N_h^k(s,{\bf a}) \vee 1}
 + \sqrt{\tfrac{1}{K}}.
\end{align*}
where \( \iota = \log(S^3 (\prod_{i=1}^m A_i) H^2 K^{3/2} / \delta) \) and \( c_1, c_2 > 0 \) are constants.
\end{lem}

\begin{proof}
    The proof-lines are similar to \citep[Lemma E.4]{Arxiv2024_DRORLwithInteractiveData_Lu} or \citep[Lemma K.3]{ghosh2025provablynearoptimaldistributionallyrobust}. Recall the bonus term defined in eq. \ref{eq:Bonus_term_TV_Bernstein}. We need to bound the first and second term of eq. \ref{eq:Bonus_term_TV_Bernstein}. We first bound the second term of $\beta^{k}_{i,h}(s,{\bf a})$ by using \Cref{lem:non_robust_TV_Bernstein}, and we get
    \begin{align}
        \label{eq:Bonus_term_TV_Bernstein_step1}
        \frac{2\mathbb{E}_{\widehat{P}^k_h(\cdot|s, {\bf a})}\left[\up{V}^{k,\rho_i}_{i,h+1}- \low{V}^{k,\rho_i}_{i,h+1}\right]}{H} &\leq \bigg( \frac{2}{H} + \frac{2}{H^2}\bigg)\mathbb{E}_{P^\star_h(\cdot|s, {\bf a})}\left[\up{V}^{k,\rho_i}_{i,h+1}- \low{V}^{k,\rho_i}_{i,h+1}\right] + \frac{c'_2 H S\iota}{\{N_h^k(s,{\bf a}) \vee 1\}}\nonumber\\
        &\leq \frac{4\mathbb{E}_{P^\star_h(\cdot|s, {\bf a})}\left[\up{V}^{k,\rho_i}_{i,h+1}- \low{V}^{k,\rho_i}_{i,h+1}\right]}{H} + \frac{c'_2 H S\iota}{\{N_h^k(s,{\bf a}) \vee 1\}},
    \end{align}
    where the second inequality is from $H\geq 1$. We now bound the first term (variance term) of eq. \ref{eq:Bonus_term_TV_Bernstein} by using \Cref{lem:variance_analysis_2_Bernstein}, which gives
    % \begin{align}
    % \label{eq:Bonus_term_TV_Bernstein_step2}
    %     \sqrt{\frac{c_1\iota \mathrm{Var}_{\widehat{P}_h^k(\cdot|s,{\bf a})} \left[ \left( \frac{ \overline{V}_{i,h+1}^{k,\rho_i} + \underline{V}_{i,h+1}^{k,\rho_i} }{2} \right) \right]}{\{N^k_h(s,{\bf a})\vee 1\}}} &\leq \sqrt{\frac{c'_1\iota\mathrm{Var}_{P^\star_h(\cdot|s,{\bf a})}\left[V^{\pi^k,\rho_i}_{i,h+1} \right]}{\{N_h^k(s,{\bf a}) \vee 1\}}} + \frac{\mathbb{E}_{P^\star_h(\cdot|s, {\bf a})}\left[\up{V}^{k,\rho_i}_{i,h+1}- \low{V}^{k,\rho_i}_{i,h+1}\right]}{H} + \frac{c_3H^2 S\iota}{\{N_h^k(s,{\bf a}) \vee 1\}},
    % \end{align}
    \begin{equation} \label{eq:Bonus_term_TV_Bernstein_step2}
    \begin{split}
    \sqrt{\frac{c_1 \iota 
       \,\mathrm{Var}_{\widehat{P}_h^k(\cdot|s,{\bf a})} 
       \!\left[ \frac{\overline{V}_{i,h+1}^{k,\rho_i} 
       + \underline{V}_{i,h+1}^{k,\rho_i}}{2} \right]}
       {\,N^k_h(s,{\bf a}) \vee 1}} 
    &\leq \sqrt{\frac{c'_1 \iota 
       \,\mathrm{Var}_{P^\star_h(\cdot|s,{\bf a})} 
       \!\left[V^{\pi^k,\rho_i}_{i,h+1} \right]}
       {\,N_h^k(s,{\bf a}) \vee 1}} \\
    &\quad + \frac{\E_{P^\star_h(\cdot|s, {\bf a})}
       \!\left[\up{V}^{k,\rho_i}_{i,h+1} 
       - \low{V}^{k,\rho_i}_{i,h+1}\right]}{H} \\
    &\quad + \frac{c_3 H^2 S \iota}{\,N_h^k(s,{\bf a}) \vee 1}.
    \end{split}
    \end{equation}

    where $c_3>0$ is an absolutely constant. Thus by combining eq. \ref{eq:Bonus_term_TV_Bernstein_step1} and eq. \ref{eq:Bonus_term_TV_Bernstein_step2} with the choice of bonus term in eq. \ref{eq:Bonus_term_TV_Bernstein}, we can conclude the proof of \Cref{lem:Control_Bonus_TV_Bernstein}.
\end{proof}

\subsubsection{NE Version: Optimistic and pessimistic estimation of the robust values for {\RMGTV}.}
Here we will proof the optimistic estimations are indeed upper bounds of the corresponding robust V-value and robust Q-value functions fro NE version.

\begin{lem}[Optimistic and pessimistic estimation of the robust values for {\RMGTV} for NE version]
\label{lem:Optimistic_pessimism_TV_NE_Bernstein}
By setting the bonus term $\beta^{k}_{i,h}$ as in eq. \ref{eq:Bonus_term_TV_Bernstein}, with probability $1-\delta$, for any $(s, {\bf a}, h, i)$ and $k \in [K]$, it holds that
    \begin{equation}
        \label{eq:general-Q-UCB-TV_Bernstein}
    Q_{i,h}^{\dagger,\pi _{-i}^{k}, \rho_i}\left( s,\bm{a} \right) \le \up{Q}_{i,h}^{k, \rho_i}\left( s,\bm{a} \right)  ,   \,\,\,\,  \low{Q}_{i,h}^{k,\rho_i}\left( s,\bm{a} \right) \le Q_{i,h}^{\pi^{k}, \rho_i}\left( s,\bm{a} \right),
    \end{equation}
    \begin{equation}
        \label{eq:general-V-UCB-TV_Bernstein}
        V_{i,h}^{\dagger,\pi _{-i}^{k}, \rho_i}\left( s \right) \le \up{V}_{i,h}^{k, \rho_i}\left( s \right)  , \,\,\,\,  \low{V}_{i,h}^{k, \rho_i}\left( s \right) \le V_{i,h}^{\pi^{k}, \rho_i}\left( s \right).
    \end{equation}
\end{lem}
\begin{proof}
The proof-lines are similar to \citep{ghosh2025provablynearoptimaldistributionallyrobust} adapted to the multi-agent case.\\
We will run a proof for each inequality outlined in Lemma \ref{lem:Optimistic_pessimism_TV_NE_Bernstein}.
\begin{itemize}
    \item \textbf{Ineq. 1:} To prove $Q_{i,h}^{\dagger,\pi_{-i}^k,\rho_i}(s,\bm{a}) \leq \overline{Q}_{i,h}^{k,\rho_i}(s,\bm{a})$.
    \item \textbf{Ineq. 2:} To prove $\underline{Q}_{i,h}^{k,\rho_i}(s,\bm{a}) \leq Q_{i,h}^{\pi^k,\rho_i}(s,\bm{a})$.
\end{itemize}
We know that, at step $h=H+1$, $\up{V}^{k,\rho_i}_{i,H+1}(s)= V^{\dagger,\pi^k_{-1},\rho_i}_{i,H+1}(s)=0$. Now, we assume that both eq. \ref{eq:general-Q-UCB-TV_Bernstein} and eq. \ref{eq:general-V-UCB-TV_Bernstein} hold at the $(h+1)$-th step. 

\begin{itemize}
    \item \textbf{Proof of Ineq. 1:}
We first consider robust $Q$ at the $h$-th step. Then, by \Cref{prop:Robust_Bellman_eq} (Robust Bellman Equation) and eq. \ref{eq:robust_Qupper_k}, we have that
\begin{align}
 &\overline{Q}_{i,h}^{k,\rho_i}(s,\bm{a})-  Q_{i,h}^{\dagger,\pi_{-i}^k,\rho_i}(s,\bm{a})\nonumber\\
&= \min \bigg\{ \sigma_{\widehat{\mathcal{P}^{\rho_i}_{i,h}}(s,\bm{a})} \left[\overline{V}^{k,\rho_i}_{i,h+1} \right]- \sigma_{\mathcal{P}^{\rho_i}_{i,h}(s,\bm{a})} \left[ V_{i,h+1}^{\dagger,\pi_{-i}^k,\rho_i} \right]  + \beta^k_{i,h}(s,{\bf a}), 
\nu^{\rho_i}_H- Q_{i,h}^{\dagger,\pi_{-i}^k,\rho_i}(s,\bm{a}) \bigg\} \nonumber\\
&\geq \min \Big\{
\sigma_{\widehat{\mathcal{P}^{\rho_i}_{i,h}}(s,\bm{a})} 
\left[ V_{i,h+1}^{\dagger,\pi_{-i}^k,\rho_i} \right] 
- \sigma_{\mathcal{P}^{\rho_i}_{i,h}(s,\bm{a})} 
\left[ V_{i,h+1}^{\dagger,\pi_{-i}^k,\rho_i} \right] + \beta^k_{i,h}(s,{\bf a}), 0\Big\},\label{eq:optimism_pessimism_ineq_NE_TV_step1_Bernstein}
\end{align}
where the second inequality follows from the induction of $V_{i,h+1}^{\dagger,\pi_{-i}^k,\rho_i} \leq \overline{V}_{i,h+1}^{k,\rho_i}$ at the $h+1$-th step and the fact that $Q_{i,h}^{\dagger,\pi_{-i}^k,\rho_i} \leq \nu^{\rho_i}_H$ by \Cref{lem:gap_TV_Bernstein}. By \cref{lem:Bernstein_concentration_TV_optimism_pessimism_join_bestresponse}, we get
\begin{align}
  \label{eq:optimism_pessimism_ineq_NE_TV_step2_Bernstein}
  \sigma_{\widehat{\mathcal{P}^{\rho_i}_{i,h}}(s,\bm{a})} \left[ V_{i,h+1}^{\dagger,\pi_{-i}^k,\rho_i} \right]- \sigma_{\mathcal{P}^{\rho_i}_{i,h}(s,\bm{a})} \left[  V_{i,h+1}^{\dagger,\pi_{-i}^k,\rho_i} \right] \leq   \sqrt{ \frac{c_1\mathrm{Var}_{\widehat{P}_h^k}\left(V_{i,h+1}^{\dagger, \pi^k_{-i},\rho_i} \right) \cdot \iota }{\{N^k_h(s,{\bf a})\vee 1\}} } \nonumber\\ 
+ \frac{ c_2H \iota }{\{N^k_h(s,{\bf a})\vee 1\}} + \frac{1}{\sqrt{K}}.
\end{align}

Now by further applying \Cref{lem:variance_analysis_1_Bernstein} to the variance term in the above inequality, we can obtain that
\begin{align}
    \label{eq:optimism_pessimism_ineq_NE_TV_step3_Bernstein}
    &\sigma_{\widehat{\mathcal{P}^{\rho_i}_{i,h}}(s,\bm{a})} \left[ V_{i,h+1}^{\dagger,\pi_{-i}^k,\rho_i} \right]- \sigma_{\mathcal{P}^{\rho_i}_{i,h}(s,\bm{a})} \left[  V_{i,h+1}^{\dagger,\pi_{-i}^k,\rho_i} \right] \nonumber\\
    &\leq \sqrt{\frac{c_1\bigg(\mathrm{Var}_{\widehat{P}_h^k(\cdot|s,{\bf a})} \left[ \left( \frac{ \overline{V}_{i,h+1}^{k,\rho_i} + \underline{V}_{i,h+1}^{k,\rho_i} }{2} \right) \right] + 4H\mathbb{E}_{\widehat{P}^k_h(\cdot|s, {\bf a})}\left[\overline{V}_{i,h+1}^{k,\rho_i} - \underline{V}_{i,h+1}^{k,\rho_i} \right]\bigg)\iota}{\{N^k_h(s,{\bf a})\vee 1\}}} \nonumber\\ &+ \frac{ c_2H \iota }{\{N^k_h(s,{\bf a})\vee 1\}} + \frac{1}{\sqrt{K}}\nonumber\\
    &\overset{(i)}{\leq} \sqrt{\frac{c_1\iota \mathrm{Var}_{\widehat{P}_h^k(\cdot|s,{\bf a})} \left[ \left( \frac{ \overline{V}_{i,h+1}^{k,\rho_i} + \underline{V}_{i,h+1}^{k,\rho_i} }{2} \right) \right]}{\{N^k_h(s,{\bf a})\vee 1\}}} + \sqrt{\frac{4Hc_1\iota\mathbb{E}_{\widehat{P}^k_h(\cdot|s, {\bf a})}\left[\overline{V}_{i,h+1}^{k,\rho_i} - \underline{V}_{i,h+1}^{k,\rho_i} \right]\bigg) }{\{N^k_h(s,{\bf a})\vee 1\}}} \nonumber\\ &+  \frac{ c_2H \iota }{\{N^k_h(s,{\bf a})\vee 1\}} + \frac{1}{\sqrt{K}}\nonumber\\
    &\overset{(ii)}{\leq} \sqrt{\frac{c_1\iota \mathrm{Var}_{\widehat{P}_h^k(\cdot|s,{\bf a})} \left[ \left( \frac{ \overline{V}_{i,h+1}^{k,\rho_i} + \underline{V}_{i,h+1}^{k,\rho_i} }{2} \right) \right]}{\{N^k_h(s,{\bf a})\vee 1\}}} + \frac{\mathbb{E}_{\widehat{P}^k_h(\cdot|s, {\bf a})}\left[\overline{V}_{i,h+1}^{k,\rho_i} - \underline{V}_{i,h+1}^{k,\rho_i} \right]\bigg) }{H} \nonumber\\ &+  \frac{H^2c'_2\iota}{\{N^k_h(s,{\bf a})\vee 1\}}  + \frac{1}{\sqrt{K}},
\end{align}
where the inequality (i) is due to $\sqrt{a+b}\leq \sqrt{a} +\sqrt{b}$,
and the last inequality (ii) is from $\sqrt{ab}\leq a+b$ where $c'_2 > 0$ is an absolute constant. Therefore, combining eqns.  \ref{eq:optimism_pessimism_ineq_NE_TV_step1_Bernstein}, \ref{eq:optimism_pessimism_ineq_NE_TV_step2_Bernstein}, \ref{eq:optimism_pessimism_ineq_NE_TV_step3_Bernstein}, and the choice of bonus in \ref{eq:Bonus_term_TV_Bernstein}, we can conclude that $ \overline{Q}_{i,h}^{k,\rho_i}(s,\bm{a})-  Q_{i,h}^{\dagger,\pi_{-i}^k,\rho_i}(s,\bm{a})\geq 0$.

 \item \textbf{Proof of Ineq. 2:} By \Cref{prop:Robust_Bellman_eq} (Robust Bellman Equation) and eq. \ref{eq:robust_Qlower_k}, we have that
 \begin{align}
     &\underline{Q}_{i,h}^{k,\rho_i}(s,\bm{a}) - Q_{i,h}^{\pi^k,\rho_i}(s,\bm{a}) \nonumber\\
     &= \max \bigg\{ \sigma_{\widehat{\mathcal{P}^{\rho_i}_{i,h}}(s,\bm{a})} \left[\underline{V}^{k,\rho_i}_{i,h+1} \right]- \sigma_{\mathcal{P}^{\rho_i}_{i,h}(s,\bm{a})} \left[ V_{i,h+1}^{\pi^k,\rho_i} \right]  - \beta^k_{i,h}(s,{\bf a}),
0- Q_{i,h}^{\dagger,\pi_{-i}^k,\rho_i}(s,\bm{a}) \bigg\}, \nonumber\\
&\leq \max \Big\{
\sigma_{\widehat{\mathcal{P}^{\rho_i}_{i,h}}(s,\bm{a})} 
\left[ V_{i,h+1}^{\pi^k,\rho_i} \right] 
- \sigma_{\mathcal{P}^{\rho_i}_{i,h}(s,\bm{a})} 
\left[ V_{i,h+1}^{\pi^k,\rho_i} \right]  - \beta^k_{i,h}(s,{\bf a}), 0
\Big\},\label{eq:optimism_pessimism_ineq_NE_TV_step4_Bernstein}
 \end{align}
where the second inequality follows from the induction of $V_{i,h+1}^{\pi^k,\rho_i} \geq \underline{V}_{i,h+1}^{k,\rho_i}$ at the $h+1$-th step and the fact that $Q_{i,h}^{\pi^k,\rho_i} \geq 0$. By \cref{lem:Bernstein_concentration_TV_optimism_pessimism_join_bestresponse}, we can confirm that
\begin{align}
  \label{eq:optimism_pessimism_ineq_NE_TV_step5_Bernstein}
  \sigma_{\widehat{\mathcal{P}^{\rho_i}_{i,h}}(s,\bm{a})} \left[V_{i,h+1}^{\pi^k,\rho_i} \right]- \sigma_{\mathcal{P}^{\rho_i}_{i,h}(s,\bm{a})} \left[ V_{i,h+1}^{\pi^k,\rho_i}\right]& \leq   \sqrt{ \frac{c_1\mathrm{Var}_{\widehat{P}_h^k}\left(V_{i,h+1}^{\dagger, \pi^k_{-i},\rho_i} \right) \cdot \iota }{\{N^k_h(s,{\bf a})\vee 1\}} } 
\nonumber\\&+ \frac{\mathbb{E}_{\widehat{P}^k_h(\cdot|s, {\bf a})}\left[\up{V}^{k,\rho_i}_{i,h+1}- \low{V}^{k,\rho_i}_{i,h+1})\right]}{H}\nonumber\\ &+ \frac{c'_2H^2S\iota}{\{N^k_h(s, {\bf a})\vee 1\}} + \frac{1}{\sqrt{K}}.
\end{align}

Now by further applying \Cref{lem:variance_analysis_1_Bernstein} to the variance term in the above inequality, with an argument similar to eq. \ref{eq:optimism_pessimism_ineq_NE_TV_step2_Bernstein} we can obtain that
\begin{align}
\label{eq:optimism_pessimism_ineq_NE_TV_step6_Bernstein}
    \sigma_{\widehat{\mathcal{P}^{\rho_i}_{i,h}}(s,\bm{a})} \left[V_{i,h+1}^{\pi^k,\rho_i} \right]- \sigma_{\mathcal{P}^{\rho_i}_{i,h}(s,\bm{a})} \left[ V_{i,h+1}^{\pi^k,\rho_i}\right]& \leq 
   \sqrt{ \frac{c_1\mathrm{Var}_{\widehat{P}_h^k}\left(V_{i,h+1}^{\dagger, \pi^k_{-i},\rho_i} \right) \cdot \iota }{\{N^k_h(s,{\bf a})\vee 1\}} } \nonumber\\
&+ \frac{\mathbb{E}_{\widehat{P}^k_h(\cdot|s, {\bf a})}\left[\up{V}^{k,\rho_i}_{i,h+1}- \low{V}^{k,\rho_i}_{i,h+1})\right]}{H}\nonumber\\
&\qquad + \frac{c''_2H^2S\iota}{\{N^k_h(s, {\bf a})\vee 1\}} + \frac{1}{\sqrt{K}}.
\end{align}
where  $c''_2 > 0$ is an absolute constant. Therefore, combining eqns.  \ref{eq:optimism_pessimism_ineq_NE_TV_step4_Bernstein}, \ref{eq:optimism_pessimism_ineq_NE_TV_step5_Bernstein}, \ref{eq:optimism_pessimism_ineq_NE_TV_step6_Bernstein}, and the choice of bonus in \ref{eq:Bonus_term_TV_Bernstein}, $\underline{Q}_{i,h}^{k,\rho_i}(s,\bm{a})-  Q_{i,h}^{\pi^k,\rho_i}(s,\bm{a})\leq 0$.
\end{itemize}

Therefore, by eq. \ref{eq:optimism_pessimism_ineq_NE_TV_step3_Bernstein} and eq. \ref{eq:optimism_pessimism_ineq_NE_TV_step6_Bernstein}, we have proved that at step $h$, it holds that
\begin{align}
\label{eq:optimism_pessimism_ineq_Qbound_final_Bernstein}
   Q_{i,h}^{\dagger,\pi _{-i}^{k}, \rho_i}\left( s,\bm{a} \right) \le \up{Q}_{i,h}^{k, \rho_i}\left( s,\bm{a} \right)  ,   \,\,\,\,  \low{Q}_{i,h}^{k,\rho_i}\left( s,\bm{a} \right) \le Q_{i,h}^{\pi^{k}, \rho_i}\left( s,\bm{a} \right).
\end{align}
We now assume that eq. \ref{eq:general-Q-UCB-TV_Bernstein} hold for $h$-th step. Then, by the definition of robust value function as given by robust Bellman equation (\Cref{prop:Robust_Bellman_eq}), and eq. \ref{eq:robust_V_values_k}, and NASH Equilibrium, we get
\begin{align}
\label{eq:optimism_pessimism_ineq_TV_step3_Bernstein}
    \up{V}_{i,h}^{k, \rho_i}\left( s \right) = \mathbb{E}_{\bm{a} \sim \pi^k(\cdot|s)}\left[\overline{Q}_{i,h}^{k, \rho_i}(s,{\bf a}) \right] = \max_{\pi^\prime_i}\mathbb{E}_{\bm{a} \sim \pi^\prime_i\times\pi_{-i}^k(\cdot|s)}\left[\overline{Q}_{i,h}^{k, \rho_i}(s,{\bf a}) \right].
\end{align}
By the definition of $ V_{i,h}^{\dagger,\pi _{-i}^{k}, \rho_i}\left( s \right)$ in eq. \ref{eq:robust_best_response_agent_i}, we get
\begin{align}
\label{eq:optimism_pessimism_ineq_TV_step4_Bernstein}
    V_{i,h}^{\dagger,\pi _{-i}^{k}, \rho_i}\left( s \right) = \max_{\pi^\prime_i}\mathbb{E}_{\bm{a} \sim \pi^\prime_i\times\pi_{-i}^k(\cdot|s)}\left[Q_{i,h}^{\dagger,\pi _{-i}^{k}, \rho_i}(s,{\bf a}) \right].
\end{align}
Since by induction, for any $(s,{\bf a})$, $\overline{Q}_{i,h}^{k, \rho_i}(s,{\bf a}) \geq Q_{i,h}^{\dagger,\pi _{-i}^{k}, \rho_i}(s,{\bf a})$. As a result, we also have $ \up{V}_{i,h}^{k, \rho_i}\left( s \right) \geq  V_{i,h}^{\dagger,\pi _{-i}^{k}, \rho_i}\left( s \right)$, which is eq. \ref{eq:general-V-UCB-TV_Bernstein} for $h$-th step. Similarly, we can show that
\begin{align}
\label{eq:optimism_pessimism_ineq_TV_step5_Bernstein}
    \low{V}_{i,h}^{k, \rho_i}\left( s \right) &= \mathbb{E}_{\bm{a} \sim \pi^k(\cdot|s)}\left[\underline{Q}_{i,h}^{k, \rho_i}(s,{\bf a}) \right], \nonumber\\
    &\overset{(i)}{\leq }\mathbb{E}_{\bm{a} \sim \pi^k(\cdot|s)}\left[Q_{i,h}^{\pi^{k}, \rho_i}(s,{\bf a}) \right],\nonumber\\
    &\overset{(ii)}{=}V_{i,h}^{\pi^{k}, \rho_i}\left( s \right),
\end{align}
where (i) is due to the fact that $\low{Q}_{i,h}^{k,\rho_i}\left( s,\bm{a} \right) \le Q_{i,h}^{\pi^{k}, \rho_i}\left( s,\bm{a} \right)$ and (ii) is by definition of $V_{i,h}^{\pi^{k}, \rho_i}\left( s \right)$ as given by Bellman equation in \Cref{prop:Robust_Bellman_eq}.
\end{proof}

\subsubsection{CCE Version: Optimistic and pessimistic estimation of the robust values for {\RMGTV}.}
Here we will proof the optimistic estimations are indeed upper bounds of the corresponding robust V-value and robust Q-value functions for CCE version.

\begin{lem}[Optimistic and pessimistic estimation of the robust values for {\RMGTV} for CCE version]
\label{lem:Optimistic_pessimism_TV_CCE_Bernstein}
By setting the bonus term $\beta^{k}_{i,h}$ as in eq. \ref{eq:Bonus_term_TV_Bernstein}, with probability $1-\delta$, for any $(s, {\bf a}, h, i)$ and $k \in [K]$, it holds that
    \begin{equation}
        \label{eq:general-Q-UCB-TV_Bernstein-CCE}
    \max_{\phi \in \Phi_i} Q^{\phi \circ \pi^k, \rho_i}_{{i,h}}\left( s,\bm{a} \right) \le \up{Q}_{i,h}^{k, \rho_i}\left( s,\bm{a} \right)  ,   \,\,\,\,  \low{Q}_{i,h}^{k,\rho_i}\left( s,\bm{a} \right) \le Q_{i,h}^{\pi^{k}, \rho_i}\left( s,\bm{a} \right),
    \end{equation}
    \begin{equation}
        \label{eq:general-V-UCB-TV_Bernstein-CCE}
       \max_{\phi \in \Phi_i} V^{\phi \circ \pi^k, \rho_i}_{{i,h}}(s) \le \up{V}_{i,h}^{k, \rho_i}\left( s \right)  , \,\,\,\,  \low{V}_{i,h}^{k, \rho_i}\left( s \right) \le V_{i,h}^{\pi^{k}, \rho_i}\left( s \right).
    \end{equation}
\end{lem}
\begin{proof}
The proof-lines are similar to \citep{ghosh2025provablynearoptimaldistributionallyrobust} adapted to the multi-agent case.\\
We will run a proof for each inequality outlined in \Cref{lem:Optimistic_pessimism_TV_CCE_Bernstein}.
\begin{itemize}
    \item \textbf{Ineq. 1:} To prove $Q_{i,h}^{\dagger,\pi_{-i}^k,\rho_i}(s,\bm{a}) \leq \overline{Q}_{i,h}^{k,\rho_i}(s,\bm{a})$.
    \item \textbf{Ineq. 2:} To prove $\underline{Q}_{i,h}^{k,\rho_i}(s,\bm{a}) \leq Q_{i,h}^{\pi^k,\rho_i}(s,\bm{a})$.
\end{itemize}
We know that, at step $h=H+1$, $\up{V}^{k,\rho_i}_{i,H+1}(s)= V^{\dagger,\pi^k_{-1},\rho_i}_{i,H+1}(s)=0$. Now, we assume that both eq. \ref{eq:general-Q-UCB-TV_Bernstein-CCE} and eq. \ref{eq:general-V-UCB-TV_Bernstein-CCE} hold at the $(h+1)$-th step. 

\begin{itemize}
    \item \textbf{Proof of Ineq. 1:}
We first consider robust $Q$ at the $h$-th step. Then, by \Cref{prop:Robust_Bellman_eq} (Robust Bellman Equation) and eq. \ref{eq:robust_Qupper_k}, we have that
\begin{align}
   &\overline{Q}_{i,h}^{k,\rho_i}(s,\bm{a})-  Q_{i,h}^{\dagger,\pi_{-i}^k,\rho_i}(s,\bm{a})
\nonumber\\
&= \min \bigg\{ \sigma_{\widehat{\mathcal{P}^{\rho_i}_{i,h}}(s,\bm{a})} \left[\overline{V}^{k,\rho_i}_{i,h+1} \right]- \sigma_{\mathcal{P}^{\rho_i}_{i,h}(s,\bm{a})} \left[ V_{i,h+1}^{\dagger,\pi_{-i}^k,\rho_i} \right] + \beta^k_{i,h}(s,{\bf a}),  
\nu^{\rho_i}_H- Q_{i,h}^{\dagger,\pi_{-i}^k,\rho_i}(s,\bm{a}) \bigg\}, \nonumber\\
&\geq \min \Big\{
\sigma_{\widehat{\mathcal{P}^{\rho_i}_{i,h}}(s,\bm{a})} 
\left[ V_{i,h+1}^{\dagger,\pi_{-i}^k,\rho_i} \right]
- \sigma_{\mathcal{P}^{\rho_i}_{i,h}(s,\bm{a})} 
\left[ V_{i,h+1}^{\dagger,\pi_{-i}^k,\rho_i} \right]  + \beta^k_{i,h}(s,{\bf a}), \; 0
\Big\},\label{eq:optimism_pessimism_ineq_CCE_TV_step1_Bernstein}
\end{align}
where the second inequality follows from the induction of $V_{i,h+1}^{\dagger,\pi_{-i}^k,\rho_i} \leq \overline{V}_{i,h+1}^{k,\rho_i}$ at the $h+1$-th step and the fact that $Q_{i,h}^{\dagger,\pi_{-i}^k,\rho_i} \leq \nu^{\rho_i}_H$ by \Cref{lem:gap_TV_Bernstein}. By \cref{lem:Bernstein_concentration_TV_optimism_pessimism_join_bestresponse}, we get
\begin{align}
  \label{eq:optimism_pessimism_ineq_CCE_TV_step2_Bernstein}
  \sigma_{\widehat{\mathcal{P}^{\rho_i}_{i,h}}(s,\bm{a})} \left[ V_{i,h+1}^{\dagger,\pi_{-i}^k,\rho_i} \right]- \sigma_{\mathcal{P}^{\rho_i}_{i,h}(s,\bm{a})} \left[  V_{i,h+1}^{\dagger,\pi_{-i}^k,\rho_i} \right] \leq   \sqrt{ \frac{c_1\mathrm{Var}_{\widehat{P}_h^k}\left(V_{i,h+1}^{\dagger, \pi^k_{-i},\rho_i} \right) \cdot \iota }{\{N^k_h(s,{\bf a})\vee 1\}} } \nonumber\\
+ \frac{ c_2H \iota }{\{N^k_h(s,{\bf a})\vee 1\}} + \frac{1}{\sqrt{K}}.
\end{align}

Now by further applying \Cref{lem:variance_analysis_1_Bernstein} to the variance term in the above inequality, we can obtain that
\begin{align}
    \label{eq:optimism_pessimism_ineq_CCE_TV_step3_Bernstein}
    &\sigma_{\widehat{\mathcal{P}^{\rho_i}_{i,h}}(s,\bm{a})} \left[ V_{i,h+1}^{\dagger,\pi_{-i}^k,\rho_i} \right]- \sigma_{\mathcal{P}^{\rho_i}_{i,h}(s,\bm{a})} \left[  V_{i,h+1}^{\dagger,\pi_{-i}^k,\rho_i} \right] \nonumber\\
    &\leq \sqrt{\frac{c_1\bigg(\mathrm{Var}_{\widehat{P}_h^k(\cdot|s,{\bf a})} \left[ \left( \frac{ \overline{V}_{i,h+1}^{k,\rho_i} + \underline{V}_{i,h+1}^{k,\rho_i} }{2} \right) \right] + 4H\mathbb{E}_{\widehat{P}^k_h(\cdot|s, {\bf a})}\left[\overline{V}_{i,h+1}^{k,\rho_i} - \underline{V}_{i,h+1}^{k,\rho_i} \right]\bigg)\iota}{\{N^k_h(s,{\bf a})\vee 1\}}} \nonumber\\ &+ \frac{ c_2H \iota }{\{N^k_h(s,{\bf a})\vee 1\}} + \frac{1}{\sqrt{K}}\nonumber\\
    &\overset{(i)}{\leq} \sqrt{\frac{c_1\iota \mathrm{Var}_{\widehat{P}_h^k(\cdot|s,{\bf a})} \left[ \left( \frac{ \overline{V}_{i,h+1}^{k,\rho_i} + \underline{V}_{i,h+1}^{k,\rho_i} }{2} \right) \right]}{\{N^k_h(s,{\bf a})\vee 1\}}} + \sqrt{\frac{4Hc_1\iota\mathbb{E}_{\widehat{P}^k_h(\cdot|s, {\bf a})}\left[\overline{V}_{i,h+1}^{k,\rho_i} - \underline{V}_{i,h+1}^{k,\rho_i} \right]\bigg) }{\{N^k_h(s,{\bf a})\vee 1\}}} \nonumber\\ &+ \frac{ c_2H \iota }{\{N^k_h(s,{\bf a})\vee 1\}} + \frac{1}{\sqrt{K}}\nonumber\\
    &\overset{(ii)}{\leq} \sqrt{\frac{c_1\iota \mathrm{Var}_{\widehat{P}_h^k(\cdot|s,{\bf a})} \left[ \left( \frac{ \overline{V}_{i,h+1}^{k,\rho_i} + \underline{V}_{i,h+1}^{k,\rho_i} }{2} \right) \right]}{\{N^k_h(s,{\bf a})\vee 1\}}} + \frac{\mathbb{E}_{\widehat{P}^k_h(\cdot|s, {\bf a})}\left[\overline{V}_{i,h+1}^{k,\rho_i} - \underline{V}_{i,h+1}^{k,\rho_i} \right]\bigg) }{H} \nonumber\\ &+ \frac{H^2c'_2\iota}{\{N^k_h(s,{\bf a})\vee 1\}}  + \frac{1}{\sqrt{K}},
\end{align}
where the inequality (i) is due to $\sqrt{a+b}\leq \sqrt{a} +\sqrt{b}$,
and the last inequality (ii) is from $\sqrt{ab}\leq a+b$ where $c'_2 > 0$ is an absolute constant. Therefore, combining eqns. \ref{eq:optimism_pessimism_ineq_CCE_TV_step1_Bernstein}, \ref{eq:optimism_pessimism_ineq_CCE_TV_step2_Bernstein}, \ref{eq:optimism_pessimism_ineq_CCE_TV_step3_Bernstein}, and the choice of bonus in \ref{eq:Bonus_term_TV_Bernstein}, we can conclude that $ \overline{Q}_{i,h}^{k,\rho_i}(s,\bm{a})-  Q_{i,h}^{\dagger,\pi_{-i}^k,\rho_i}(s,\bm{a})\geq 0$.

 \item \textbf{Proof of Ineq. 2:} By \Cref{prop:Robust_Bellman_eq} (Robust Bellman Equation) and eq. \ref{eq:robust_Qlower_k}, we have that
 \begin{align}
     &\underline{Q}_{i,h}^{k,\rho_i}(s,\bm{a}) - Q_{i,h}^{\pi^k,\rho_i}(s,\bm{a}) \nonumber\\
     &= \max \bigg\{ \sigma_{\widehat{\mathcal{P}^{\rho_i}_{i,h}}(s,\bm{a})} \left[\underline{V}^{k,\rho_i}_{i,h+1} \right]- \sigma_{\mathcal{P}^{\rho_i}_{i,h}(s,\bm{a})} \left[ V_{i,h+1}^{\pi^k,\rho_i} \right] - \beta^k_{i,h}(s,{\bf a}), 
0- Q_{i,h}^{\dagger,\pi_{-i}^k,\rho_i}(s,\bm{a}) \bigg\}, \nonumber\\
&\leq \max \Big\{
\sigma_{\widehat{\mathcal{P}^{\rho_i}_{i,h}}(s,\bm{a})} 
\left[ V_{i,h+1}^{\pi^k,\rho_i} \right]
- \sigma_{\mathcal{P}^{\rho_i}_{i,h}(s,\bm{a})} 
\left[ V_{i,h+1}^{\pi^k,\rho_i} \right] - \beta^k_{i,h}(s,{\bf a}), \; 0
\Big\},\label{eq:optimism_pessimism_ineq_CCE_TV_step4_Bernstein}
 \end{align}
where the second inequality follows from the induction of $V_{i,h+1}^{\pi^k,\rho_i} \geq \underline{V}_{i,h+1}^{k,\rho_i}$ at the $h+1$-th step and the fact that $Q_{i,h}^{\pi^k,\rho_i} \geq 0$. By \cref{lem:Bernstein_concentration_TV_optimism_pessimism_join_bestresponse}, we can confirm that
\begin{align}
  \label{eq:optimism_pessimism_ineq_CCE_TV_step5_Bernstein}
  \sigma_{\widehat{\mathcal{P}^{\rho_i}_{i,h}}(s,\bm{a})} \left[V_{i,h+1}^{\pi^k,\rho_i} \right]- \sigma_{\mathcal{P}^{\rho_i}_{i,h}(s,\bm{a})} \left[ V_{i,h+1}^{\pi^k,\rho_i}\right]& \leq   \sqrt{ \frac{c_1\mathrm{Var}_{\widehat{P}_h^k}\left(V_{i,h+1}^{\dagger, \pi^k_{-i},\rho_i} \right) \cdot \iota }{\{N^k_h(s,{\bf a})\vee 1\}} } \nonumber\\
&+ \frac{\mathbb{E}_{\widehat{P}^k_h(\cdot|s, {\bf a})}\left[\up{V}^{k,\rho_i}_{i,h+1}- \low{V}^{k,\rho_i}_{i,h+1})\right]}{H}\nonumber\\
&\qquad + \frac{c'_2H^2S\iota}{\{N^k_h(s, {\bf a})\vee 1\}} + \frac{1}{\sqrt{K}}.
\end{align}

Now by further applying \Cref{lem:variance_analysis_1_Bernstein} to the variance term in the above inequality, with an argument similar to eq. \ref{eq:optimism_pessimism_ineq_CCE_TV_step2_Bernstein} we can obtain that
\begin{align}
\label{eq:optimism_pessimism_ineq_CCE_TV_step6_Bernstein}
    \sigma_{\widehat{\mathcal{P}^{\rho_i}_{i,h}}(s,\bm{a})} \left[V_{i,h+1}^{\pi^k,\rho_i} \right]- \sigma_{\mathcal{P}^{\rho_i}_{i,h}(s,\bm{a})} \left[ V_{i,h+1}^{\pi^k,\rho_i}\right]& \leq 
   \sqrt{ \frac{c_1\mathrm{Var}_{\widehat{P}_h^k}\left(V_{i,h+1}^{\dagger, \pi^k_{-i},\rho_i} \right) \cdot \iota }{\{N^k_h(s,{\bf a})\vee 1\}} } \nonumber\\
&+ \frac{\mathbb{E}_{\widehat{P}^k_h(\cdot|s, {\bf a})}\left[\up{V}^{k,\rho_i}_{i,h+1}- \low{V}^{k,\rho_i}_{i,h+1})\right]}{H}\nonumber\\
&\qquad + \frac{c''_2H^2S\iota}{\{N^k_h(s, {\bf a})\vee 1\}} + \frac{1}{\sqrt{K}}.
\end{align}
where  $c''_2 > 0$ is an absolute constant. Therefore, combining eqns.  \ref{eq:optimism_pessimism_ineq_CCE_TV_step4_Bernstein}, \ref{eq:optimism_pessimism_ineq_CCE_TV_step5_Bernstein}, \ref{eq:optimism_pessimism_ineq_CCE_TV_step6_Bernstein}, and the choice of bonus in \ref{eq:Bonus_term_TV_Bernstein}, $\underline{Q}_{i,h}^{k,\rho_i}(s,\bm{a})-  Q_{i,h}^{\pi^k,\rho_i}(s,\bm{a})\leq 0$.
\end{itemize}

Therefore, by eq. \ref{eq:optimism_pessimism_ineq_CCE_TV_step3_Bernstein} and eq. \ref{eq:optimism_pessimism_ineq_CCE_TV_step6_Bernstein}, we have proved that at step $h$, it holds that
\begin{align}
\label{eq:optimism_pessimism_ineq_CCE_Qbound_final_Bernstein}
   Q_{i,h}^{\dagger,\pi _{-i}^{k}, \rho_i}\left( s,\bm{a} \right) \le \up{Q}_{i,h}^{k, \rho_i}\left( s,\bm{a} \right)  ,   \,\,\,\,  \low{Q}_{i,h}^{k,\rho_i}\left( s,\bm{a} \right) \le Q_{i,h}^{\pi^{k}, \rho_i}\left( s,\bm{a} \right).
\end{align}
We now assume that eq. \ref{eq:general-Q-UCB-TV_Bernstein-CCE} hold for $h$-th step. Then, by the definition of robust value function as given by robust Bellman equation (\Cref{prop:Robust_Bellman_eq}), eq. \ref{eq:robust_V_values_k}, and CCE Equilibrium, we get
\begin{align}
\label{eq:optimism_pessimism_ineq_CCE_TV_step7_Bernstein}
    \up{V}_{i,h}^{k, \rho_i}\left( s \right) = \mathbb{E}_{\bm{a} \sim \pi^k(\cdot|s)}\left[\overline{Q}_{i,h}^{k, \rho_i}(s,{\bf a}) \right] \geq  \max_{\pi^\prime_i}\mathbb{E}_{\bm{a} \sim \pi^\prime_i\times\pi_{-i}^k(\cdot|s)}\left[\overline{Q}_{i,h}^{k, \rho_i}(s,{\bf a}) \right],
\end{align}
By the definition of $ V_{i,h}^{\dagger,\pi _{-i}^{k}, \rho_i}\left( s \right)$ in eq. \ref{eq:robust_best_response_agent_i}, we get
\begin{align}
\label{eq:optimism_pessimism_ineq_CCE_TV_step8_Bernstein}
    V_{i,h}^{\dagger,\pi _{-i}^{k}, \rho_i}\left( s \right) = \max_{\pi^\prime_i}\mathbb{E}_{\bm{a} \sim \pi^\prime_i\times\pi_{-i}^k(\cdot|s)}\left[Q_{i,h}^{\dagger,\pi _{-i}^{k}, \rho_i}(s,{\bf a}) \right].
\end{align}
Since by induction, for any $(s,{\bf a})$, $\overline{Q}_{i,h}^{k, \rho_i}(s,{\bf a}) \geq Q_{i,h}^{\dagger,\pi _{-i}^{k}, \rho_i}(s,{\bf a})$. As a result, we also have $ \up{V}_{i,h}^{k, \rho_i}\left( s \right) \geq  V_{i,h}^{\dagger,\pi _{-i}^{k}, \rho_i}\left( s \right)$, which is eq. \ref{eq:general-V-UCB-TV_Bernstein-CCE} for $h$-th step. Similarly, we can show that
\begin{align}
\label{eq:optimism_pessimism_ineq_CCE_TV_step9_Bernstein}
    \low{V}_{i,h}^{k, \rho_i}\left( s \right) &= \mathbb{E}_{\bm{a} \sim \pi^k(\cdot|s)}\left[\underline{Q}_{i,h}^{k, \rho_i}(s,{\bf a}) \right], \nonumber\\
    &\overset{(i)}{\leq }\mathbb{E}_{\bm{a} \sim \pi^k(\cdot|s)}\left[Q_{i,h}^{\pi^{k}, \rho_i}(s,{\bf a}) \right],\nonumber\\
    &\overset{(ii)}{=}V_{i,h}^{\pi^{k}, \rho_i}\left( s \right),
\end{align}
where (i) is due to the fact that $\low{Q}_{i,h}^{k,\rho_i}\left( s,\bm{a} \right) \le Q_{i,h}^{\pi^{k}, \rho_i}\left( s,\bm{a} \right)$ and (ii) is by definition of $V_{i,h}^{\pi^{k}, \rho_i}\left( s \right)$ as given by Bellman equation in \Cref{prop:Robust_Bellman_eq}.
\end{proof}

\subsubsection{CE Version: Optimistic and pessimistic estimation of the robust values for {\RMGTV}.}
Here we will proof the optimistic estimations are indeed upper bounds of the corresponding robust V-value and robust Q-value functions for CE version.

\begin{lem}[Optimistic and pessimistic estimation of the robust values for {\RMGTV} for CE version]
\label{lem:Optimistic_pessimism_TV_CE_Bernstein}
By setting the bonus term $\beta^{k}_{i,h}$ as in eq. \ref{eq:Bonus_term_TV_Bernstein}, with probability $1-\delta$, for any $(s, {\bf a}, h, i)$ and $k \in [K]$, it holds that
    \begin{equation}
        \label{eq:general-Q-UCB-TV_Bernstein-CE}
    Q_{i,h}^{\dagger,\pi _{-i}^{k}, \rho_i}\left( s,\bm{a} \right) \le \up{Q}_{i,h}^{k, \rho_i}\left( s,\bm{a} \right)  ,   \,\,\,\,  \low{Q}_{i,h}^{k,\rho_i}\left( s,\bm{a} \right) \le Q_{i,h}^{\pi^{k}, \rho_i}\left( s,\bm{a} \right),
    \end{equation}
    \begin{equation}
        \label{eq:general-V-UCB-TV_Bernstein-CE}
        V_{i,h}^{\dagger,\pi _{-i}^{k}, \rho_i}\left( s \right) \le \up{V}_{i,h}^{k, \rho_i}\left( s \right)  , \,\,\,\,  \low{V}_{i,h}^{k, \rho_i}\left( s \right) \le V_{i,h}^{\pi^{k}, \rho_i}\left( s \right).
    \end{equation}
\end{lem}
\begin{proof}
The proof-lines are similar to \citep{ghosh2025provablynearoptimaldistributionallyrobust} adapted to the multi-agent case.\\
We will run a proof for each inequality outlined in \Cref{lem:Optimistic_pessimism_TV_CE_Bernstein}.
\begin{itemize}
    \item \textbf{Ineq. 1:} To prove $Q_{i,h}^{\dagger,\pi_{-i}^k,\rho_i}(s,\bm{a}) \leq \overline{Q}_{i,h}^{k,\rho_i}(s,\bm{a})$.
    \item \textbf{Ineq. 2:} To prove $\underline{Q}_{i,h}^{k,\rho_i}(s,\bm{a}) \leq Q_{i,h}^{\pi^k,\rho_i}(s,\bm{a})$.
\end{itemize}
We know that, at step $h=H+1$, $\up{V}^{k,\rho_i}_{i,H+1}(s)= V^{\dagger,\pi^k_{-1},\rho_i}_{i,H+1}(s)=0$. Now, we assume that both eq. \ref{eq:general-Q-UCB-TV_Bernstein-CE} and eq. \ref{eq:general-V-UCB-TV_Bernstein-CE} hold at the $(h+1)$-th step. 

\begin{itemize}
    \item \textbf{Proof of Ineq. 1:}
We first consider robust $Q$ at the $h$-th step. Then, by \Cref{prop:Robust_Bellman_eq} (Robust Bellman Equation) and eq. \ref{eq:robust_Qupper_k}, we have that
\begin{align}
   &\overline{Q}_{i,h}^{k,\rho_i}(s,\bm{a})-  Q_{i,h}^{\dagger,\pi_{-i}^k,\rho_i}(s,\bm{a})
\nonumber\\
&= \min \bigg\{ \sigma_{\widehat{\mathcal{P}^{\rho_i}_{i,h}}(s,\bm{a})} \left[\overline{V}^{k,\rho_i}_{i,h+1} \right]- \sigma_{\mathcal{P}^{\rho_i}_{i,h}(s,\bm{a})} \left[ V_{i,h+1}^{\dagger,\pi_{-i}^k,\rho_i} \right] + \beta^k_{i,h}(s,{\bf a}),  
\nu^{\rho_i}_H- Q_{i,h}^{\dagger,\pi_{-i}^k,\rho_i}(s,\bm{a}) \bigg\}, \nonumber\\
&\geq \min \Big\{
\sigma_{\widehat{\mathcal{P}^{\rho_i}_{i,h}}(s,\bm{a})}
\left[ V_{i,h+1}^{\dagger,\pi_{-i}^k,\rho_i} \right]
- \sigma_{\mathcal{P}^{\rho_i}_{i,h}(s,\bm{a})}
\left[ V_{i,h+1}^{\dagger,\pi_{-i}^k,\rho_i} \right]  + \beta^k_{i,h}(s,{\bf a}), \; 0
\Big\}.\label{eq:optimism_pessimism_ineq_CE_TV_step1_Bernstein}
\end{align}
where the second inequality follows from the induction of $V_{i,h+1}^{\dagger,\pi_{-i}^k,\rho_i} \leq \overline{V}_{i,h+1}^{k,\rho_i}$ at the $h+1$-th step and the fact that $Q_{i,h}^{\dagger,\pi_{-i}^k,\rho_i} \leq \nu^{\rho_i}_H$ by \Cref{lem:gap_TV_Bernstein}. By \cref{lem:Bernstein_concentration_TV_optimism_pessimism_join_bestresponse}, we get
\begin{align}
  \label{eq:optimism_pessimism_ineq_CE_TV_step2_Bernstein}
  \sigma_{\widehat{\mathcal{P}^{\rho_i}_{i,h}}(s,\bm{a})} \left[ V_{i,h+1}^{\dagger,\pi_{-i}^k,\rho_i} \right]- \sigma_{\mathcal{P}^{\rho_i}_{i,h}(s,\bm{a})} \left[  V_{i,h+1}^{\dagger,\pi_{-i}^k,\rho_i} \right] \leq   \sqrt{ \frac{c_1\mathrm{Var}_{\widehat{P}_h^k}\left(V_{i,h+1}^{\dagger, \pi^k_{-i},\rho_i} \right) \cdot \iota }{\{N^k_h(s,{\bf a})\vee 1\}} } \nonumber\\
+ \frac{ c_2H \iota }{\{N^k_h(s,{\bf a})\vee 1\}} + \frac{1}{\sqrt{K}}.
\end{align}

Now by further applying \Cref{lem:variance_analysis_1_Bernstein} to the variance term in the above inequality, we can obtain that
\begin{align}
    \label{eq:optimism_pessimism_ineq_CE_TV_step3_Bernstein}
    &\sigma_{\widehat{\mathcal{P}^{\rho_i}_{i,h}}(s,\bm{a})} \left[ V_{i,h+1}^{\dagger,\pi_{-i}^k,\rho_i} \right]- \sigma_{\mathcal{P}^{\rho_i}_{i,h}(s,\bm{a})} \left[  V_{i,h+1}^{\dagger,\pi_{-i}^k,\rho_i} \right] \nonumber\\
    &\leq \sqrt{\frac{c_1\bigg(\mathrm{Var}_{\widehat{P}_h^k(\cdot|s,{\bf a})} \left[ \left( \frac{ \overline{V}_{i,h+1}^{k,\rho_i} + \underline{V}_{i,h+1}^{k,\rho_i} }{2} \right) \right] + 4H\mathbb{E}_{\widehat{P}^k_h(\cdot|s, {\bf a})}\left[\overline{V}_{i,h+1}^{k,\rho_i} - \underline{V}_{i,h+1}^{k,\rho_i} \right]\bigg)\iota}{\{N^k_h(s,{\bf a})\vee 1\}}} \nonumber\\ &+ \frac{ c_2H \iota }{\{N^k_h(s,{\bf a})\vee 1\}} + \frac{1}{\sqrt{K}}\nonumber\\
    &\overset{(i)}{\leq} \sqrt{\frac{c_1\iota \mathrm{Var}_{\widehat{P}_h^k(\cdot|s,{\bf a})} \left[ \left( \frac{ \overline{V}_{i,h+1}^{k,\rho_i} + \underline{V}_{i,h+1}^{k,\rho_i} }{2} \right) \right]}{\{N^k_h(s,{\bf a})\vee 1\}}} + \sqrt{\frac{4Hc_1\iota\mathbb{E}_{\widehat{P}^k_h(\cdot|s, {\bf a})}\left[\overline{V}_{i,h+1}^{k,\rho_i} - \underline{V}_{i,h+1}^{k,\rho_i} \right]\bigg) }{\{N^k_h(s,{\bf a})\vee 1\}}} \nonumber\\ &+ \frac{ c_2H \iota }{\{N^k_h(s,{\bf a})\vee 1\}} + \frac{1}{\sqrt{K}}\nonumber\\
    &\overset{(ii)}{\leq} \sqrt{\frac{c_1\iota \mathrm{Var}_{\widehat{P}_h^k(\cdot|s,{\bf a})} \left[ \left( \frac{ \overline{V}_{i,h+1}^{k,\rho_i} + \underline{V}_{i,h+1}^{k,\rho_i} }{2} \right) \right]}{\{N^k_h(s,{\bf a})\vee 1\}}} + \frac{\mathbb{E}_{\widehat{P}^k_h(\cdot|s, {\bf a})}\left[\overline{V}_{i,h+1}^{k,\rho_i} - \underline{V}_{i,h+1}^{k,\rho_i} \right]\bigg) }{H}\nonumber\\ &+ \frac{H^2c'_2\iota}{\{N^k_h(s,{\bf a})\vee 1\}}  + \frac{1}{\sqrt{K}},
\end{align}
where the inequality (i) is due to $\sqrt{a+b}\leq \sqrt{a} +\sqrt{b}$,
and the last inequality (ii) is from $\sqrt{ab}\leq a+b$ where $c'_2 > 0$ is an absolute constant. Therefore, combining eqns. \ref{eq:optimism_pessimism_ineq_CE_TV_step1_Bernstein}, \ref{eq:optimism_pessimism_ineq_CE_TV_step2_Bernstein}, \ref{eq:optimism_pessimism_ineq_CE_TV_step3_Bernstein}, and the choice of bonus in \ref{eq:Bonus_term_TV_Bernstein}, we can conclude that $ \overline{Q}_{i,h}^{k,\rho_i}(s,\bm{a})-  Q_{i,h}^{\dagger,\pi_{-i}^k,\rho_i}(s,\bm{a})\geq 0$.

 \item \textbf{Proof of Ineq. 2:} By \Cref{prop:Robust_Bellman_eq} (Robust Bellman Equation) and eq. \ref{eq:robust_Qlower_k}, we have that
 \begin{align}
     &\underline{Q}_{i,h}^{k,\rho_i}(s,\bm{a}) - Q_{i,h}^{\pi^k,\rho_i}(s,\bm{a}) \nonumber\\
     &= \max \bigg\{ \sigma_{\widehat{\mathcal{P}^{\rho_i}_{i,h}}(s,\bm{a})} \left[\underline{V}^{k,\rho_i}_{i,h+1} \right]- \sigma_{\mathcal{P}^{\rho_i}_{i,h}(s,\bm{a})} \left[ V_{i,h+1}^{\pi^k,\rho_i} \right] - \beta^k_{i,h}(s,{\bf a}), 
0- Q_{i,h}^{\dagger,\pi_{-i}^k,\rho_i}(s,\bm{a}) \bigg\}, \nonumber\\
&\leq \max \Big\{
\sigma_{\widehat{\mathcal{P}^{\rho_i}_{i,h}}(s,\bm{a})}
\left[ V_{i,h+1}^{\pi^k,\rho_i} \right]
- \sigma_{\mathcal{P}^{\rho_i}_{i,h}(s,\bm{a})}
\left[ V_{i,h+1}^{\pi^k,\rho_i} \right]  - \beta^k_{i,h}(s,{\bf a}), \; 0
\Big\},\label{eq:optimism_pessimism_ineq_CE_TV_step4_Bernstein}
 \end{align}
where the second inequality follows from the induction of $V_{i,h+1}^{\pi^k,\rho_i} \geq \underline{V}_{i,h+1}^{k,\rho_i}$ at the $h+1$-th step and the fact that $Q_{i,h}^{\pi^k,\rho_i} \geq 0$. By \cref{lem:Bernstein_concentration_TV_optimism_pessimism_join_bestresponse}, we can confirm that
\begin{align}
  \label{eq:optimism_pessimism_ineq_CE_TV_step5_Bernstein}
  \sigma_{\widehat{\mathcal{P}^{\rho_i}_{i,h}}(s,\bm{a})} \left[V_{i,h+1}^{\pi^k,\rho_i} \right]- \sigma_{\mathcal{P}^{\rho_i}_{i,h}(s,\bm{a})} \left[ V_{i,h+1}^{\pi^k,\rho_i}\right]& \leq   \sqrt{ \frac{c_1\mathrm{Var}_{\widehat{P}_h^k}\left(V_{i,h+1}^{\dagger, \pi^k_{-i},\rho_i} \right) \cdot \iota }{\{N^k_h(s,{\bf a})\vee 1\}} } \nonumber\\
&+ \frac{\mathbb{E}_{\widehat{P}^k_h(\cdot|s, {\bf a})}\left[\up{V}^{k,\rho_i}_{i,h+1}- \low{V}^{k,\rho_i}_{i,h+1})\right]}{H}\nonumber\\
&\qquad + \frac{c'_2H^2S\iota}{\{N^k_h(s, {\bf a})\vee 1\}} + \frac{1}{\sqrt{K}}.
\end{align}

Now by further applying \Cref{lem:variance_analysis_1_Bernstein} to the variance term in the above inequality, with an argument similar to eq. \ref{eq:optimism_pessimism_ineq_CE_TV_step2_Bernstein} we can obtain that
\begin{align}
\label{eq:optimism_pessimism_ineq_CE_TV_step6_Bernstein}
    \sigma_{\widehat{\mathcal{P}^{\rho_i}_{i,h}}(s,\bm{a})} \left[V_{i,h+1}^{\pi^k,\rho_i} \right]- \sigma_{\mathcal{P}^{\rho_i}_{i,h}(s,\bm{a})} \left[ V_{i,h+1}^{\pi^k,\rho_i}\right]& \leq 
   \sqrt{ \frac{c_1\mathrm{Var}_{\widehat{P}_h^k}\left(V_{i,h+1}^{\dagger, \pi^k_{-i},\rho_i} \right) \cdot \iota }{\{N^k_h(s,{\bf a})\vee 1\}} } \nonumber\\ 
&+ \frac{\mathbb{E}_{\widehat{P}^k_h(\cdot|s, {\bf a})}\left[\up{V}^{k,\rho_i}_{i,h+1}- \low{V}^{k,\rho_i}_{i,h+1})\right]}{H}\nonumber\\
&\qquad + \frac{c''_2H^2S\iota}{\{N^k_h(s, {\bf a})\vee 1\}} + \frac{1}{\sqrt{K}},
\end{align}
where  $c''_2 > 0$ is an absolute constant. Therefore, combining eqns.  \ref{eq:optimism_pessimism_ineq_CE_TV_step4_Bernstein}, \ref{eq:optimism_pessimism_ineq_CE_TV_step5_Bernstein}, \ref{eq:optimism_pessimism_ineq_CE_TV_step6_Bernstein}, and the choice of bonus in \ref{eq:Bonus_term_TV_Bernstein}, $\underline{Q}_{i,h}^{k,\rho_i}(s,\bm{a})-  Q_{i,h}^{\pi^k,\rho_i}(s,\bm{a})\leq 0$.
\end{itemize}

Therefore, by eq. \ref{eq:optimism_pessimism_ineq_CE_TV_step3_Bernstein} and eq. \ref{eq:optimism_pessimism_ineq_CE_TV_step6_Bernstein}, we have proved that at step $h$, it holds that
\begin{align}
\label{eq:optimism_pessimism_ineq_CE_Qbound_final_Bernstein}
   Q_{i,h}^{\dagger,\pi _{-i}^{k}, \rho_i}\left( s,\bm{a} \right) \le \up{Q}_{i,h}^{k, \rho_i}\left( s,\bm{a} \right)  ,   \,\,\,\,  \low{Q}_{i,h}^{k,\rho_i}\left( s,\bm{a} \right) \le Q_{i,h}^{\pi^{k}, \rho_i}\left( s,\bm{a} \right).
\end{align}
We now assume that eq. \ref{eq:general-Q-UCB-TV_Bernstein-CE} hold for $h$-th step. Then, by the definition of robust value function as given by robust Bellman equation (\Cref{prop:Robust_Bellman_eq}), eq. \ref{eq:robust_V_values_k}, and CE Equilibrium, we get
\begin{align}
\label{eq:optimism_pessimism_ineq_CE_TV_step7_Bernstein}
    \up{V}_{i,h}^{k, \rho_i}\left( s \right) = \mathbb{E}_{\bm{a} \sim \pi^k(\cdot|s)}\left[\overline{Q}_{i,h}^{k, \rho_i}(s,{\bf a}) \right] = \max_{\phi \in \Phi_i}\mathbb{E}_{\bm{a} \sim \phi \circ \pi^k(\cdot|s)}\left[\overline{Q}_{i,h}^{k, \rho_i}(s,{\bf a}) \right].
\end{align}
By the definition of $\max\limits_{\phi \in \Phi_i} V^{\phi \circ \pi^k, \rho_i}_{{i,h}}\left( s \right)$ in eq. \ref{eq:robust_best_response_agent_i}, we get
\begin{align}
\label{eq:optimism_pessimism_ineq_CE_TV_step8_Bernstein}
    \max_{\phi \in \Phi_i} V^{\phi \circ \pi^k, \rho_i}_{{i,h}}\left( s \right) = \max_{\phi \in \Phi_i} \mathbb{E}_{\bm{a} \sim \phi \circ \pi^k(\cdot|s)}\left[\max_{\phi^{\prime} }Q^{\phi^{\prime} \circ \pi^k, \rho_i}_{{i,h}}(s,{\bf a}) \right].
\end{align}
Since by induction, for any $(s,{\bf a})$, $\overline{Q}_{i,h}^{k, \rho_i}(s,{\bf a}) \geq \max\limits_{\phi \in \Phi_i} Q^{\phi \circ \pi^k, \rho_i}_{{i,h}}(s,{\bf a})$. As a result, we also have $ \up{V}_{i,h}^{k, \rho_i}\left( s \right) \geq  \max\limits_{\phi \in \Phi_i} V^{\phi \circ \pi^k, \rho_i}_{{i,h}}\left( s \right)$, which is eq. \ref{eq:general-V-UCB-KL-CE} for $h$-th step. Similarly, we can show that
\begin{align}
\label{eq:optimism_pessimism_ineq_CE_TV_step9_Bernstein}
    \low{V}_{i,h}^{k, \rho_i}\left( s \right) &= \mathbb{E}_{\bm{a} \sim \pi^k(\cdot|s)}\left[\underline{Q}_{i,h}^{k, \rho_i}(s,{\bf a}) \right], \nonumber\\
    &\overset{(i)}{\leq }\mathbb{E}_{\bm{a} \sim \pi^k(\cdot|s)}\left[Q_{i,h}^{\pi^{k}, \rho_i}(s,{\bf a}) \right],\nonumber\\
    &\overset{(ii)}{=}V_{i,h}^{\pi^{k}, \rho_i}\left( s \right),
\end{align}
where (i) is due to the fact that $\low{Q}_{i,h}^{k,\rho_i}\left( s,\bm{a} \right) \le Q_{i,h}^{\pi^{k}, \rho_i}\left( s,\bm{a} \right)$ and (ii) is by definition of $V_{i,h}^{\pi^{k}, \rho_i}\left( s \right)$ as given by Bellman equation in \Cref{prop:Robust_Bellman_eq}.
\end{proof}

\subsection{Auxiliary Lemmas for {\RMGTV}}
\label{subsubsec:aux_lemma_TV_Bernstein}

\begin{lem}[Bernstein bound for {\RMGTV} and the robust value functions of $\pi^k$ and $\pi^{\dagger}$]
\label{lem:Bernstein_concentration_TV_optimism_pessimism_join_bestresponse}
Under event \( \mathcal{E}_{\text{TV}} \) in eq. \ref{eq:Bernstein_TV_event} and definition of $\pi^{\dagger}$ as given in eq. \ref{eq:joint_best_response}, we assume that for any $\Eq \in \{\text{NASH}, \text{CE}, \text{CCE}\}$ the optimism and pessimism inequalities holds at \( (h+1, k) \), where these inequalities can correspond to any of the following cases of $\Eq$:
\begin{itemize}
    \item \textbf{NE:} \Cref{lem:Optimistic_pessimism_TV_NE_Bernstein} using eq. \ref{eq:general-Q-UCB-TV_Bernstein} and eq. \ref{eq:general-V-UCB-TV_Bernstein},
    \item \textbf{CCE:} \Cref{lem:Optimistic_pessimism_TV_CCE_Bernstein} using eq. \ref{eq:general-Q-UCB-TV_Bernstein-CCE} and eq. \ref{eq:general-V-UCB-TV_Bernstein-CCE},
    \item \textbf{CE:} \Cref{lem:Optimistic_pessimism_TV_CE_Bernstein} using eq. \ref{eq:general-Q-UCB-TV_Bernstein-CE} and eq. \ref{eq:general-V-UCB-TV_Bernstein-CE},
\end{itemize}
Then, it holds that
\begin{align*}
    &\abs{\sigma_{\widehat{\mathcal{P}^{\rho_i}_{i,h}}(s,\bm{a})}[V_{i,h+1}^{\pi^k,\rho_i}] - \sigma_{\mathcal{P}_{i,h}^{\rho_i}(s,\bm{a})}[V_{i,h+1}^{\pi^k,\rho_i}]}\nonumber\\
    &\leq 
    \begin{cases}
    \sqrt{ \frac{c_1\mathrm{Var}_{\widehat{P}_h^k}\left(V_{i,h+1}^{\dagger, \pi^k_{-i},\rho_i} \right) \cdot \iota }{\{N^k_h(s,{\bf a})\vee 1\}} } 
+ \frac{ c_2H \iota }{\{N^k_h(s,{\bf a})\vee 1\}} + \frac{1}{\sqrt{K}}, \quad &\text{if $\pi^k=\pi^{\dagger}$}\\
    \sqrt{ \frac{c_1\mathrm{Var}_{\widehat{P}_h^k}\left(V_{i,h+1}^{\dagger, \pi^k_{-i},\rho_i} \right) \cdot \iota }{\{N^k_h(s,{\bf a})\vee 1\}} } 
+ \frac{\mathbb{E}_{\widehat{P}^k_h(\cdot|s, {\bf a})}\left[\up{V}^{k,\rho_i}_{i,h+1}- \low{V}^{k,\rho_i}_{i,h+1})\right]}{H} + \frac{c'_2H^2S\iota}{\{N^k_h(s, {\bf a})\vee 1\}} + \frac{1}{\sqrt{K}}, \qquad &\text{otherwise},
    \end{cases}
\end{align*}
where $\iota = \log\bigg(\frac{S^2(\prod_{i=1}^m A_i)H^2K^{3/2}}{\delta}\bigg)$ and $c_1, c'_2 >0$ are absolute constants. 
\end{lem}

\begin{proof}
By our definition of the operator \( \sigma_{\mathcal{P}_{i,h}^{\rho_i}(s,\bm{a})}[V_{i,h+1}^{\pi^k,\rho_i}] \) in eq. \ref{eq:dual_TV}, we can arrive at,
\begin{align}
\big|\sigma_{\widehat{\mathcal{P}^{\rho_i}_{i,h}}(s,\mathbf{a})}[V_{i,h+1}^{\pi^k,\rho_i}]
- \sigma_{\mathcal{P}_{i,h}^{\rho_i}(s,\mathbf{a})}[V_{i,h+1}^{\pi^k,\rho_i}]\big|
&\le \sup_{\eta\in[0,H]} \Big|\bigg\{
  \mathbb{E}_{\widehat{P}_h^k(\cdot\mid s,\mathbf{a})}\big[(\eta - V_{i,h+1}^{\pi^k,\rho_i})_+\big ] \nonumber\\
&\qquad\qquad - \mathbb{E}_{P_h^\star(\cdot\mid s,\mathbf{a})}\big[(\eta - V_{i,h+1}^{\pi^k,\rho_i})_+\big]\bigg\}
\Big| \nonumber\\
&= \text{Term (i) + Term (ii)}. \label{eq:Bernstein_TV_bound_step1}
\end{align}

where we denote
\begin{align}
    \text{Term (i)} := 
    &\sup_{\eta \in [0,H]} 
    \Bigg|\bigg\{
    \mathbb{E}_{\widehat{P}_h^k(\cdot|s,{\bf a})} \Big[ (\eta - V_{i,h+1}^{\dagger,\pi^k_{-i},\rho_i})_+ \Big] - \mathbb{E}_{P_h^\star(\cdot|s,{\bf a})} \Big[ (\eta - V_{i,h+1}^{\dagger,\pi^k_{-i},\rho_i})_+ \Big] 
    \bigg\} \Bigg| \label{eq:Term_i_TV_Bernstein}\\
    % &\phantom{=} \text{ } % optional spacing/alignment
\text{Term (ii)} &:= \sup_{\eta \in [0,H]} \Bigg|\bigg\{ 
  \mathbb{E}_{\widehat{P}_h^k(\cdot|s,{\bf a})}\left[ \left( \eta - V_{i,h+1}^{\pi^k,\rho_i}] \right)_+- \left( \eta - V_{i,h+1}^{\dagger, \pi^k_{-1i},\rho_i}] \right)_+ \right] \nonumber\\
  &\qquad - \mathbb{E}_{P_h^\star(\cdot|s,{\bf a})}\left[ \left( \eta - V_{i,h+1}^{\pi^k,\rho_i}] \right)_+- \left( \eta - V_{i,h+1}^{\dagger,\pi^k_{-i},\rho_i}] \right)_+ \right]  \bigg\}\Bigg|.\label{eq:Term_ii_TV_Bernstein}
\end{align}

We deal with Term (i) and Term (ii) respectively.

\noindent\textbf{Bound for Term (i):} Term (i) is referred to Bernstein bound for Bernstein bound for {\RMGTV} and the robust value function of the robust best response $\pi^{\dagger,\rho_i}_i(\pi_{-i})$. More specifically, we find the Bernstein bound on the gap $\abs{\sigma_{\widehat{\mathcal{P}^{\rho_i}_{i,h}}(s,\bm{a})}[V_{i,h+1}^{\dagger, \pi^k_{-i},\rho_i}] - \sigma_{\mathcal{P}_{i,h}^{\rho_i}(s,\bm{a})}[V_{i,h+1}^{\dagger, \pi^k_{-i},\rho_i}]}$. Therefore, by the definition of the operator \( \sigma_{\mathcal{P}_{i,h}^{\rho_i}(s,\bm{a})}[V_{i,h+1}^{\dagger, \pi^k_{-i},\rho_i}] \) in eq. \ref{eq:dual_TV}), we can arrive at,
    \begin{align}
    &\abs{\sigma_{\widehat{\mathcal{P}^{\rho_i}_{i,h}}(s,\bm{a})}[V_{i,h+1}^{\dagger, \pi^k_{-i},\rho_i}] - \sigma_{\mathcal{P}_{i,h}^{\rho_i}(s,\bm{a})}[V_{i,h+1}^{\dagger, \pi^k_{-i},\rho_i}]} \nonumber\\
    &\quad \leq \sup_{\eta \in [0,H]} \abs{\bigg\{ 
  \mathbb{E}_{\widehat{P}_h^k(\cdot|s,{\bf a})}\left[ \left( \eta - V_{i,h+1}^{\dagger, \pi^k_{-1i},\rho_i}] \right)_+ \right] - \mathbb{E}_{P_h^\star(\cdot|s,{\bf a})}\left[ \left( \eta - V_{i,h+1}^{\dagger,\pi^k_{-i},\rho_i}] \right)_+ \right]  \bigg\}} \nonumber\\
  &\quad = \text{Term (i)}. \label{eq:Term_i_Bernstein_step1}
    \end{align}

By now according to the first inequality of event \(\mathcal{E}\) in eq. \ref{eq:Bernstein_TV_event}, we can bound eq. \ref{eq:Term_i_Bernstein_step1} as
\begin{align}
\text{Term (i)} &\leq \sqrt{ \frac{c_1\mathrm{Var}_{\widehat{P}_h^k}\left( \eta - V_{i,h+1}^{\dagger, \pi^k_{-i},\rho_i} \right)_+ \cdot \iota }{\{N^k_h(s,{\bf a})\vee 1\}} } 
+ \frac{ c_2H \iota }{\{N^k_h(s,{\bf a})\vee 1\}} \nonumber \\
&\leq \sqrt{ \frac{c_1\mathrm{Var}_{\widehat{P}_h^k}\left(V_{i,h+1}^{\dagger, \pi^k_{-i},\rho_i} \right) \cdot \iota }{\{N^k_h(s,{\bf a})\vee 1\}} } 
+ \frac{ c_2H \iota }{\{N^k_h(s,{\bf a})\vee 1\}}, \label{eq:Term_i_Bernstein_step2}
\end{align}
for any \(\eta \in \mathcal{N}_{1/(S\sqrt{K})}([0,H])\). Here the second inequality is because $\mathrm{Var}[(a - X)_+] \leq \mathrm{Var}[X]$. Therefore, by applying the covering argument in eq. \ref{eq:Term_i_Bernstein_step2}, for any \(\eta \in [0,H]\), it holds that
\begin{align}
\label{eq:Term_i_Bernstein_bound}
\text{Term (i)} & \leq \sqrt{ \frac{c_1\mathrm{Var}_{\widehat{P}_h^k}\left(V_{i,h+1}^{\dagger, \pi^k_{-i},\rho_i} \right) \cdot \iota }{\{N^k_h(s,{\bf a})\vee 1\}} } 
+ \frac{ c_2H \iota }{\{N^k_h(s,{\bf a})\vee 1\}} + \frac{1}{\sqrt{K}}.
\end{align}

\noindent \textbf{Bound for Term (ii):} For Term (ii), we apply the second inequality of event \(\mathcal{E}\) in eq. \ref{eq:Bernstein_TV_event}, and we obtain that
    \begin{align}
    \label{eq:Term_ii_Bernstein_step1}
        \text{Term (ii)} &\leq \sup_{\eta \in [0,H]}\Bigg\{\sum_{s'\in \mathcal{S}}\Bigg(\sqrt{ \frac{ c_1\min\left\{ P^\star_h(s' \mid s,{\bf a}),P^k_h(s' \mid s,{\bf a}) \right\} \cdot \iota }
{\{N^k_h(s,{\bf a})\vee 1\}} } + \frac{ c_2\iota }{\{N^k_h(s,{\bf a})\vee 1\}}\Bigg)\nonumber\\
&\qquad \times \abs{ \left( \eta - V_{i,h+1}^{\pi^k,\rho_i}] \right)_+- \left( \eta - V_{i,h+1}^{\dagger,\pi^k_{-i},\rho_i}] \right)_+}\Bigg\}.
\end{align}
Now by assuming that eq. \ref{eq:general-V-UCB-TV_Bernstein} holds at $(h+1,k)$, we can upper bound the absolute value above by
\begin{align}
\label{eq:Term_ii_Bernstein_step2}
    \abs{\left( \eta - V_{i,h+1}^{\pi^k,\rho_i}] \right)_+- \left( \eta - V_{i,h+1}^{\dagger,\pi^k_{-i},\rho_i}] \right)_+}   \overset{(i)}{\leq} \abs{V_{i,h+1}^{\pi^k,\rho_i} -  V_{i,h+1}^{\dagger,\pi^k_{-i},\rho_i}} \nonumber \\ \overset{(ii)}{\leq} \up{V}^{k,\rho_i}_{i,h+1}(s')- \low{V}^{k,\rho_i}_{i,h+1}(s'),
    \end{align}
    where the first inequality (i) is due to the $1$-Lipschitz continuity of $\psi_{\eta}(x) = (\eta-x)_+$, and the second inequality (ii) is due to eq. \ref{eq:general-V-UCB-TV_Bernstein}. Thus combining eq. \ref{eq:Term_ii_Bernstein_step1} and eq. \ref{eq:Term_ii_Bernstein_step2}, we get
    \begin{align}
    \label{eq:Term_ii_Bernstein_bound}
            \text{Term (ii)} &\leq \sum_{s'\in \mathcal{S}}\Bigg(\sqrt{ \frac{ c_1\widehat{P}^k_h(s' \mid s,{\bf a}) \cdot \iota }
{\{N^k_h(s,{\bf a})\vee 1\}} } + \frac{ c_2\iota }{\{N^k_h(s,{\bf a})\vee 1\}}\Bigg)\cdot \bigg(\up{V}^{k,\rho_i}_{i,h+1}(s')- \low{V}^{k,\rho_i}_{i,h+1}(s')\bigg)\nonumber\\
&\overset{(i)}{\leq} \sum_{s'\in \mathcal{S}}\Bigg(\frac{\widehat{P}^k_h(s' \mid s,{\bf a})}{H} + \frac{c_1H\iota}{\{N^k_h(s,{\bf a})\vee 1\}} + \frac{c_2\iota}{\{N^k_h(s,{\bf a})\vee 1\}}\Bigg)\nonumber \\
&\qquad \quad.\bigg(\up{V}^{k,\rho_i}_{i,h+1}(s')- \low{V}^{k,\rho_i}_{i,h+1}(s')\bigg)\nonumber\\
&\overset{(ii)}{\leq} \frac{\mathbb{E}_{\widehat{P}^k_h(\cdot|s, {\bf a})}\left[\up{V}^{k,\rho_i}_{i,h+1}- \low{V}^{k,\rho_i}_{i,h+1}\right]}{H} + \frac{c'_2H^2S\iota}{\{N^k_h(s, {\bf a})\vee 1\}},
    \end{align} 
    where $c'_2>0$ is an absolute constant. The first inequality (i) is by $\sqrt{ab}\leq a +b$ and the second inequality (ii) is due to $\up{V}^{k,\rho_i}_{i,h+1}, \low{V}^{k,\rho_i}_{i,h+1} \in [0,H]$. Finally, by combining eq. \ref{eq:Term_i_Bernstein_bound} and eq. \ref{eq:Term_ii_Bernstein_bound} and applying in eq. \ref{eq:Bernstein_TV_bound_step1}, we get the required bound as
    \begin{align}
        \text{Term (ii)} &\leq \sqrt{ \frac{c_1\mathrm{Var}_{\widehat{P}_h^k}\left(V_{i,h+1}^{\dagger, \pi^k_{-i},\rho_i} \right) \cdot \iota }{\{N^k_h(s,{\bf a})\vee 1\}} } 
+ \frac{\mathbb{E}_{\widehat{P}^k_h(\cdot|s, {\bf a})}\left[\up{V}^{k,\rho_i}_{i,h+1}- \low{V}^{k,\rho_i}_{i,h+1}\right]}{H} + \frac{c'_2H^2S\iota}{\{N^k_h(s, {\bf a})\vee 1\}} \nonumber \\&+ \frac{1}{\sqrt{K}}.
\end{align}
This concludes the proof of \Cref{lem:Bernstein_concentration_TV_optimism_pessimism_join_bestresponse}.
\end{proof}

\begin{lem}[Bernstein bound for {\RMGTV} and optimistic and pessimistic robust value estimators]
\label{lem:Bernstein_concentration_TV_optimism_pessimism}
Under event \( \mathcal{E}_{\text{TV}} \) in eq. \ref{eq:Bernstein_TV_event} and definition of $\pi^{\dagger}$ as given in eq. \ref{eq:joint_best_response}, we assume that for any $\Eq \in \{\text{NASH}, \text{CE}, \text{CCE}\}$ the optimism and pessimism inequalities holds at \( (h+1, k) \), where these inequalities can correspond to any of the following cases of $\Eq$:
\begin{itemize}
    \item \textbf{NE:} \Cref{lem:Optimistic_pessimism_TV_NE_Bernstein} using eq. \ref{eq:general-Q-UCB-TV_Bernstein} and eq. \ref{eq:general-V-UCB-TV_Bernstein},
    \item \textbf{CCE:} \Cref{lem:Optimistic_pessimism_TV_CCE_Bernstein} using eq. \ref{eq:general-Q-UCB-TV_Bernstein-CCE} and eq. \ref{eq:general-V-UCB-TV_Bernstein-CCE},
    \item \textbf{CE:} \Cref{lem:Optimistic_pessimism_TV_CE_Bernstein} using eq. \ref{eq:general-Q-UCB-TV_Bernstein-CE} and eq. \ref{eq:general-V-UCB-TV_Bernstein-CE},
\end{itemize}
Then, it holds that

\begin{align*}
    &\max\Bigg\{\abs{\sigma_{\widehat{\mathcal{P}^{\rho_i}_{i,h}}(s,\bm{a})}\left[\up{V}_{i,h+1}^{k,\rho_i}\right] - \sigma_{\mathcal{P}_{i,h}^{\rho_i}(s,\bm{a})}\left[\up{V}_{i,h+1}^{k,\rho_i}\right]}, \abs{\sigma_{\widehat{\mathcal{P}^{\rho_i}_{i,h}}(s,\bm{a})}\left[\low{V}_{i,h+1}^{k,\rho_i}\right] - \sigma_{\mathcal{P}_{i,h}^{\rho_i}(s,\bm{a})}\left[\low{V}_{i,h+1}^{k,\rho_i}\right]}\Bigg\}\nonumber\\
    &\leq 
    \sqrt{ \frac{c_1\mathrm{Var}_{\widehat{P}_h^k}\left(V_{i,h+1}^{\dagger, \pi^k_{-i},\rho_i} \right) \cdot \iota }{\{N^k_h(s,{\bf a})\vee 1\}} } 
+ \frac{\mathbb{E}_{\widehat{P}^k_h(\cdot|s, {\bf a})}\left[\up{V}^{k,\rho_i}_{i,h+1}- \low{V}^{k,\rho_i}_{i,h+1}\right]}{H} + \frac{c'_2H^2S\iota}{\{N^k_h(s, {\bf a})\vee 1\}} + \frac{1}{\sqrt{K}},
\end{align*}
where $\iota = \log\bigg(\frac{S^2(\prod_{i=1}^m A_i)H^2K^{3/2}}{\delta}\bigg)$ and $c_1, c'_2 >0$ are absolute constants.
\end{lem}

\begin{proof}
    This follows from the same proof as \Cref{lem:Bernstein_concentration_TV_optimism_pessimism_join_bestresponse} and is thus omitted.
\end{proof}

\begin{lem}[Non-robust Concentration for {\RMGTV}]
\label{lem:non_robust_TV_Bernstein}
Under event \( \mathcal{E}_{\text{TV}} \) in eq. \ref{eq:Bernstein_TV_event} and definition of $\pi^{\dagger}$ as given in eq. \ref{eq:joint_best_response}, we assume that for any $\Eq \in \{\text{NASH}, \text{CE}, \text{CCE}\}$ the optimism and pessimism inequalities holds at \( (h+1, k) \), where these inequalities can correspond to any of the following cases of $\Eq$:
\begin{itemize}
    \item \textbf{NE:} \Cref{lem:Optimistic_pessimism_TV_NE_Bernstein} using eq. \ref{eq:general-Q-UCB-TV_Bernstein} and eq. \ref{eq:general-V-UCB-TV_Bernstein},
    \item \textbf{CCE:} \Cref{lem:Optimistic_pessimism_TV_CCE_Bernstein} using eq. \ref{eq:general-Q-UCB-TV_Bernstein-CCE} and eq. \ref{eq:general-V-UCB-TV_Bernstein-CCE},
    \item \textbf{CE:} \Cref{lem:Optimistic_pessimism_TV_CE_Bernstein} using eq. \ref{eq:general-Q-UCB-TV_Bernstein-CE} and eq. \ref{eq:general-V-UCB-TV_Bernstein-CE},
\end{itemize}
Then, it holds that
\begin{align*}
    \left| \mathbb{E}_{P^\star_h(\cdot|s,{\bf a})}[\up{V}^{k,\rho_i}_{i,h+1}- \low{V}^{k,\rho_i}_{i,h+1}] - \mathbb{E}_{\widehat{P}^k_h(\cdot|s,{\bf a})}[\up{V}^{k,\rho_i}_{i,h+1}- \low{V}^{k,\rho_i}_{i,h+1}] \right| 
&\leq \frac{\mathbb{E}_{\widehat{P}^k_h(\cdot|s, {\bf a})}\left[\up{V}^{k,\rho_i}_{i,h+1}- \low{V}^{k,\rho_i}_{i,h+1}\right]}{H} \\
&\quad+ \frac{c'_2H^2S\iota}{\{N^k_h(s, {\bf a})\vee 1\}}, 
\end{align*}
where $\iota = \log\bigg(\frac{S^2(\prod_{i=1}^m A_i)H^2K^{3/2}}{\delta}\bigg)$ and $c'_2 >0$ are absolute constants.
\end{lem}

\begin{proof}
Assuming that eq. \ref{eq:general-V-UCB-TV_Bernstein} holds for $(h+1,k)$, we apply the second inequality of event \(\mathcal{E}\) in eq. \ref{eq:Bernstein_TV_event} to get the required bound \Cref{lem:non_robust_TV_Bernstein}.
\end{proof}
%%%%%%%%%%%%

\begin{lem}[Variance analysis for $\pi^{\dagger}$ for {\RMGTV}]
\label{lem:variance_analysis_1_Bernstein}
Under the definition of $\pi^{\dagger}$ as given in eq. \ref{eq:joint_best_response}, we assume that for any $\Eq \in \{\text{NASH}, \text{CE}, \text{CCE}\}$ the optimism and pessimism inequalities holds at \( (h+1, k) \), where these inequalities can correspond to any of the following cases of $\Eq$:
\begin{itemize}
    \item \textbf{NE:} \Cref{lem:Optimistic_pessimism_TV_NE_Bernstein} using eq. \ref{eq:general-Q-UCB-TV_Bernstein} and eq. \ref{eq:general-V-UCB-TV_Bernstein},
    \item \textbf{CCE:} \Cref{lem:Optimistic_pessimism_TV_CCE_Bernstein} using eq. \ref{eq:general-Q-UCB-TV_Bernstein-CCE} and eq. \ref{eq:general-V-UCB-TV_Bernstein-CCE},
    \item \textbf{CE:} \Cref{lem:Optimistic_pessimism_TV_CE_Bernstein} using eq. \ref{eq:general-Q-UCB-TV_Bernstein-CE} and eq. \ref{eq:general-V-UCB-TV_Bernstein-CE},
\end{itemize}
Then, it holds that

% \begin{align*}
% %\label{eq:variance_analysis_1_Bernstein}
% \abs{\mathrm{Var}_{\widehat{P}_h^k(\cdot|s,{\bf a})} \left[ \left( \frac{ \overline{V}_{i,h+1}^{k,\rho_i} + \underline{V}_{i,h+1}^{k,\rho_i} }{2} \right) \right]-\mathrm{Var}_{\widehat{P}_h^k(\cdot|s,{\bf a})} 
% \left[ V_{i,h+1}^{\dagger,\pi^k_{-i},\rho_i} \right]}\leq 4H\mathbb{E}_{\widehat{P}_h^k(\cdot|s,{\bf a})} \left[ \overline{V}_{h+1}^{k,\rho_i} - \underline{V}_{h+1}^{k,\rho_i} \right].
% \end{align*}
\begin{align*}
\abs{
  \mathrm{Var}_{\widehat{P}_h^k(\cdot|s,{\bf a})} \!\left[ \tfrac{ \overline{V}_{i,h+1}^{k,\rho_i} + \underline{V}_{i,h+1}^{k,\rho_i} }{2} \right]
  - \mathrm{Var}_{\widehat{P}_h^k(\cdot|s,{\bf a})} \!\left[ V_{i,h+1}^{\dagger,\pi^k_{-i},\rho_i} \right]
}
&\leq 4H \,
\mathbb{E}_{\widehat{P}_h^k(\cdot|s,{\bf a})} 
\left[ \overline{V}_{h+1}^{k,\rho_i} - \underline{V}_{h+1}^{k,\rho_i} \right].
\end{align*}
\end{lem}

\begin{proof}
Our proof closely follows the lines of Lemma 22 in \citep{liu2021sharp} and Lemma E.11 in \citep{Arxiv2024_DRORLwithInteractiveData_Lu}, with detailed elaboration on each step for clarity. The left hand side of the inequality in \Cref{lem:variance_analysis_1_Bernstein} can be upper bounded by the following
\begin{align}
 &\abs{\mathrm{Var}_{\widehat{P}_h^k(\cdot|s,{\bf a})} \left[ \left( \frac{ \overline{V}_{i,h+1}^{k,\rho_i} + \underline{V}_{i,h+1}^{k,\rho_i} }{2} \right) \right]-\mathrm{Var}_{\widehat{P}_h^k(\cdot|s,{\bf a})} 
\left[ V_{i,h+1}^{\dagger,\pi^k_{-i},\rho_i} \right]} \nonumber\\
& \leq \abs{\mathbb{E}_{\widehat{P}_h^k(\cdot|s,{\bf a})} \left[ \left( \frac{ \overline{V}_{i,h+1}^{k,\rho_i} + \underline{V}_{i,h+1}^{k,\rho_i} }{2} \right)^2 \right]-\mathbb{E}_{\widehat{P}_h^k(\cdot|s,{\bf a})} 
\left[ \left(V_{i,h+1}^{\dagger,\pi^k_{-i},\rho_i}\right)^2 \right]} \nonumber\\
&\qquad + \abs{\left(\mathbb{E}_{\widehat{P}_h^k(\cdot|s,{\bf a})} \left[ \left( \frac{ \overline{V}_{i,h+1}^{k,\rho_i} + \underline{V}_{i,h+1}^{k,\rho_i} }{2} \right) \right]\right)^2-\left(\mathbb{E}_{\widehat{P}_h^k(\cdot|s,{\bf a})} 
\left[ V_{i,h+1}^{\dagger,\pi^k_{-i},\rho_i} \right]\right)^2}.\label{eq:variance1_TV_step1}
\end{align}
By applying eq. \ref{eq:general-V-UCB-TV_Bernstein} and the facts that  $\overline{V}_{i,h+1}^{k,\rho_i}$  and $\underline{V}_{i,h+1}^{k,\rho_i}$, $\overline{V}_{i,h+1}^{k,\rho_i}, \underline{V}_{i,h+1}^{k,\rho_i},  V_{i,h+1}^{\dagger,\pi^k_{-i},\rho_i} \in [0,H]$, we can further upper bound eq. \ref{eq:variance1_TV_step1} as
\begin{align}
    &\abs{\mathrm{Var}_{\widehat{P}_h^k(\cdot|s,{\bf a})} \left[ \left( \frac{ \overline{V}_{i,h+1}^{k,\rho_i} + \underline{V}_{i,h+1}^{k,\rho_i} }{2} \right) \right]-\mathrm{Var}_{\widehat{P}_h^k(\cdot|s,{\bf a})} 
\left[ V_{i,h+1}^{\dagger,\pi^k_{-i},\rho_i} \right]} \nonumber\\
&\leq 4H\,\mathbb{E}_{\widehat{P}_h^k(\cdot|s,{\bf a})}
\left[\Biggl|\tfrac{\overline{V}_{i,h+1}^{k,\rho_i}+\underline{V}_{i,h+1}^{k,\rho_i}}{2}
- V_{i,h+1}^{\dagger,\pi^k_{-i},\rho_i}\Biggr|\right]
\leq 4H\,\mathbb{E}_{\widehat{P}_h^k(\cdot|s,{\bf a})}
\left[\overline{V}_{i,h+1}^{k,\rho_i}-\underline{V}_{i,h+1}^{k,\rho_i}\right].
\end{align}
This concludes the proof of \Cref{lem:variance_analysis_1_Bernstein}.
\end{proof}

\begin{lem}[Variance analysis for any robust joint policy $\pi^k$ for {\RMGTV}]
\label{lem:variance_analysis_2_Bernstein}
Under event \( \mathcal{E}_{\text{TV}} \) in eq. \ref{eq:Bernstein_TV_event} and definition of $\pi^{\dagger}$ as given in eq. \ref{eq:joint_best_response}, we assume that for any $\Eq \in \{\text{NASH}, \text{CE}, \text{CCE}\}$ the optimism and pessimism inequalities holds at \( (h+1, k) \), where these inequalities can correspond to any of the following cases of $\Eq$:
\begin{itemize}
    \item \textbf{NE:} \Cref{lem:Optimistic_pessimism_TV_NE_Bernstein} using eq. \ref{eq:general-Q-UCB-TV_Bernstein} and eq. \ref{eq:general-V-UCB-TV_Bernstein},
    \item \textbf{CCE:} \Cref{lem:Optimistic_pessimism_TV_CCE_Bernstein} using eq. \ref{eq:general-Q-UCB-TV_Bernstein-CCE} and eq. \ref{eq:general-V-UCB-TV_Bernstein-CCE},
    \item \textbf{CE:} \Cref{lem:Optimistic_pessimism_TV_CE_Bernstein} using eq. \ref{eq:general-Q-UCB-TV_Bernstein-CE} and eq. \ref{eq:general-V-UCB-TV_Bernstein-CE},
\end{itemize}
Then, then the following inequality holds,
\begin{align*}
&\abs{\mathrm{Var}_{\widehat{P}_h^k(\cdot|s,{\bf a})} \left[ \left( \frac{ \overline{V}_{i,h+1}^{k,\rho_i} + \underline{V}_{i,h+1}^{k,\rho_i} }{2} \right) \right]-\mathrm{Var}_{P_h^\star(\cdot|s,{\bf a})} 
\left[ V_{i,h+1}^{\pi^k,\rho_i} \right]}\nonumber\\
&\leq 4H\mathbb{E}_{P_h^\star(\cdot|s,{\bf a})} \left[ \overline{V}_{h+1}^{k,\rho_i} - \underline{V}_{h+1}^{k,\rho_i} \right] +\frac{c'_2H^4S\iota}{\{N^k_h(s,{\bf a})\vee 1\}} + 1.
\end{align*}
\end{lem}
\begin{proof}
    We follow the proof-lines of Lemma 23 in \citep{liu2021sharp} and Lemma E.12 of \citep{Arxiv2024_DRORLwithInteractiveData_Lu}. We present a detailed derivation as follows. We first relate the variance on $\widehat{P}^k_h$ to the variance on $P^\star_h$. Specifically, we have
    \begin{align}
    \label{eq:variance_analysis_2_Bernstein_step1}
\abs{\mathrm{Var}_{\widehat{P}_h^k(\cdot|s,{\bf a})} \left[ \left( \frac{ \overline{V}_{i,h+1}^{k,\rho_i} + \underline{V}_{i,h+1}^{k,\rho_i} }{2} \right) \right]-\mathrm{Var}_{P_h^\star(\cdot|s,{\bf a})} 
\left[ V_{i,h+1}^{\pi^k,\rho_i} \right]}\leq  \text{Term (i)} + \text{Term (ii)},
    \end{align}
    where we denote
    \begin{align}
        \text{Term (i)} 
        &:= \left| \mathrm{Var}_{\widehat{P}_h^k(\cdot|s,{\bf a})}
          \left[ \tfrac{\overline{V}_{i,h+1}^{k,\rho_i} + \underline{V}_{i,h+1}^{k,\rho_i}}{2} \right] \right. \left. - \mathrm{Var}_{P_h^\star(\cdot|s,{\bf a})}
          \left[ \tfrac{\overline{V}_{i,h+1}^{k,\rho_i} + \underline{V}_{i,h+1}^{k,\rho_i}}{2} \right] \right|. 
        \label{eq:Termi_variance2_TV_Bernstein}\\
    \text{Term (ii)} &:= \abs{ \mathrm{Var}_{P_h^\star(\cdot|s,{\bf a})} \left[ \left( \frac{ \overline{V}_{i,h+1}^{k,\rho_i} + \underline{V}_{i,h+1}^{k,\rho_i} }{2} \right) \right] - \mathrm{Var}_{\widehat{P}_h^k(\cdot|s,{\bf a})} 
\left[ V_{i,h+1}^{\pi^k,\rho_i} \right]}.   \label{eq:Termii_variance2_TV_Bernstein}
    \end{align}
    We will now bound Term (i) and Term (ii) respectively.
    \begin{itemize}
        \item \textbf{Term (i):} By applying the fact $\Big( \overline{V}_{i,h+1}^{k,\rho_i} + \underline{V}_{i,h+1}^{k,\rho_i}\Big)\Big/2 \in [0,H]$ in the variance terms on Term (i), we can upper bound Term (i) as
        \begin{align}
           \text{Term (i)} &\leq H^2\sum_{s'\in \mathcal{S}}\abs{P^\star_h(s'|s,{\bf a})-\widehat{P}^k_h(s'|s,{\bf a})}\nonumber\\
           &\overset{(i)}{\leq} H^2\sum_{s'\in \mathcal{S}}\Bigg(\sqrt{ \frac{ c_1\widehat{P}^k_h(s' \mid s,{\bf a}) \cdot \iota }
{\{N^k_h(s,{\bf a})\vee 1\}} } + \frac{ c_2\iota }{\{N^k_h(s,{\bf a})\vee 1\}}\Bigg)\nonumber\\
&\overset{(ii)}{\leq} H^2\Bigg(\sqrt{ \frac{ c_1S\iota }
{\{N^k_h(s,{\bf a})\vee 1\}} } + \frac{ c_2S\iota }{\{N^k_h(s,{\bf a})\vee 1\}}\Bigg)\nonumber\\
&\overset{(iii)}{\leq} 1 + \frac{c'_2H^4S\iota}{\{N^k_h(s,{\bf a})\vee 1\}},\label{eq:Termi_variance2_TV_Bernstein_bound}
        \end{align}
where the inequality (i) is by the second inequality in event $\mathcal{E}$ in eq. \ref{eq:Bernstein_TV_event}, the  inequality (ii) is by Cauchy-
Schwartz inequality and the probability distribution sums up to 1, and the last inequality (iii) is from the fact $\sqrt{ab} \leq a+b$.

        \item \textbf{Term (ii):} By using the proof-lines of \Cref{lem:variance_analysis_1_Bernstein} and assuming that the optimism and pessimism inequality eq. \ref{eq:general-V-UCB-TV_Bernstein} holds for $(h+1,k)$, we can bound Term (ii) as
        \begin{align}
        \label{eq:Termii_variance2_TV_Bernstein_bound}
            \text{Term (ii)} \leq 4H\mathbb{E}_{P_h^\star(\cdot|s,{\bf a})} \left[ \overline{V}_{h+1}^{k,\rho_i} - \underline{V}_{h+1}^{k,\rho_i} \right].
        \end{align}
    \end{itemize}
Applying eq. \ref{eq:Termi_variance2_TV_Bernstein_bound} and eq. \ref{eq:Termii_variance2_TV_Bernstein_bound}, we get the required bound in \Cref{lem:variance_analysis_2_Bernstein}.
\end{proof}

\section{Proof of regret bound of {\AlgonameKL}}
\label{app:thm:Regret_KL_bound}
% For the KL-divergence uncertainty set, we adopt the following standard assumption \cite{Arxiv2022_OfflineDRRLLinearFunctionApprox_Ma, AnnalsStat2022_TheoreticalUnderstandingRMDP_Yang,NeuRIPS2023_CuriousPriceDRRLGenerativeMdel_Shi}, which ensures the regularity of the dual formulation of the distributionally robust optimization over the KL-divergence uncertainty set.

% \begin{assu}
% \label{ass:KL_P_min}
%     We assume there exists a constant
% $P^{\star}_{\min}>0$, such that for any $(h,s,\bm{a},s') \in [H]\times \mathcal{S}\times \mathcal{A} \times \mathcal{S}$, if $P^\star_h(s'|s,\bm{a})>0$, then $P^\star_h(s'|s,\bm{a})>P^\star_{\min}$.
% \end{assu}

Similar to \citep{ghosh2025provablynearoptimaldistributionallyrobust}, we consider the following definitions:
\begin{align}
    \widehat{P}^k_{\min,h}(s,\bm{a}) &:= \min_{s'\in \mathcal{S}} \left\{ \widehat{P}^k_{h}(s'|s,\bm{a}) : \widehat{P}^k_{h}(s'|s,\bm{a}) > 0 \right\}, \label{eq:hat_p_min_KL} \\
    P^\star_{\min,h}(s,\bm{a}) &:= \min_{s'\in \mathcal{S}} \left\{ P^\star_h(s'|s,\bm{a}) : P^\star_h(s'|s,\bm{a}) > 0 \right\}, \label{eq:p_star_min_h_KL} \\
    P^\star_{\min} &:= \min_{(h,s) \in [H] \times \mathcal{S}} P^\star_{\min,h}(s, \pi^\star_h(s)), \label{ass:KL_P_min}
\end{align}
where the following inequality is satisfied: $P^\star_h(s'|s,\bm{a}) \geq P^\star_{\min,h}(s, \pi^\star_h(s)) \geq P^\star_{\min}$.\\

We now recall the bonus term of {\AlgonameKL} for agent $i$ in episode $k$ at step $h$, as follows:
\begin{align}
\label{eq:Bonus_term_KL}
 \beta^k_{i,h}(s,\bm{a}) =  \frac{2 c_f H}{\sigma_i} 
\sqrt{
    \frac{\iota}{
        \big(N_h^k(s,\bm{a}) \vee 1\big) \widehat{P}^k_{\min,h}(s,\bm{a})
    }
} + \sqrt{\frac{1}{K}},
\end{align}
where $\widehat{P}^k_{\min,h}(s,\bm{a}) = \min\limits_{s'\in \mathcal{S}}\{\widehat{P}^k_{h}(s'|s,\bm{a}): \widehat{P}^k_{h}(s'|s,\bm{a})>0\}$, $\iota = \log\Big(S^2(\prod_{i=1}^m A_i)H^2K^{3/2}/\delta\Big)$, and  $c_f$ is an absolute constant.

Before proceeding to all key lemmas, we introduce the high-probability “typical” event 
$\mathcal{E}_{\text{KL}}$ in the lemma below. The proof strategy follows \citep{Arxiv2024_DRORLwithInteractiveData_Lu} and \citep{ghosh2025provablynearoptimaldistributionallyrobust}.

\begin{lem}[Uniform Concentration Bound of event $\mathcal{E}_{\text{KL}}$]
\label{lem:KL_event_bound}
Let $\mathcal{E}_{\text{KL}}$ be the event in which, for all $(s, \mathbf{a}, s', h, k) \in \mathcal{S} \times \mathcal{A} \times \mathcal{S} \times [H] \times [K]$, and for all $\eta$ in a $\frac{1}{\rho_{\min}S\sqrt{K}}$-cover of $[0,H/\rho_{\min}]$, and is defined as
\begin{align}
\label{eq:Event_KL}
    \mathcal{E}_{\text{KL}} =\Bigg\{&\abs{\log\bigg(\mathbb{E}_{\widehat{P}_h^k(\cdot|s,\bm{a})}\bigg[\exp\bigg\{-\frac{V_{h+1}}{\eta}\bigg\}\bigg]\bigg)  - \log\bigg(\mathbb{E}_{P^{\star}_h(\cdot|s,\bm{a})}\bigg[\exp\bigg\{-\frac{V_{h+1}}{\eta}\bigg\}\bigg]\bigg)} \nonumber\\
    &\qquad \qquad \leq c_1\sqrt{\frac{\iota}{\{N^k_h(s,\bm{a})\vee 1\}\widehat{P}^k_{\min,h}(s,\bm{a})}},\nonumber\\
    &\quad \forall (h,s,\bm{a},s^{\prime},k) \in [H]\times \mathcal{S}\times \mathcal{A}\times \mathcal{S}\times[K], \forall \eta \in \mathcal{N}_{\frac{1}{\rho_{\min}S\sqrt{K}}}\left(\left[0,\frac{H}{\rho_{\min}}\right]\right) \Bigg\},
\end{align}
where $\widehat{P}^k_{\min,h}(s,\bm{a})$ is defined in eq. \ref{eq:hat_p_min_KL}, $\iota = \log\Big(S^3 \big(\prod_{i=1}^m A_i\big) H^2 K^{3/2} / \delta\Big)$, $c_1> 0$ is an absolute constant and $\eta \in \mathcal{N}_{\frac{1}{\rho_{\min}S\sqrt{K}}}([0,H/\rho_{\min}])$, where $\rho_{\min}=\min\limits_{i \in \mathcal{M}}\rho_i$ and $\mathcal{N}_{\frac{1}{\rho_{\min}S\sqrt{K}}}([0,H/\rho_{\min}])$ denotes an $1/(\rho_{\min}S\sqrt{K})$-cover of the interval $[0,H/\rho_{\min}]$.

Then, this event $ \mathcal{E}_{\text{KL}}$ occurs with high probability, i.e., $\Pr(\mathcal{E}_{\text{KL}}) \geq 1 - \delta$.
\end{lem}
% \textbf{Define the event \( \mathcal{E}_{\text{KL}} \) for DRMG-KL:}
% Before presenting all key lemmas, we define the typical event \( \mathcal{E}_{\text{KL}} \) as
% \begin{align}
% \label{eq:Event_KL}
%     \mathcal{E}_{\text{KL}} =\Bigg\{&\abs{\log\bigg(\mathbb{E}_{\widehat{P}_h^k(\cdot|s,\bm{a})}\bigg[\exp\bigg\{-\frac{V_{h+1}}{\eta}\bigg\}\bigg]\bigg)  - \log\bigg(\mathbb{E}_{P^{\star}_h(\cdot|s,\bm{a})}\bigg[\exp\bigg\{-\frac{V_{h+1}}{\eta}\bigg\}\bigg]\bigg)} \nonumber\\
%     &\qquad \qquad \leq c_1\sqrt{\frac{\iota}{\{N^k_h(s,\bm{a})\vee 1\}\widehat{P}^k_{\min,h}(s,\bm{a})}},\nonumber\\
%     &\quad \forall (h,s,\bm{a},s^{\prime},k) \in [H]\times \mathcal{S}\times \mathcal{A}\times \mathcal{S}\times[K], \forall \eta \in \mathcal{N}_{\frac{1}{\rho_{\min}S\sqrt{K}}}\left(\left[0,\frac{H}{\rho_{\min}}\right]\right) \Bigg\},
% \end{align}
% where $\widehat{P}^k_{\min,h}(s,\bm{a})$ is defined in \ref{eq:hat_p_min_KL}, $\iota = \log\Big(S^3 \big(\prod_{i=1}^m A_i\big) H^2 K^{3/2} / \delta\Big)$, $c_1> 0$ is an absolute constant and $\eta \in \mathcal{N}_{\frac{1}{\rho_{\min}S\sqrt{K}}}([0,H/\rho_{\min}])$, where $\rho_{\min}=\min\limits_{i \in \mathcal{M}}\rho_i$ and $\mathcal{N}_{\frac{1}{\rho_{\min}S\sqrt{K}}}([0,H/\rho_{\min}])$ denotes an $1/(\rho_{\min}S\sqrt{K})$-cover of the interval $[0,H/\rho_{\min}]$.

% \begin{lem}[Bound of event $\mathcal{E}_{\text{KL}}$]
%  For the typical event $\mathcal{E}_{\text{KL}}$ defined in \ref{eq:Event_KL}, it holds that $\Pr(\mathcal{E}_{\text{KL}}) \geq 1 - \delta$.
% \end{lem}

\begin{proof}
   The proof follows standard techniques: we apply classical concentration inequalities followed by a union bound. Consider a fixed tuple $(s,{\bf a}, h)$ for a fixed episode $k$. Now we consider the following equivalent random process: (i) before the agents starts, the environment samples $\{s^{(1)}, s^{(2)},\dots, s^{(k-1)}\}$ independently from $P^\star_h(\cdot|s,{\bf a})$, where $s^{(i)} \in \mathcal{S}$ denotes the state sampled at episode $i$; (ii) during the interaction between the agents and the environment, the $i$-th time the state and joint actions $(s, {\bf a})$ tuple is visited at step $h$, the environment will make the agents transit to the next state $s^{(i)}$. Note that the randomness induced by this interaction procedure is exactly the same as the original one, which means the probability of any event in this context is the same as in the original problem. Therefore, it suffices to prove the target concentration inequality in this context.

Based on the above fact, we directly apply \citep[Lemma 16]{NeuRIPS2024_UnifiedPessimismOfflineRL_Yue}. To extend the bound uniformly, we apply a union bound over all tuples $
(h, s, {\bf a}, s', k, \eta) \in [H] \times \mathcal{S} \times \mathcal{A} \times \mathcal{S} \times [K] \times \mathcal{N}_{1/(\rho_{\min}S\sqrt{K})}\big([0, H/\rho_{\min}]\big)$. Note that the $\eta$-cover for each agent $i$ lies in the interval $[0,H/\rho_i] \leq [0,H/\rho_{\min}]$ for all $i \in \mathcal{M}$, and this cover contains a valid $\frac{1}{\rho_iS\sqrt{K}}$-cover for each agent-specific interval $\left[0, \frac{H}{\rho_i} \right]$. Therefore, we define the common $\eta$-cover as $\eta \in \mathcal{N}_{\frac{1}{\rho_{\min}S\sqrt{K}}} \left( \left[0, \frac{H}{\rho_{\min}} \right] \right),$
where $\mathcal{N}_{\frac{1}{\rho_{\min}S\sqrt{K}}} \left( \left[0, \frac{H}{\rho_{\min}} \right] \right)$ denotes a $\frac{1}{\rho_{\min}S\sqrt{K}}$-cover of the interval $\left[0, \frac{H}{\rho_{\min}} \right]$.\qedhere

%    This is a direct application of \citep[Lemma 16]{NeuRIPS2024_UnifiedPessimismOfflineRL_Yue}, together with a union bound over $(h, s, \bm{a}, s', k, \eta) \in [H]\times \mathcal{S}\times \mathcal{A} \times \mathcal{S}\times [K]\times\mathcal{N}_{1/S\sqrt{K}}([0,H/\rho_i])$. Note that the size of $\mathcal{N}_{1/S\sqrt{K}}([0,H/\rho_i])$ is of order $SH\sqrt{K}$.
\end{proof}

\subsubsection{Proof of Theorem \ref{thm:Regret_KL_bound} ({\RMGKL} Setting)}
\label{app:proof_thm_Regret_KL_bound}
\begin{proof}
With \Cref{lem:Optimistic_pessimism_KL_NE}, we can establish an upper bound on the regret by considering the difference between our optimistic and pessimistic value functions:
\begin{align}
\Reg_{\sf NASH}(K) = \sum_{k=1}^K \max_{i \in \mathcal{M}}(V_{i,1}^{\dagger, \pi_{-i}^k, \rho_i}-V_{i,1}^{\pi^k, \rho_i} )(s^k_1) \leq \sum_{k=1}^K \max_{i \in \mathcal{M}}(\up{V}_{i,1}^{k, \rho_i}-\low{V}_{i,1}^{k, \rho_i} )(s^k_1).
\label{eq:Regret_KL_step1}
\end{align}
For the KL-divergence uncertainty set, we will refer to the bonus term as $\beta^k_{i,h}(s,\bm{a})$, as given in eq. \ref{eq:Bonus_term_KL}. Our first step is to establish a bound on the difference between the upper and lower Q-values. Given our definitions for $\overline{Q}_{i,h}^{k,\rho_i},\underline{Q}_{i,h}^{k,\rho_i},\overline{V}_{i,h}^{k,\rho_i},\underline{V}_{i,h}^{k,\rho_i}$, and the bonus term $\beta_{i,h}^{k,\rho_i}(s,\bm{a})$ as defined in eq. \ref{eq:robust_Qupper_k} through eq. \ref{eq:Bonus_term_KL}, for any $(i,h,k,s,\bm{a}) \in \mathcal{M} \times [H] \times [K] \times \mathcal{S} \times \mathcal{A}$, we have
\begin{align}
\overline{Q}_{i,h}^{k,\rho_i}(s,\bm{a}) - \underline{Q}_{i,h}^k(s,\bm{a}) &\leq \sigma_{\widehat{\mathcal{P}^{\rho_i}_{i,h}}(s,\bm{a})}\left[ \overline{V}_{i,h+1}^{k,\rho_i} \right]
- \sigma_{\widehat{\mathcal{P}^{\rho_i}_{i,h}}(s,\bm{a})}\left[ \underline{V}_{i,h+1}^{k,\rho_i} \right] + 2\beta^{k,\rho_i}_{i,h}(s,\bm{a}).\label{eq:Regret_KL_step2}
\end{align}
We define the following terms, $A$ and $B$, to simplify our analysis:
\begin{align}
A &:= \sigma_{\widehat{\mathcal{P}^{\rho_i}_{i,h}}(s,\bm{a})}\left[ \overline{V}_{i,h+1}^{k,\rho_i} \right] - \sigma_{\mathcal{P}^{\rho_i}_{i,h}(s,\bm{a})}\left[ \overline{V}_{i,h+1}^{k,\rho_i} \right] + \sigma_{\mathcal{P}^{\rho_i}_{i,h}(s,\bm{a})}\left[ \underline{V}_{i,h+1}^{k,\rho_i} \right] - \sigma_{\widehat{\mathcal{P}^{\rho_i}_{i,h}}(s,\bm{a})}\left[ \underline{V}_{i,h+1}^{k,\rho_i} \right]. \label{eq:Regret_KL_A}\\
B &:= \sigma_{\mathcal{P}^{\rho_i}_{i,h}(s,\bm{a})}\left[ \overline{V}_{i,h+1}^{k,\rho_i} \right] - \sigma_{\mathcal{P}^{\rho_i}_{i,h}(s,\bm{a})}\left[ \underline{V}_{i,h+1}^{k,\rho_i} \right].\label{eq:Regret_KL_B}
\end{align}
By applying eq. \ref{eq:Regret_KL_A} and eq. \ref{eq:Regret_KL_B} to eq. \ref{eq:Regret_KL_step2}, we obtain:
\begin{align}
\overline{Q}_{i,h}^{k,\rho_i}(s,\bm{a}) - \underline{Q}_{i,h}^{k,\rho_i}(s,\bm{a})&\leq A + B + 2\beta^{k,\rho_i}_{i,h}(s,\bm{a}). \label{eq:Regret_KL_step3}
\end{align}
We can upper bound term $A$ using a concentration argument tailored for KL robust expectations from \Cref{lem:Proper_bound_optimism_pessimism_bound_KL}, which shows that 
\begin{align}
\label{eq:Regret_bound_KL_A}
A \leq 2\beta^{k,\rho_i}_{i,h}(s,\bm{a}).
\end{align}
For term $B$, we use the definition of $\mathbb{E}_{\mathcal{P}_h^\rho(s,a)}[V]$ from eq. \ref{eq:dual_KL} to establish the following bound:
\begin{align}
B &= \sup_{\eta \in \left[0,\frac{H}{\rho_i}\right]} \bigg\{-\eta\log\bigg(\mathbb{E}_{P^{\star}_h(\cdot|s,\bm{a})}\bigg[\exp\bigg\{-\frac{\overline{V}^{k,\rho_i}_{i,h+1}}{\eta}\bigg\}\bigg]\bigg) - \eta\rho_i\bigg\} \nonumber\\
&\qquad \qquad \quad - \sup_{\eta \in \left[0,\frac{H}{\rho_i}\right]} \bigg\{-\eta\log\bigg(\mathbb{E}_{P^{\star}_h(\cdot|s,\bm{a})}\bigg[\exp\bigg\{-\frac{\underline{V}^{k,\rho_i}_{i,h+1}}{\eta}\bigg\}\bigg]\bigg) - \eta\rho_i\bigg\}\nonumber\\
&\leq \sup_{\eta \in [0,H/\rho_i]} \eta \Bigg\{ \log\bigg(\mathbb{E}_{P^{\star}_h(\cdot|s,\bm{a})}\bigg[\exp\bigg\{-\frac{\underline{V}^{k,\rho_i}_{i,h+1}}{\eta}\bigg\}\bigg]\bigg) \nonumber\\
&\qquad \qquad \qquad - \log\bigg(\mathbb{E}_{P^{\star}_h(\cdot|s,\bm{a})}\bigg[\exp\bigg\{-\frac{\overline{V}^{k,\rho_i}_{i,h+1}}{\eta}\bigg\}\bigg]\bigg)\Bigg\} \nonumber\\
&= \sup_{\eta \in [0,H/\rho_i]} \eta \log\Bigg( \frac{\mathbb{E}_{P^{\star}_h(\cdot|s,\bm{a})}\bigg[\exp\bigg\{-\frac{\underline{V}^{k,\rho_i}_{i,h+1}}{\eta}\bigg\}\bigg]}{\mathbb{E}_{P^{\star}_h(\cdot|s,\bm{a})}\bigg[\exp\bigg\{-\frac{\overline{V}^{k,\rho_i}_{i,h+1}}{\eta}\bigg\}\bigg]}\Bigg)\nonumber\\
 &=\sup_{\eta \in [0,H/\rho_i]} \eta \log\Bigg(1 + \frac{\mathbb{E}_{P^{\star}_h(\cdot|s,\bm{a})}\bigg[\exp\bigg\{-\frac{\underline{V}^{k,\rho_i}_{i,h+1}}{\eta}\bigg\}- \exp\bigg\{-\frac{\overline{V}^{k,\rho_i}_{i,h+1}}{\eta}\bigg\}\bigg]}{\mathbb{E}_{P^{\star}_h(\cdot|s,\bm{a})}\bigg[\exp\bigg\{-\frac{\overline{V}^{k,\rho_i}_{i,h+1}}{\eta}\bigg\}\bigg]}\Bigg)\nonumber\\
&\overset{\text{(a)}}{\leq} \sup_{\eta \in [0,H/\rho_i]} \eta \frac{\mathbb{E}_{P^{\star}_h(\cdot|s,\bm{a})}\bigg[\exp\bigg\{-\frac{\underline{V}^{k,\rho_i}_{i,h+1}}{\eta}\bigg\}- \exp\bigg\{-\frac{\overline{V}^{k,\rho_i}_{i,h+1}}{\eta}\bigg\}\bigg]}{\mathbb{E}_{P^{\star}_h(\cdot|s,\bm{a})}\bigg[\exp\bigg\{-\frac{\overline{V}^{k,\rho_i}_{i,h+1}}{\eta}\bigg\}\bigg]}\nonumber\\
&\overset{\text{(b)}}{\leq} \sup_{\eta \in [\underline{\eta},H/\rho_i]} \eta \exp\bigg\{\frac{H}{\underline{\eta}}\bigg\}\mathbb{E}_{P^{\star}_h(\cdot|s,\bm{a})}\bigg[\exp\bigg\{-\frac{\underline{V}^{k,\rho_i}_{i,h+1}}{\eta}\bigg\}- \exp\bigg\{-\frac{\overline{V}^{k,\rho_i}_{i,h+1}}{\eta}\bigg\}\bigg] \nonumber\\
&\overset{\text{(c)}}{\leq} \exp\bigg\{\frac{H}{\underline{\eta}}\bigg\}\mathbb{E}_{P^{\star}_h(s,{\bf a})}\left[\overline{V}^{k,\rho_i}_{i,h+1} - \underline{V}^{k,\rho_i}_{i,h+1}\right],\label{eq:Regret_bound_KL_B}
\end{align}
where inequality (a) uses the fact that $\log(1 + x) \leq x$, inequality (b) holds because $0 \leq \overline{V}^{k,\rho_i}_{i,h+1}\leq H$ and $\eta \in [\underline{\eta}, H/\rho_i]$, and inequality (c) is due to the $\frac{1}{\eta}$-Lipschitz continuity of $\phi_{\eta}(x) = \exp\big\{-\frac{x}{\eta}\big\}$ for $x\geq 0$, as well as $\underline{V}^{k,\rho_i}_{i,h+1} \leq \overline{V}^{k,\rho_i}_{i,h+1}$.

By applying the bounds for $A$ and $B$ to eq. \ref{eq:Regret_KL_step3}, we get
\begin{align}
&\overline{Q}_{i,h}^{k,\rho_i}(s,\bm{a}) - \underline{Q}_{i,h}^{k,\rho_i}(s,\bm{a}) \leq\exp\bigg\{\frac{H}{\underline{\eta}}\bigg\}\mathbb{E}_{P^{\star}_h(s,{\bf a})}\left[\overline{V}^{k,\rho_i}_{i,h+1} - \underline{V}^{k,\rho_i}_{i,h+1}\right] + 4\beta^{k,\rho_i}_{h}(s,\bm{a}). \label{eq:Regret_KL_step4}
\end{align}
Using \Cref{lem:Control_Bonus_KL} to upper bound the bonus term, and rearranging the terms, we further obtain:
\begin{align}
\overline{Q}_{i,h}^{k,\rho_i}(s,{\bf a}) - \underline{Q}_{i,h}^{k,\rho_i}(s,{\bf a}) &\leq \exp\bigg\{\frac{H}{\underline{\eta}}\bigg\}\mathbb{E}_{P^{\star}_h(s,{\bf a})}\left[\overline{V}^{k,\rho_i}_{i,h+1} - \underline{V}^{k,\rho_i}_{i,h+1}\right] \nonumber\\ &\qquad + \frac{4c_1H}{\rho_{\min}}\sqrt{\frac{\iota^2}{\{N^k_h(s,\bm{a})\vee 1\}P^{\star}_{\min}}} + \sqrt{\frac{4}{K}},\label{eq:Regret_KL_step5}
\end{align}
where $c_1>0$ is an absolute constant. From the definitions in eq. \ref{eq:robust_V_values_k}, the difference in V-functions is given by:
\begin{align}
\overline{V}_{i,h}^{k,\rho_i}(s) - \underline{V}_{i,h}^{k,\rho_i}(s) & = \mathbb{E}_{{\bf a}\sim\pi^k(\cdot|s)}\bigg[\overline{Q}_{i,h}^{k,\rho_i}(s,{\bf a}) - \underline{Q}_{i,h}^{k,\rho_i}(s,{\bf a}) \bigg].\label{eq:Regret_KL_step6}
\end{align}
We now define a new recursive value function $\widetilde{V}_{h}^{k,\rho_{\min}}$ and a corresponding Q-function $\widetilde{Q}_{h}^{k,\rho_{\min}}$ with $\widetilde{V}_{H+1}^{k,\rho_{\min}}=0$, where $\rho_{\min}=\min\limits_{i \in \mathcal{M}}\rho_i$:
\begin{align}
& \widetilde{Q}^{k,\rho_{\min}}_{h}(s,\bm{a})
= \exp\bigg\{\frac{H}{\underline{\eta}}\bigg\}\mathbb{E}_{P^{\star}_h(s,{\bf a})}\left[\widetilde{V}^{k,\rho_{\min}}_{h+1}\right] + \frac{4c_1H}{\rho_{\min}}\sqrt{\frac{\iota^2}{\{N^k_h(s,\bm{a})\vee 1\}P^{\star}_{\min}}} + \sqrt{\frac{4}{K}}. \label{eq:Q_tilde}\\
& \widetilde{V}^{k,\rho_{\min}}_{h}(s)
= \mathbb{E}_{\bm{a} \sim \pi_{h}^{k}(\cdot|s)}\left[\widetilde{Q}^{k,\rho_{\min}}_{i,h}(s,{\bf a})\right]. \label{eq:V_tilde}
\end{align}
By an inductive proof, we can show that for any $(i, h, s, {\bf a}) \in \mathcal{M} \times [H] \times \mathcal{S} \times \mathcal{A}$, the following bounds hold:
\begin{align}
& \max_{i \in \mathcal{M}}(\overline{Q}^{k,\rho_i}_{i,h} - \underline{Q}^{k,\rho_i}_{i,h})(s,\bm{a}) \le \widetilde{Q}^{k,\rho_{\min}}_{h}(s,\bm{a}),\label{eq:bound_Q_tilde}\\
& \max_{i \in \mathcal{M}}(\overline{V}^{k,\rho_i}_{i,h} - \underline{V}^{k,\rho_i}_{i,h})(s) \le \widetilde{V}^{k,\rho_{\min}}_{h}(s).\label{eq:bound_V_tilde}
\end{align}
Therefore, our analysis can focus on bounding the sum $\sum_{k=1}^K \widetilde{V}_1^{k,\rho_{\min}}(s^k_1)$. For simplicity, we introduce the following notations for the differences at any $(h,k) \in [H] \times [K]$:
\begin{align}
\Delta^k_{h} &:= \widetilde{V}_h^{k,\rho_{\min}}(s^k_h),\label{eq:Delta_k_h_KL}\\
\zeta_{h}^k &:= \Delta_{h}^k - \widetilde{Q}_h^{k,\rho_{\min}}(s_h^k, {\bf a}_h^k), \label{eq:zeta_h_k_KL}\\
\xi_{h}^k &:= \mathbb{E}_{P^\star_h(\cdot|s^k_h,{\bf a}^k_h)}[\widetilde{V}_{h+1}^{k,\rho_{\min}}] - \Delta_{h+1}^k. \label{eq:xi_h_k_KL}
\end{align}
We can confirm that $\{\zeta_{h}^k\}_{(h,k)}$ and $\{\xi_{h}^k\}_{(h,k)}$ are martingale difference sequences with respect to their respective filtrations. By substituting eq. \ref{eq:Q_tilde} into eq. \ref{eq:zeta_h_k_KL}, we obtain the recursive relationship:
\begin{align}
\Delta_{i,h}^k &= \zeta_{i,h}^k + \widetilde{Q}_h^{k,\rho_{\min}}(s_h^k, {\bf a}_h^k) \nonumber \\
&\le \zeta_{i,h}^k + \exp\bigg\{\frac{H}{\underline{\eta}}\bigg\}\mathbb{E}_{P^{\star}_h(s,{\bf a})}\left[\widetilde{V}^{k,\rho_{\min}}_{h+1}\right] + \frac{4c_1H}{\rho_{\min}}\sqrt{\frac{\iota^2}{\{N^k_h(s,\bm{a})\vee 1\}P^{\star}_{\min}}} + \sqrt{\frac{4}{K}}\nonumber \\
&= \zeta_{i,h}^k + \exp\bigg\{\frac{H}{\underline{\eta}}\bigg\}\xi_{i,h}^k + \exp\bigg\{\frac{H}{\underline{\eta}}\bigg\}\Delta_{i,h+1}^k + \frac{4c_1H}{\rho_{\min}}\sqrt{\frac{\iota^2}{\{N^k_h(s,\bm{a})\vee 1\}P^{\star}_{\min}}} \nonumber\\
&\qquad + \sqrt{\frac{4}{K}}.\label{eq:Regret_KL_step7}
\end{align}
By recursively applying eq. \ref{eq:Regret_KL_step7} and noting that $1 \leq \left(\exp\Big\{\frac{H}{\underline{\eta}}\Big\}\right)^h \leq \left(\exp\Big\{\frac{H}{\underline{\eta}}\Big\}\right)^H := d_H$, we can upper bound the right hand side of eq. \ref{eq:Regret_KL_step1} as:
\begin{align}
\label{eq:Regret_KL_step8}
\Reg_{\sf NASH}(K) \leq \sum_{k=1}^K \Delta_{1}^k &\leq c'd_H\sum_{k=1}^K \sum_{h=1}^H \Bigg\{(\zeta_{h}^k + \xi_{h}^k) \nonumber\\
&\qquad \qquad \qquad + \left(\frac{4c_1H}{\rho_{\min}}\sqrt{\frac{\iota^2}{\{N^k_h(s,\bm{a})\vee 1\}P^{\star}_{\min}}} + \sqrt{\frac{4}{K}}\right)\Bigg\}.
\end{align}
Next, we bound each of these two main terms.
The first term, a sum of martingale differences, is bounded using the Azuma-Hoeffding inequality from \Cref{lem:Azuma-Hoeffding}, yielding:
\begin{align}
\label{eq:Term_i_KL_step9}
\sum_{k=1}^{K} \sum_{h=1}^{H} (\zeta_{i,h}^k + \xi_{i,h}^k)
\leq c'_1 \sqrt{H^3KL},
\end{align}
where $c'_1 > 0$ is an absolute constant.
For the second term, we apply the proof lines of \citep[Theorem 3]{liu2021sharp} to bound the sum of the inverse counts:
\begin{align}
\label{eq:Regret_KL_step10}
\sum_{k=1}^K \sum_{h=1}^H \sqrt{ \frac{ 1 }{\{N_h^k(s_h^k, \bm{a}_h^k) \vee 1\}}} \le c'_2\Bigg(\sqrt{H^2KS\prod_{i\in \mathcal{M}}A_i} + HS\prod_{i\in \mathcal{M}}A_i \Bigg).
\end{align}
By applying eq. \ref{eq:Regret_KL_step10} to the second term of eq. \ref{eq:Regret_KL_step8}, we get the following:
\begin{align}
\label{eq:Term_ii_KL_step11}
\sum_{k=1}^K \sum_{h=1}^H \left(\frac{4c_1H}{\rho_{\min}}\sqrt{\frac{\iota^2}{\{N^k_h(s,\bm{a})\vee 1\}P^{\star}_{\min}}} + \sqrt{\frac{4}{K}}\right) &\le c'_2\Bigg(\sqrt{\frac{H^4KS\big(\prod_{i\in \mathcal{M}}A_i\big)\iota^2 }{\rho^2_{\min}P^\star_{\min}}} \nonumber\\
&+ \frac{H^2S\big(\prod_{i\in \mathcal{M}}A_i\big)\iota}{\rho_{\min}\sqrt{P^\star_{\min}}} + \sqrt{H^2K} \Bigg).
\end{align}
By combining the bounds for both terms in eq. \ref{eq:Regret_KL_step8}, we can upper bound the final regret as follows:
\begin{align}
\label{eq:Regret_KL_step12}
\Reg_{\sf NASH}(K) &\leq c'd_H\Bigg(\sqrt{\frac{H^4KS\big(\prod_{i\in \mathcal{M}}A_i\big)\iota^2 }{\rho^2_{\min}P^\star_{\min}}}\Bigg) \nonumber\\
&= \mathcal{O}\Bigg(\sqrt{\frac{H^4\exp(2H^2)KS\big(\prod_{i\in \mathcal{M}}A_i\big)(\iota^{\prime})^3 }{\rho^2_{\min}P^\star_{\min}}}\Bigg).
\end{align}
This completes the proof of \Cref{thm:Regret_KL_bound}. \qedhere

\begin{rem}
\label{rem:proof_regret_KL_bound_CCE_CE}
The proof techniques for bounding $\Reg_{\sf CCE}(K)$ and $\Reg_{\sf CE}(K)$ follow the same lines of proof for $\Reg_{\sf NASH}(K)$, leveraging \Cref{lem:Optimistic_pessimism_KL_CCE} and \Cref{lem:Optimistic_pessimism_KL_CE}, respectively, in the context of \RMGKL.
\end{rem}
\end{proof}

\subsection{Key Lemmas for {\RMGKL}}
\label{subsubsec:key_lemma_KL}

\begin{lem}[Concentration Bound for Robust Value Estimators in \RMGKL]
\label{lem:Proper_bound_optimism_pessimism_bound_KL}
Let $\mathcal{E}_{KL}$ be the typical event and let the bonus term $ \beta^{k}_{i,h}$ be set 
defined in eq. \ref{eq:Bonus_term_KL}. Then, the following inequality holds:
\begin{align}
\label{eq:Proper_bound_optimism_pessimism_bound_KL}
&\sigma_{\widehat{\mathcal{P}_{i,h}^{\rho_i}}(s,\bm{a})}\left[\overline{V}_{i,h+1}^{k, \rho_i}\right] -\sigma_{\mathcal{P}_{i,h}^{\rho_i}(s,\bm{a})}\left[\overline{V}_{i,h+1}^{k, \rho_i}\right] + \sigma_{\mathcal{P}_{i,h}^{\rho_i}(s,\bm{a})}\left[\underline{V}_{i,h+1}^{k, \rho_i}\right] -\sigma_{\widehat{\mathcal{P}_{i,h}^{\rho_i}}(s,\bm{a})}\left[\underline{V}_{i,h+1}^{k, \rho_i}\right] \nonumber\\
&\qquad \qquad \leq \frac{2c_1H}{\rho_{\min}}\sqrt{\frac{\iota}{\{N^k_h(s,\bm{a})\vee 1\}\widehat{P}^k_{\min,h}(s,\bm{a})}} + \sqrt{\frac{2}{K}},
\end{align}
where $\iota = \log\left(S^3 \left(\prod_{i=1}^m A_i\right) H^2 K^{3/2} / \delta\right)$, and \( c_1 > 0 \) is an absolute constant.
\end{lem}

\begin{proof}
We begin by defining the term that we need to bound. Let's denote this term by $A$:
\begin{align}
\label{eq:A_KL}
A:= \sigma_{\widehat{\mathcal{P}_h^\rho}(s,\bm{a})}\left[\overline{V}_{h+1}^k\right] -\sigma_{\mathcal{P}_h^\rho(s,\bm{a})}\left[\overline{V}_{h+1}^k\right] + \sigma_{\mathcal{P}_h^\rho(s,\bm{a})}\left[\underline{V}_{h+1}^k\right] -\sigma_{\widehat{\mathcal{P}_h^\rho}(s,\bm{a})}\left[\underline{V}_{h+1}^k\right].
\end{align}
Under the high-probability event $\mathcal{E}_{KL}$, we can directly apply the concentration inequality given in \Cref{lem:Bound_KL_optimism_pessimism}. This allows us to upper bound $A$ as follows:
\begin{align}
A \leq \frac{2c_1H}{\rho_{\min}}\sqrt{\frac{\iota}{\{N^k_h(s,\bm{a})\vee 1\}\widehat{P}^k_{\min,h}(s,\bm{a})}} + \sqrt{\frac{2}{K}}, \label{eq:Proper_bouns_optimism_pessimism_KL_bound_step1}
\end{align}
where \( c_1 > 0 \) is an absolute constant and $\iota = \log\left(S^3 \left(\prod_{i=1}^m A_i\right) H^2 K^{3/2} / \delta\right)$. This bound is exactly the bonus term multiplied by a constant. Therefore, based on our choice of \(\beta^{k}_{i,h}(s,\bm{a})\) as defined in eq. \ref{eq:Bonus_term_KL}, the inequality in eq. \ref{eq:Proper_bound_optimism_pessimism_bound_KL} holds. This completes the proof of \Cref{lem:Proper_bound_optimism_pessimism_bound_KL}.
\end{proof}

\begin{lem}[Bound of the bonus term for {\RMGKL}]
\label{lem:Control_Bonus_KL}
\textit{Let \( \mathcal{E}_{KL} \) be the typical event, the bonus term \(  \beta^{k}_{i,h} \) in eq. \ref{eq:Bonus_term_KL} is bounded by}
\begin{align}
    \beta^{k}_{i,h}(s,\bm{a}) \leq \frac{c_1H}{\rho_{\min}}\sqrt{\frac{\iota^2}{\{N^k_h(s,\bm{a})\vee 1\}P^{\star}_{\min}}} + \sqrt{\frac{1}{K}},
\end{align}
where $\iota = \log\left(S^3 \left(\prod_{i=1}^m A_i\right) H^2 K^{3/2} / \delta\right)$, and \( c_1>0\) is an absolute constant.
\end{lem}

\begin{proof}
       The proof-lines are similar to \citep[Lemma K.7]{ghosh2025provablynearoptimaldistributionallyrobust}. We recall the choice of $\beta^{k}_{i,h}$ as given in eq. \ref{eq:Bonus_term_KL}, i.e.
\begin{align}
\label{eq:Bonus_term_KL_step1}
   \beta^{k}_{i,h}(s,\bm{a}) = \frac{2c_fH}{\rho_{i}}\sqrt{\frac{\iota}{\{N^k_h(s,\bm{a})\vee 1\}\widehat{P}^k_{\min,h}(s,\bm{a})}} + \sqrt{\frac{1}{K}},
\end{align}
where $\iota = \log\left(S^3 \left(\prod_{i=1}^m A_i\right) H^2 K^{3/2} / \delta\right)$, $\widehat{P}^k_{\min,h}(s,\bm{a})$ is defined in eq. \ref{eq:hat_p_min_KL}, and $c_f>0$ is an absolute constant. 
   
By \Cref{lem:binomial_rv_bound} and the union bound, it holds that with probability at least $1 - \delta$ that for all $(h,s,\bm{a}) \in [H]\times\mathcal{S}\times\mathcal{A}$, we get
\begin{align}
\label{eq:binomial_bound_step1}
    \forall s' \in \mathcal{S}:\quad 
P^{\star}_h(s' \mid s,\bm{a}) \geq \frac{\widehat{P}^k_h(s' \mid s,\bm{a})}{e^2} \geq \frac{P^{\star}_h(s' \mid s,\bm{a})}{8e^2\iota}.
\end{align}

To characterize the relation between $P^{\star}_{\min,h}(s,\bm{a})$ and $\widehat{P}^k_{\min,h}(s,\bm{a})$ for any $(h,s,\bm{a}) \in [H]\times\mathcal{S}\times\mathcal{A}$, we suppose—without loss of generality—that $P^{\star}_{\min,h}(s,\bm{a}) = P^{\star}_h(s_1 \mid s,\bm{a})$ and $\widehat{P}^k_{\min,h}(s,\bm{a}) = \widehat{P}^k_h(s_2 \mid s,\bm{a})$ for some $s_1, s_2 \in \mathcal{S}$. Then, it follows that
% \begin{align}
% P^{\star}_{\min,h}(s,\bm{a}) = P^{\star}_h(s_1 \mid s,\bm{a}) 
% &\overset{\text{(i)}}{\geq} \frac{\widehat{P}^k_h(s_1 \mid s,\bm{a})}{e^2} 
% \geq \frac{\widehat{P}^k_{\min,h}(s,\bm{a})}{e^2} \nonumber\\
% &\qquad = \frac{\widehat{P}^k_h(s_2 \mid s,\bm{a})}{e^2} \overset{\text{(ii)}}{\geq} \frac{P^{\star}_h(s_2 \mid s,\bm{a})}{8e^2 \iota} 
% \geq \frac{P^{\star}_{\min,h}(s,\bm{a})}{8e^2\iota} \overset{(iii)}{\geq} \frac{P^{\star}_{\min}}{8e^2\iota} ,\label{eq:binomial_bound_step2}
% \end{align}
\begin{align}
P^{\star}_{\min,h}(s,\bm{a}) 
&= P^{\star}_h(s_1 \mid s,\bm{a}) \nonumber\\
&\overset{\text{(i)}}{\geq} \frac{\widehat{P}^k_h(s_1 \mid s,\bm{a})}{e^2} 
   \geq \frac{\widehat{P}^k_{\min,h}(s,\bm{a})}{e^2} \nonumber\\
&= \frac{\widehat{P}^k_h(s_2 \mid s,\bm{a})}{e^2} 
   \overset{\text{(ii)}}{\geq} \frac{P^{\star}_h(s_2 \mid s,\bm{a})}{8e^2 \iota} \nonumber\\
&\geq \frac{P^{\star}_{\min,h}(s,\bm{a})}{8e^2 \iota} 
   \overset{\text{(iii)}}{\geq} \frac{P^{\star}_{\min}}{8e^2 \iota}.
\label{eq:binomial_bound_step2}
\end{align}

where the inequalities (i) and (ii) follow from eq. \ref{eq:binomial_bound_step1}, and inequality (iii) follows by eq. \ref{ass:KL_P_min}.

By applying eq. \ref{eq:binomial_bound_step2} in eq. \ref{eq:Bonus_term_KL_step1}, we get
\begin{align}
      \beta^{k}_{i,h}(s,\bm{a}) \leq \frac{2c_fH}{\rho_i}\sqrt{\frac{\iota^2}{\{N^k_h(s,\bm{a})\vee 1\}P^{\star}_{\min}}} + \sqrt{\frac{1}{K}} &\leq  \frac{c_1H}{\rho_{\min}}\sqrt{\frac{\iota^2}{\{N^k_h(s,\bm{a})\vee 1\}P^{\star}_{\min}}} \nonumber\\ &\qquad + \sqrt{\frac{1}{K}}. 
\end{align}
This concludes the proof of \Cref{lem:Control_Bonus_KL}.
\end{proof}

\subsubsection{NE Version: Optimistic and pessimistic estimation of the robust values for {\RMGKL}.}
Here we will proof the optimistic estimations are indeed upper bounds of the corresponding robust V-value and robust Q-value functions fro NE version.

\begin{lem}[Optimistic and pessimistic estimation of the robust values for {\RMGKL} for NE Version]
\label{lem:Optimistic_pessimism_KL_NE}
Under the event $\mathcal{E}_{KL}$ and by setting the bonus term $\beta^{k}_{i,h}$ as in eq. \ref{eq:Bonus_term_KL}, it holds that
    \begin{equation}
        \label{eq:general-Q-UCB-KL}
    Q_{i,h}^{\dagger,\pi _{-i}^{k}, \rho_i}\left( s,\bm{a} \right) \le \up{Q}_{i,h}^{k, \rho_i}\left( s,\bm{a} \right)  ,   \,\,\,\,  \low{Q}_{i,h}^{k,\rho_i}\left( s,\bm{a} \right) \le Q_{i,h}^{\pi^{k}, \rho_i}\left( s,\bm{a} \right),
    \end{equation}
    \begin{equation}
        \label{eq:general-V-UCB-KL}
        V_{i,h}^{\dagger,\pi _{-i}^{k}, \rho_i}\left( s \right) \le \up{V}_{i,h}^{k, \rho_i}\left( s \right)  , \,\,\,\,  \low{V}_{i,h}^{k, \rho_i}\left( s \right) \le V_{i,h}^{\pi^{k}, \rho_i}\left( s \right).
    \end{equation}
\end{lem}

\begin{proof}
The proof-lines are similar to \citep{ghosh2025provablynearoptimaldistributionallyrobust} adapted to the multi-agent case.\\
We will run a proof for each inequality outlined in  \Cref{lem:Optimistic_pessimism_KL_NE}
\begin{itemize}
    \item \textbf{Ineq. 1:} To prove $Q_{i,h}^{\dagger,\pi_{-i}^k,\rho_i}(s,\bm{a}) \leq \overline{Q}_{i,h}^{k,\rho_i}(s,\bm{a})$.
    \item \textbf{Ineq. 2:} To prove $\underline{Q}_{i,h}^{k,\rho_i}(s,\bm{a}) \leq Q_{i,h}^{\pi^k,\rho_i}(s,\bm{a})$.
\end{itemize}
Assume that both eq. \ref{eq:general-Q-UCB-KL} and eq. \ref{eq:general-V-UCB-KL} hold at the $(h+1)$-th step.
\begin{itemize}
    \item \textbf{Proof of Ineq. 1:}
We first consider robust $Q$ at the $h$-th step. Then, by \Cref{prop:Robust_Bellman_eq} (Robust Bellman Equation) and eq. \ref{eq:robust_Qupper_k}, we have that
\begin{align}
   & Q_{i,h}^{\dagger,\pi_{-i}^k,\rho_i}(s,\bm{a}) - \overline{Q}_{i,h}^{k,\rho_i}(s,\bm{a})\nonumber\\
&= \max \bigg\{ \sigma_{\mathcal{P}^{\rho_i}_{i,h}(s,\bm{a})} \left[ V_{i,h+1}^{\dagger,\pi_{-i}^k,\rho_i} \right]
- \sigma_{\widehat{\mathcal{P}^{\rho_i}_{i,h}}(s,\bm{a})} \left[\overline{V}^{k,\rho_i}_{i,h+1} \right] - \beta^{k}_{i,h}(s,\bm{a}), \nonumber\\
&\qquad \qquad \qquad \qquad  Q_{i,h}^{\dagger,\pi_{-i}^k,\rho_i}(s,\bm{a}) - H \bigg\}, \nonumber\\
&\leq \max\Biggl\{
   \sigma_{\mathcal{P}^{\rho_i}_{i,h}(s,\bm{a})}\!\left[ V_{i,h+1}^{\dagger,\pi_{-i}^k,\rho_i} \right]
   - \sigma_{\widehat{\mathcal{P}^{\rho_i}_{i,h}}(s,\bm{a})}\!\left[ V_{i,h+1}^{\dagger,\pi_{-i}^k,\rho_i} \right]  - \beta^{k}_{i,h}(s,\bm{a}),
   0
\Biggr\},\label{eq:optimism_pessimism_ineq_KL_step1}
\end{align}
where the second inequality follows from the induction of $V_{i,h+1}^{\dagger,\pi_{-i}^k,\rho_i} \leq \overline{V}_{i,h+1}^{k,\rho_i}$ at the $h+1$-th step and the fact that $Q_{i,h}^{\dagger,\pi_{-i}^k,\rho_i} \leq H$. By \Cref{lem:Bound_KL_optimal_policy} and by the definition of $\widehat{P}^k_{\min,h}(s,\bm{a})$ as given in eq. \ref{eq:hat_p_min_KL}, we have that
\begin{align}
\label{eq:optimism_pessimism_ineq_KL_step2}
\sigma_{\mathcal{P}^{\rho_i}_{i,h}(s,\bm{a})} \left[ V_{i,h+1}^{\dagger,\pi_{-i}^k,\rho_i} \right]
- \sigma_{\widehat{\mathcal{P}^{\rho_i}_{i,h}}(s,\bm{a})} \left[V_{i,h+1}^{\dagger,\pi_{-i}^k,\rho_i} \right] &\leq \frac{c_1H}{\rho_i}\sqrt{\frac{L}{\{N^k_h(s,\bm{a})\vee 1\}\widehat{P}^k_{\min,h}(s,\bm{a})}} \nonumber\\ &\qquad + \sqrt{\frac{1}{K}}.
\end{align}
% By the choice of $ \beta^{k}_{i,h}$ in eq. \ref{eq:Bonus_term_KL} and by the fact $\rho_i \geq \min\limits_{i \in \mathcal{M}}\rho_i,$ $\forall i \in \mathcal{M}$, we get
% \begin{align}
% \label{eq:bonus_bound_KL_NE}
%     \beta^k_{i,h}(s, {\bf a}) \leq \frac{c_1H}{\rho_{\min}}\sqrt{\frac{L}{\{N^k_h(s,\bm{a})\vee 1\}\widehat{P}^k_{\min,h}(s,\bm{a})}} + \sqrt{\frac{1}{K}}.
% \end{align}
% Therefore, 
By the choice of $ \beta^{k}_{i,h}$ in eq. \ref{eq:Bonus_term_KL} and eq. \ref{eq:optimism_pessimism_ineq_KL_step2} and applying in eq. \ref{eq:optimism_pessimism_ineq_KL_step1}, we conclude that
\begin{align}
\label{eq:optimism_pessimism_ineq_KL_bound_case1}
Q_{i,h}^{\dagger,\pi_{-i}^k,\rho_i}(s,\bm{a}) \leq \overline{Q}_{i,h}^{k,\rho_i}(s,\bm{a}).
\end{align}

\item \textbf{Proof of Ineq. 2:} By using \Cref{prop:Robust_Bellman_eq} (Robust Bellman Equation) and eq. \ref{eq:robust_Qlower_k}, we have that
\begin{align}
&\underline{Q}_{i,h}^{k,\rho_i}(s,\bm{a}) - Q_{i,h}^{\pi^k,\rho_i}(s,\bm{a})\nonumber\\
&= \max\Biggl\{
   \sigma_{\widehat{\mathcal{P}_{i,h}^{\rho_i}}(s,\bm{a})}\!\left[ \underline{V}_{i,h+1}^{k,\rho_i} \right]
   - \sigma_{\mathcal{P}_{i,h}^{\rho_i}(s,\bm{a})}\!\left[ V_{i,h+1}^{\pi^k,\rho_i} \right]
   - \beta^{k}_{i,h}(s,\bm{a}), \nonumber \\
&\qquad \qquad \qquad \qquad  0 - Q_{i,h}^{\pi^k,\rho_i}(s,\bm{a})\Biggr\}\\
&\leq \max\Biggl\{
   \sigma_{\widehat{\mathcal{P}_{i,h}^{\rho_i}}(s,\bm{a})}\!\left[ V_{i,h+1}^{\pi^k,\rho_i} \right]
   - \sigma_{\mathcal{P}_{i,h}^{\rho_i}(s,\bm{a})}\!\left[ V_{i,h+1}^{\pi^k,\rho_i} \right]
   - \beta^{k}_{i,h}(s,\bm{a}), 0\Biggr\},
 \label{eq:optimism_pessimism_ineq_KL_step4}
\end{align}
where the second inequality follows from the induction of $\underline{V}_{i,h+1}^{k,\rho_i} \leq V_{i,h+1}^{\pi^k,\rho_i}$ at the $(h+1)$-th step and the fact that $Q_{i,h}^{\pi^k,\rho_i} \geq 0$. By \Cref{lem:Bound_KL_policy_k}, we get
\begin{align}
\sigma_{\widehat{\mathcal{P}_{i,h}^{\rho_i}}(s,\bm{a})} \left[ V_{i,h+1}^{\pi^k,\rho_i} \right]
- \sigma_{\mathcal{P}_{i,h}^{\rho_i}(s,\bm{a})} \left[ V_{i,h+1}^{\pi^k,\rho_i} \right] &\leq \frac{c_1H}{\rho_i}\sqrt{\frac{L}{\{N^k_h(s,\bm{a})\vee 1\}\widehat{P}^k_{\min,h}(s,\bm{a})}} \nonumber\\
&\qquad + \sqrt{\frac{1}{K}}.\label{eq:optimism_pessimism_ineq_KL_step5}
\end{align}

By the choice of $ \beta^{k}_{i,h}$ in eq. \ref{eq:Bonus_term_KL} and eq. \ref{eq:optimism_pessimism_ineq_KL_step5} and applying in eq. \ref{eq:optimism_pessimism_ineq_KL_step4}, we conclude that
\begin{align}
\label{eq:optimism_pessimism_ineq_KL_bound_case2}
Q_{i,h}^{\dagger,\pi_{-i}^k,\rho_i}(s,\bm{a}) \leq \overline{Q}_{i,h}^{k,\rho_i}(s,\bm{a}).
\end{align}

\end{itemize}

Therefore, by eq. \ref{eq:optimism_pessimism_ineq_KL_bound_case1} and eq. \ref{eq:optimism_pessimism_ineq_KL_bound_case2}, we have proved that at step $h$, it holds that
\begin{align}
\label{eq:optimism_pessimism_ineq_KL_bound_final_NE}
   Q_{i,h}^{\dagger,\pi _{-i}^{k}, \rho_i}\left( s,\bm{a} \right) \le \up{Q}_{i,h}^{k, \rho_i}\left( s,\bm{a} \right)  ,   \,\,\,\,  \low{Q}_{i,h}^{k,\rho_i}\left( s,\bm{a} \right) \le Q_{i,h}^{\pi^{k}, \rho_i}\left( s,\bm{a} \right).
\end{align}
We now assume that eq. \ref{eq:general-Q-UCB-KL} hold for $h$-th step. Then, by the definition of robust value function as given by robust Bellman equation (\Cref{prop:Robust_Bellman_eq}), eq. \ref{eq:robust_V_values_k}, and NASH Equilibrium, we get
\begin{align}
\label{eq:optimism_pessimism_ineq_KL_step6}
    \up{V}_{i,h}^{k, \rho_i}\left( s \right) = \mathbb{E}_{\bm{a} \sim \pi^k(\cdot|s)}\left[\overline{Q}_{i,h}^{k, \rho_i}(s,{\bf a}) \right] = \max_{\pi^\prime_i}\mathbb{E}_{\bm{a} \sim \pi^\prime_i\times\pi_{-i}^k(\cdot|s)}\left[\overline{Q}_{i,h}^{k, \rho_i}(s,{\bf a}) \right].
\end{align}
By the definition of $ V_{i,h}^{\dagger,\pi _{-i}^{k}, \rho_i}\left( s \right)$ in eq. \ref{eq:robust_best_response_agent_i}, we get
\begin{align}
\label{eq:optimism_pessimism_ineq_KL_step7}
    V_{i,h}^{\dagger,\pi _{-i}^{k}, \rho_i}\left( s \right) = \max_{\pi^\prime_i}\mathbb{E}_{\bm{a} \sim \pi^\prime_i\times\pi_{-i}^k(\cdot|s)}\left[Q_{i,h}^{\dagger,\pi _{-i}^{k}, \rho_i}(s,{\bf a}) \right].
\end{align}
Sine by induction, for any $(s,{\bf a})$, $\overline{Q}_{i,h}^{k, \rho_i}(s,{\bf a}) \geq Q_{i,h}^{\dagger,\pi _{-i}^{k}, \rho_i}(s,{\bf a})$. As a result, we also have $ \up{V}_{i,h}^{k, \rho_i}\left( s \right) \geq  V_{i,h}^{\dagger,\pi _{-i}^{k}, \rho_i}\left( s \right)$, which is eq. \ref{eq:general-V-UCB-KL} for $h$-th step. Similarly, we can show that
\begin{align}
\label{eq:optimism_pessimism_ineq_KL_step8}
    \low{V}_{i,h}^{k, \rho_i}\left( s \right) &= \mathbb{E}_{\bm{a} \sim \pi^k(\cdot|s)}\left[\underline{Q}_{i,h}^{k, \rho_i}(s,{\bf a}) \right], \nonumber\\
    &\overset{(i)}{\leq }\mathbb{E}_{\bm{a} \sim \pi^k(\cdot|s)}\left[Q_{i,h}^{\pi^{k}, \rho_i}(s,{\bf a}) \right],\nonumber\\
    &\overset{(ii)}{=}V_{i,h}^{\pi^{k}, \rho_i}\left( s \right),
\end{align}
where (i) is due to the fact that $\low{Q}_{i,h}^{k,\rho_i}\left( s,\bm{a} \right) \le Q_{i,h}^{\pi^{k}, \rho_i}\left( s,\bm{a} \right)$ and (ii) is by definition of $V_{i,h}^{\pi^{k}, \rho_i}\left( s \right)$ as given by Bellman equation in \Cref{prop:Robust_Bellman_eq}.
\end{proof}

\subsubsection{CCE Version: Optimistic and pessimistic estimation of the robust values for {\RMGKL}.}
Here we will proof the optimistic estimations are indeed upper bounds of the corresponding robust V-value and robust Q-value functions fro CCE version.

\begin{lem}[Optimistic and pessimistic estimation of the robust values for {\RMGKL} for CCE Version]
\label{lem:Optimistic_pessimism_KL_CCE}
Under the event $\mathcal{E}_{KL}$ and by setting the bonus term $\beta^{k}_{i,h}$ as in eq. \ref{eq:Bonus_term_KL}, it holds that
    \begin{equation}
        \label{eq:general-Q-UCB-KL-CCE}
    Q_{i,h}^{\dagger,\pi _{-i}^{k}, \rho_i}\left( s,\bm{a} \right) \le \up{Q}_{i,h}^{k, \rho_i}\left( s,\bm{a} \right)  ,   \,\,\,\,  \low{Q}_{i,h}^{k,\rho_i}\left( s,\bm{a} \right) \le Q_{i,h}^{\pi^{k}, \rho_i}\left( s,\bm{a} \right),
    \end{equation}
    \begin{equation}
        \label{eq:general-V-UCB-KL-CCE}
        V_{i,h}^{\dagger,\pi _{-i}^{k}, \rho_i}\left( s \right) \le \up{V}_{i,h}^{k, \rho_i}\left( s \right)  , \,\,\,\,  \low{V}_{i,h}^{k, \rho_i}\left( s \right) \le V_{i,h}^{\pi^{k}, \rho_i}\left( s \right).
    \end{equation}
\end{lem}

\begin{proof}
The proof-lines are similar to \citep{ghosh2025provablynearoptimaldistributionallyrobust} adapted to the multi-agent case.\\
We will run a proof for each inequality outlined in \Cref{lem:Optimistic_pessimism_KL_CCE}
\begin{itemize}
    \item \textbf{Ineq. 1:} To prove $Q_{i,h}^{\dagger,\pi_{-i}^k,\rho_i}(s,\bm{a}) \leq \overline{Q}_{i,h}^{k,\rho_i}(s,\bm{a})$.
    \item \textbf{Ineq. 2:} To prove $\underline{Q}_{i,h}^{k,\rho_i}(s,\bm{a}) \leq Q_{i,h}^{\pi^k,\rho_i}(s,\bm{a})$.
\end{itemize}
Assume that both eq. \ref{eq:general-Q-UCB-KL-CCE} and eq. \ref{eq:general-V-UCB-KL-CCE} hold at the $(h+1)$-th step.
\begin{itemize}
    \item \textbf{Proof of Ineq. 1:}
We first consider robust $Q$ at the $h$-th step. Then, by \Cref{prop:Robust_Bellman_eq} (Robust Bellman Equation) and eq. \ref{eq:robust_Qupper_k}, we have that
\begin{align}
   & Q_{i,h}^{\dagger,\pi_{-i}^k,\rho_i}(s,\bm{a}) - \overline{Q}_{i,h}^{k,\rho_i}(s,\bm{a})\nonumber\\
&= \max \bigg\{ \sigma_{\mathcal{P}^{\rho_i}_{i,h}(s,\bm{a})} \left[ V_{i,h+1}^{\dagger,\pi_{-i}^k,\rho_i} \right]
- \sigma_{\widehat{\mathcal{P}^{\rho_i}_{i,h}}(s,\bm{a})} \left[\overline{V}^{k,\rho_i}_{i,h+1} \right] - \beta^{k}_{i,h}(s,\bm{a}),\nonumber\\
&\qquad  \qquad \qquad \qquad \, Q_{i,h}^{\dagger,\pi_{-i}^k,\rho_i}(s,\bm{a}) - H \bigg\}, \nonumber\\
&\leq \max\biggl\{
   \sigma_{\mathcal{P}^{\rho_i}_{i,h}(s,\bm{a})}\!\left[ V_{i,h+1}^{\dagger,\pi_{-i}^k,\rho_i} \right]
   - \sigma_{\widehat{\mathcal{P}^{\rho_i}_{i,h}}(s,\bm{a})}\!\left[V_{i,h+1}^{\dagger,\pi_{-i}^k,\rho_i} \right]  - \beta^{k}_{i,h}(s,\bm{a}), \; 0
\biggr\}
,\label{eq:optimism_pessimism_ineq_KL_step1_CCE}
\end{align}
where the second inequality follows from the induction of $V_{i,h+1}^{\dagger,\pi_{-i}^k,\rho_i} \leq \overline{V}_{i,h+1}^{k,\rho_i}$ at the $h+1$-th step and the fact that $Q_{i,h}^{\dagger,\pi_{-i}^k,\rho_i} \leq H$. By \Cref{lem:Bound_KL_optimal_policy} and by the definition of $\widehat{P}^k_{\min,h}(s,\bm{a})$ as given in eq. \ref{eq:hat_p_min_KL}, we have that

\begin{align}
\label{eq:optimism_pessimism_ineq_KL_step2_CCE}
\sigma_{\mathcal{P}^{\rho_i}_{i,h}(s,\bm{a})} \left[ V_{i,h+1}^{\dagger,\pi_{-i}^k,\rho_i} \right]
- \sigma_{\widehat{\mathcal{P}^{\rho_i}_{i,h}}(s,\bm{a})} \left[V_{i,h+1}^{\dagger,\pi_{-i}^k,\rho_i} \right] &\leq \frac{c_1H}{\rho_i}\sqrt{\frac{L}{\{N^k_h(s,\bm{a})\vee 1\}\widehat{P}^k_{\min,h}(s,\bm{a})}} \nonumber\\ &\qquad + \sqrt{\frac{1}{K}}.
\end{align}

By the choice of $ \beta^{k}_{i,h}$ in eq. \ref{eq:Bonus_term_KL} and eq. \ref{eq:optimism_pessimism_ineq_KL_step2_CCE} and applying in eq. \ref{eq:optimism_pessimism_ineq_KL_step1_CCE}, we conclude that
\begin{align}
\label{eq:optimism_pessimism_ineq_KL_bound_case1_CCE}
Q_{i,h}^{\dagger,\pi_{-i}^k,\rho_i}(s,\bm{a}) \leq \overline{Q}_{i,h}^{k,\rho_i}(s,\bm{a}).
\end{align}

\item \textbf{Proof of Ineq. 2:} By using \Cref{prop:Robust_Bellman_eq} (Robust Bellman Equation) and eq. \ref{eq:robust_Qlower_k}, we have that
\begin{align}
&\underline{Q}_{i,h}^{k,\rho_i}(s,\bm{a}) - Q_{i,h}^{\pi^k,\rho_i}(s,\bm{a} \nonumber\\
&= \max\Biggl\{
   \sigma_{\widehat{\mathcal{P}_{i,h}^{\rho_i}}(s,\bm{a})} \left[ \underline{V}_{i,h+1}^{k,\rho_i} \right]
   - \sigma_{\mathcal{P}_{i,h}^{\rho_i}(s,\bm{a})} \left[ V_{i,h+1}^{\pi^k,\rho_i} \right]
   - \beta^{k}_{i,h}(s,\bm{a}),\nonumber\\
&\qquad \qquad \qquad \qquad  0 - Q_{i,h}^{\pi^k,\rho_i}(s,\bm{a})
\Biggr\} \notag\\
&\leq \max\Biggl\{
   \sigma_{\widehat{\mathcal{P}_{i,h}^{\rho_i}}(s,\bm{a})} \left[ V_{i,h+1}^{\pi^k,\rho_i} \right]
   - \sigma_{\mathcal{P}_{i,h}^{\rho_i}(s,\bm{a})} \left[ V_{i,h+1}^{\pi^k,\rho_i} \right]  - \beta^{k}_{i,h}(s,\bm{a}),  0
\Biggr\}, \label{eq:optimism_pessimism_ineq_KL_step4_CCE}
\end{align}
where the second inequality follows from the induction of $\underline{V}_{i,h+1}^{k,\rho_i} \leq V_{i,h+1}^{\pi^k,\rho_i}$ at the $(h+1)$-th step and the fact that $Q_{i,h}^{\pi^k,\rho_i} \geq 0$. By \Cref{lem:Bound_KL_policy_k}, we get
\begin{align}
\sigma_{\widehat{\mathcal{P}_{i,h}^{\rho_i}}(s,\bm{a})} \left[ V_{i,h+1}^{\pi^k,\rho_i} \right]
- \sigma_{\mathcal{P}_{i,h}^{\rho_i}(s,\bm{a})} \left[ V_{i,h+1}^{\pi^k,\rho_i} \right] &\leq \frac{c_1H}{\rho_i}\sqrt{\frac{L}{\{N^k_h(s,\bm{a})\vee 1\}\widehat{P}^k_{\min,h}(s,\bm{a})}} \nonumber\\
&\qquad \qquad + \sqrt{\frac{1}{K}}.\label{eq:optimism_pessimism_ineq_KL_step5_CCE}
\end{align}

By the choice of $ \beta^{k}_{i,h}$ in eq. \ref{eq:Bonus_term_KL} and eq. \ref{eq:optimism_pessimism_ineq_KL_step5_CCE} and applying in eq. \ref{eq:optimism_pessimism_ineq_KL_step4_CCE}, we conclude that
\begin{align}
\label{eq:optimism_pessimism_ineq_KL_bound_case2_CCE}
Q_{i,h}^{\dagger,\pi_{-i}^k,\rho_i}(s,\bm{a}) \leq \overline{Q}_{i,h}^{k,\rho_i}(s,\bm{a}).
\end{align}

\end{itemize}

Therefore, by eq. \ref{eq:optimism_pessimism_ineq_KL_bound_case1_CCE} and eq. \ref{eq:optimism_pessimism_ineq_KL_bound_case2_CCE}, we have proved that at step $h$, it holds that
\begin{align}
\label{eq:optimism_pessimism_ineq_KL_bound_final_CCE}
   Q_{i,h}^{\dagger,\pi _{-i}^{k}, \rho_i}\left( s,\bm{a} \right) \le \up{Q}_{i,h}^{k, \rho_i}\left( s,\bm{a} \right)  ,   \,\,\,\,  \low{Q}_{i,h}^{k,\rho_i}\left( s,\bm{a} \right) \le Q_{i,h}^{\pi^{k}, \rho_i}\left( s,\bm{a} \right).
\end{align}
We now assume that eq. \ref{eq:general-Q-UCB-KL-CCE} hold for $h$-th step. Then, by the definition of robust value function as given by robust Bellman equation (\Cref{prop:Robust_Bellman_eq}), eq. \ref{eq:robust_V_values_k}, and CCE Equilibrium, we get
\begin{align}
\label{eq:optimism_pessimism_ineq_KL_step6_CCE}
    \up{V}_{i,h}^{k, \rho_i}\left( s \right) = \mathbb{E}_{\bm{a} \sim \pi^k(\cdot|s)}\left[\overline{Q}_{i,h}^{k, \rho_i}(s,{\bf a}) \right] \geq \max_{\pi^\prime_i}\mathbb{E}_{\bm{a} \sim \pi^\prime_i\times\pi_{-i}^k(\cdot|s)}\left[\overline{Q}_{i,h}^{k, \rho_i}(s,{\bf a}) \right].
\end{align}
By the definition of $ V_{i,h}^{\dagger,\pi _{-i}^{k}, \rho_i}\left( s \right)$ in eq. \ref{eq:robust_best_response_agent_i}, we get
\begin{align}
\label{eq:optimism_pessimism_ineq_KL_step7_CCE}
    V_{i,h}^{\dagger,\pi _{-i}^{k}, \rho_i}\left( s \right) = \max_{\pi^\prime_i}\mathbb{E}_{\bm{a} \sim \pi^\prime_i\times\pi_{-i}^k(\cdot|s)}\left[Q_{i,h}^{\dagger,\pi _{-i}^{k}, \rho_i}(s,{\bf a}) \right].
\end{align}
Sine by induction, for any $(s,{\bf a})$, $\overline{Q}_{i,h}^{k, \rho_i}(s,{\bf a}) \geq Q_{i,h}^{\dagger,\pi _{-i}^{k}, \rho_i}(s,{\bf a})$. As a result, we also have $ \up{V}_{i,h}^{k, \rho_i}\left( s \right) \geq  V_{i,h}^{\dagger,\pi _{-i}^{k}, \rho_i}\left( s \right)$, which is eq. \ref{eq:general-V-UCB-KL-CCE} for $h$-th step. Similarly, we can show that
\begin{align}
\label{eq:optimism_pessimism_ineq_KL_step8_CCE}
    \low{V}_{i,h}^{k, \rho_i}\left( s \right) &= \mathbb{E}_{\bm{a} \sim \pi^k(\cdot|s)}\left[\underline{Q}_{i,h}^{k, \rho_i}(s,{\bf a}) \right], \nonumber\\
    &\overset{(i)}{\leq }\mathbb{E}_{\bm{a} \sim \pi^k(\cdot|s)}\left[Q_{i,h}^{\pi^{k}, \rho_i}(s,{\bf a}) \right],\nonumber\\
    &\overset{(ii)}{=}V_{i,h}^{\pi^{k}, \rho_i}\left( s \right),
\end{align}
where (i) is due to the fact that $\low{Q}_{i,h}^{k,\rho_i}\left( s,\bm{a} \right) \le Q_{i,h}^{\pi^{k}, \rho_i}\left( s,\bm{a} \right)$ and (ii) is by definition of $V_{i,h}^{\pi^{k}, \rho_i}\left( s \right)$ as given by Bellman equation in \Cref{prop:Robust_Bellman_eq}.
\end{proof}

\subsubsection{CE Version: Optimistic and pessimistic estimation of the robust values for {\RMGKL}.}
Here we will proof the optimistic estimations are indeed upper bounds of the corresponding robust V-value and robust Q-value functions fro CE version.

\begin{lem}[Optimistic and pessimistic estimation of the robust values for {\RMGKL} for CE version]
\label{lem:Optimistic_pessimism_KL_CE}
By setting the bonus term $\beta^{k}_{i,h}$ as in eq. \ref{eq:Bonus_term_KL}, with probability $1-\delta$, for any $(s, {\bf a}, h, i)$ and $k \in [K]$, it holds that
    \begin{equation}
        \label{eq:general-Q-UCB-KL-CE}
   \max_{\phi \in \Phi_i} Q^{\phi \circ \pi^k, \rho_i}_{{i,h}}\left( s,\bm{a} \right) \le \up{Q}_{i,h}^{k, \rho_i}\left( s,\bm{a} \right)  ,   \,\,\,\,  \low{Q}_{i,h}^{k,\rho_i}\left( s,\bm{a} \right) \le Q_{i,h}^{\pi^{k}, \rho_i}\left( s,\bm{a} \right),
    \end{equation}
    \begin{equation}
        \label{eq:general-V-UCB-KL-CE}
       \max_{\phi \in \Phi_i} V^{\phi \circ \pi^k, \rho_i}_{{i,h}}(s) \le \up{V}_{i,h}^{k, \rho_i}\left( s \right)  , \,\,\,\,  \low{V}_{i,h}^{k, \rho_i}\left( s \right) \le V_{i,h}^{\pi^{k}, \rho_i}\left( s \right).
    \end{equation}
\end{lem}

\begin{proof}
    The proof-lines are similar to \citep{ghosh2025provablynearoptimaldistributionallyrobust} adapted to the multi-agent case.\\
    We will run a proof for each inequality outlined in \Cref{lem:Optimistic_pessimism_KL_CE}
\begin{itemize}
    \item \textbf{Ineq. 1:} To prove $\displaystyle\max_{\substack{\phi \in \Phi_i}} Q^{\phi \circ \pi^k, \rho_i}_{{i,h}}\left( s,\bm{a} \right) \le \up{Q}_{i,h}^{k, \rho_i}(s,\bm{a})$.
  \item \textbf{Ineq. 2:} To prove $\underline{Q}_{i,h}^{k,\rho_i}(s,\bm{a}) \leq Q_{i,h}^{\pi^k,\rho_i}(s,\bm{a})$.
\end{itemize}
Assume that both eq. \ref{eq:general-Q-UCB-KL-CE} and eq. \ref{eq:general-V-UCB-KL-CE} hold at the $(h+1)$-th step.
\begin{itemize}
    \item \textbf{Proof of Ineq. 1:}
We first consider robust $Q$ at the $h$-th step. Then, by \Cref{prop:Robust_Bellman_eq} (Robust Bellman Equation) and eq. \ref{eq:robust_Qupper_k}, we have that
\begin{align}
  & \max_{\phi \in \Phi_i} Q^{\phi \circ \pi^k, \rho_i}_{{i,h}}\left( s,\bm{a} \right) - \overline{Q}_{i,h}^{k,\rho_i}(s,\bm{a})\nonumber\\
&= \max \Bigg\{ 
    \sigma_{\mathcal{P}^{\rho_i}_{i,h}(s,\bm{a})}
    \left[\max_{\phi \in \Phi_i} V^{\phi \circ \pi^k, \rho_i}_{{i,h}}\right]
    - \sigma_{\widehat{\mathcal{P}^{\rho_i}_{i,h}}(s,\bm{a})} 
    \left[\overline{V}^{k,\rho_i}_{i,h+1} \right]- \beta^{k}_{i,h}(s,\bm{a}), \nonumber \\ 
    &\qquad \qquad \quad \max_{\phi \in \Phi_i} Q^{\phi \circ \pi^k, \rho_i}_{{i,h}}\left( s,\bm{a} \right) - H 
\Bigg\}\nonumber \\
&\quad \leq \max \Bigg\{ 
    \sigma_{\mathcal{P}^{\rho_i}_{i,h}(s,\bm{a})}
    \left[\max_{\phi \in \Phi_i} V^{\phi \circ \pi^k, \rho_i}_{{i,h}}\right]
    - \sigma_{\widehat{\mathcal{P}^{\rho_i}_{i,h}}(s,\bm{a})} 
    \left[\max_{\phi \in \Phi_i} V^{\phi \circ \pi^k, \rho_i}_{{i,h}}\right]  - \beta^{k}_{i,h}(s,\bm{a})
, 0\Bigg\}, &\label{eq:optimism_pessimism_ineq_KL_step1_CE}
\end{align}

where the second inequality follows from the induction of $\displaystyle\max_{\substack{\phi \in \Phi_i}} V^{\phi \circ \pi^k, \rho_i}_{{i,h+1}}(s) \le \up{V}_{i,h+1}^{k, \rho_i}\left( s \right)$ at the $h+1$-th step and the fact that $\displaystyle\max_{\substack{\phi \in \Phi_i}} Q^{\phi \circ \pi^k, \rho_i}_{{i,h}}\left( s,\bm{a} \right) \leq H$. By \Cref{lem:Bound_KL_optimal_policy} and by the definition of $\widehat{P}^k_{\min,h}(s,\bm{a})$ as given in eq. \ref{eq:hat_p_min_KL}, we have that
\begin{align}
\label{eq:optimism_pessimism_ineq_KL_step2_CE}
&\sigma_{\mathcal{P}^{\rho_i}_{i,h}(s,\bm{a})}
    \left[\max_{\phi \in \Phi_i} V^{\phi \circ \pi^k, \rho_i}_{{i,h}}(s)\right]
    - \sigma_{\widehat{\mathcal{P}^{\rho_i}_{i,h}}(s,\bm{a})} 
    \left[\max_{\phi \in \Phi_i} V^{\phi \circ \pi^k, \rho_i}_{{i,h}}(s)\right] \nonumber\\
    &\qquad \leq \frac{c_1H}{\rho_i}\sqrt{\frac{L}{\{N^k_h(s,\bm{a})\vee 1\}\widehat{P}^k_{\min,h}(s,\bm{a})}} + \sqrt{\frac{1}{K}}.
\end{align}

By the choice of $ \beta^{k}_{i,h}$ in eq. \ref{eq:Bonus_term_KL} and eq. \ref{eq:optimism_pessimism_ineq_KL_step2_CE} and applying in eq. \ref{eq:optimism_pessimism_ineq_KL_step1_CE}, we conclude that
\begin{align}
\label{eq:optimism_pessimism_ineq_KL_bound_case1_CE}
\max_{\phi \in \Phi_i} Q^{\phi \circ \pi^k, \rho_i}_{{i,h}}\left( s,\bm{a} \right) \le \up{Q}_{i,h}^{k, \rho_i}\left( s,\bm{a} \right).
\end{align}

\item \textbf{Proof of Ineq. 2:} By using \Cref{prop:Robust_Bellman_eq} (Robust Bellman Equation) and eq. \ref{eq:robust_Qlower_k}, we have that
\begin{align}
&\underline{Q}_{i,h}^{k,\rho_i}(s,\bm{a}) - Q_{i,h}^{\pi^k,\rho_i}(s,\bm{a}) \nonumber\\
&\quad =\max \left\{
\sigma_{\widehat{\mathcal{P}_{i,h}^{\rho_i}}(s,\bm{a})} \left[ \underline{V}_{i,h+1}^{k,\rho_i} \right]
- \sigma_{\mathcal{P}_{i,h}^{\rho_i}(s,\bm{a})} \left[ V_{i,h+1}^{\pi^k,\rho_i} \right]- \beta^{k}_{i,h}(s,\bm{a}), \,0 - Q_{i,h}^{\pi^k,\rho_i}(s,\bm{a})
\right\}, \notag\\
&\quad \leq \max \left\{
\sigma_{\widehat{\mathcal{P}_{i,h}^{\rho_i}}(s,\bm{a})} \left[ V_{i,h+1}^{\pi^k,\rho_i} \right]
- \sigma_{\mathcal{P}_{i,h}^{\rho_i}(s,\bm{a})} \left[ V_{i,h+1}^{\pi^k,\rho_i} \right]- \beta^{k}_{i,h}(s,\bm{a}), \, 0\right\}, \label{eq:optimism_pessimism_ineq_KL_step4_CE}
\end{align}
where the second inequality follows from the induction of $\underline{V}_{i,h+1}^{k,\rho_i} \leq V_{i,h+1}^{\pi^k,\rho_i}$ at the $(h+1)$-th step and the fact that $Q_{i,h}^{\pi^k,\rho_i} \geq 0$. By \Cref{lem:Bound_KL_policy_k}, we get
\begin{align}
\sigma_{\widehat{\mathcal{P}_{i,h}^{\rho_i}}(s,\bm{a})} \left[ V_{i,h+1}^{\pi^k,\rho_i} \right]
- \sigma_{\mathcal{P}_{i,h}^{\rho_i}(s,\bm{a})} \left[ V_{i,h+1}^{\pi^k,\rho_i} \right] &\leq \frac{c_1H}{\rho_i}\sqrt{\frac{L}{\{N^k_h(s,\bm{a})\vee 1\}\widehat{P}^k_{\min,h}(s,\bm{a})}} \nonumber\\ 
&\qquad \qquad + \sqrt{\frac{1}{K}}.\label{eq:optimism_pessimism_ineq_KL_step5_CE}
\end{align}
By the choice of $ \beta^{k}_{i,h}$ in eq. \ref{eq:Bonus_term_KL} and eq. \ref{eq:optimism_pessimism_ineq_KL_step5_CE} and applying in eq. \ref{eq:optimism_pessimism_ineq_KL_step4_CE}, we conclude that
\begin{align}
\label{eq:optimism_pessimism_ineq_KL_bound_case2_CE}
    \low{Q}_{i,h}^{k,\rho_i}(s,\bm{a}) \leq Q_{i,h}^{\pi^k,\rho_i}(s,\bm{a}).
\end{align}
\end{itemize}

Therefore, by eq. \ref{eq:optimism_pessimism_ineq_KL_bound_case1_CE} and eq. \ref{eq:optimism_pessimism_ineq_KL_bound_case2_CE}, we have proved that at step $h$, it holds that
\begin{align}
\label{eq:optimism_pessimism_ineq_KL_bound_final_CE}
  \max_{\phi \in \Phi_i} Q^{\phi \circ \pi, \rho_i}_{{i,h}}\left( s,\bm{a} \right) \le \up{Q}_{i,h}^{k, \rho_i}\left( s,\bm{a} \right)  ,   \,\,\,\,  \low{Q}_{i,h}^{k,\rho_i}\left( s,\bm{a} \right) \le Q_{i,h}^{\pi^{k}, \rho_i}\left( s,\bm{a} \right).
\end{align}
We now assume that eq. \ref{eq:general-Q-UCB-KL-CE} hold for $h$-th step. Then, by the definition of robust value function as given by robust Bellman equation (\Cref{prop:Robust_Bellman_eq}), eq. \ref{eq:robust_V_values_k}, and CE Equilibrium, we get
\begin{align}
\label{eq:optimism_pessimism_ineq_KL_step6_CE}
    \up{V}_{i,h}^{k, \rho_i}\left( s \right) = \mathbb{E}_{\bm{a} \sim \pi^k(\cdot|s)}\left[\overline{Q}_{i,h}^{k, \rho_i}(s,{\bf a}) \right] = \max_{\phi \in \Phi_i}\mathbb{E}_{\bm{a} \sim \phi \circ \pi^k(\cdot|s)}\left[\overline{Q}_{i,h}^{k, \rho_i}(s,{\bf a}) \right].
\end{align}
By the definition of $\max\limits_{\phi \in \Phi_i} V^{\phi \circ \pi^k, \rho_i}_{{i,h}}\left( s \right)$ in eq. \ref{eq:robust_best_response_agent_i}, we get
\begin{align}
\label{eq:optimism_pessimism_ineq_KL_step7_CE}
    \max_{\phi \in \Phi_i} V^{\phi \circ \pi^k, \rho_i}_{{i,h}}\left( s \right) = \max_{\phi \in \Phi_i} \mathbb{E}_{\bm{a} \sim \phi \circ \pi^k(\cdot|s)}\left[\max_{\phi^{\prime} }Q^{\phi^{\prime} \circ \pi^k, \rho_i}_{{i,h}}(s,{\bf a}) \right].
\end{align}
Since by induction, for any $(s,{\bf a})$, $\overline{Q}_{i,h}^{k, \rho_i}(s,{\bf a}) \geq \max\limits_{\phi \in \Phi_i} Q^{\phi \circ \pi^k, \rho_i}_{{i,h}}(s,{\bf a})$. As a result, we also have $ \up{V}_{i,h}^{k, \rho_i}\left( s \right) \geq  \max\limits_{\phi \in \Phi_i} V^{\phi \circ \pi^k, \rho_i}_{{i,h}}\left( s \right)$, which is eq. \ref{eq:general-V-UCB-KL-CE} for $h$-th step. Similarly, we can show that
\begin{align}
\label{eq:optimism_pessimism_ineq_KL_step8_CE}
    \low{V}_{i,h}^{k, \rho_i}\left( s \right) &= \mathbb{E}_{\bm{a} \sim \pi^k(\cdot|s)}\left[\underline{Q}_{i,h}^{k, \rho_i}(s,{\bf a}) \right], \nonumber\\
    &\overset{(i)}{\leq }\mathbb{E}_{\bm{a} \sim \pi^k(\cdot|s)}\left[Q_{i,h}^{\pi^{k}, \rho_i}(s,{\bf a}) \right],\nonumber\\
    &\overset{(ii)}{=}V_{i,h}^{\pi^{k}, \rho_i}\left( s \right),
\end{align}
where (i) is due to the fact that $\low{Q}_{i,h}^{k,\rho_i}\left( s,\bm{a} \right) \le Q_{i,h}^{\pi^{k}, \rho_i}\left( s,\bm{a} \right)$ and (ii) is by definition of $V_{i,h}^{\pi^{k}, \rho_i}\left( s \right)$ as given by Bellman equation in \Cref{prop:Robust_Bellman_eq}.
\end{proof}

\subsection{Auxiliary Lemmas for {\RMGKL}}
\label{subsubsec:aux_lemma_KL}

\begin{lem}[Concentration of Value Function in \RMGKL]
\label{lem:Bound_KL_optimal_policy}
Under the typical event \( \mathcal{E}_{\text{KL}} \) as defined in eq. \ref{eq:Event_KL}, the following concentration bound holds with probability at least $1-\delta$:
\begin{align*}
\abs{ \sigma_{\widehat{\mathcal{P}_{h}^{\rho_i}}(s,\bm{a})}\left[V_{i,h+1}^{\dagger,\pi_{-i}^{k}, \rho_i}\right]
- \sigma_{\mathcal{P}_{h}^{\rho_i}(s,\bm{a})}\left[V_{i,h+1}^{\dagger,\pi_{-i}^k, \rho_i}\right]}\leq \frac{c_1H}{\rho_i}\sqrt{\frac{L}{\{N^k_h(s,\bm{a})\vee 1\}\widehat{P}^k_{\min,h}(s,\bm{a})}} + \frac{1}{\sqrt{K}},
\end{align*}
where \(\iota = \log\Big(S^3 \big(\prod_{i=1}^m A_i\big) H^2 K^{3/2} / \delta\Big) \) and \( c_1 \) is an absolute constant.
\end{lem}

\begin{proof}
This proof establishes a concentration bound for the difference between the empirical and true robust value functions. We use the definition of the KL-divergence operator $\sigma_{\mathcal{P}_{i,h}^{\rho_i}(s,\bm{a})}[V_{i,h+1}^{\dagger,\pi_{-i}^k,\rho_i}]$ from eq. \ref{eq:dual_KL} and the empirical minimum probability $\widehat{P}^k_{\min,h}(s,\bm{a})$ from eq. \ref{eq:hat_p_min_KL} to express this difference as a supremum:
% \begin{align}
% &\abs{ \sigma_{\widehat{\mathcal{P}_{i,h}^{\rho_i}}(s,\bm{a})}\left[V_{i,h+1}^{\dagger,\pi_{-i}^k, \rho_i}\right]
% - \sigma_{\mathcal{P}_{i,h}^{\rho_i}(s,\bm{a})}\left[V_{i,h+1}^{\dagger,\pi_{-i}^k, \rho_i}\right]} \nonumber\\
% &\quad \leq \sup_{\eta \in [\underline{\eta}, H/\rho_i]} \eta \abs{\log\bigg(\mathbb{E}_{\widehat{P}_h^k(\cdot|s,\bm{a})}\bigg[\exp\bigg\{-\frac{V_{i,h+1}^{\dagger,\pi_{-i}^k, \rho_i}}{\eta}\bigg\}\bigg]\bigg)- \log\bigg(\mathbb{E}_{P^{\star}_h(\cdot|s,\bm{a})}\bigg[\exp\bigg\{-\frac{V_{i,h+1}^{\dagger,\pi_{-i}^k, \rho_i}}{\eta}\bigg\}\bigg]\bigg)}.
% \end{align}
\begin{align}
&\abs{ \sigma_{\widehat{\mathcal{P}_{i,h}^{\rho_i}}(s,\bm{a})}\left[V_{i,h+1}^{\dagger,\pi_{-i}^k, \rho_i}\right]
- \sigma_{\mathcal{P}_{i,h}^{\rho_i}(s,\bm{a})}\left[V_{i,h+1}^{\dagger,\pi_{-i}^k, \rho_i}\right] } \nonumber\\
&\quad \leq \sup_{\eta \in [\underline{\eta}, H/\rho_i]} \eta 
\Biggl| \log \Biggl( \mathbb{E}_{\widehat{P}_h^k(\cdot|s,\bm{a})} 
\Bigl[ \exp\Bigl\{-\frac{V_{i,h+1}^{\dagger,\pi_{-i}^k, \rho_i}}{\eta}\Bigr\} \Bigr] \Biggr) \nonumber\\
&\qquad\qquad\qquad\quad - \log \Biggl( \mathbb{E}_{P^{\star}_h(\cdot|s,\bm{a})} 
\Bigl[ \exp\Bigl\{-\frac{V_{i,h+1}^{\dagger,\pi_{-i}^k, \rho_i}}{\eta}\Bigr\} \Bigr] \Biggr) \Biggr|.
\end{align}

Under the high-probability event $\mathcal{E}_{\text{KL}}$ (defined in eq. \ref{eq:Event_KL}), we apply a known concentration inequality from \citep[Lemma 16]{NeuRIPS2024_UnifiedPessimismOfflineRL_Yue} to bound this expression:
\begin{align}
&\abs{ \sigma_{\widehat{\mathcal{P}_{i,h}^{\rho_i}}(s,\bm{a})}\left[V_{i,h+1}^{\dagger,\pi_{-i}^k, \rho_i}\right]
- \sigma_{\mathcal{P}_{i,h}^{\rho_i}(s,\bm{a})}\left[V_{i,h+1}^{\dagger,\pi_{-i}^k, \rho_i}\right]} \leq \frac{c_1H}{\rho_i}\sqrt{\frac{L}{\{N^k_h(s,\bm{a})\vee 1\}\widehat{P}^k_{\min,h}(s,\bm{a})}},
\end{align}
This bound holds for any $\eta$ within a fine-grained cover of the interval $[0,H/\rho_{\min}]$. By applying a standard covering argument, we extend this bound to hold for all $\eta \in [0,H/\rho_{\min}]$, thereby concluding the proof of \Cref{lem:Bound_KL_optimal_policy}.
\end{proof}

% \begin{lem}[Bound for RMDP-KL and the robust value function of \( \pi^k \)]
% \label{lem:Bound_KL_policy_k}
% Under event \( \mathcal{E}_{\text{KL}} \) in \ref{eq:Event_KL}, suppose that the optimism and pessimism inequalities~\ref{eq:general-Q-UCB-KL} and~\ref{eq:general-V-UCB-KL}, of Lemma~\ref{lem:Optimistic_pessimism_KL_NE}, hold at \( (h+1, k) \), then it holds that
% \begin{align}
% \label{eq:Bound_KL_policy_k}
% &\abs{\mathbb{E}_{\widehat{\mathcal{P}_{i,h}^{\rho_i}}(s,\bm{a})} 
% \left[ V_{i,h+1}^{\pi^k,\rho_i} \right]-\mathbb{E}_{\mathcal{P}_{i,h}^{\rho_i}(s,\bm{a})} 
% \left[ V_{i,h+1}^{\pi^k,\rho_i} \right]}\leq \frac{c_1H}{\rho_{\min}}\sqrt{\frac{L}{\{N^k_h(s,\bm{a})\vee 1\}\widehat{P}^k_{\min,h}(s,\bm{a})}} + \frac{1}{\sqrt{K}},
% \end{align}
% where \(\iota = \log\Big(S^3 \big(\prod_{i=1}^m A_i\big) H^2 K^{3/2} / \delta\Big) \) and \( c_1 \) is an absolute constant.
% \end{lem}

\begin{lem}[Bound for DRMG-KL and the robust value function of \( \pi^k \)]
\label{lem:Bound_KL_policy_k}
Under event \( \mathcal{E}_{\text{KL}} \) in eq. \ref{eq:Event_KL} and for any $\Eq \in \{\text{NASH}, \text{CE}, \text{CCE}\}$, we assume that the optimism and pessimism inequalities hold at \( (h+1, k) \), where these inequalities can correspond to any of the following cases of $\Eq$:
\begin{itemize}
    \item \textbf{NE:} \Cref{lem:Optimistic_pessimism_KL_NE} using eq. \ref{eq:general-Q-UCB-KL} and eq. \ref{eq:general-V-UCB-KL},
    \item \textbf{CCE:} \Cref{lem:Optimistic_pessimism_KL_CCE} using eq. \ref{eq:general-Q-UCB-KL-CCE} and eq. \ref{eq:general-V-UCB-KL-CCE},
    \item \textbf{CE:} \Cref{lem:Optimistic_pessimism_KL_CE} using eq. \ref{eq:general-Q-UCB-KL-CE} and eq. \ref{eq:general-V-UCB-KL-CE}.
\end{itemize}
Then the following bound holds:
\begin{align*}
%\label{eq:Bound_KL_policy_k}
&\left|\sigma_{\widehat{\mathcal{P}_{i,h}^{\rho_i}}(s,\bm{a})} 
\left[ V_{i,h+1}^{\pi^k,\rho_i} \right]-\sigma_{\mathcal{P}_{i,h}^{\rho_i}(s,\bm{a})} 
\left[ V_{i,h+1}^{\pi^k,\rho_i} \right]\right| \leq \frac{c_1H}{\rho_i}\sqrt{\frac{L}{\{N^k_h(s,\bm{a})\vee 1\}\widehat{P}^k_{\min,h}(s,\bm{a})}} + \frac{1}{\sqrt{K}},
\end{align*}
where \( \iota = \log\left(S^3 \left(\prod_{i=1}^m A_i\right) H^2 K^{3/2} / \delta\right) \), and \( c_1 \) is an absolute constant.
\end{lem}

% \begin{proof}
% By our definition of the operator \( \sigma_{\mathcal{P}_{i,h}^{\rho_i}(s,\bm{a})}[V_{i,h+1}^{\pi^k,\rho_i}] \) in~\ref{eq:dual_KL} and $\widehat{P}^k_{\min,h}(s,\bm{a})$ in \ref{eq:hat_p_min_KL},  we can arrive at
% \begin{align}
%    &\abs{ 
% \sigma_{\widehat{\mathcal{P}_{i,h}^{\rho_i}}(s,\bm{a})} \left[ V_{i,h+1}^{\pi^k,\rho_i} \right] 
% - \sigma_{\mathcal{P}_{i,h}^{\rho_i}(s,\bm{a})} \left[ V_{i,h+1}^{\pi^k,\rho_i} \right]} \nonumber\\
% &\quad \leq \sup_{\eta \in [\underline{\eta}, H/\rho_i]} \eta \abs{\log\bigg(\mathbb{E}_{\widehat{P}_h^k(\cdot|s,\bm{a})}\bigg[\exp\bigg\{-\frac{V_{i,h+1}^{\pi^k,\rho_i}}{\eta}\bigg\}\bigg]\bigg)  - \log\bigg(\mathbb{E}_{P^{\star}_h(\cdot|s,\bm{a})}\bigg[\exp\bigg\{-\frac{V_{i,h+1}^{\pi^k,\rho_i}}{\eta}\bigg\}\bigg]\bigg)}.
% \end{align}
% By the definition of $\mathcal{E}_{\text{KL}}$ as defined in \ref{eq:Event_KL} and by applying \citep[Lemma 17]{NeuRIPS2024_UnifiedPessimismOfflineRL_Yue}, we can arrive at
% \begin{align}
% &\abs{ 
% \sigma_{\widehat{\mathcal{P}_{i,h}^{\rho_i}}(s,\bm{a})} \left[ V_{i,h+1}^{\pi^k,\rho_i} \right] 
% - \sigma_{\mathcal{P}_{i,h}^{\rho_i}(s,\bm{a})} \left[ V_{i,h+1}^{\pi^k,\rho_i} \right]}\leq \frac{c_1H}{\rho_i}\sqrt{\frac{L}{\{N^k_h(s,\bm{a})\vee 1\}\widehat{P}^k_{\min,h}(s,\bm{a})}}.
% \end{align}
% for any $\eta \in \mathcal{N}_{\frac{1}{\rho_{\min} S\sqrt{K}}}\big([0,H/\rho_{\min}]\big)$. Therefore, by a covering argument, for any $\eta \in [0,H/\rho_{\min}]$, we conclude the proof of \Cref{lem:Bound_KL_policy_k}.\qedhere
% \end{proof}

\begin{proof}
This proof establishes a concentration bound for the difference between the empirical and true robust value functions under the KL-divergence. By using the definition of the robust operator $\sigma_{\mathcal{P}_{i,h}^{\rho_i}(s,\bm{a})}[V_{i,h+1}^{\pi^k,\rho_i}]$ from eq. \ref{eq:dual_KL} and the empirical minimum probability $\widehat{P}^k_{\min,h}(s,\bm{a})$ from eq. \ref{eq:hat_p_min_KL}, we can bound the absolute difference as follows:
\begin{align}
&\Biggl|
\sigma_{\widehat{\mathcal{P}_{i,h}^{\rho_i}}(s,\bm{a})} \left[ V_{i,h+1}^{\pi^k,\rho_i} \right]
- \sigma_{\mathcal{P}_{i,h}^{\rho_i}(s,\bm{a})} \left[ V_{i,h+1}^{\pi^k,\rho_i} \right]
\Biggr| \nonumber\\
&\quad \leq \sup_{\eta \in [\underline{\eta}, H/\rho_i]} \eta 
\Biggl| 
\log \Biggl( \mathbb{E}_{\widehat{P}_h^k(\cdot|s,\bm{a})} 
\Bigl[ \exp\Bigl\{-\frac{V_{i,h+1}^{\pi^k,\rho_i}}{\eta}\Bigr\} \Bigr] \Biggr) \nonumber\\
&\qquad\qquad\qquad\quad - \log \Biggl( \mathbb{E}_{P^{\star}_h(\cdot|s,\bm{a})} 
\Bigl[ \exp\Bigl\{-\frac{V_{i,h+1}^{\pi^k,\rho_i}}{\eta}\Bigr\} \Bigr] \Biggr)
\Biggr|.
\end{align}

Under the high-probability event $\mathcal{E}_{\text{KL}}$ (defined in eq. \ref{eq:Event_KL}), and by applying a known concentration inequality from \citep[Lemma 17]{NeuRIPS2024_UnifiedPessimismOfflineRL_Yue}, we can establish a uniform bound on this difference:
\begin{align}
&\abs{
\sigma_{\widehat{\mathcal{P}_{i,h}^{\rho_i}}(s,\bm{a})} \left[ V_{i,h+1}^{\pi^k,\rho_i} \right]
- \sigma_{\mathcal{P}_{i,h}^{\rho_i}(s,\bm{a})} \left[ V_{i,h+1}^{\pi^k,\rho_i} \right]}\leq \frac{c_1H}{\rho_i}\sqrt{\frac{L}{\{N^k_h(s,\bm{a})\vee 1\}\widehat{P}^k_{\min,h}(s,\bm{a})}}.
\end{align}
This inequality holds for any $\eta$ in a fine-grained cover of the interval $[0,H/\rho_{\min}]$. We conclude the proof of \Cref{lem:Bound_KL_policy_k} by using a standard covering argument to extend the bound to all $\eta \in [0,H/\rho_{\min}]$. \qedhere
\end{proof}

\begin{lem}[Bounds for RMG-KL and optimistic and pessimistic robust value estimators]
\label{lem:Bound_KL_optimism_pessimism}
Under event \( \mathcal{E}_{\text{KL}} \) in eq. \ref{eq:Event_KL} and for any $\Eq \in \{\text{NASH}, \text{CE}, \text{CCE}\}$, we assume that the optimism and pessimism inequalities hold at \( (h+1, k) \), where these inequalities can correspond to any of the following cases of $\Eq$:
\begin{itemize}
    \item \textbf{NE:} \Cref{lem:Optimistic_pessimism_KL_NE} using eq. \ref{eq:general-Q-UCB-KL} and eq. \ref{eq:general-V-UCB-KL},
    \item \textbf{CCE:} \Cref{lem:Optimistic_pessimism_KL_CCE} using eq. \ref{eq:general-Q-UCB-KL-CCE} and eq. \ref{eq:general-V-UCB-KL-CCE},
    \item \textbf{CE:} \Cref{lem:Optimistic_pessimism_KL_CE} using eq. \ref{eq:general-Q-UCB-KL-CE} and eq. \ref{eq:general-V-UCB-KL-CE}.
\end{itemize}
Then the following bound holds:
\begin{align*}
&\max \Bigg\{\abs{\sigma_{\widehat{\mathcal{P}_{i,h}^{\rho_i}}(s,\bm{a})}\left[\overline{V}_{i,h+1}^{k,\rho_i}\right] - \sigma_{\mathcal{P}_{i,h}^{\rho_i}(s,\bm{a})}\left[\overline{V}_{i,h+1}^{k, \rho_i}\right]},\abs{\sigma_{\widehat{\mathcal{P}_{i,h}^{\rho_i}}(s,\bm{a})}\left[\underline{V}_{i,h+1}^{k, \rho_i}\right] - \sigma_{\mathcal{P}_{i,h}^{\rho_i}(s,\bm{a})}\left[\underline{V}_{i,h+1}^{k, \rho_i}\right]}
\Bigg\}\nonumber\\
&\qquad \qquad \leq  \frac{c_1H}{\rho_i}\sqrt{\frac{L}{\{N^k_h(s,a)\vee 1\}\widehat{P}^k_{\min,h}(s,\bm{a})}} +\sqrt{\frac{1}{K}},
\end{align*}
where $\iota = \log\left(S^3 \left(\prod_{i=1}^m A_i\right) H^2 K^{3/2} / \delta\right))$ and $c_1$ is an absolute constant.
\end{lem}

\begin{proof}
    We follow the same proof lines as \Cref{lem:Bound_KL_policy_k}, and thereby we omit it.
\end{proof}

\begin{lem}[Bound on Binomal random variable]
\label{lem:binomial_rv_bound}
    Suppose $X \sim \text{Binomial}(n, p)$, where $n \geq 1$ and $p \in [0, 1]$. For any $\delta \in (0,1)$, we have
\begin{align}
X &\geq \frac{np}{8 \log\left( \frac{1}{\delta} \right)}, \hspace{9mm} \text{if } np \geq 8 \log\left( \frac{1}{\delta} \right), \label{eq:binomial_lower_bound}\\
X &\leq 
\begin{cases}
e^2 np, & \text{if } np \geq \log\left( \frac{1}{\delta} \right), \\
2e^2 \log\left( \frac{1}{\delta} \right), & \text{if } np \leq 2 \log\left( \frac{1}{\delta} \right),
\end{cases} \label{eq:binomial_upper_bound}
\end{align}
hold with probability at least $1 - 4\delta$.
\end{lem}

\begin{proof}
    Refer to \citep[Lemma 8]{NeuRIPS2023_CuriousPriceDRRLGenerativeMdel_Shi} for details.
\end{proof}

%%%%%%%%%%%%%%Key Lemmas %%%%%%%%%
% \section{Key Lemmas}
% Here, we present some key lemmas which are useful in the proof for {\RMGTV} and {\RMGKL}.

\section{Other Technical Lemmas}
Here, we present some auxiliary lemmas which are useful in the proof.

\begin{lem}[Azuma Hoeffding's Inequality]
    \label{lem:Azuma-Hoeffding}
    Let $\{Z_t\}_{t \in \mathbb{Z}_+}$ be a martingale with respect to the filtration $\{\mathcal{F}_t\}_{t\in \mathbb{Z}_+}$. Assume that there are predictable processes $\{A_t\}_{t \in \mathbb{Z}_+}$ and $\{B_t\}_{t \in \mathbb{Z}_+}$ with respect to $\{\mathcal{F}_t\}_{t\in \mathbb{Z}_+}$, i.e., for all $t$, $A_t$ and $B_t$ are $\mathcal{F}_{t-1}$-measurable, and constants $0 < c_1, c_2, \dots < +\infty$ such that $A_t \leq Z_t - Z_{t-1} \leq B_t$ and $B_t - A_t \leq c_t$ almost surely. Then, for all $\beta > 0$
    \begin{align}
\mathbb{P}\bigg(\abs{Z_t - Z_0} \geq \beta\bigg) \leq \exp\Bigg\{-\frac{2\beta^2}{\sum\limits_{i \leq t}c^2_t}\Bigg\}.
\end{align}
\end{lem}
\begin{proof}
    Refer to the proof of Theorem 5.1 of \citep{Book2009_ConcIneq_Dubhashi}.
\end{proof}

\begin{lem}[Self-bounding variance inequality {\citep[Theorem 10]{Arxiv2009_EmpBernsteinBounds_Maurer}}]
\label{lem:self_bound_variance}
Let $X_1, \ldots, X_T$ be independent and identically distributed random variables with finite variance, that is, $\operatorname{Var}(X_1) < \infty$. Assume that $X_t \in [0, M]$ for every $t$ with $M > 0$, and let
\[
S_T^2 = \frac{1}{T} \sum_{t=1}^T X_t^2 - \left( \frac{1}{T} \sum_{t=1}^T X_t \right)^2.
\]
Then, for any $\varepsilon > 0$, we have
\[
\mathbb{P} \left( \left| S_T - \sqrt{\operatorname{Var}(X_1)} \right| \geq \varepsilon \right)
\leq 2 \exp\left( - \frac{T \varepsilon^2}{2M^2} \right).
\]
\end{lem}

\begin{proof}
    Refer to the proof of Lemma 7 of \citep{PMLR2022_SampleComplexityRORL_Panaganti}.
\end{proof}

\end{document}